%% file: main.tex
\documentclass{article}

\usepackage[nonatbib,preprint]{neurips_2022}
\makeatletter
\renewcommand{\@noticestring}{}
\makeatother

\usepackage[utf8]{inputenc} %
\usepackage[T1]{fontenc}    %
\usepackage[USenglish]{babel}
\usepackage{csquotes}
\usepackage{amsmath,amssymb,amsthm}
\usepackage{hyperref}       %
\usepackage{url}            %
\usepackage{booktabs}       %
\usepackage{amsfonts}       %
\usepackage{nicefrac}       %
\usepackage{microtype}      %
\usepackage{xcolor}         %

\usepackage{enumitem}

\usepackage{graphicx}
\usepackage[caption = false]{subfig}
\usepackage{wrapfig,booktabs}
\usepackage{todonotes}
\usepackage{soul}
\usepackage[font=small,labelfont=bf]{caption}
\usepackage{multirow}

\usepackage[style=numeric-comp,maxbibnames=9,maxcitenames=2,sortcites]{biblatex}
\let\citet\textcite
\let\citep\parencite
\renewbibmacro{in:}{}

\usepackage[capitalize]{cleveref}

\title{Evaluating Graph Generative Models with Contrastively Learned Features}

\newtheorem{prop}{Proposition}

\author{%
  Hamed Shirzad\\
  UBC\\
  \texttt{shirzad@cs.ubc.ca}%
  \And 
  Kaveh Hassani\\
  Autodesk\\
  \texttt{kaveh.hassani@autodesk.com}%
  \And 
  Danica J.\ Sutherland\\
  UBC \& Amii\\
  \texttt{dsuth@cs.ubc.ca}%
}

\begin{document}

\maketitle

\begin{abstract}
A wide range of models have been proposed for Graph Generative Models, necessitating effective methods to evaluate their quality. So far, most techniques use either traditional metrics based on subgraph counting, or the representations of randomly initialized Graph Neural Networks (GNNs). We propose using representations from contrastively trained GNNs, rather than random GNNs, and show this gives more reliable evaluation metrics. Neither traditional approaches nor GNN-based approaches dominate the other, however: we give examples of graphs that each approach is unable to distinguish. We demonstrate that Graph Substructure Networks (GSNs), which in a way combine both approaches, are better at distinguishing the distances between graph datasets. The code used for this project is available in: \url{https://github.com/hamed1375/Self-Supervised-Models-for-GGM-Evaluation}.
\end{abstract}

\section{Introduction}
\label{intro}

Quantitative evaluation of generative models is challenging \cite{theis2015note};
evaluating purely by visual inspection can introduce biases, particularly towards ``precision'' at cost of ``recall'' \cite{sajjadi2018assessing}.
The traditional metric, likelihood,
is also not only difficult to evaluate for implicit generative models on complex, highly structured datasets,
but can also be a poor fit to users' goals with generative modeling 
\cite{theis2015note, sajjadi2018assessing}.
Many quantitative proxy measures for the discrepancy between generator and real distributions have thus been used in recent work on generative modeling of images,
some of the most important including the
Inception Score (IS) \cite{salimans2016improved}, Fréchet 
Inception Distance (FID) \cite{heusel2017gans}, Precision/Recall (PR) \cite{sajjadi2018assessing}, and Density/Coverage \citep{naeem2020reliable}.
FID and PR measures in particular have been shown to correlate with human 
judgments in some practical settings.
All of these measures use representations extracted from CNNs pretrained on ImageNet classification \citep{ILSVRC15}. Nevertheless, it is recently shown that self-supervised representations archives more reasonable ranking in terms of FID/Precision/Recall, while the ranking with ImageNet-pretrained embeddings often can be misleading \cite{morozov2021on}.

These problems are exacerbated for the evaluation of generative models of graph data.
Graphs are often used to represent concepts in a broad range of specialized domains, including various fields of engineering, bioinformatics, and so on,
so that it is more difficult or perhaps impossible to find ``generally good'' features like those available for natural images via ImageNet models.
Many methods for evaluating graph generative models are thus based on measuring discrepancies between statistics of generic low-level graph features,
including local measurements like
degree distributions, 
clustering coefficient distributions, and four-node orbit counts \cite{li2018learning, you2018graphrnn}, or simple global properties such as 
spectra \cite{liao2019efficient}.
Differences between distributions of these features are generally measured via the maximum mean discrepancy (MMD) \cite{gretton:mmd}
or total variation (TV) distance.

More recent work \cite{thompson2022evaluation} proposes instead extracting features via randomly initialized Graph Isomorphism Networks (GINs) \cite{xu2018powerful}.
These features provide a representation for each graph, so that metrics like FID, Precision/Recall, and Density/Coverage can be estimated using these representations. 
For datasets with natural labels, we can alternatively find representations based on training a classifier;
such features, however, are optimized for a very different task.

One significant challenge of graph data is that it is computationally difficult to even tell whether two graphs are the same (are isomorphic to one another).
GINs are known to be as powerful at detecting isomorphism as the standard Weisfeiler-Lehman (WL) test \cite{xu2018powerful}, but it is not known whether \emph{random} GIN representations have this ability. Experimental evidence \cite{xu2018powerful} shows that trained GIN networks are as powerful as the WL test, but have substantially worse power at random initialization.

Unlike for distributions of natural images, however, the relevant graph features and particularly the semantics of node or edge features varies widely between datasets.
On molecular graphs, adding any single edge will almost always violate physical constraints about stable atomic bonds;
for house layouts, the feasibility of adding an edge depends significantly on the rest of the layout and the type of rooms (few front doors open into a bathroom, and direct connections between rooms on the first and third floors are unlikely);
Traditional metrics and random graph neural networks are both unlikely to be able to capture these complex inter-dependencies functioning very differently across domains,
and it similarly seems quite unlikely that pretraining on some ``graph ImageNet'' would be able to find useful generic graph features.

We therefore propose to train graph encoders on the same data used to train the generative model, using self-supervised contrastive learning to find meaningful feature extractors \cite{velickovic_2019_iclr,Sun2020InfoGraph,pmlr-v119-hassani20a,qiu2020gcc,you2020graph,you2021graph,hassani2022learning,zhu2020deep,zhu2021graph, thakoor2021bootstrapped}, which we then use to compare the generated graphs to a test set.
The set of perturbations introduced in contrastive learning teaches the model which kinds of graphs should be considered similar to one another.
In this work we use types of perturbations traditionally used for training contrastive learning methods on graphs, but point out that future work focusing on domain-specific modeling can directly incorporate knowledge about which graphs should be considered similar by choosing different perturbation sets.

Inspired by the theoretical results for contrastive learning of \cite{wang2022chaos},
we also propose two variations to our representation learning procedures.
This upper bound shows that contrastively learned representations work well for downstream tasks as long as the probability of data points yielding overlapping augmentations is relatively large for \emph{within-class} dataset pairs and relatively small for \emph{cross-class} pairs.
Edit distances between graphs on typical training sets are large, however;
we propose to use subgraphs (``crops'') in our set of data augmentations for contrastive learning,
which is more likely to yield overlaps.
We also suggest enforcing a layer-wise Lipschitz constraint on feature extractors,
which encourages similar graphs to have similar learned representations.
We show experimentally that both changes improve learning.

With all of these improvements, we further ask: can we theoretically guarantee that our learned GNN representations outperform traditional local metrics?
We prove that we cannot:
we give examples of graphs easily distinguishable by local metrics that first-order GNNs cannot distinguish.
Yet the converse is also true: we show graphs easily distinguishable by GNNs that appear equivalent to local metrics.
We thus propose to use models based on Graph Substructure Networks (GSNs) \citep{bouritsas2022improving}, (using node degrees, and node clustering coefficients), and as a result,
explicitly incorporating local metrics into our models
and surpassing the power of the WL test, and yield further improvements.

\section{Related Work}
\paragraph{Graph generative models.}
Our work is not particular to any type of graph generative model, as it focuses on simply evaluating samples;
nonetheless, it is worth briefly reviewing some methods.
Sequential generation models, such as GraphRNN \cite{you2018graphrnn} and GRAN \cite{liao2019efficient}, generate nodes and edges in an auto-regressive manner. These models are efficient, but can struggle to capture global properties of the graphs, and require a predefined ordering. One-shot models such as MolGAN \cite{de2018molgan}, GraphDeconv \cite{flam2020graph}, GraphVAE \cite{kipf2016semi}, and GraphNVP \cite{madhawa2019graphnvp}, on the other hand, generate all nodes and edges in one shot; they can model long-range dependencies, but generally are not efficient and cannot scale to large graphs.
For a detailed overview of graph generative models, see \cite{guo2020systematic}.

\paragraph{Evaluating graph generative models.}
Graph generative models can be even more challenging to evaluate than visual or textual models,
because it is generally more difficult for humans to judge the quality of a generated graph.
The classic measure of the quality of a generative model, likelihood,
also has significant issues with graphs:
in addition to the kinds of issues that appear in generative models of images \citep{theis2015note},
the likelihood is particularly hard to evaluate on graphs where even checking equality is quite difficult \citep{o2021evaluation}. The most common method is to compute the Maximum Mean Discrepancy (MMD) \citep{gretton:mmd} between distributions of local graph statistics,
including node degrees, cluster coefficients, and counts of orbits up to a particular size \citep{you2018graphrnn,liao2019efficient}.
Global statistics, such as eigenvalues of the normalized graph Laplacian, are also used \citep{you2018graphrnn}.
These metrics, however, focus only on low-level structure of the graph, and ignore any features that might be present on nodes or edges \citep{thompson2022evaluation}.
Choosing an appropriate kernel is also very important to consistency of these metrics \citep{o2021evaluation}.
These types of graph statistics might also encourage models to overfit to the training data, rather than truly learning the target distribution \citep{shirzad2021td}.

For unlabeled image-based generative models, most work focuses on metrics including the Inception Score (IS) \cite{salimans2016improved}, Fréchet 
Inception Distance (FID) \cite{heusel2017gans}, Precision/Recall (PR) \cite{sajjadi2018assessing}, and Density/Coverage \citep{naeem2020reliable},
all of which compare distributions in a fixed latent space (typically activations of a late layer in an InceptionV3 \citep{inception} model trained on ImageNet).
These methods are rarely adopted in graph generative models,
due to challenges with ``general-purpose'' graph models discussed in \cref{intro}.
\citet{thompson2022evaluation} thus used random (untrained) graph encoders in these metrics.
We discuss these methods in more detail in the appendix.
Our work, inspired by \citeauthor{morozov2020self}'s similar proposal in image domains \cite{morozov2020self}, explores the use of self-supervised contrastive training to find representations that work better than random initializations.

\paragraph{Graph Contrastive Learning.}
A popular method for self-supervised learning,
contrastive learning generally aims to find a representation
roughly invariant to various operations (e.g.\ for images, taking random crops, horizontal flipping, shifting colors)
but able to identify different source data points.
Ideally, such a representation will be useful for downstream tasks not known when learning the representation.
In graph settings,
learned representations may be at the node, edge, or graph levels.

DGI \cite{velickovic_2019_iclr} 
and InfoGraph \cite{Sun2020InfoGraph} adopt DeepInfoMax \cite{hjelm_2019_iclr}, and enforce consistency between local (node) and 
global (graph) representation.
MVGRL \cite{pmlr-v119-hassani20a} augments a graph via graph diffusion, and constructs two views by 
randomly sampling sub-graphs from the adjacency and diffusion matrices. GCC \cite{qiu2020gcc} uses sub-graph instance discrimination 
and contrasts sub-graphs from a multi-hop ego network. GraphCL \cite{you2020graph} uses trial-and-error to hand-pick graph augmentations 
and the corresponding parameters of each augmentation. JOAO \cite{you2021graph} extends the GraphCL using a bi-level min-max optimization 
that learns to select the augmentations. LG2AR \cite{hassani2022learning} learns a policy to sample augmentations. GRACE \cite{zhu2020deep} uses a 
similar approach to GraphCL to learn node representations. GCA \cite{zhu2021graph} uses a set of heuristics to adaptively pick the augmentation 
parameters. BGRL \cite{thakoor2021bootstrapped} adopts BYOL \cite{NEURIPS2020_f3ada80d} and uses random augmentations to learn node representations.

\section{GNNs Versus Local Metrics} \label{sec:gnn-v-local}
\begin{figure}
    \centering
    \includegraphics[width=110mm]{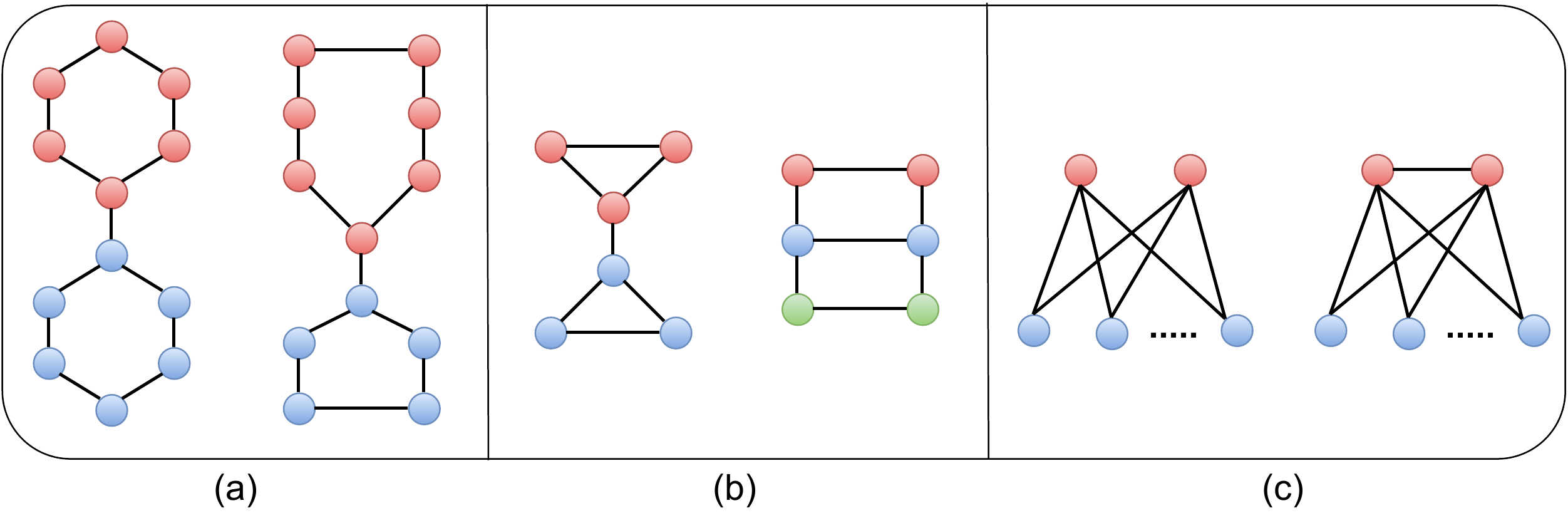}
    \caption{(a) An example of two graphs that can be differentiated by GNNs but not local metrics. Both graphs have same degree distribution, clustering coefficient and 4 node orbits. (b) An example of two graphs that can be differentiated by local metrics (clustering coefficient and smallest cycle)  but not GNNs. (c) Adding a single edge can drastically change cluster coefficient and 4-node orbits.}
    \label{fig:local_iso}
\end{figure}

Different methods for understanding graphs can ``understand'' the difference between graphs in very different ways:
a particular change to a graph might barely affect some features,
while drastically changing others.
One extreme case is when a given metric cannot detect that two distinct (non-isomorphic) graphs are different.
Since graph isomorphism is a computationally difficult problem,
we expect that all efficiently computable graph representations ``collapse'' some pairs of input graphs.\footnote{Alternatively, an efficient graph representation might be able to distinguish non-isomorphic graphs, if it also sometimes distinguishes isomorphic graphs. We do not consider such representations in this paper.}
It is conceivable, however, that one method could be strictly more powerful than another.
For instance,
since recent GNN models have overcome traditional models based on local metric representations in a variety of problems \citep{kipf2016semi, velivckovic2017graph, xu2018powerful},
is it the case that GNNs are strictly more powerful than local metrics?

We show, constructively, that the answer is no:
there are indeed graphs which GNNs can distinguish and local metrics cannot,
but there are also graphs which local metrics can distinguish but first-order GNNs cannot.
\Cref{fig:local_iso}(a) shows a pair of graphs with the same degree distribution, clustering coefficient, and four-node orbits,
which can nonetheless be distinguished by GNNs
(proof in \cref{app:proofs}).
On the other hand, the graphs in \cref{fig:local_iso}(b) have different clustering coefficient and smallest cycle,
but first-order GNNs cannot tell them apart.
Thus, neither method strictly outperforms the other on all problems,
and so there are theoretical generative models which perfectly match in GNN-based representations but differ in local metrics,
and vice versa.
This motivates our addition of local features to our graph representation models (\cref{sec:add-local-feats}).

It is much easier to incorporated such hard-coded structures into GNNs than it would be to add learning to feature metrics; in particular, counting higher-order local patterns quickly becomes prohibitively expensive, with a super-exponential time complexity. GNNs can also easily handle node and/or edge features on the underlying graphs, which is far more difficult to add to local metrics.

Another quality we would like in our graph representations, in addition to the ability to distinguish distinct graphs, is some form of stability: if a distribution of graphs only changes slightly, we would like our evaluation methods to give ``partial credit'' as opposed to a distribution where all graphs are dramatically different.
(This is closely related to issues in training and evaluating image-based generative models \citep{arjovsky:towards-principled-gans,wgan,smmdgan,theis2015note}.)
As previously discussed,
the notion of ``similar graphs'' is very domain-dependent,
but traditional local metrics can be highly sensitive to changes like adding a single edge:
\cref{fig:local_iso}(c) shows an example where two graphs differing only by a single edge can have drastically different statistics.
By learning GNN representations,
we can have some control over these types of smoothness properties;
we exploit this explicitly in our methodology in \cref{sec:adding-smoothness}.

    \begin{figure}
        \centering
        \includegraphics[width=130mm]{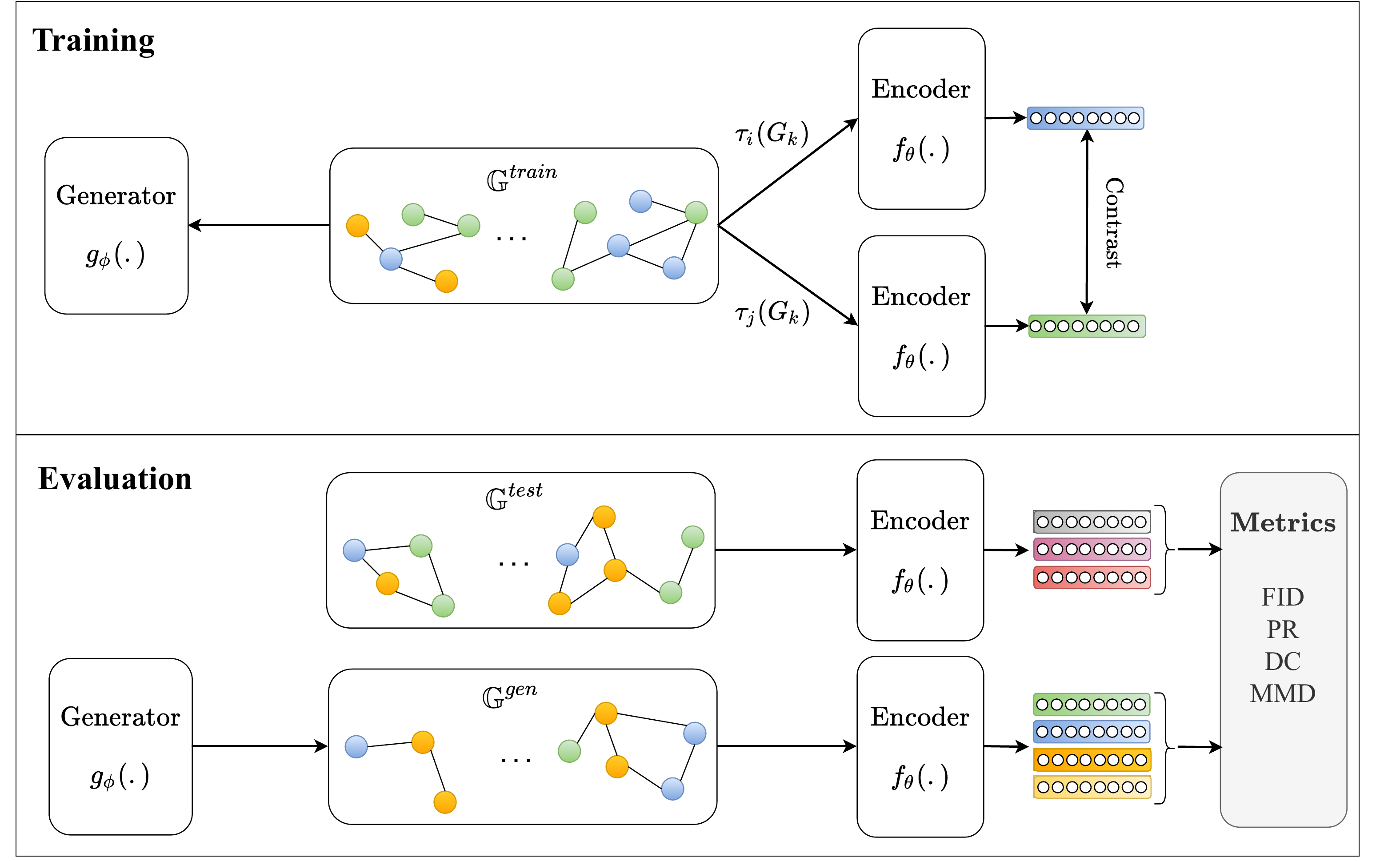}
        \caption{During the \textbf{training} phase, we use the same training data ($\mathbb{G}^{train}$) to train both the generator and the encoder networks. The encoder is trained using a contrastive loss where 
        two augmentation $\tau_{i}$ and $\tau_{j}$ are randomly sampled
        from a set of rational augmentations $\mathcal{T}$ to construct two views of a sampled graph $G_k$. During \textbf{evaluation} phase, we sample the generator to form a generated set
        of graphs $\mathbb{G}^{gen}$ and feed it along with a held-out set of real graphs $\mathbb{G}^{test}$ to the encoder to compute the graph representations. The representations are then used to compute robust metrics to quantify the discrepency between real and generated graphs.
        }
        \label{fig:gen-arc}
    \end{figure}

\section{Self-supervised Training of Graph Representations for Evaluation}
Suppose we have two sets of graphs $\mathbb{G}^{train} = \{G_1,\dots\,G_S\}$ and $\mathbb{G}^{test} = \{G_1,\dots\,G_M\}$, each sampled from the same data distribution $p\left(G\right)$. Also suppose that we have access to an unconditional graph generative model $g_\phi(.)$, which is trained on $\mathbb{G}^{train}$ to learn the distribution of the observed set of graphs. We sample a set of generated graphs $\mathbb{G}^{gen} = \{G_1,\dots\,G_N\} \sim p_{g_\phi}(G)$ from this model. In order to evaluate the quality of the sampled graphs (i.e., to decide whether the model $g_\phi(.)$ has successfully recovered the underlying distribution $p\left(G\right)$), we can define a measure of divergence $\mathcal{D}\left(\mathbb{G}^{test}, \mathbb{G}^{gen} \right)$ to quantify the discrepancy
between distributions of the real and generated graphs. One robust way to achieve this is
to define the metric on latent vector representation spaces, and expect representations of graphs rather than the original objects. Thus, to use these metrics, we need to train a shared encoder $f_{\theta}(.)$ and then compute the discrepancy as $\mathcal{D}\left(f_{\theta}(\mathbb{G}^{test}, f_{\theta}(\mathbb{G}^{gen}\right)$
There are a few such metrics well-studied in visual domains that can differentiate the fidelity and diversity of the model, and which we can adopt in graph domains.

For evaluating image generative models, due to similar feature space across image datasets and also availability of large-scale data, the trunk of a model trained over ImageNet with explicit supervisory signals is usually chosen as the encoder.
However, it is not straightforward to adopt the same trick to graph-structured data: there are no ImageNet-scale graph datasets available, and more importantly,
the semantics of graphs and their features vary wildly across commonly used graph datasets,
far more than occurs across distributions of natural images.
For instance, even datasets of molecular graphs may use different feature sets to represent the molecules \cite{hassani2022cross},
in addition to the many cross-domain challenges  discussed in \cref{intro}.
Thus, it is not feasible to imagine a single ``universal'' graph representation;
we would like a general-purpose method for finding representations useful for a new graph dataset.

\paragraph{Training with Graph Contrastive Learning}
\label{sec:using-gcl}

To find expressive representation of real and generated graphs, we train the encoder using a contrastive objective. Assuming a set of rational augmentations $\mathcal{T}$over $\mathbb{G}$ where each augmentation 
$\tau_{i} \in\mathcal{T}$ is defined as a function over graph $G_k$ that generates an identity-preserving view of the graph: $G^+_k=\tau_{i}(G_k)$, a  
contrastive framework with negative sampling strategy uses $\mathcal{T}$to draw positive samples from the joint distribution $p\left(\tau_{i}(G_k), \tau_{j}(G_k)\right)$ 
in order to maximize the agreement between different views of the same graph $G_k$ and to draw negative samples from the product of 
marginals $p\left(\tau_{i}(G_k)\right) \times p\left(\tau_{j}(G_{k'})\right)$ to minimize it for views from two distinct graphs $G_k$ and $G_{k'}, k\neq k'$.

As mentioned previously, we use a GIN architecture for our feature extractor.
Other than the training process, the rest of our evaluation pipeline is similar to that of \citet{thompson2022evaluation},
who use random GIN weights.We consider two methods for training our GIN's parameters:
GraphCL \citep{you2020graph} and InfoGraph \citep{sun2019infograph}. GraphCL randomly samples augmentation from four possible augmentaiotns including \textit{node dropping}, \textit{edge perturbation}, \textit{attribute masking}, and \textit{sub-graph inducing} based on a random walk to which
we would like the representation to be roughly invariant. . GraphCL uses normalized temperature-scaled 
cross-entropy (NT-Xent) objective \cite{chen_2020_arxiv} to maximize the agreement between positive graph-level representations. InfoGraph works differently: it contrasts the graph-level representation with the representations of individual nodes, which encode neighborhood structure information. InfoGraph uses 
DeepInfoMax \cite{hjelm_2019_iclr} objective to maximize the mutual information between graph-level and node-level representations of each individual graph.

\paragraph{Using Local Subgraph Counts as Input Features for GINs} \label{sec:add-local-feats}

To build on the insight of \cref{sec:gnn-v-local},
we also consider various methods for adding information about the local graph structure as node features,
similarly to Graph Substructure Networks \citep{bouritsas2022improving}. The simplest such method is to add the degree of a node as an explicit feature for the GIN to consider.
We do this, but also add, by concatenation on node features,
higher-order local information as well: three-node and four-node clustering features for each node. Four node clustering coefficient is calculated as:
\begin{equation}
    C_{4}(v)=\frac{\sum_{(u,w) \in Nei(v)} q_{v}(u, w)}{\sum_{(u,w) \in Nei(v)}\left[deg(u) + deg(w) - q_v(u,w) - 2\mathbb{I}(u \in Nei(w))\right]}
\end{equation}
where $q_v(u,w)$ defined as the number of common neighbors of $u$ and $w$, not counting $v$.
Aggregating these features across the whole graph would give information on distribution of $4$ node orbits of the graph,
but this provides more localized information across the graph that is nonetheless difficult or impossible for a GIN to examine otherwise \citep{chen2020can}.

\paragraph{Choice of Augmentations}
\label{sec:choosing-augmentations}
\citet{wang2022chaos}, building off work of \citet{wang2020understanding}, study contrastive learning in general settings (with an eye towards vision applications),
and provide an intriguing bound for the performance of downstream classifiers.
To explain it,
consider the ``augmentation graph'' of the training samples:
if $G_i$ is the $i$th training example
and $G_i^+$ is a random augmentation of that example,
we connect the nodes $G_i$ and $G_j$ in the augmentation graph
if there are feasible augmentations $G_i^+$ and $G_j^+$ such that $G_i^+ = G_j^+$.\footnote{%
  Technically, the augmentations we use (described in \cref{sec:using-gcl}) would result in a densely-connected augmentation graph:
  for instance, each graph has some vanishingly small probability of being reduced to an empty graph.
  We can instead think about the augmentation graph based on the augmentations from a high-probability set for each training example.
}
\citeauthor{wang2022chaos} argue that downstream linear classifiers are likely to succeed if this augmentation graph has stronger intra-class connectivity than it does cross-class connectivity,
proving a tight connection between the two under a particular setting with strong assumptions about this and related aspects of the setup.

Given the connection between the distribution metrics discussed in \cref{sec:eval-methods-on-latents} and classifier performance \citep[e.g.][Section 4]{liu:deep-testing},
if we accept the argument of \citet{wang2022chaos},
evaluation metrics based on a contrastively trained graph representation will give poor values (good classifier performance) when generated samples are not well-connected to real samples in the augmentation graph,
and vice versa.
If we choose augmentations appropriately, this is sensible behavior.

\paragraph{Enforcing Lipschitz Layers in Representation Networks}
\label{sec:adding-smoothness}
The prior line of reasoning also suggests that we should choose augmentations that are capable of making real graphs look like one another.
Edit distances between graphs, however, are typically large on the datasets we consider,
and so augmentations based on adding or deleting individual nodes and/or edges will struggle to do this.
The same is true for many of the augmentations used on images,
except -- as \citet{wang2022chaos} note -- for crop-type augmentations,
where e.g.\ two different car pictures might become quite similar if we crop to just a wheel.
On graphs, an analogous operation is subgraph sampling,
which we include in our GraphCL setup;
InfoGraph already naturally looks at subgraph features as a core component.

Taking this line of reasoning as well as the general motivations of contrastive learning further,
it is also natural to think that if we can inherently enforce ``similar graphs'' to have similar representations,
this could improve the process of contrastive learning:
we would save on needing to train the model to learn these similarities,
and it could help decrease the classifier performance for good generative models whose output graphs are legitimately near the distribution of target graphs.

A line of work on GANs in visual settings \citep{arjovsky:towards-principled-gans,wgan,smmdgan,roth:stabilizing,mescheder:converge,spectral-norm}
has made clear the importance of this type of reasoning in losses for training generative models:
the loss should smoothly improve as generator samples approach the target distribution, even if the supports differ.
Viewing model evaluation metrics as a kind of ``out-of-the-loop'' loss function for training generative models -- hyperparameter selection and model development focusing on variants with better evaluation metrics -- suggests that these kinds of properties may be important for the problem of evaluation as well.

We thus explicitly enforce the layers of our GIN to have a controlled Lipschitz constant, similarly to e.g.\ spectral normalization in GAN discriminators \citep{spectral-norm}. To this end, we fix the $\lambda$ Lipschitz factor to $1.0$ in the experiments. For each linear layer with weights $\mathbf{W}_\ell$, we use projected gradient descent; after each update on the weights, if $\|\mathbf{W}_\ell\| > 1.0$, we update the value of the weights to $\mathbf{W}_\ell = \frac{\mathbf{W}_\ell}{\|\mathbf{W}_\ell\|}$.
This guarantees small changes in the graphs, such as adding/removing edges, or change in the input features, will not drastically change the final representation.

If we assume node representations for input of a GIN layer has a maximum norm of $\lVert B \rVert$, changing one connection changes two input messages in the graph by at most $\lVert B \rVert$. Each MLP of the GIN has a Lipschitz constant of $1$, between the normalization above and using ReLU activations. As a result, each node's representation will change by at most $\lVert B \rVert$. 
As the GIN concatenates the representations from all layers, this difference will propagate in the layers and increase as the number of layers increase. For fixed-depth networks, however, such as our depth-3 networks here, controlling this Lipschitz constant will still give reasonable control of the changes in the final representation.

\section{Experimental Results}
In all of the experiments we train the model on the full dataset in an self-supervised manner. Following \citep{thompson2022evaluation}, we take the dataset and make perturbations on the dataset and see what is the trend in the measurements as the perturbation degree increases. We denote perturbation degree with $r$, and define it for each type of perturbation. We use these type of perturbations:

\begin{figure*}[ht!]
\captionsetup[subfloat]{farskip=-2pt,captionskip=-8pt}
\centering
\subfloat[][]{\includegraphics[width = 2.55in]{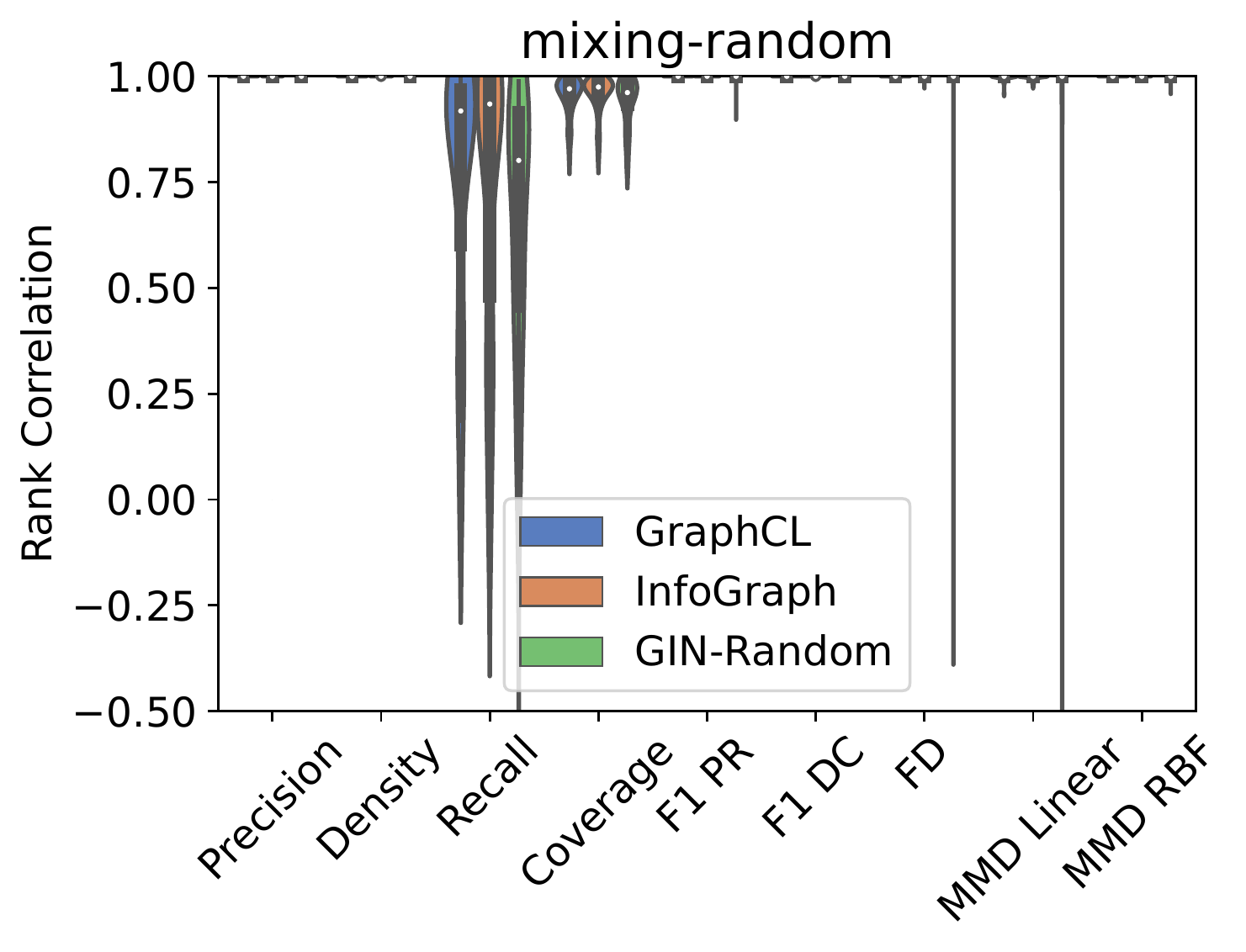}} 
\subfloat[][]{\includegraphics[width = 2.55in]{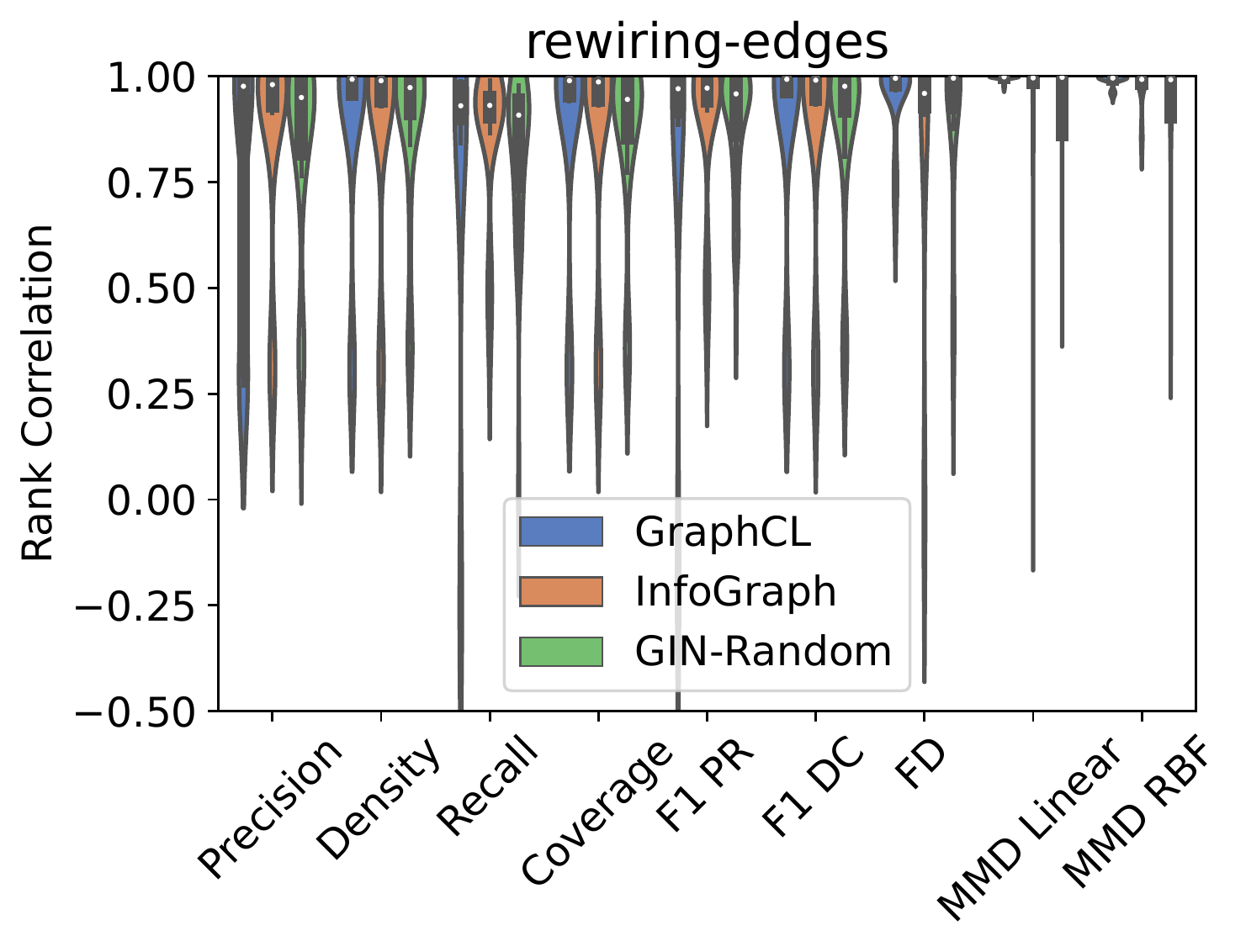}}
\\
\subfloat[][]{\includegraphics[width = 2.55in]{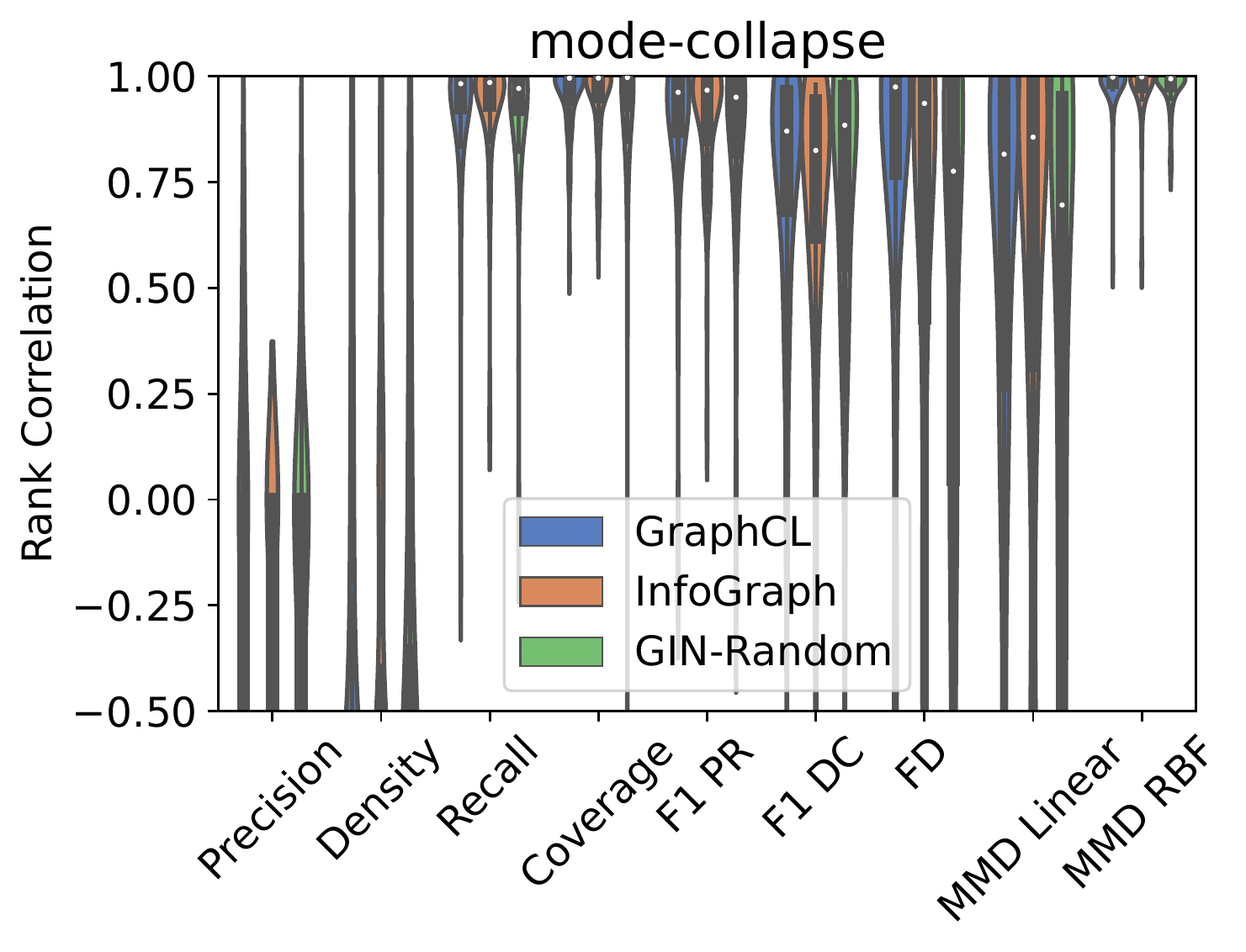}}
\subfloat[][]{\includegraphics[width = 2.55in]{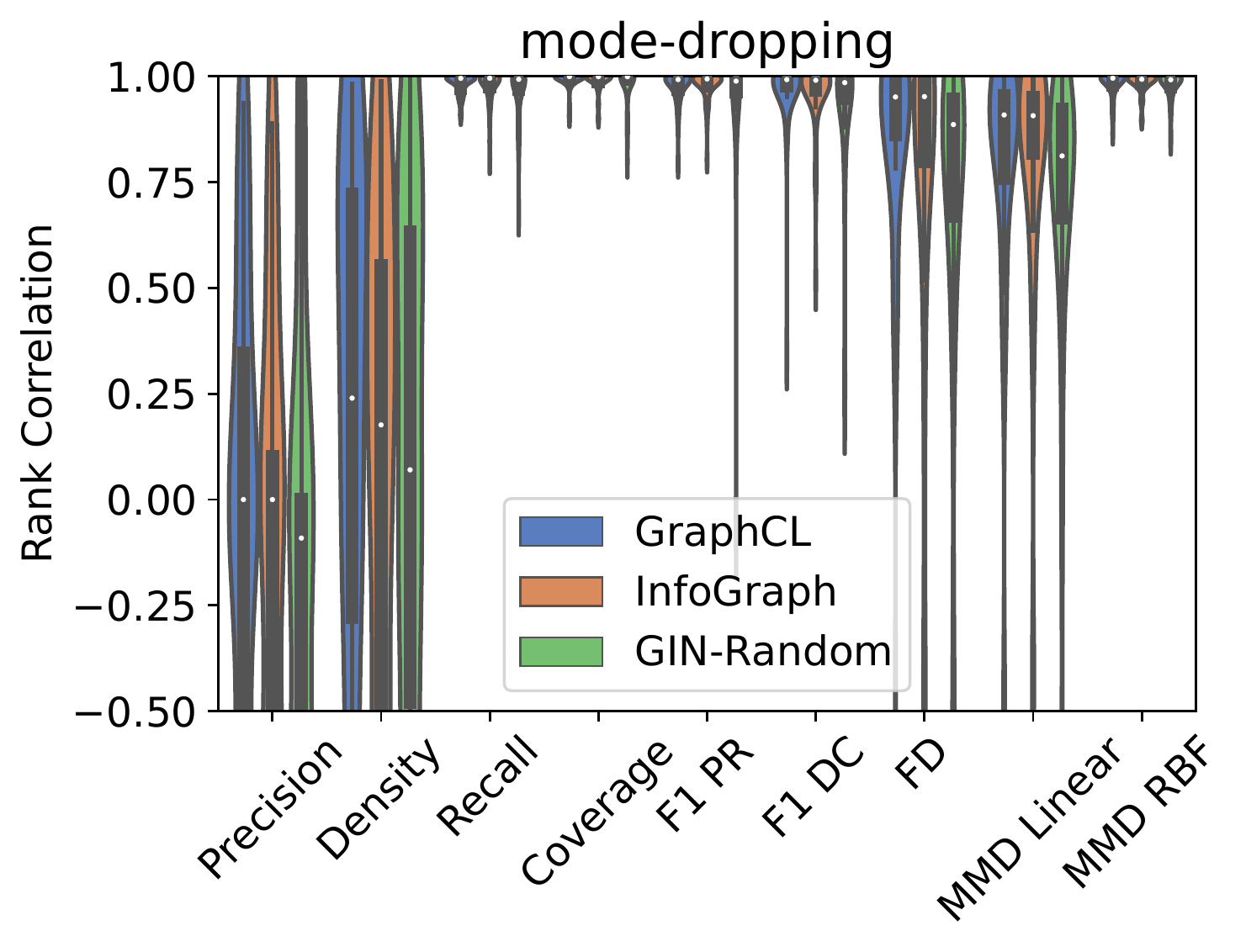}}
\caption{Experimental results for the pretrained models versus random GIN, without structural features.}
\label{fig:no_feat}
\end{figure*}

\begin{itemize}
    \item \textbf{Mixing-Random:} In perturbation with ratio $r$, we remove $r$ chunk of the reference samples, and replace them with Erdős–Rényi (ER) graphs with the same ratio of edges.
    \item \textbf{Rewiring Edges:} This perturbation, rewires each edge of the graph with probability $r$. Each rewired edge, will change one of the sides of the edge with equal probability to another node that is not already connected to the stable node.
    \item \textbf{Mode Collapse and Mode Dropping:} For these perturbation, first we cluster the graphs using the WL-Kernel. First, we choose $r$ ratio of clusters, then, for mode collapse, we replace each graph on that dataset with the center of the cluster. For, mode dropping, we remove the selected clusters and then for making size of the dataset fixed, we randomly repeat samples from other clusters.
\end{itemize}

For each experiment, we measure the Spearman rand correlation between the perturbation ratio, $r$, and the value of the measurement. For measurements that supposed to decrease by the increase of ratio, we flip the values. As a result, in all experiments higher is better. We gather the results among different datasets and several random seeds and plot them for distribution of the correlations. For detailed experiments on the individual datasets see \ref{app:further_experiments}.

\paragraph{Datasets:} Following \cite{thompson2022evaluation}, we use six diverse datasets ( three synthetic and three real-world) that are frequently used in literature
including: (1) \textbf{Lobster} \cite{dai2020scalable} consisting of stochastic graphs with each node at most 2-hops away from a backbone path, (2) 2D \textbf{Grid} 
graphs \cite{dai2020scalable,you2018graphrnn,liao2019efficient}, (3) \textbf{Proteins} \cite{dobson2003distinguishing} where represents connectivity among amino acids 
if their physical distance is less than a threshold, (4) 3-hop \textbf{Ego} networks \cite{you2018graphrnn} extracted from the CiteSeer network \cite{sen2008collective} representing citation-based connectivity among documents, (5) \textbf{Community} \cite{you2018graphrnn} graphs generated by Erdős–Rényi model \cite{erdos1960evolution}, and (6) \textbf{Zinc} \cite{irwin2012zinc} is a dataset of attributed molecular graphs which allows to study sensitivity of metrics to changes in node/edge feature distributions.  We follow the exact protocols used in  \cite{thompson2022evaluation, dai2020scalable,you2018graphrnn,liao2019efficient} as follows. We randomly sample 1000 graphs from Zinc dataset and use one-hot encoding for edge/node features. For community dataset, we set $n=|\mathcal{V}|/2$, $p=0.3$, and add $0.05|\mathcal{V}|$  inter-community edges with uniform probability. For all datasets, we use the node size range indicated in Table \ref{table:dataset} in Appendix.

\paragraph{Contrastive Training Versus Random GIN without Structural Features:} In the first experiment, we compare the methods without adding structural features. Results can be seen in Figure \ref{fig:no_feat}. In general we can see that in most measurements pretraining shows superior performance compared to the random network. In this experiments we use $10$ random seeds and gather data on all datasets. In \ref{app:further_experiments}, same results separated for each dataset can be seen. In our experience, pretraining shows near perfect results on larger datasets, but for Lobster and Grid datasets correlations are not near to $1$. Our observation is these measurements are moving with the correct trend up to some perturbation ratio; but for example precision/recall become zero after some perturbation and keeps still. Our intuition here is because of highly regular structures in these datasets, model learns instantly to discrimate the real graphs from the fake ones very easily. And after small amount of perturbation the perturbed graphs are all very far from the reference graphs.

\paragraph{Contrastive Training Versus Random GIN with Node Degree Features:} In this experiment, similar to the first one, we compare random and pretrained models. The difference is that we use just node degree features for the nodes, the easiest and fastest structural information that could be added to the model. Results are provided in Figure \ref{fig:deg_feat}.
\begin{figure*}[ht!]
\captionsetup[subfloat]{farskip=-2pt,captionskip=-8pt}
\centering
\subfloat[][]{\includegraphics[width = 2.55in]{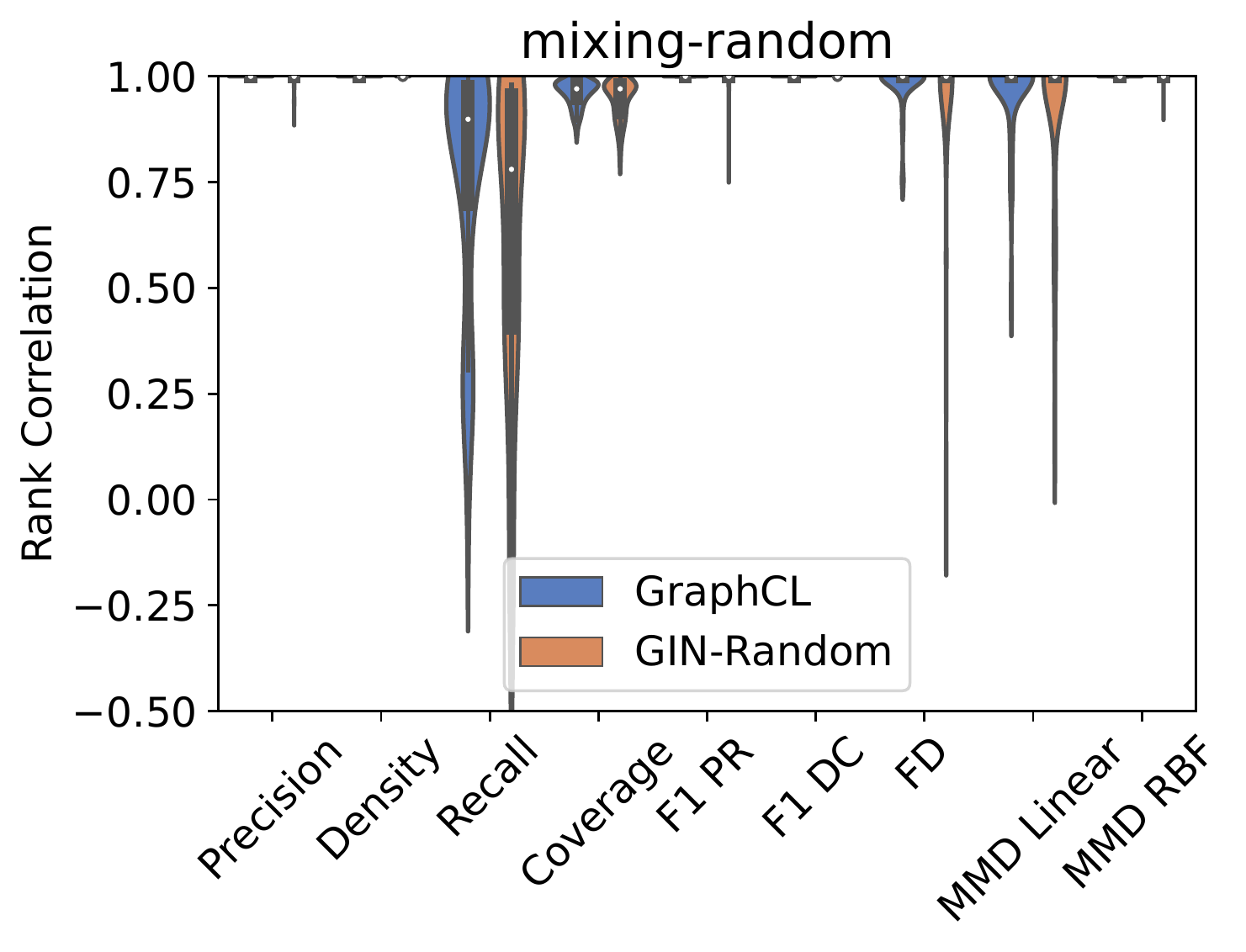}} 
\subfloat[][]{\includegraphics[width = 2.55in]{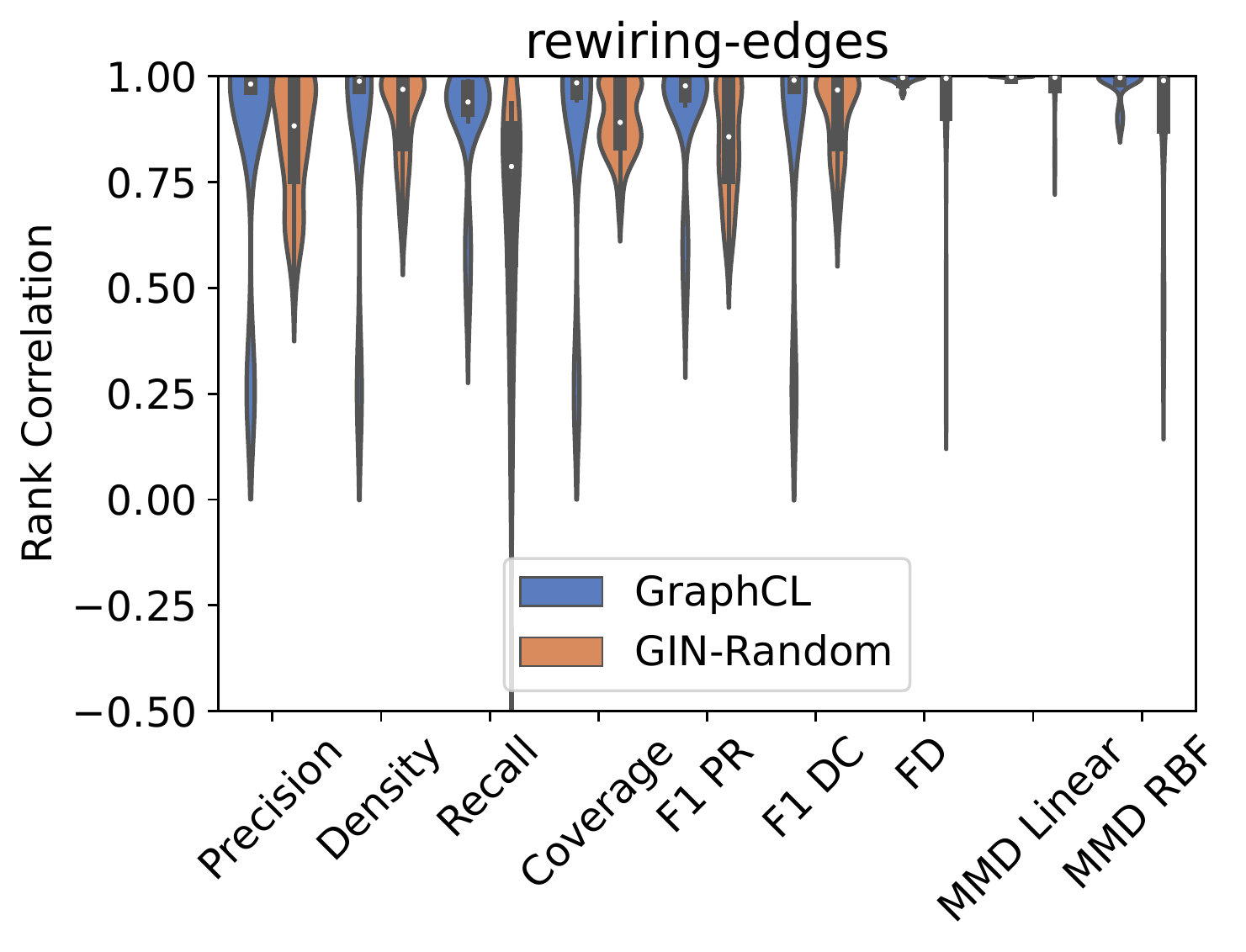}}
\\
\subfloat[][]{\includegraphics[width = 2.55in]{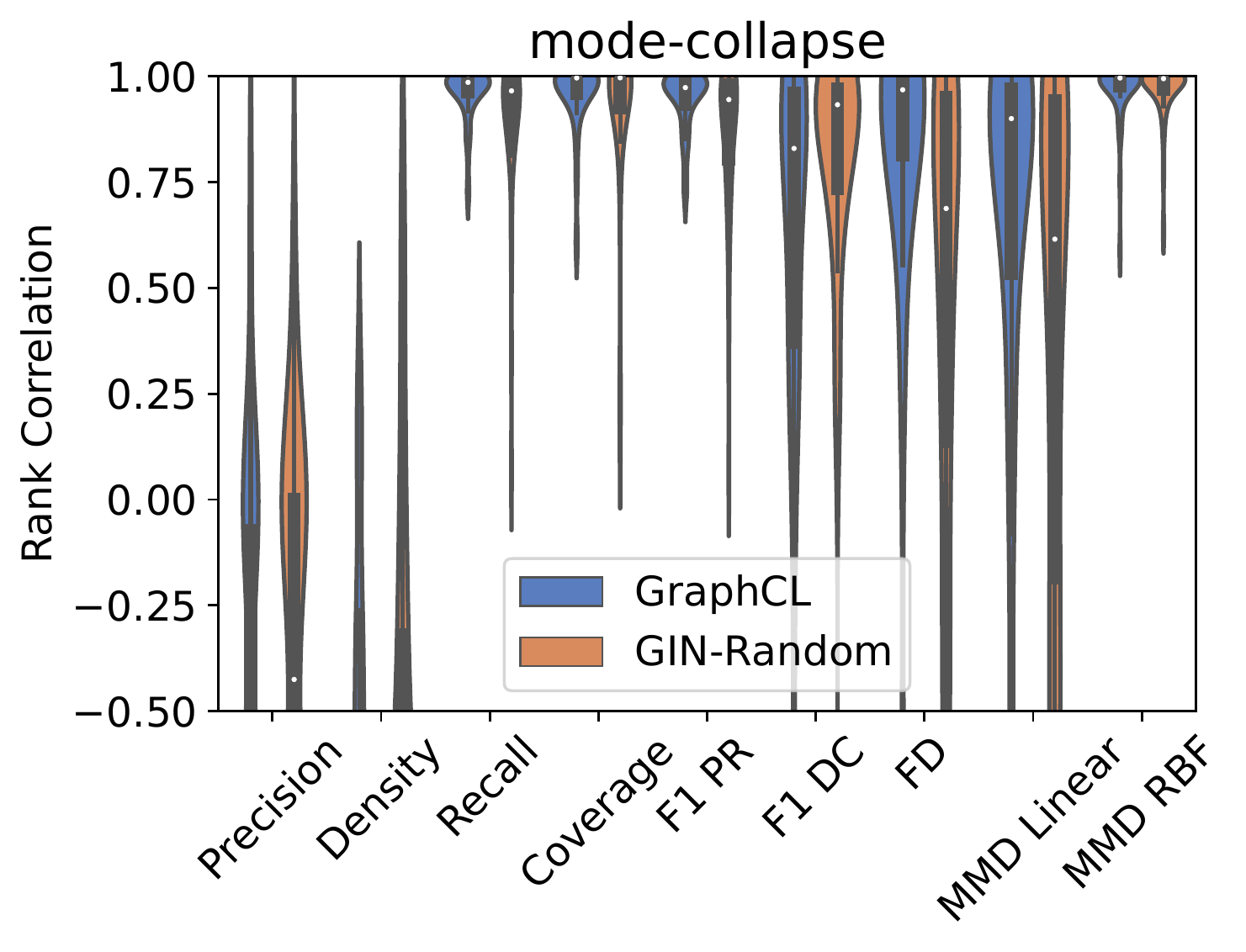}}
\subfloat[][]{\includegraphics[width = 2.55in]{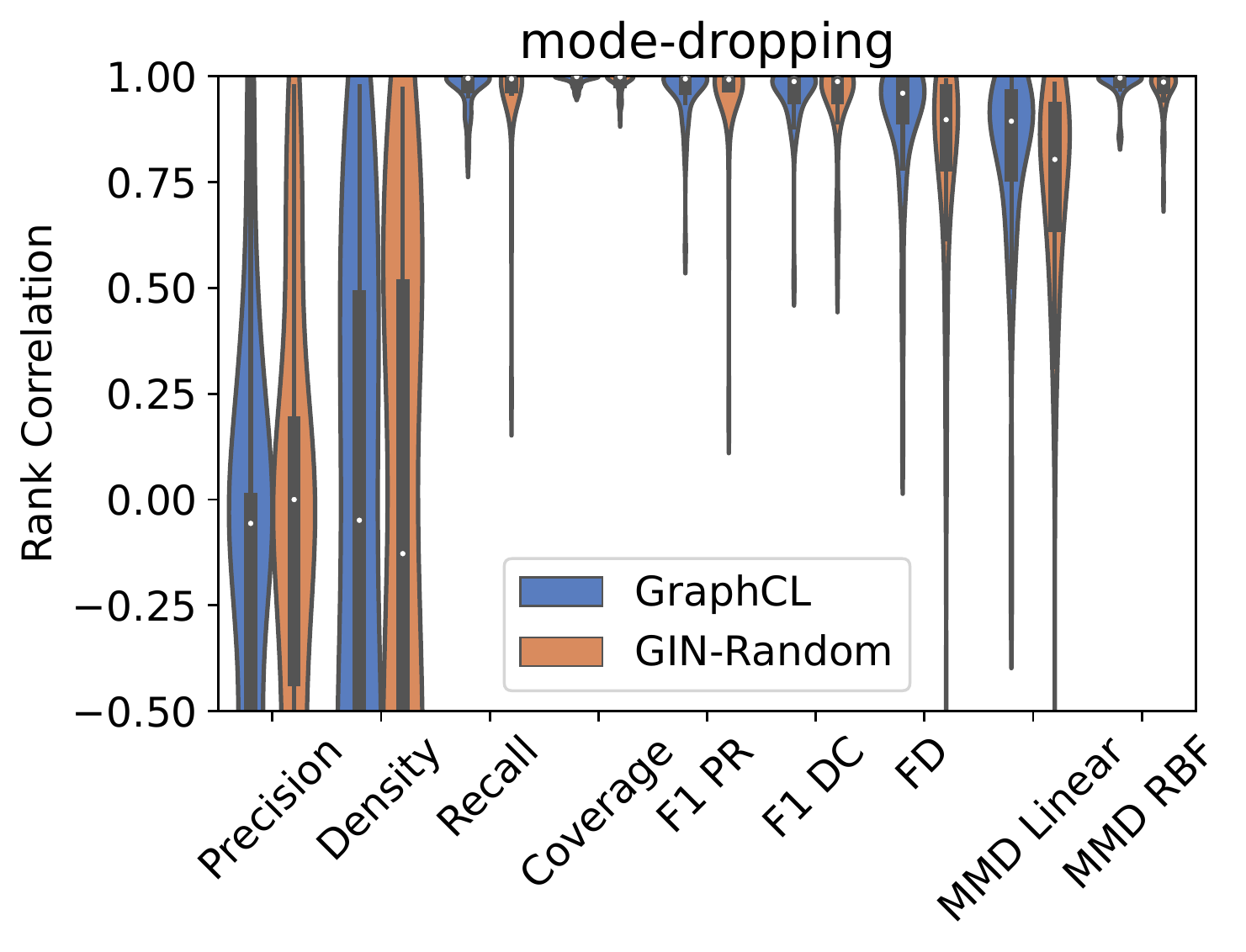}}
\caption{Experimental results for the pretrained models versus random GIN, with node degree features.}
\label{fig:deg_feat}
\end{figure*}
\begin{figure*}[ht!]
\captionsetup[subfloat]{farskip=-2pt,captionskip=-8pt}
\centering
\subfloat[][]{\includegraphics[width = 2.55in]{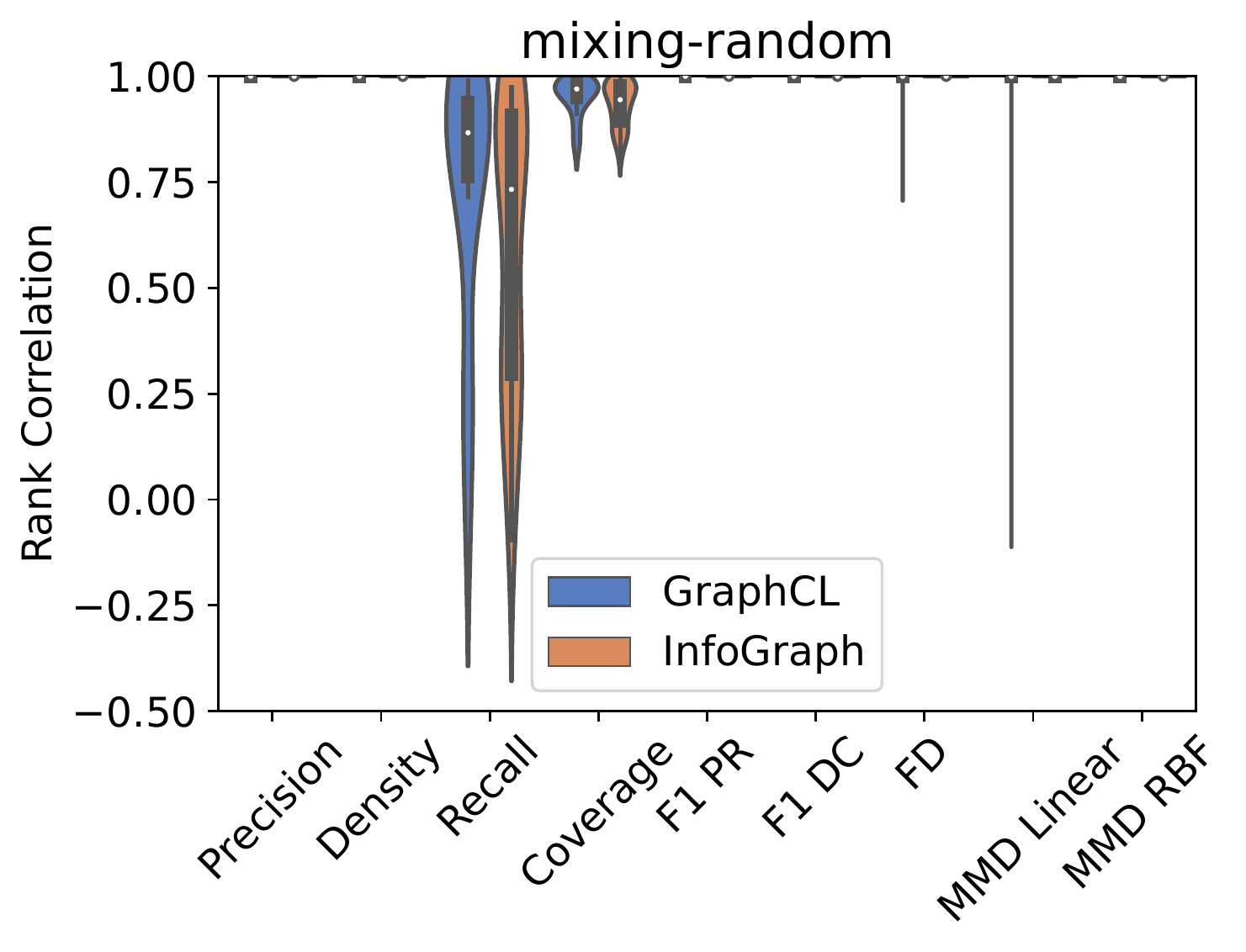}} 
\subfloat[][]{\includegraphics[width = 2.55in]{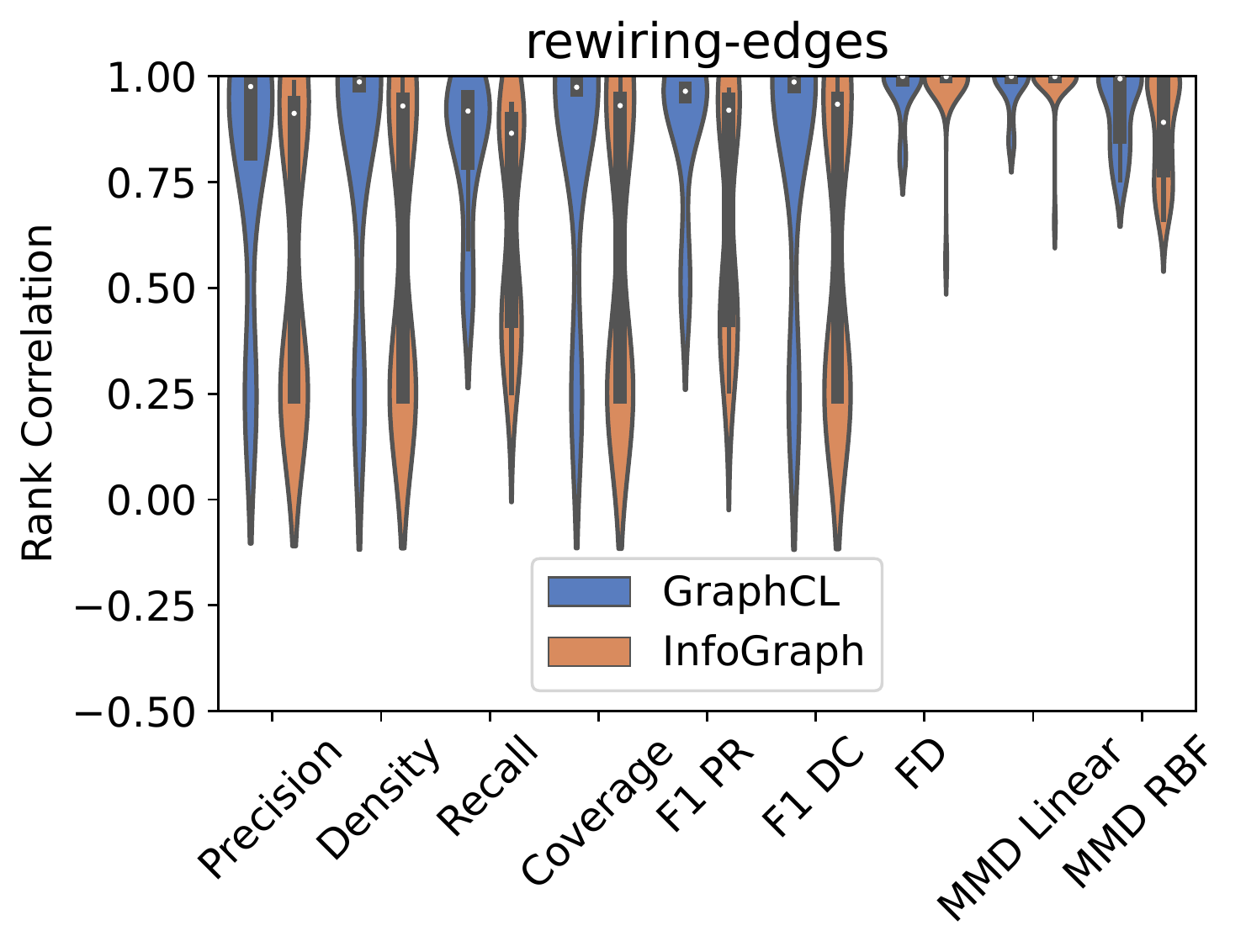}}
\\
\subfloat[][]{\includegraphics[width = 2.55in]{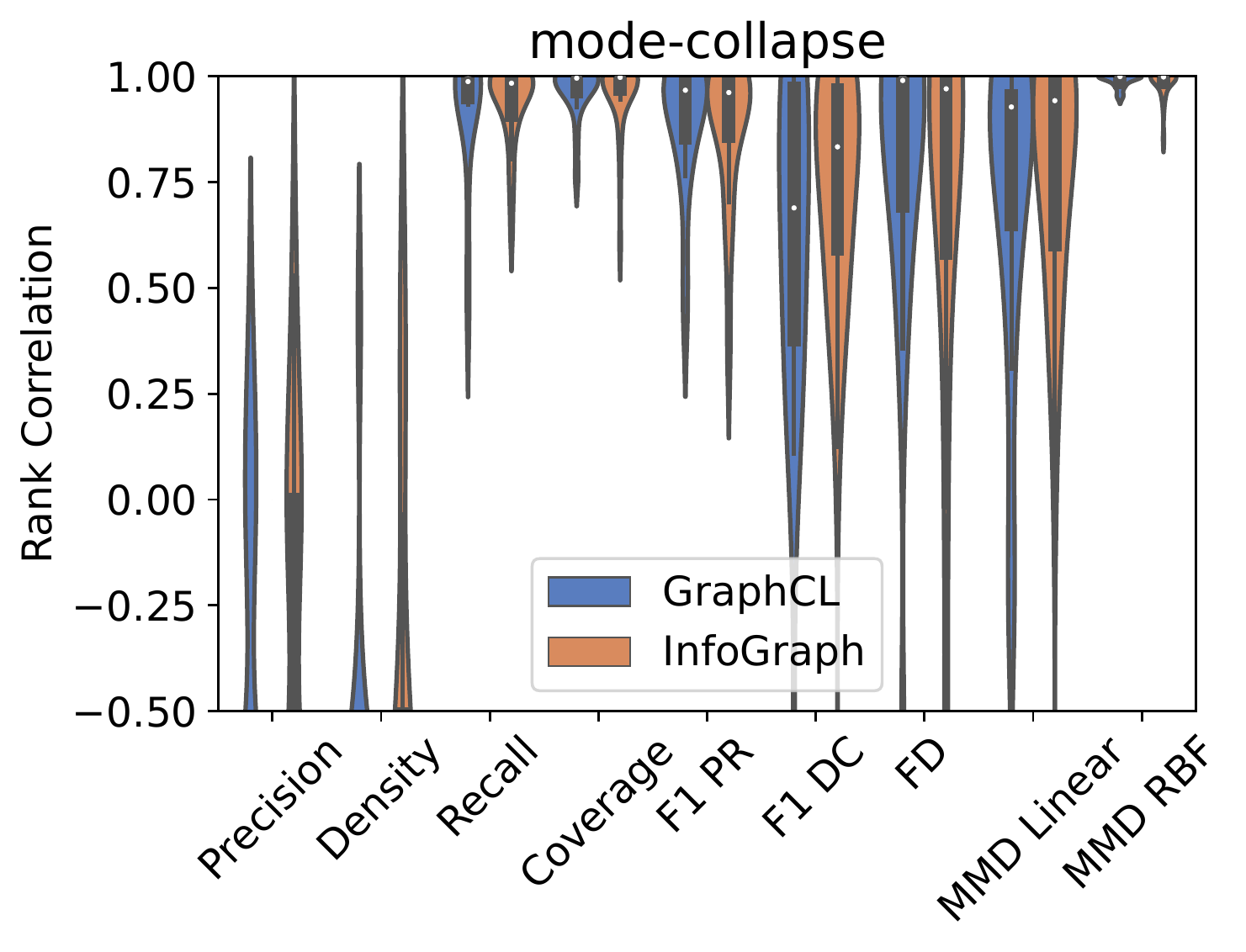}}
\subfloat[][]{\includegraphics[width = 2.55in]{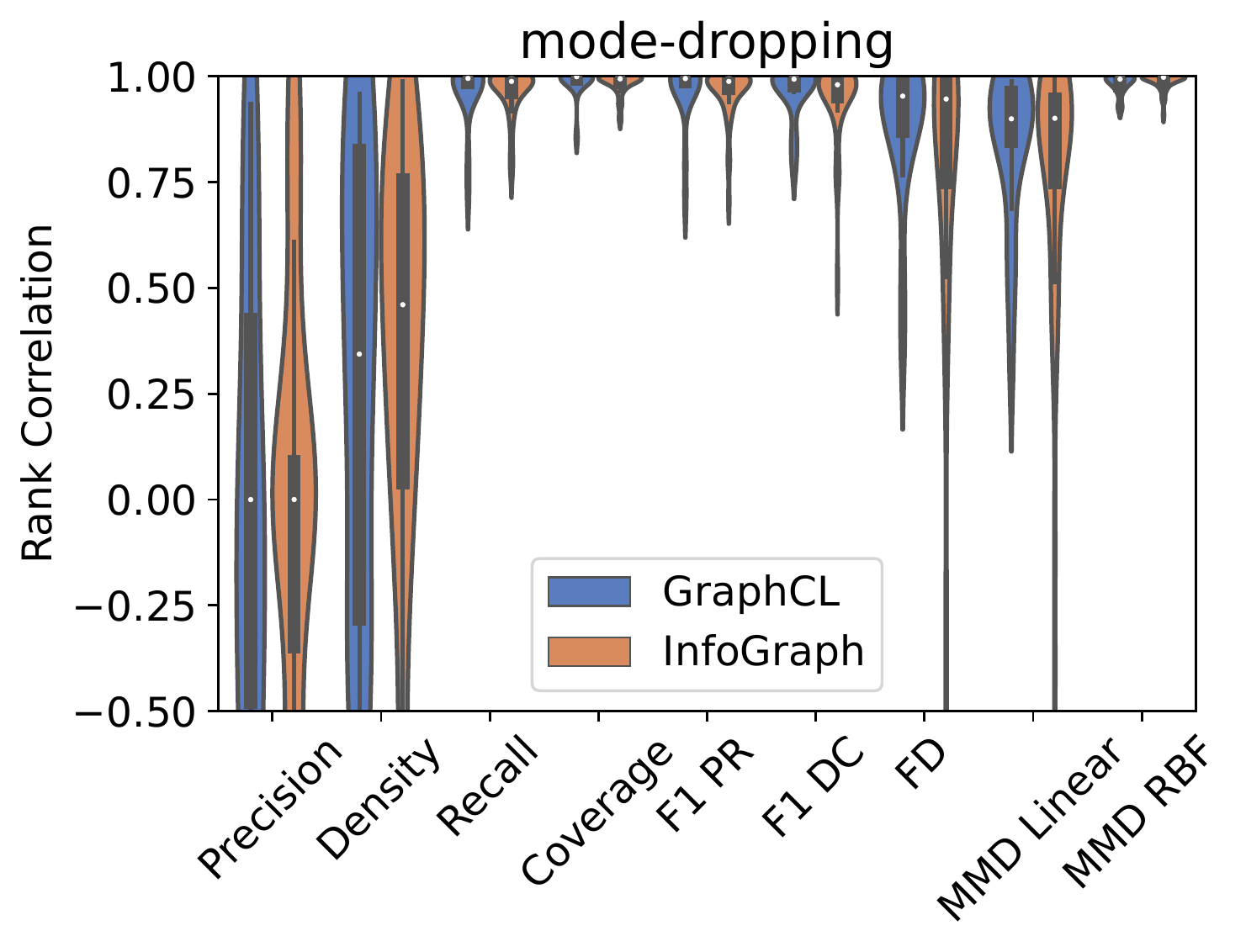}}
\caption{Experimental results for the pretrained models versus random GIN, with structural features.}
\label{fig:clustering_feat}
\end{figure*}
\begin{figure*}[ht!]
\captionsetup[subfloat]{farskip=-2pt,captionskip=-8pt}
\centering
\subfloat[][]{\includegraphics[width = 2.55in]{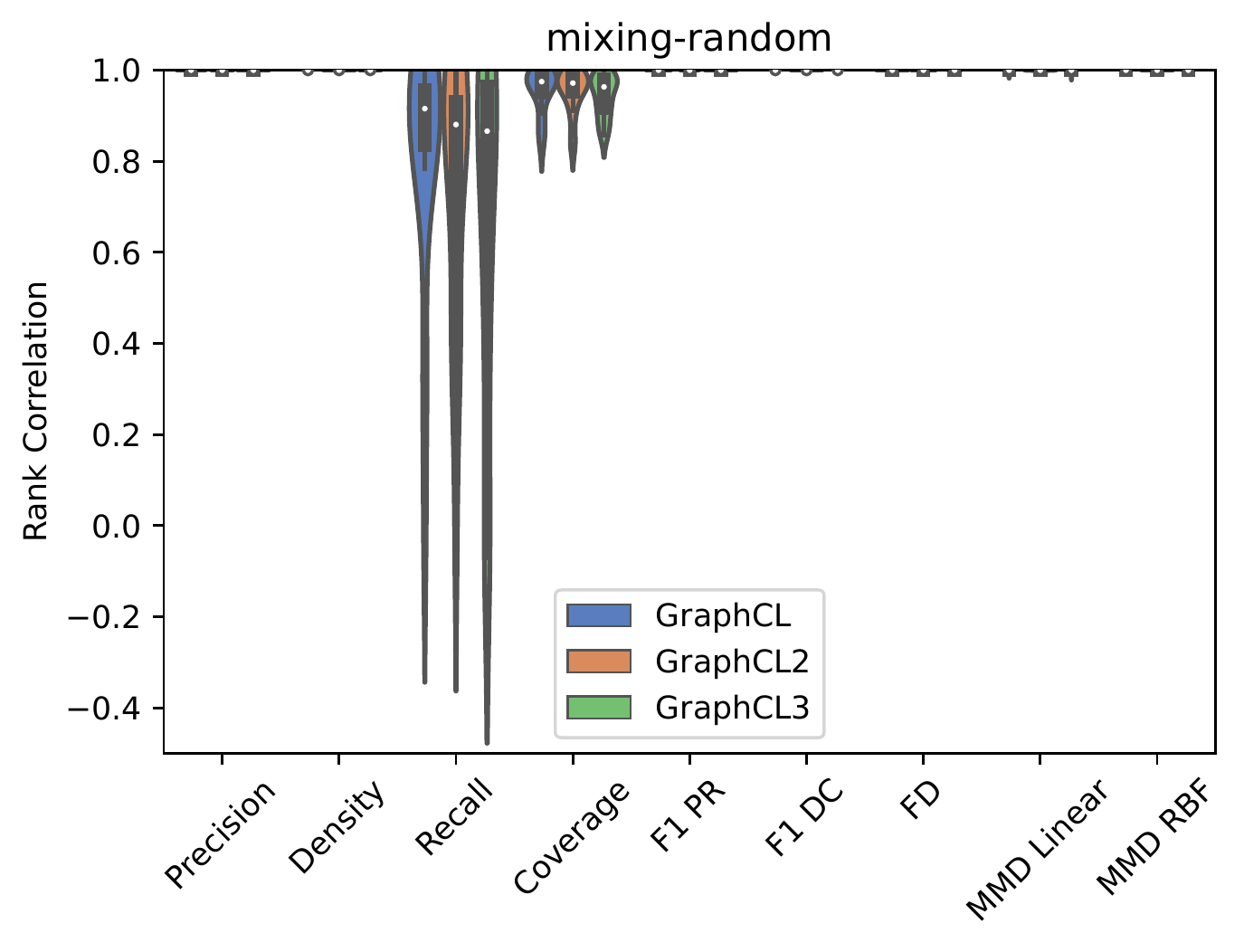}} 
\subfloat[][]{\includegraphics[width = 2.55in]{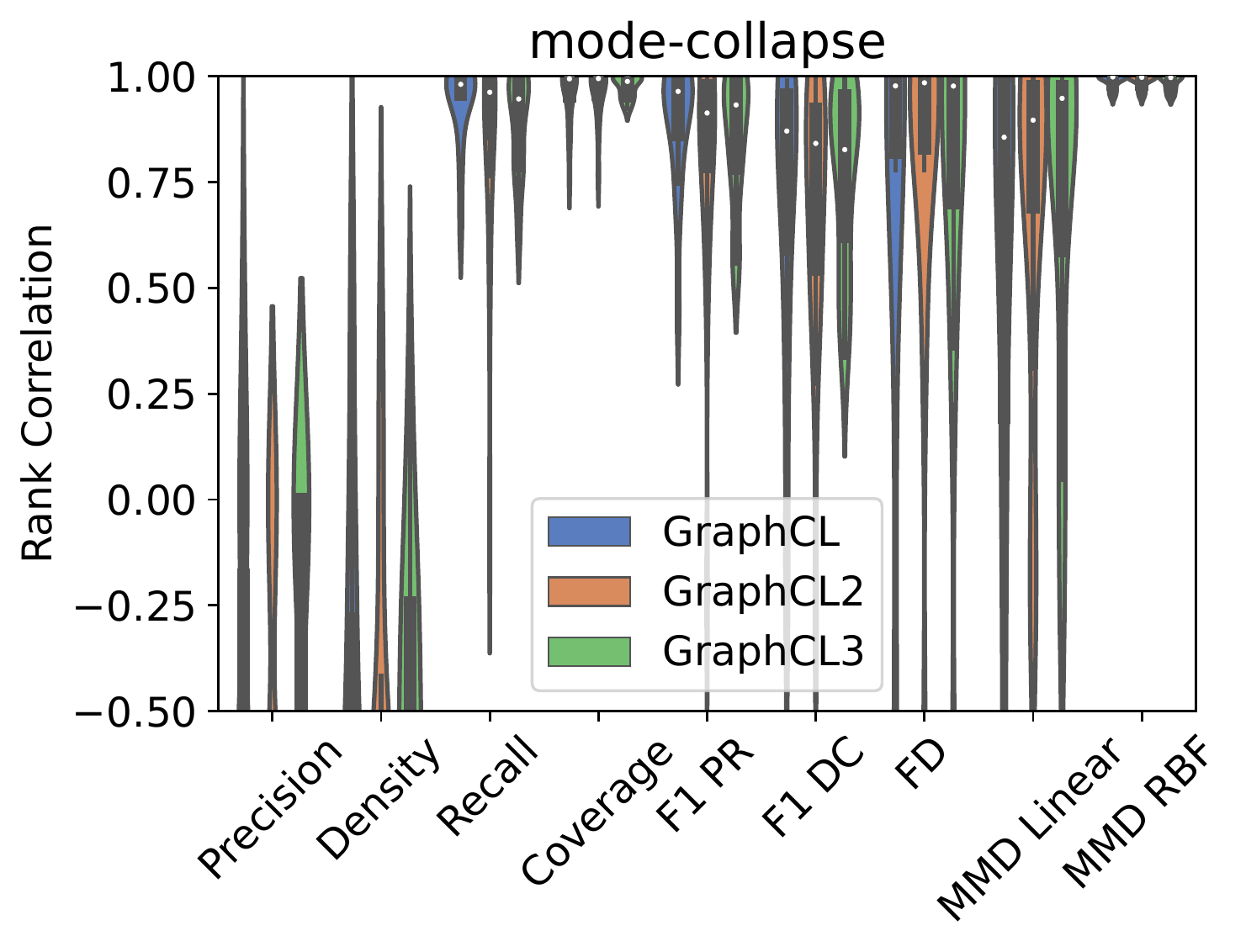}}
\caption{Comparison of the full method on no feature data versus removing layer normalization versus removing subgraph data augmentations on pretraining the network. GraphCL is normal training. In GraphCL2 we do not enforce lipschitzness. In GraphCL3 we remove the subgraph constraints and reduce the probability of node and edge dropping in the augmentations.}
\label{fig:ablation}
\end{figure*}
\paragraph{Contrastive Training with Structural Information of Node Clustering coefficients:} In this experiment we concatenate node degrees and node clustering coefficients for each node feature. Here we can see that adding these features improves the models power for some measurements. However, on small datasets Grid and Lobster results become poorer. Again intuition is more powerful model can easier distinguish main graphs from the fake ones. Break down of the results on the datasets in Appendix \ref{app:further_experiments} proves this point.
\paragraph{Ablation Study:} We analyze how layer normalization and using subgraph augmentations in GraphCL reflects on the final results. We have used no structural feature setup and conducted the experiment for mixing-random and mode-collapse experiments. Figure \ref{fig:ablation}, shows the results on this task. The results prove that both of these improvements are essential for getting better results.

\section{Conclusion}
We have demonstrated that self-supervised pre-training of representations can yield significantly better metrics for graph evaluation than random ones,
particularly when incorporating local graph features with Lipschitz control,
as inspired by theory.
We suggest graph generative modeling papers should consider evaluating with these metrics in addition to or instead of their existing ones.

\begin{ack}
This research was enabled in part by support, computational resources, and services provided by the Canada CIFAR AI Chairs program, the Natural Sciences and Engineering Research Council of Canada, WestGrid, and the Digital Research Alliance of Canada.
\end{ack}

\printbibliography
\newpage

\clearpage
\appendix

\section{Preliminaries}
We now give further description of some background material.
\subsection{Graph Neural Networks}
Suppose we have a set of graphs $\mathbb{G} = \{G_1,\dots\,G_N\}$ where each graph $G=\left(\mathcal{V}, \mathcal{E}, \mathbf{X}, \mathbf{E}\right)$ consists of $|\mathcal{V}|$ nodes and $|\mathcal{E}|$ edges such that an edge $e_{ij} \in \mathcal{E}$ connects nodes $v_i, v_j \in \mathcal{V}$ ($\mathcal{E} \subseteq \mathcal{V} \times \mathcal{V}$).
If the graph contains extra data beyond simply the connectivity structure, initial node features $\textbf{X}\in \mathbb{R}^{|\mathcal{V}| \times d_x}$ and/or edge features $\textbf{E}\in \mathbb{R}^{|\mathcal{V}| \times |\mathcal{V}| \times d_e}$ will also be available, where $d_x$, and $d_e$ denote the dimension of initial node and edge features, respectively.
Message-passing GNNs learn low-dimensional node representations $h_v^{(k)} \in \mathbb{R}^{d_h}, \forall v\in \mathcal{V}$ as follows:
\begin{equation}
    h_v^{(k)}=\phi \left(h_v^{(k-1)}, \bigoplus\limits_{u\in\mathcal{N}(v)} \psi\left(h_v^{(k-1)},h_u^{(k-1)}, e_{uv}  \right) \right)
\end{equation}
where $\mathcal{N}(v)$ is a set of immediate neighbors of node $v$, $\phi$ and $\psi$ are two arbitrary functions, and $\bigoplus$ is the aggregation function. We can use a read-out function to aggregate the learned node representations into a graph representation: $h_G=\mathcal{R}\left(\{h_v|v \in \mathcal{V}\}\right)$. Different instantiations of $\phi, \psi, \bigoplus$ results in different flavors of GNNs such as message passing neural network (MPNN) \cite{gilmer_2017_icml}, graph convolutional network (GCN) \cite{kipf2016semi}, graph attention network (GAT) \cite{velivckovic2017graph}, and graph isomorphism network (GIN) \cite{xu2018powerful}. Particularly, GIN (in which $\phi$ and $\psi$ are designed as MLP and identity functions, respectively) is shown to be as powerful at detecting isomorphism as the standard Weisfeiler-Lehman (WL) test. 

\subsection{Evaluating Graph Generative Models in Latent Space} \label{sec:eval-methods-on-latents}
Suppose we have two sets of graphs $\mathbb{G}^{train} = \{G_1,\dots\,G_S\}$ and $\mathbb{G}^{test} = \{G_1,\dots\,G_M\}$, each sampled from the same data distribution $p\left(G\right)$. Also suppose that we have access to an unconditional graph generative model $g_\phi(.)$, which is trained on $\mathbb{G}^{train}$ to learn the distribution of the observed set of graphs. We sample a set of generated graphs $\mathbb{G}^{gen} = \{G_1,\dots\,G_N\} \sim p_{g_\phi}(G)$ from this model. In order to evaluate the quality of the sampled graphs (i.e., to decide whether the model $g_\phi(.)$ has successfully recovered the underlying distribution $p\left(G\right)$), we can define a measure of divergence $\mathcal{D}\left(\mathbb{G}^{test}, \mathbb{G}^{gen} \right)$ to quantify the discrepancy
between distributions of the real and generated graphs. One robust way to achieve this is
to define the metric on latent vector representation spaces, and expect representations of graphs rather than the original objects. Thus, to use these metrics, we need to train a shared encoder $f_{\theta}(.)$ and then compute the discrepancy as $\mathcal{D}\left(\textbf{H}^{t}, \textbf{H}^{g} \right)$ where $\textbf{H}^{t}=f_{\theta}(\mathbb{G}^{test})\in\mathbb{R}^{M \times d_h}$ and $\textbf{H}^{g}=f_{\theta}(\mathbb{G}^{gen})\in\mathbb{R}^{N \times d_h}$. There are a few such metrics well-studied in visual domains that can differentiate the fidelity and diversity of the model, and which we can adopt in graph domains.

\textbf{Fréchet Distance (FD)} \cite{heusel2017gans} models the graph representations as multivariate Gaussian distributions with mean and covariance $\mu$ and $\textbf{C}$, and then computes the squared Wasserstein-2 distance between them: $\mathcal{D}\left(\textbf{H}^{t}, \textbf{H}^{g} \right)=||\mu_{t}-\mu_{g}||_2^2+\text{Tr}\left(\textbf{C}_t+\textbf{C}_g-2(\textbf{C}_t\textbf{C}_g)^{1/2} \right)$.
FD based on Inception v3 features of images has been shown to correlate with human judgments in some settings. It can also detect intra-class mode collapse \cite{borji2019pros}.
The estimator, however, has low variance but high bias \citep{demystifying}.

\textbf{Precision \& Recall (PR)} \cite{sajjadi2018assessing} measure constructs a hyper-sphere for each representation by extending a radius to its kth nearest neighbor and then aggregates the hyper-spheres to define a manifold. It then computes the precision and recall based on generated samples that fall within the manifold of real samples and vice versa. PR is also shown to correlate with human judgments in visual domain. It can also detect mode collapse and mode dropping.

\textbf{Density \& Coverage (DC)} \citep{naeem2020reliable} is similar to PR but instead of aggregating the hyper-spheres, treats them independently. Density is computed as the average number of real hyper-spheres that a generated sample falls in and Coverage is computed as the average number of generated hyper-spheres that a real samples falls in. DC is shown to be more robust compared to PR.

\textbf{Maximum Mean Discrepancy (MMD)} \cite{gretton:mmd} compares two distributions of any type, based on item-level comparison by a kernel function (e.g., polynomial kernel $\mathcal{K}(x_i, x_j)=(x_i^{\top}x_j+1)^p$, Gaussian kernel $\mathcal{K}(x_i, x_j)=\text{exp}\left(- \tfrac{d(x_i, x_j)^2}{2\sigma^2} \right)$, or linear kernel $\mathcal{K}(x_i, x_j)=x_i^\top x_j$).
The usual estimator is
\begin{equation}
    \text{MMD}\left(\textbf{H}^{t}, \textbf{H}^{g} \right) = \frac{1}{M^2}\sum_{\substack{i,j=1 \\ i\neq j}}^{M}\mathcal{K}\left(h_i^t, h_j^t \right) +
    \frac{1}{N^2}\sum_{\substack{i,j=1 \\ i\neq j}}^{N}\mathcal{K}\left(h_i^g, h_j^g \right) -
    \frac{2}{NM}\sum_{i=1}^{N}\sum_{\substack{j=1 %
    }}^{M}\mathcal{K}\left(h_i^g, h_j^t \right)
.\end{equation}

\section{Proofs}
\label{app:proofs}

\begin{figure}
    \centering
    \includegraphics[width=13cm]{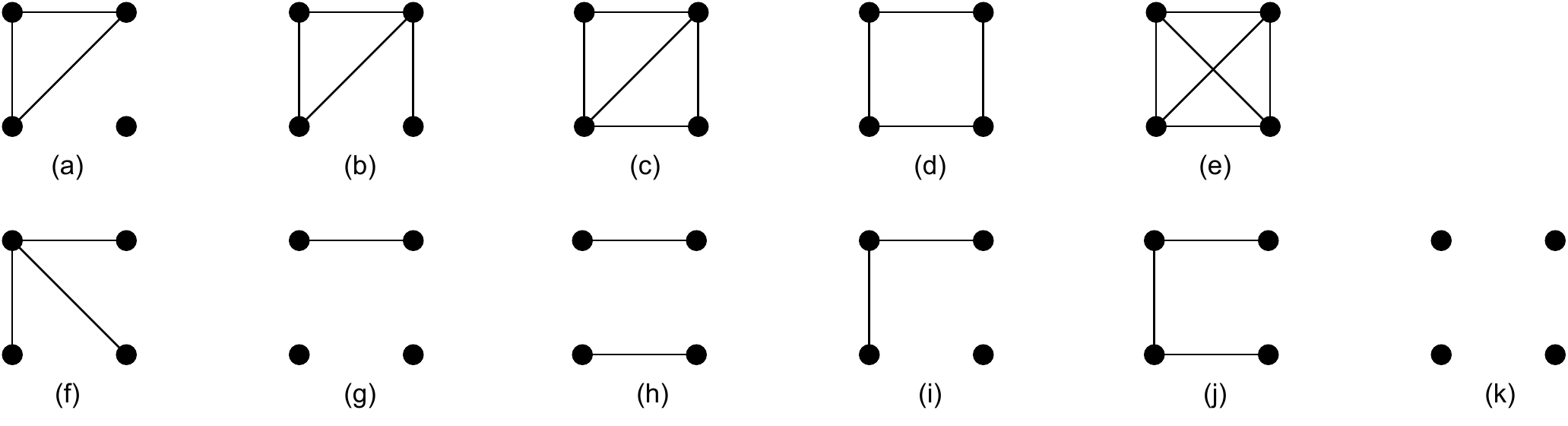}
    \caption{Different possible $4$ node orbits.}
    \label{fig:4_node_orb}
\end{figure}

\begin{figure}
        \centering
        \includegraphics[width=4.5cm]{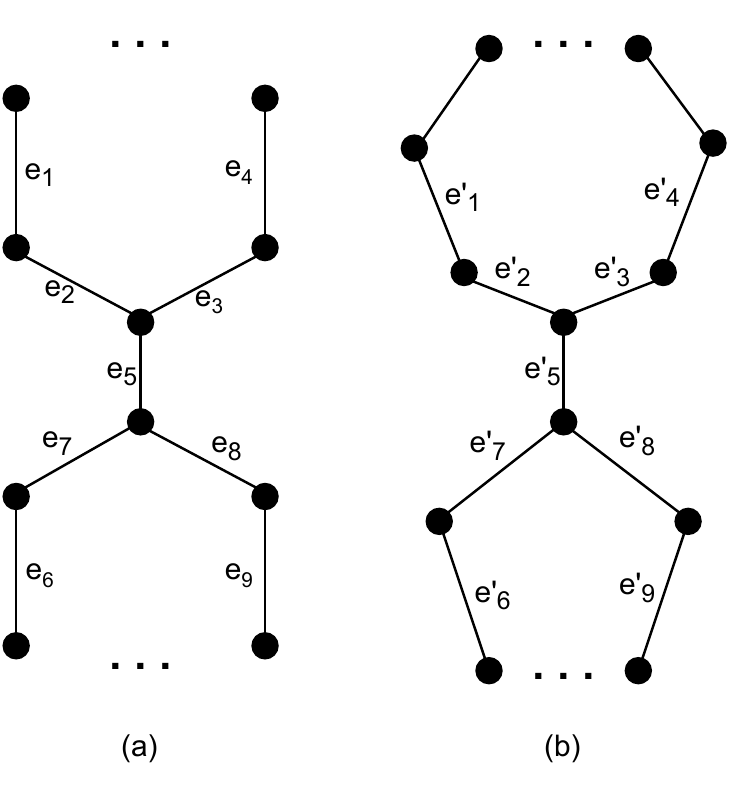}
        \caption{a and b are $\mathcal{C}_{a,b}$ and $\mathcal{C}_{c,d}$ respectively. We use naming from this figure for the proof.}
        \label{fig:iso_proof}
    \end{figure}

Let $\mathcal{C}_{i,j}$ denote the graph constructed by taking disjoint cycles of size $i$ and a cycle of size $j$, then connecting the two with a single edge between one node from each side.
\begin{prop}
For any $a, b, c, d$ satisfying $4<a<c<d<b$ and $a+b=c+d$, $\mathcal{C}_{a,b}$ and $\mathcal{C}_{c,d}$ are not distinguishable using local metrics, but there exists a GNN that can distinguish the two.
\end{prop}

\begin{proof}
First we will prove that local metrics can not identify these two class of graphs from each other:

Both graphs have $a+b-2$ nodes of degree $2$ and $2$ nodes of degree $3$ so the distribution of degrees is equal in both graphs. For clustering coefficient, none of the graphs has a cycle of size $3$, so it is $0$ for both graphs.

We will count the number of different possible orbits of $4$ nodes (all shown in \cref{fig:4_node_orb}), and show the count is the same between two graphs.

Orbits a-e occur in neither of the two graphs, so their counts are the same.

Orbit f appears just twice in both graphs.

For g and h: Consider an edge in an arbitrary graph, if we remove the edge and all nodes connected to either sides of the edge, in the remaining graph selection of each two connected nodes with the first edge will make pattern h, and two non-connected with the removed edge will make pattern g. And these are all patterns g and h that we can see in these graphs. Now, let us assume a mapping from edges of $\mathcal{C}_{a,b}$ to $\mathcal{C}_{c,d}$, by $e_i \leftrightarrow e'_i$ for the named edges in figure~\ref{fig:iso_proof}. And assume any arbitrary mapping for other edges. If we call this mapping $f$, we can see that removing edge $e$ and all its neighbors, leaves the same number of edges and nodes as removing $f(e)$ and its neighbors for the second graph. This results as having the same number of patterns of g and h in both graphs.

For pattern i, if we assume two connected edges and again remove all the nodes connected to the nodes of these edges, any of the remaining nodes + two removed edges will shape pattern i. Again by the same mapping $f$, we can see that removing any of these two connected edges from the first graph leaves the same number of nodes as removing from the second graph.

For pattern j, we can divide these patterns into two groups: 1) Patterns that do not include $e_5$ or $e'_5$. 2) Patterns that include these edges. For the first 1, we can assume we have two separate cycles on each of the graphs. Each cycle of size $n$ has $n$ of these patterns, so first graph has a total of $a+b$ of these patterns, and the second graph has a total of $c+d = a+b$ of these patterns. So, these class of patterns have same amount of repetition in two graphs. For (2), we can easily see from the~\ref{fig:iso_proof} there are $8$ of them in each of the graphs. So, the number of patterns of this type is also equal in both graphs.

For pattern k, as two graphs have same number of nodes, we have $n\choose{k}$ possible selection of four nodes, and as the number of patterns of type a-j is equal between the graphs, pattern k should also have same number of repeat in both graphs.

As a result, the distribution of local metrics is exactly same between two classes of graphs, and thus these two graphs can not be distinguished using these metrics.

Now, we prove that with a deep enough GNN, and good set of weights, we can distinguish these two graphs. Assume all nodes have no information more than their degrees, at first all of the nodes except two sides of bridge in each graph have the same message. There are two messages of degree 3, and rest of degree 2. In cycle with size $a$, after $a$ steps, the initial node with degree 3, will finally receive a message that has started from itself, and so it would be different from any message that are already spread in 3 other circles. As a result, there are at least two nodes with different messages between two graphs after $a$ steps. Thus two graphs are distinguishable using this GNN.
\end{proof}

For \cref{fig:local_iso} (b), it is well-known that these two graphs can not be identified using WL-test, and thus results of \citet{xu2018powerful} imply that GNNs of order $1$ can not distinguish these two graphs. On the other hand, for local structures, one of the graphs has two triangles and the other one has none. As a result clustering coefficients of the nodes will be different.

\section{Implementation Details and Hyperparameters}
\label{app:hyperparam}

In addition to the details here,
code is available in the supplementary materials, and will be made public with release of the paper.

 \subsection{Training Self-Supervised Methods}
 For training GraphCL networks, we use node dropping and edge dropping probabilities of $0.1$ for both. We use random walks of length $10$ as subgraphs. We do not use node feature masking augmentations; here essentially all the datasets except for the ZINC do not have any node features, and we use constant value of $1.0$ for their node features. Removing just constant would leave some nodes without any features. For the experiments with structural features, the point was to see how the structural features work, as a result again we do not use feature dropping here as well. During the training, we use structural features as default node features; this means that after augmentations like node/edge dropping, we do not recalculate the features for the new graph. This makes training process faster, but can affect the learning process in GraphCL only.

The hyperparameters used for GINs are provided in \ref{table:hyperparams}. These hyperparameters are shared among GraphCL, InfoGraph and Random GIN models. For other hyperparameters of the self-supervised training we follow the instructions from \citet{you2020graph, sun2019infograph}. We train the self-supervised methods and evaluate them based on their criteria loss. After finding proper hyperparameters for that dataset, we train the model with those hyperparameters and among all of the experiments for that dataset we use the same model (no retraining for each experiment).

\begin{table}[h!]
\centering
\caption{Hyperparameters of the networks.}
\label{table:hyperparams}
\begin{tabular}{lcccccc} 
\toprule
                      & \textbf{Lobster} & \textbf{Grid} & \textbf{Community} & \textbf{Ego} & \textbf{Proteins} & \textbf{ZINC}  \\ 
\midrule
Number of GIN Layers  & 2                & 3             & 3                  & 3            & 3                 & 3              \\
GIN Hidden Layer Size & 16               & 16            & 32                 & 32           & 32                & 32             \\
Epochs                & 40               & 20            & 100                & 100          & 100               & 100            \\
Lipchitz Factor       & 1.0              & 1.0           & 1.0                & 1.0          & 1.0               & 1.0            \\
\bottomrule
\end{tabular}
\end{table}

\subsection{Random GINs}
Experiments in \citep{thompson2022evaluation} show that results on random GINs tend to have very small variance among different setups of the structures of the GINs. For fair evaluation, we used the same GIN hyperparameters among all of the models for the same dataset. We use batch norm and orthogonal initialization as instructed in the main paper. For each experiment with a different seed, we use a different randomly initialized GIN. For experiments with structural features, we give the same features as input for the Random GIN as well. Random GIN with no structural features will be the same as the method proposed in \citep{thompson2022evaluation}. %

\subsection{Benchmark Experiment Setups}
We follow the exact instruction and \hyperlink{https://github.com/uoguelph-mlrg/ggm-metrics}{public code} released for \citep{thompson2022evaluation} for the benchmarks. Step sizes for increasing the $r$ is set to $0.01$. For mixing random we remove $r$ ratio of the graphs and for each Graph $G = (V,E)$ that has been removed we replace it with and Erdős–Rényi (ER) graph with same number of nodes and probability of connection as $p=\frac{|E|}{\binom{|V|}{2}}$. Rewiring probabilities are increasing from $0.0$ to $1.0$ with step size of $0.01$. The number of steps for mode-dropping/mode-collapse experiments is as the number of clusters found by the clustering algorithm. Experiments for no structural features and with degree features are conducted with ten random seed for each dataset. For the clustering features we have used three random seed per dataset. Ablation studies are done using $5$ random seeds per dataset.

\section{Further Experiment Results}
\label{app:further_experiments}

\Cref{table:dataset} gives details of the datasets.
\cref{fig:no_struct_feats_community,fig:no_struct_feats_ego,fig:no_struct_feats_lobster,fig:no_struct_feats_proteins,fig:no_struct_feats_zinc,fig:no_struct_feats_grid,table:no_struct_feats_community,table:no_struct_feats_ego,table:no_struct_feats_lobster,table:no_struct_feats_proteins,table:no_struct_feats_zinc,table:no_struct_feats_grid}
break down the experiments with no structural features on individual datasets,
\cref{fig:deg_feats_community,fig:deg_feats_ego,fig:deg_feats_lobster,fig:deg_feats_proteins,fig:deg_feats_zinc,fig:deg_feats_grid,table:deg_feats_community,table:deg_feats_ego,table:deg_feats_lobster,table:deg_feats_proteins,table:deg_feats_zinc,table:deg_feats_grid} those with degree features,
and 
\cref{fig:clustering_feats_community,fig:clustering_feats_ego,fig:clustering_feats_lobster,fig:clustering_feats_proteins,fig:clustering_feats_zinc,fig:clustering_feats_grid,table:clustering_feats_community,table:clustering_feats_ego,table:clustering_feats_lobster,table:clustering_feats_proteins,table:clustering_feats_zinc,table:clustering_feats_grid} those with clustering features.

\begin{table}[h!]

\caption{Statistics of graph datasets.}
\label{table:dataset}
\begin{center}
\begin{small}
\begin{sc}
\begin{tabular}{lccc|ccc}
\toprule
& \multicolumn{3}{c}{\textit{Synthetic}} & \multicolumn{3}{c}{\textit{Real-World}}\\
\cline{2-7}
&                  \textbf{Lobster} & \textbf{Grid} &  \textbf{Community}&  \textbf{Ego} & \textbf{Proteins} & \textbf{ZINC}\\
\midrule
$|$ Graphs$|$   & 100	     & 100	    & 500     &  757 		& 918	 & 1000		\\			
$|$ Nodes$|$    & 10-100   &  100-400  &   60-160  &  50-399    &  100-500  & 10-50   \\
$|$ Edges$|$     & 10-100   &  360-1368  &   300-1800  &  57-1071  & 186-1575   & 22-82    \\
$|$ Node Features$|$ & 0  & 0  &   0  &  0	 &  0 & 28 \\
$|$ Edge Features$|$ & 0  & 0  &   0  &  0	 &  0 & 4 \\
\bottomrule
\end{tabular}
\end{sc}	
\end{small}
\end{center}
\end{table}

\include{no_feat_experiments_figs_and_tables}

\include{deg_experiments_figs_and_tables}

\include{clustering_experiments_figs_and_tables}

\end{document}

%% file: no_feat_experiments_figs_and_tables.tex
\begin{figure*}[h!]
    \captionsetup[subfloat]{farskip=-2pt,captionskip=-8pt}
    \centering
    \subfloat[][]{\includegraphics[width = 2.55in]{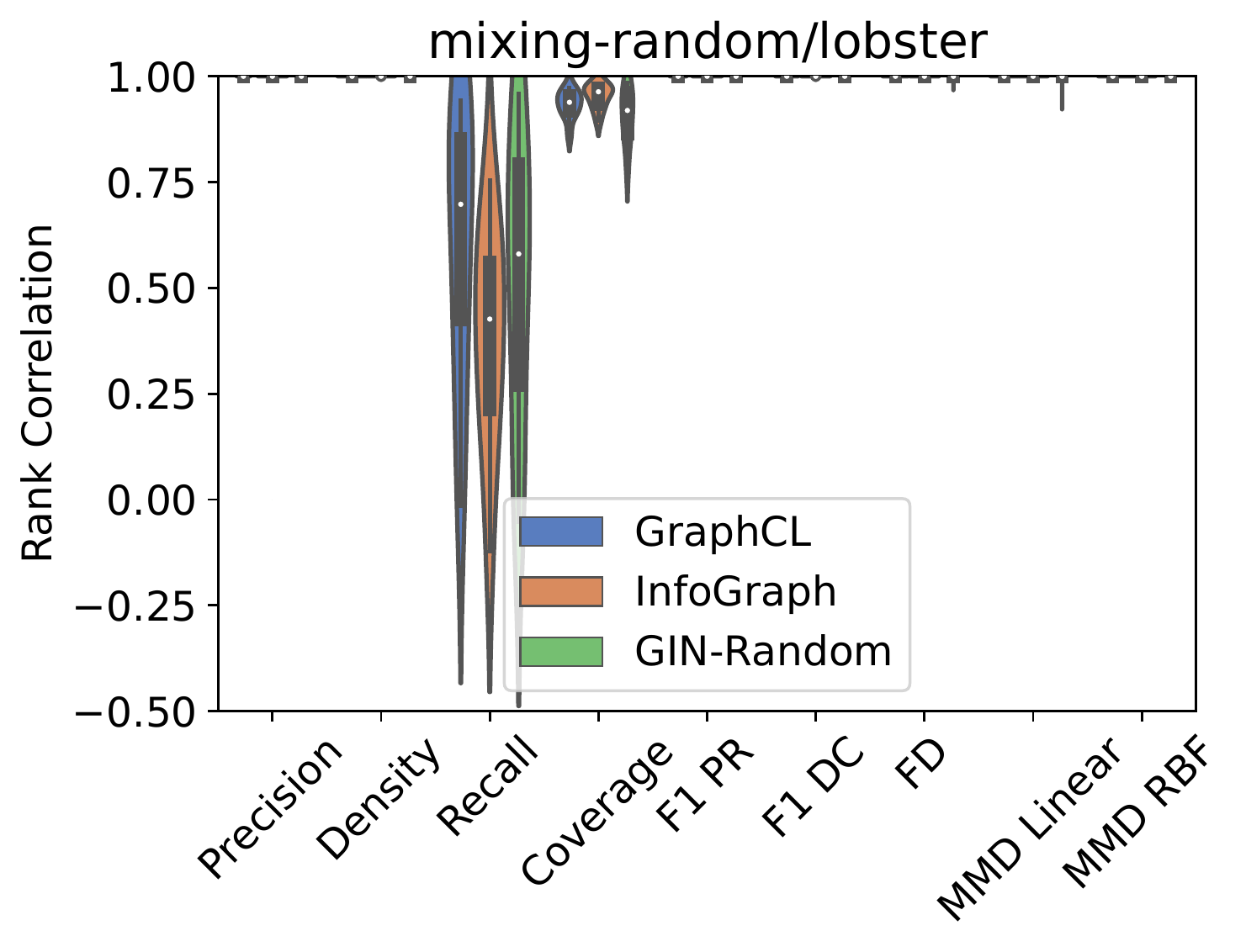}}
    \subfloat[][]{\includegraphics[width = 2.55in]{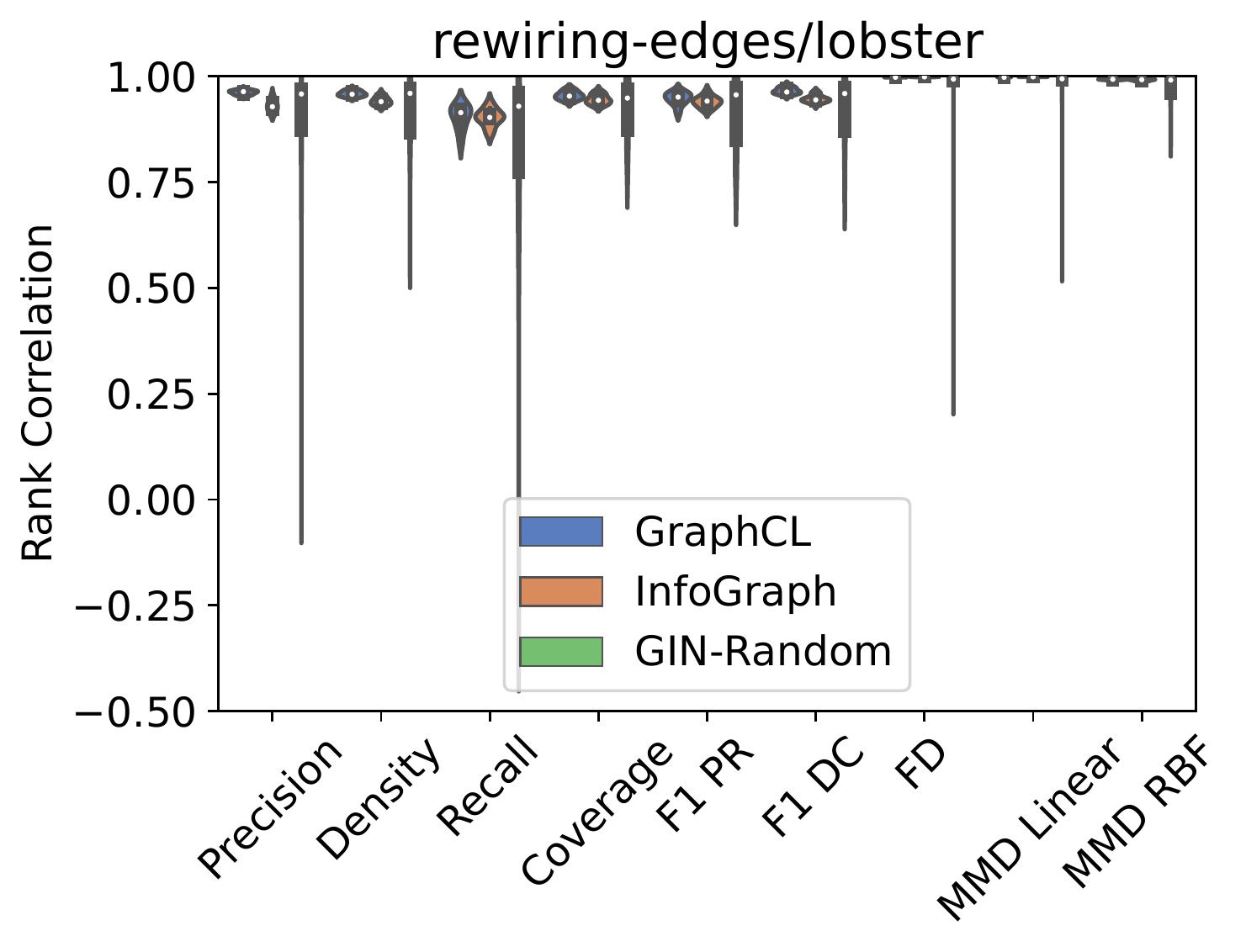}}
    \\
    \subfloat[][]{\includegraphics[width = 2.55in]{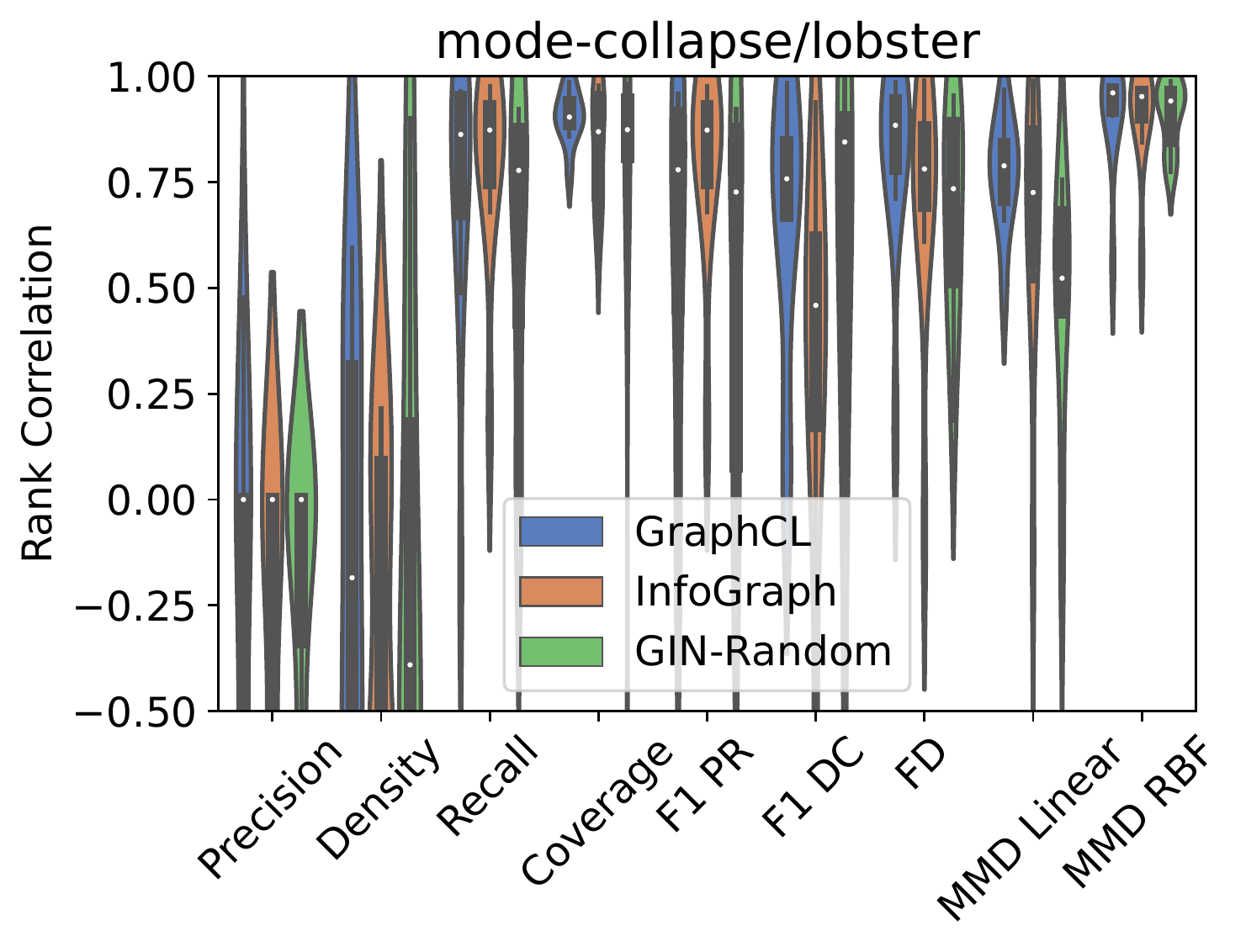}}
    \subfloat[][]{\includegraphics[width = 2.55in]{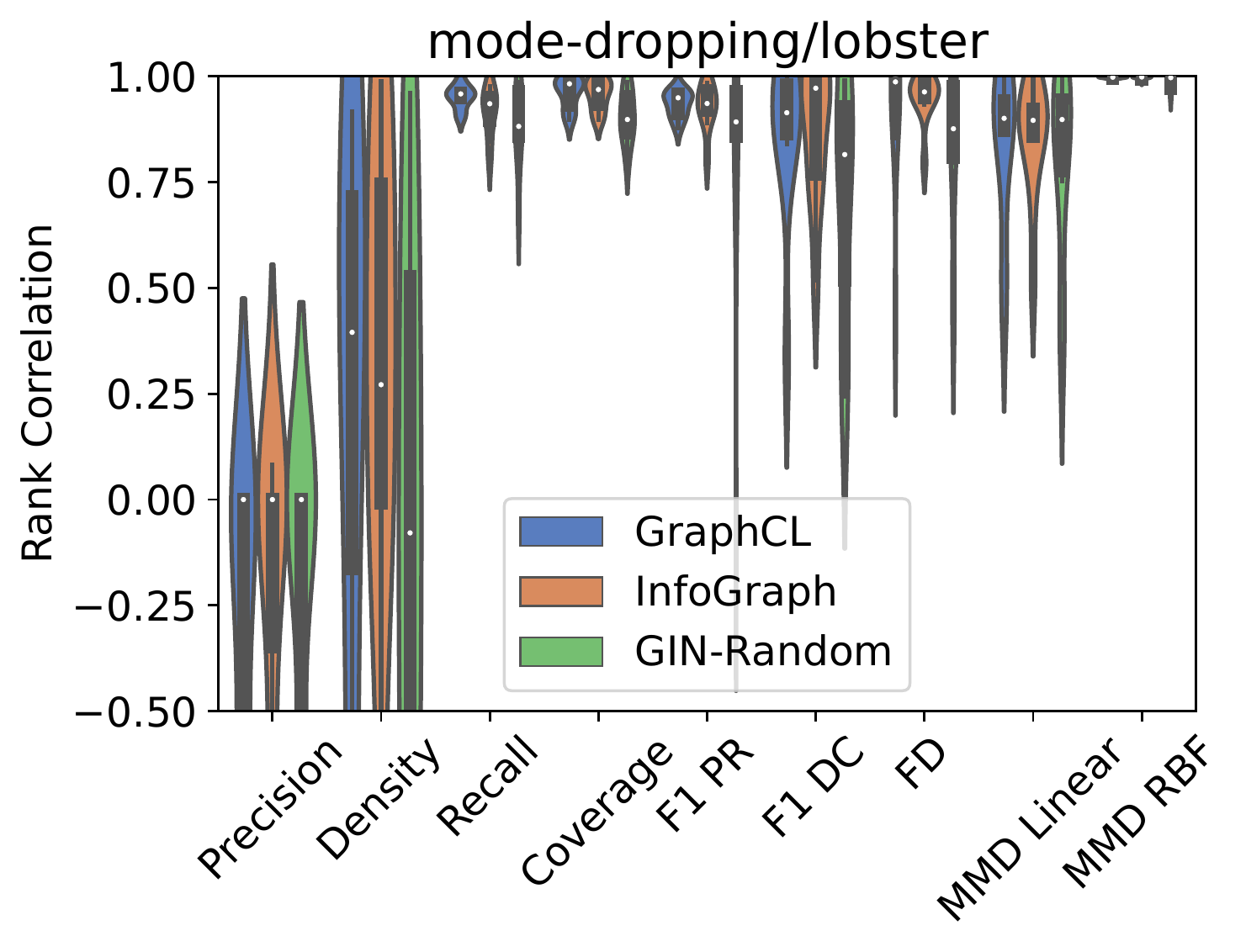}}
    \caption{Violing comparative results among the methods, with no structural features for lobster dataset.}
    \label{fig:no_struct_feats_lobster}
\end{figure*}
 
\begin{table}[h!]
\centering
\begin{small}
\scalebox{0.7}{
\begin{tabular}{l|l|l|l|l|l|l|l|l|l|l|l} 
\toprule
Model Name               & Experiment                      &        & Precision & Density & Recall & Coverage & F1PR & F1DC & FD & MMD Lin & MMD RBF  \\ 
\hline
\multirow{8}{*}{GraphCL}  &    \multirow{2}{*}{Mixing Random} & Mean & 1.0 & 1.0 & 0.6 & 0.93 & 1.0 & 1.0 & 1.0 & 1.0 & 1.0 \\
\cline{3-12}
                   &     & Median & 1.0 & 1.0 & 0.7 & 0.94 & 1.0 & 1.0 & 1.0 & 1.0 & 1.0 \\ 
\cline{2-12}
 &    \multirow{2}{*}{Rewiring Edges} & Mean & 0.96 & 0.96 & 0.9 & 0.95 & 0.95 & 0.97 & 1.0 & 1.0 & 0.99 \\
\cline{3-12}
                   &     & Median & 0.96 & 0.96 & 0.91 & 0.95 & 0.95 & 0.96 & 1.0 & 1.0 & 0.99 \\ 
\cline{2-12}
 &    \multirow{2}{*}{Mode Collapse} & Mean & -0.31 & -0.13 & 0.73 & 0.9 & 0.59 & 0.66 & 0.8 & 0.77 & 0.91 \\
\cline{3-12}
                   &     & Median & 0.0 & -0.18 & 0.86 & 0.9 & 0.78 & 0.76 & 0.88 & 0.79 & 0.96 \\ 
\cline{2-12}
 &    \multirow{2}{*}{Mode Dropping} & Mean & -0.29 & 0.17 & 0.95 & 0.96 & 0.93 & 0.86 & 0.91 & 0.83 & 1.0 \\
\cline{3-12}
                   &     & Median & 0.0 & 0.4 & 0.96 & 0.98 & 0.95 & 0.91 & 0.99 & 0.9 & 1.0 \\ 
\cline{1-12}
\multirow{8}{*}{InfoGraph}  &    \multirow{2}{*}{Mixing Random} & Mean & 1.0 & 1.0 & 0.38 & 0.95 & 1.0 & 1.0 & 1.0 & 1.0 & 1.0 \\
\cline{3-12}
                   &     & Median & 1.0 & 1.0 & 0.43 & 0.96 & 1.0 & 1.0 & 1.0 & 1.0 & 1.0 \\ 
\cline{2-12}
 &    \multirow{2}{*}{Rewiring Edges} & Mean & 0.93 & 0.94 & 0.9 & 0.95 & 0.94 & 0.95 & 1.0 & 1.0 & 0.99 \\
\cline{3-12}
                   &     & Median & 0.93 & 0.94 & 0.9 & 0.94 & 0.94 & 0.94 & 1.0 & 1.0 & 0.99 \\ 
\cline{2-12}
 &    \multirow{2}{*}{Mode Collapse} & Mean & -0.32 & -0.41 & 0.79 & 0.83 & 0.79 & 0.38 & 0.71 & 0.67 & 0.9 \\
\cline{3-12}
                   &     & Median & 0.0 & -0.59 & 0.87 & 0.87 & 0.87 & 0.46 & 0.78 & 0.73 & 0.95 \\ 
\cline{2-12}
 &    \multirow{2}{*}{Mode Dropping} & Mean & -0.21 & 0.27 & 0.92 & 0.96 & 0.93 & 0.87 & 0.95 & 0.85 & 1.0 \\
\cline{3-12}
                   &     & Median & 0.0 & 0.27 & 0.94 & 0.97 & 0.94 & 0.97 & 0.96 & 0.9 & 1.0 \\ 
\cline{1-12}
\multirow{8}{*}{GIN-Random}  &    \multirow{2}{*}{Mixing Random} & Mean & 1.0 & 1.0 & 0.51 & 0.91 & 1.0 & 1.0 & 1.0 & 0.99 & 1.0 \\
\cline{3-12}
                   &     & Median & 1.0 & 1.0 & 0.58 & 0.92 & 1.0 & 1.0 & 1.0 & 1.0 & 1.0 \\ 
\cline{2-12}
 &    \multirow{2}{*}{Rewiring Edges} & Mean & 0.86 & 0.9 & 0.79 & 0.92 & 0.91 & 0.92 & 0.94 & 0.96 & 0.96 \\
\cline{3-12}
                   &     & Median & 0.96 & 0.96 & 0.93 & 0.95 & 0.96 & 0.96 & 0.99 & 0.99 & 0.99 \\ 
\cline{2-12}
 &    \multirow{2}{*}{Mode Collapse} & Mean & -0.21 & -0.18 & 0.61 & 0.74 & 0.5 & 0.5 & 0.67 & 0.46 & 0.91 \\
\cline{3-12}
                   &     & Median & 0.0 & -0.39 & 0.78 & 0.87 & 0.73 & 0.84 & 0.73 & 0.52 & 0.94 \\ 
\cline{2-12}
 &    \multirow{2}{*}{Mode Dropping} & Mean & -0.23 & -0.03 & 0.88 & 0.9 & 0.81 & 0.71 & 0.82 & 0.81 & 0.99 \\
\cline{3-12}
                   &     & Median & 0.0 & -0.08 & 0.88 & 0.9 & 0.89 & 0.82 & 0.88 & 0.9 & 1.0 \\ \bottomrule
\end{tabular}
}
\end{small}
\caption{Mean and median values for measurements in experiments with no structural features by models.} 
\label{table:no_struct_feats_lobster}
\end{table}
 
\begin{figure*}[h!]
    \captionsetup[subfloat]{farskip=-2pt,captionskip=-8pt}
    \centering
    \subfloat[][]{\includegraphics[width = 2.55in]{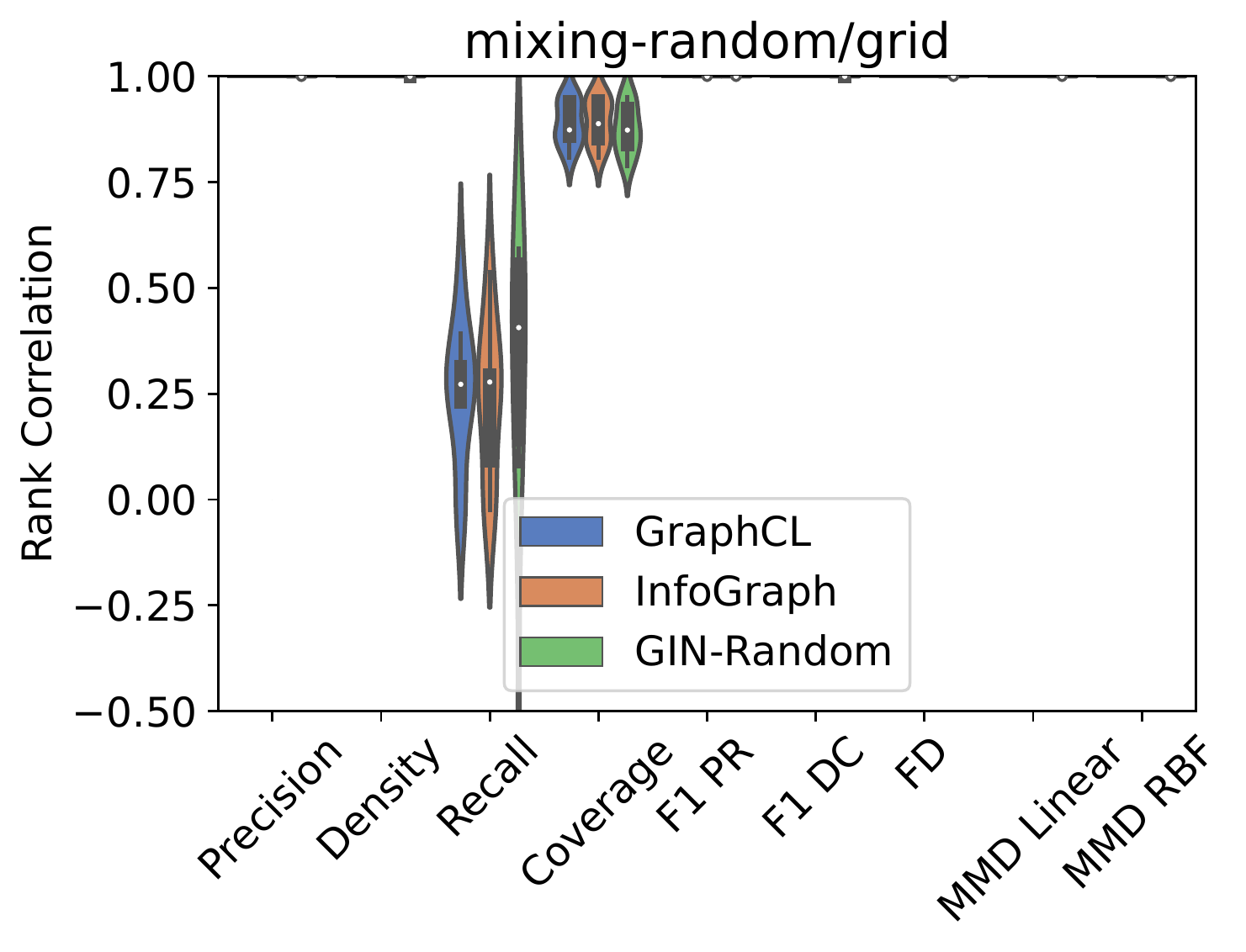}}
    \subfloat[][]{\includegraphics[width = 2.55in]{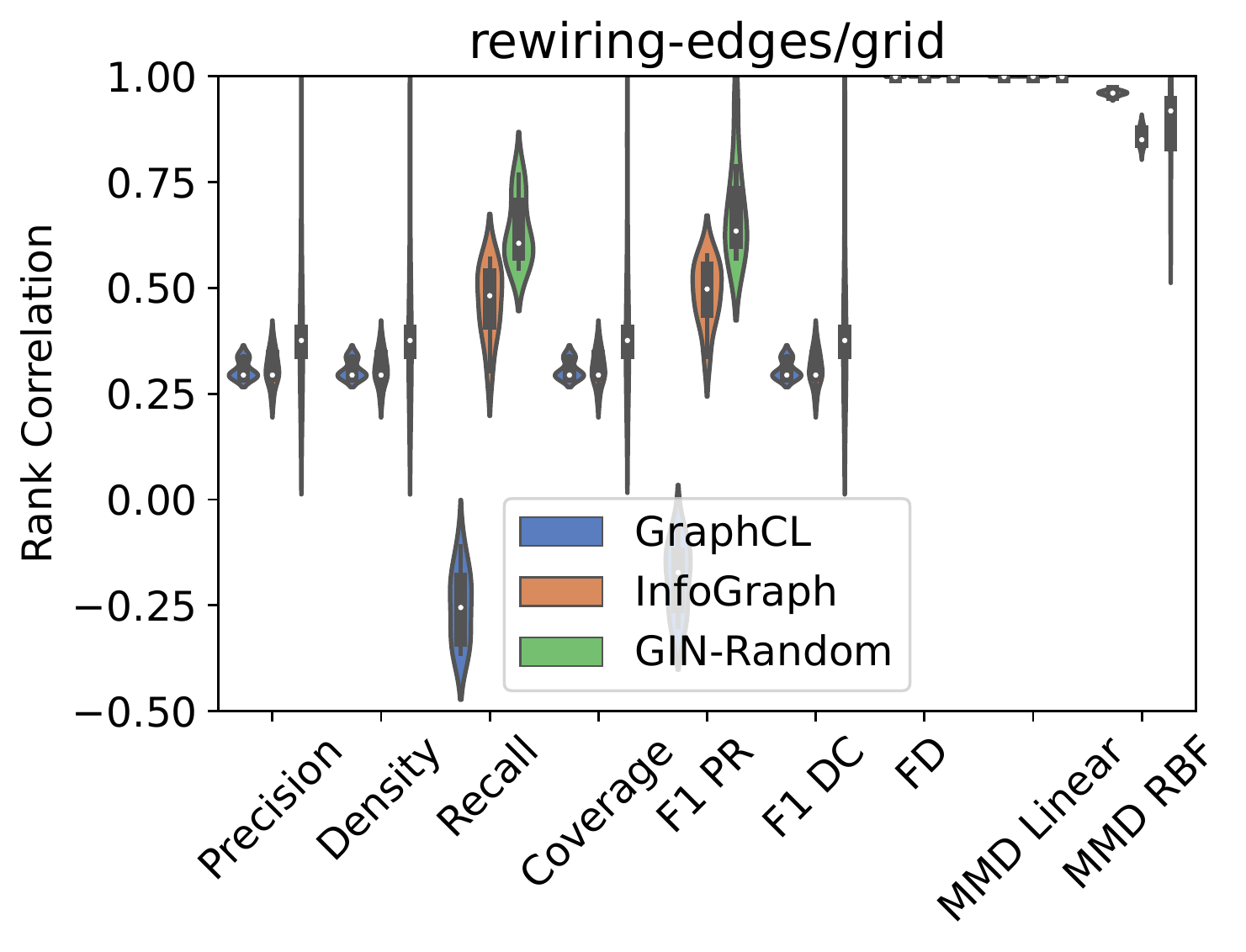}}
    \\
    \subfloat[][]{\includegraphics[width = 2.55in]{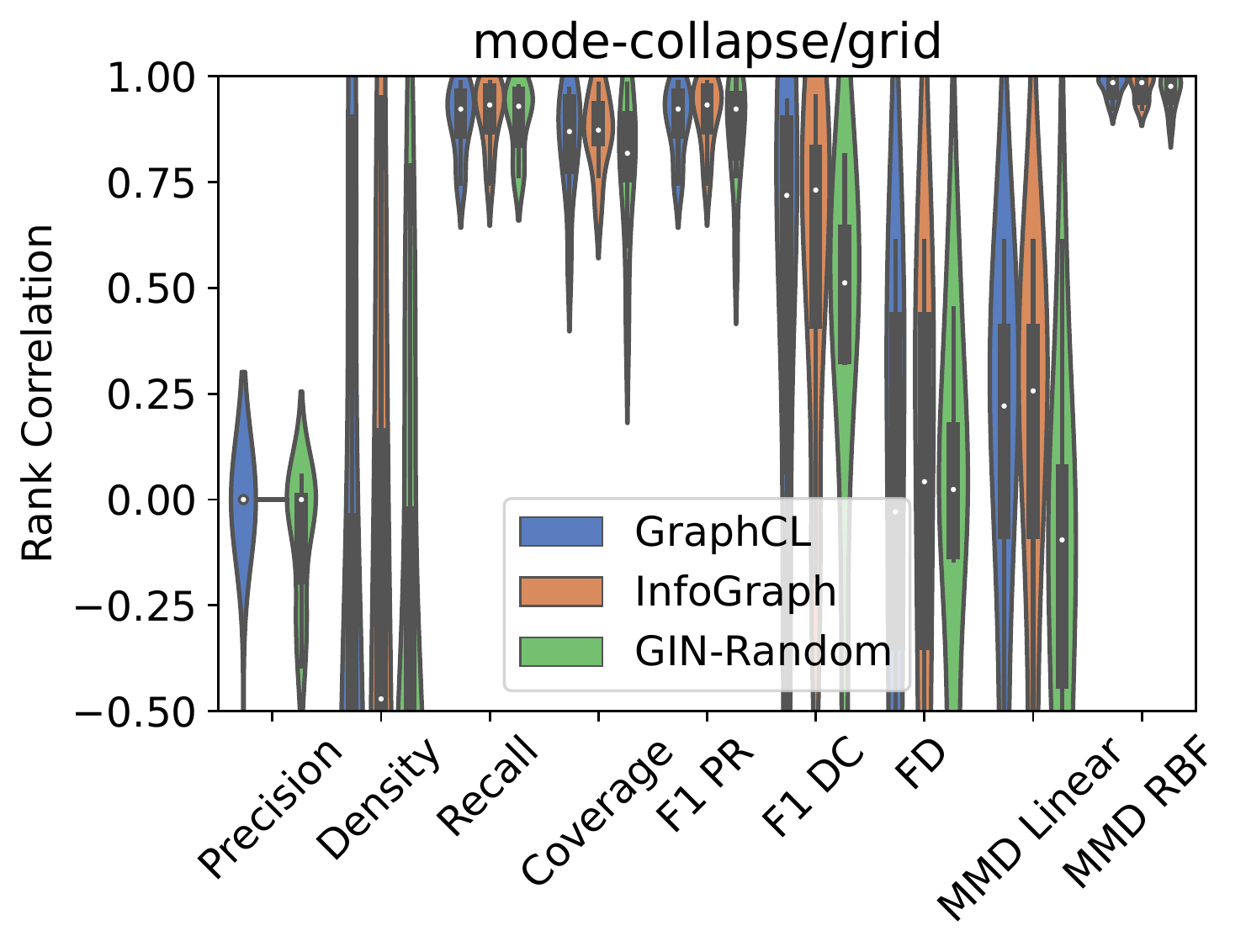}}
    \subfloat[][]{\includegraphics[width = 2.55in]{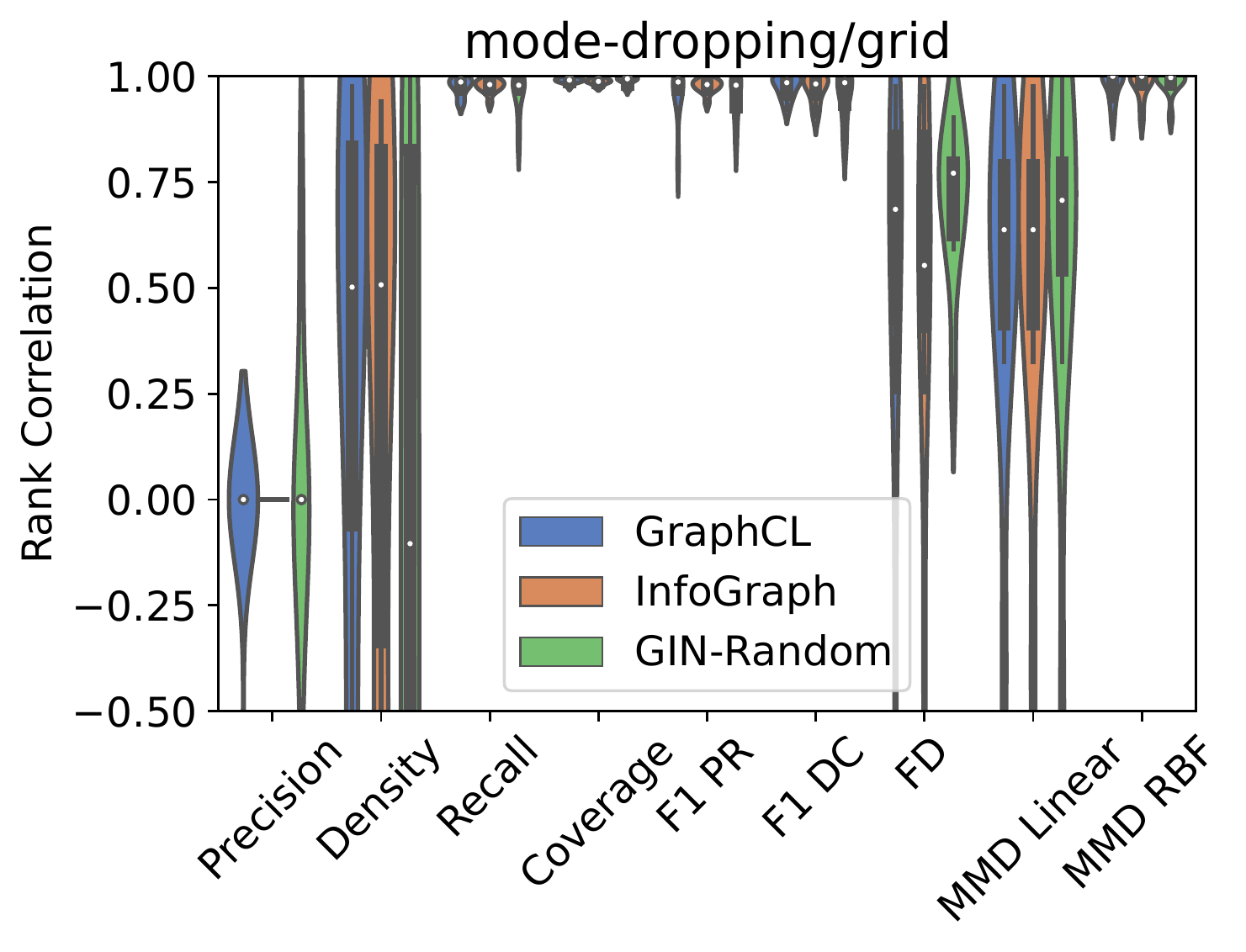}}
    \caption{Violing comparative results among the methods, with no structural features for grid dataset.}
    \label{fig:no_struct_feats_grid}
\end{figure*}
 
\begin{table}[h!]
\centering
\begin{small}
\scalebox{0.7}{
\begin{tabular}{l|l|l|l|l|l|l|l|l|l|l|l} 
\toprule
Model Name               & Experiment                      &        & Precision & Density & Recall & Coverage & F1PR & F1DC & FD & MMD Lin & MMD RBF  \\ 
\hline
\multirow{8}{*}{GraphCL}  &    \multirow{2}{*}{Mixing Random} & Mean & 1.0 & 1.0 & 0.25 & 0.89 & 1.0 & 1.0 & 1.0 & 1.0 & 1.0 \\
\cline{3-12}
                   &     & Median & 1.0 & 1.0 & 0.27 & 0.87 & 1.0 & 1.0 & 1.0 & 1.0 & 1.0 \\ 
\cline{2-12}
 &    \multirow{2}{*}{Rewiring Edges} & Mean & 0.31 & 0.31 & -0.25 & 0.31 & -0.19 & 0.31 & 1.0 & 1.0 & 0.96 \\
\cline{3-12}
                   &     & Median & 0.29 & 0.29 & -0.25 & 0.29 & -0.17 & 0.29 & 1.0 & 1.0 & 0.96 \\ 
\cline{2-12}
 &    \multirow{2}{*}{Mode Collapse} & Mean & -0.08 & -0.35 & 0.89 & 0.85 & 0.89 & 0.46 & -0.05 & 0.09 & 0.98 \\
\cline{3-12}
                   &     & Median & 0.0 & -0.7 & 0.92 & 0.87 & 0.92 & 0.72 & -0.03 & 0.22 & 0.99 \\ 
\cline{2-12}
 &    \multirow{2}{*}{Mode Dropping} & Mean & -0.08 & 0.3 & 0.98 & 0.99 & 0.96 & 0.98 & 0.57 & 0.56 & 0.98 \\
\cline{3-12}
                   &     & Median & 0.0 & 0.5 & 0.99 & 0.99 & 0.99 & 0.98 & 0.69 & 0.64 & 1.0 \\ 
\cline{1-12}
\multirow{8}{*}{InfoGraph}  &    \multirow{2}{*}{Mixing Random} & Mean & 1.0 & 1.0 & 0.24 & 0.89 & 1.0 & 1.0 & 1.0 & 1.0 & 1.0 \\
\cline{3-12}
                   &     & Median & 1.0 & 1.0 & 0.28 & 0.89 & 1.0 & 1.0 & 1.0 & 1.0 & 1.0 \\ 
\cline{2-12}
 &    \multirow{2}{*}{Rewiring Edges} & Mean & 0.31 & 0.31 & 0.47 & 0.31 & 0.49 & 0.31 & 1.0 & 1.0 & 0.86 \\
\cline{3-12}
                   &     & Median & 0.29 & 0.29 & 0.48 & 0.29 & 0.5 & 0.29 & 1.0 & 1.0 & 0.85 \\ 
\cline{2-12}
 &    \multirow{2}{*}{Mode Collapse} & Mean & 0.0 & -0.23 & 0.91 & 0.87 & 0.91 & 0.54 & -0.03 & 0.1 & 0.97 \\
\cline{3-12}
                   &     & Median & 0.0 & -0.47 & 0.93 & 0.87 & 0.93 & 0.73 & 0.04 & 0.26 & 0.99 \\ 
\cline{2-12}
 &    \multirow{2}{*}{Mode Dropping} & Mean & 0.0 & 0.24 & 0.98 & 0.99 & 0.98 & 0.97 & 0.54 & 0.56 & 0.98 \\
\cline{3-12}
                   &     & Median & 0.0 & 0.51 & 0.98 & 0.99 & 0.98 & 0.98 & 0.55 & 0.64 & 1.0 \\ 
\cline{1-12}
\multirow{8}{*}{GIN-Random}  &    \multirow{2}{*}{Mixing Random} & Mean & 1.0 & 1.0 & 0.28 & 0.87 & 1.0 & 1.0 & 1.0 & 1.0 & 1.0 \\
\cline{3-12}
                   &     & Median & 1.0 & 1.0 & 0.41 & 0.87 & 1.0 & 1.0 & 1.0 & 1.0 & 1.0 \\ 
\cline{2-12}
 &    \multirow{2}{*}{Rewiring Edges} & Mean & 0.46 & 0.46 & 0.63 & 0.46 & 0.68 & 0.46 & 1.0 & 1.0 & 0.88 \\
\cline{3-12}
                   &     & Median & 0.38 & 0.38 & 0.61 & 0.38 & 0.63 & 0.38 & 1.0 & 1.0 & 0.92 \\ 
\cline{2-12}
 &    \multirow{2}{*}{Mode Collapse} & Mean & -0.09 & -0.41 & 0.9 & 0.79 & 0.87 & 0.38 & -0.01 & -0.16 & 0.96 \\
\cline{3-12}
                   &     & Median & 0.0 & -0.71 & 0.93 & 0.82 & 0.92 & 0.51 & 0.02 & -0.1 & 0.98 \\ 
\cline{2-12}
 &    \multirow{2}{*}{Mode Dropping} & Mean & 0.05 & 0.01 & 0.96 & 0.99 & 0.96 & 0.96 & 0.71 & 0.58 & 0.98 \\
\cline{3-12}
                   &     & Median & 0.0 & -0.1 & 0.98 & 0.99 & 0.98 & 0.98 & 0.77 & 0.71 & 1.0 \\ \bottomrule
\end{tabular}
}
\end{small}
\caption{Mean and median values for measurements in experiments with no structural features by models.} 
\label{table:no_struct_feats_grid}
\end{table}
 
\begin{figure*}[h!]
    \captionsetup[subfloat]{farskip=-2pt,captionskip=-8pt}
    \centering
    \subfloat[][]{\includegraphics[width = 2.55in]{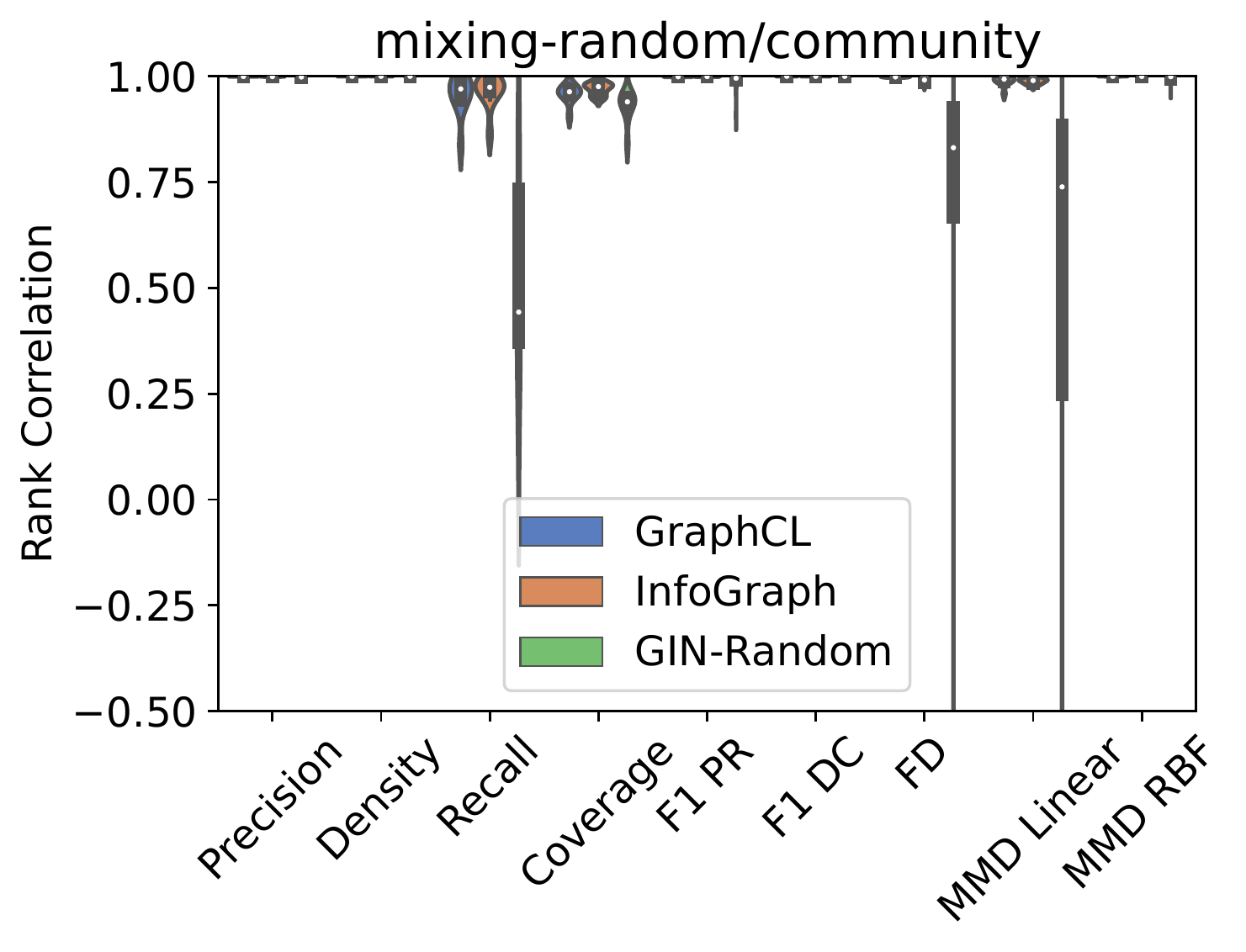}}
    \subfloat[][]{\includegraphics[width = 2.55in]{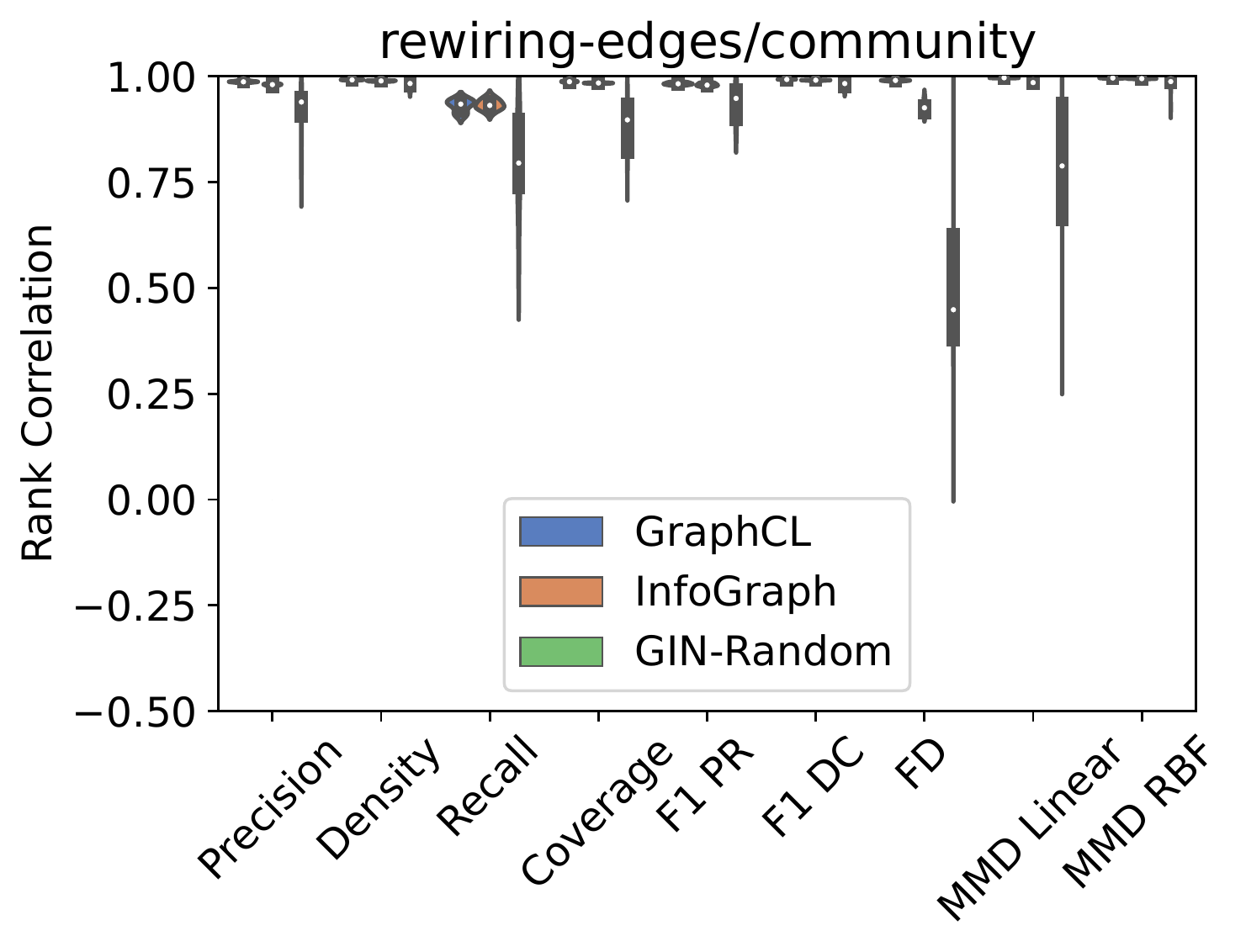}}
    \\
    \subfloat[][]{\includegraphics[width = 2.55in]{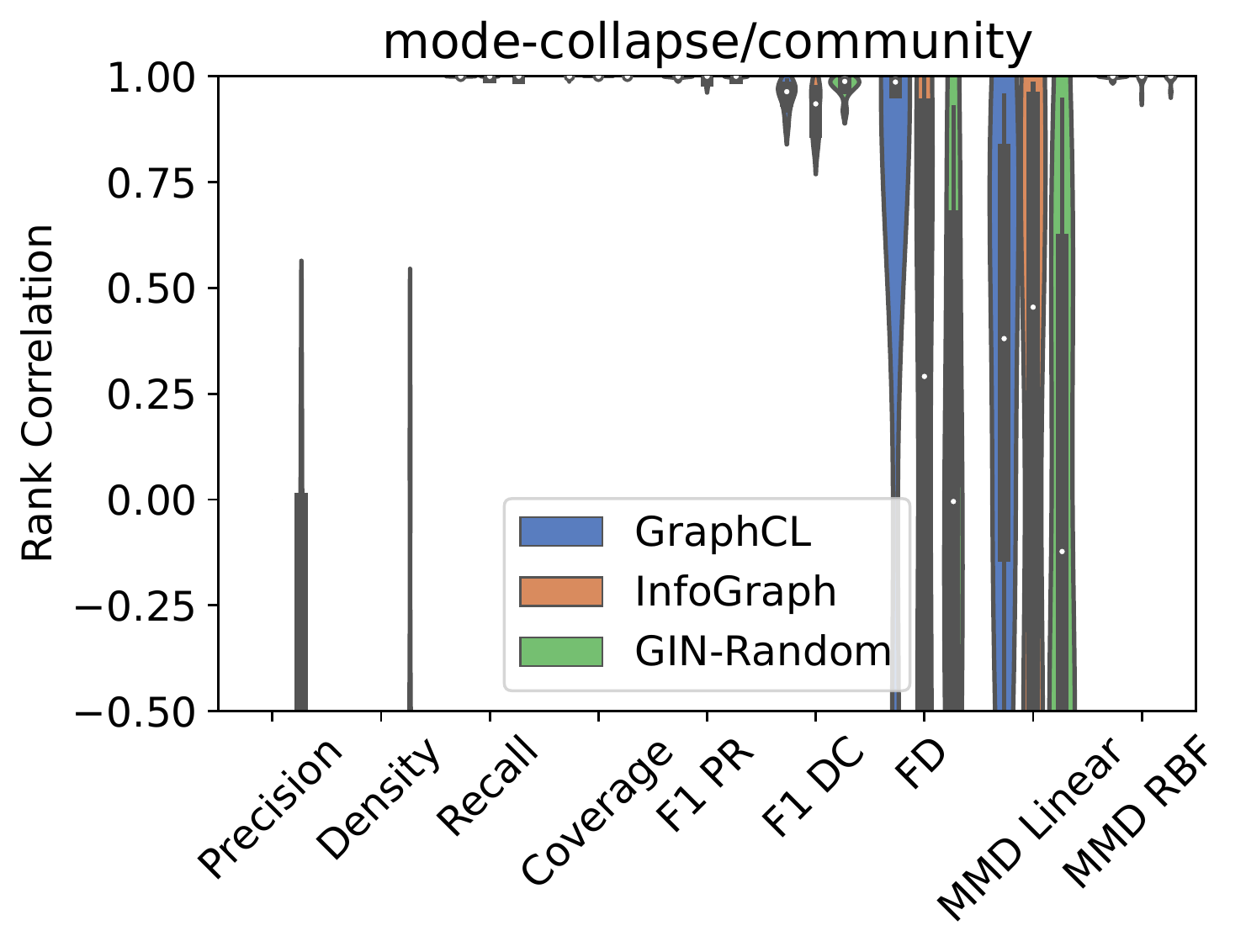}}
    \subfloat[][]{\includegraphics[width = 2.55in]{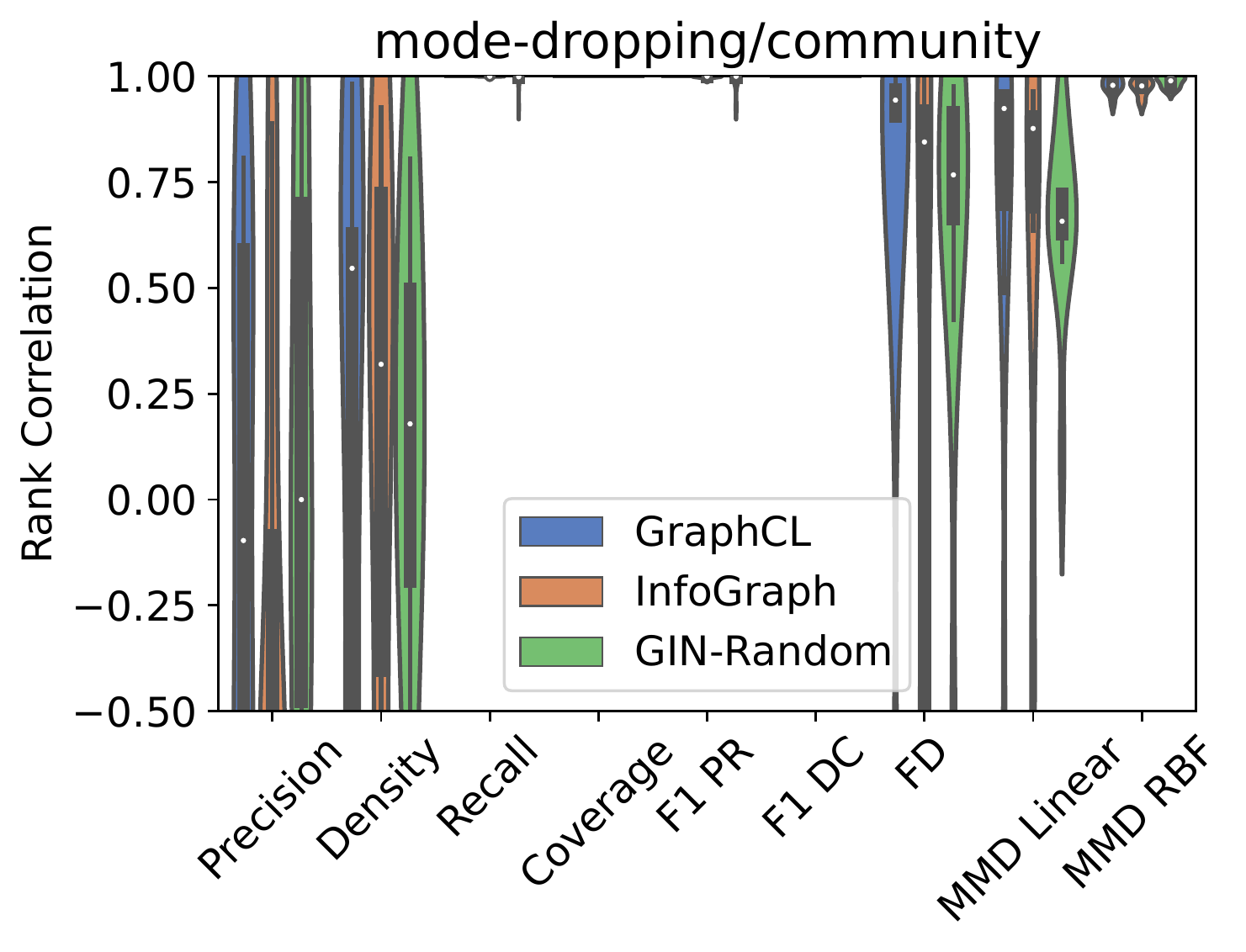}}
    \caption{Violing comparative results among the methods, with no structural features for community dataset.}
    \label{fig:no_struct_feats_community}
\end{figure*}
 
\begin{table}[h!]
\centering
\begin{small}
\scalebox{0.7}{
\begin{tabular}{l|l|l|l|l|l|l|l|l|l|l|l} 
\toprule
Model Name               & Experiment                      &        & Precision & Density & Recall & Coverage & F1PR & F1DC & FD & MMD Lin & MMD RBF  \\ 
\hline
\multirow{8}{*}{GraphCL}  &    \multirow{2}{*}{Mixing Random} & Mean & 1.0 & 1.0 & 0.96 & 0.96 & 1.0 & 1.0 & 1.0 & 0.99 & 1.0 \\
\cline{3-12}
                   &     & Median & 1.0 & 1.0 & 0.97 & 0.96 & 1.0 & 1.0 & 1.0 & 0.99 & 1.0 \\ 
\cline{2-12}
 &    \multirow{2}{*}{Rewiring Edges} & Mean & 0.99 & 0.99 & 0.93 & 0.99 & 0.98 & 0.99 & 0.99 & 1.0 & 1.0 \\
\cline{3-12}
                   &     & Median & 0.99 & 0.99 & 0.93 & 0.99 & 0.98 & 0.99 & 0.99 & 1.0 & 1.0 \\ 
\cline{2-12}
 &    \multirow{2}{*}{Mode Collapse} & Mean & -0.91 & -0.9 & 1.0 & 1.0 & 1.0 & 0.96 & 0.66 & 0.27 & 1.0 \\
\cline{3-12}
                   &     & Median & -0.95 & -0.93 & 1.0 & 1.0 & 1.0 & 0.96 & 0.99 & 0.38 & 1.0 \\ 
\cline{2-12}
 &    \multirow{2}{*}{Mode Dropping} & Mean & -0.09 & 0.14 & 1.0 & 1.0 & 1.0 & 1.0 & 0.73 & 0.72 & 0.98 \\
\cline{3-12}
                   &     & Median & -0.1 & 0.55 & 1.0 & 1.0 & 1.0 & 1.0 & 0.94 & 0.92 & 0.98 \\ 
\cline{1-12}
\multirow{8}{*}{InfoGraph}  &    \multirow{2}{*}{Mixing Random} & Mean & 1.0 & 1.0 & 0.96 & 0.97 & 1.0 & 1.0 & 0.99 & 0.99 & 1.0 \\
\cline{3-12}
                   &     & Median & 1.0 & 1.0 & 0.97 & 0.98 & 1.0 & 1.0 & 0.99 & 0.99 & 1.0 \\ 
\cline{2-12}
 &    \multirow{2}{*}{Rewiring Edges} & Mean & 0.98 & 0.99 & 0.93 & 0.98 & 0.98 & 0.99 & 0.93 & 0.98 & 0.99 \\
\cline{3-12}
                   &     & Median & 0.98 & 0.99 & 0.93 & 0.98 & 0.98 & 0.99 & 0.93 & 0.99 & 0.99 \\ 
\cline{2-12}
 &    \multirow{2}{*}{Mode Collapse} & Mean & -0.96 & -0.96 & 1.0 & 1.0 & 0.99 & 0.92 & 0.18 & 0.22 & 1.0 \\
\cline{3-12}
                   &     & Median & -0.98 & -0.97 & 1.0 & 1.0 & 1.0 & 0.93 & 0.29 & 0.45 & 1.0 \\ 
\cline{2-12}
 &    \multirow{2}{*}{Mode Dropping} & Mean & -0.41 & 0.13 & 1.0 & 1.0 & 1.0 & 1.0 & 0.31 & 0.61 & 0.97 \\
\cline{3-12}
                   &     & Median & -0.74 & 0.32 & 1.0 & 1.0 & 1.0 & 1.0 & 0.84 & 0.88 & 0.98 \\ 
\cline{1-12}
\multirow{8}{*}{GIN-Random}  &    \multirow{2}{*}{Mixing Random} & Mean & 1.0 & 1.0 & 0.51 & 0.93 & 0.99 & 1.0 & 0.69 & 0.48 & 0.99 \\
\cline{3-12}
                   &     & Median & 1.0 & 1.0 & 0.44 & 0.94 & 1.0 & 1.0 & 0.83 & 0.74 & 1.0 \\ 
\cline{2-12}
 &    \multirow{2}{*}{Rewiring Edges} & Mean & 0.91 & 0.98 & 0.8 & 0.88 & 0.93 & 0.98 & 0.5 & 0.79 & 0.98 \\
\cline{3-12}
                   &     & Median & 0.94 & 0.98 & 0.79 & 0.9 & 0.95 & 0.98 & 0.45 & 0.79 & 0.99 \\ 
\cline{2-12}
 &    \multirow{2}{*}{Mode Collapse} & Mean & -0.49 & -0.78 & 1.0 & 1.0 & 1.0 & 0.98 & 0.03 & -0.02 & 1.0 \\
\cline{3-12}
                   &     & Median & -0.6 & -0.92 & 1.0 & 1.0 & 1.0 & 0.99 & -0.0 & -0.12 & 1.0 \\ 
\cline{2-12}
 &    \multirow{2}{*}{Mode Dropping} & Mean & 0.05 & 0.11 & 0.99 & 1.0 & 0.99 & 1.0 & 0.64 & 0.65 & 0.99 \\
\cline{3-12}
                   &     & Median & 0.0 & 0.18 & 1.0 & 1.0 & 1.0 & 1.0 & 0.77 & 0.66 & 0.99 \\ \bottomrule
\end{tabular}
}
\end{small}
\caption{Mean and median values for measurements in experiments with no structural features by models.} 
\label{table:no_struct_feats_community}
\end{table}
 
\begin{figure*}[h!]
    \captionsetup[subfloat]{farskip=-2pt,captionskip=-8pt}
    \centering
    \subfloat[][]{\includegraphics[width = 2.55in]{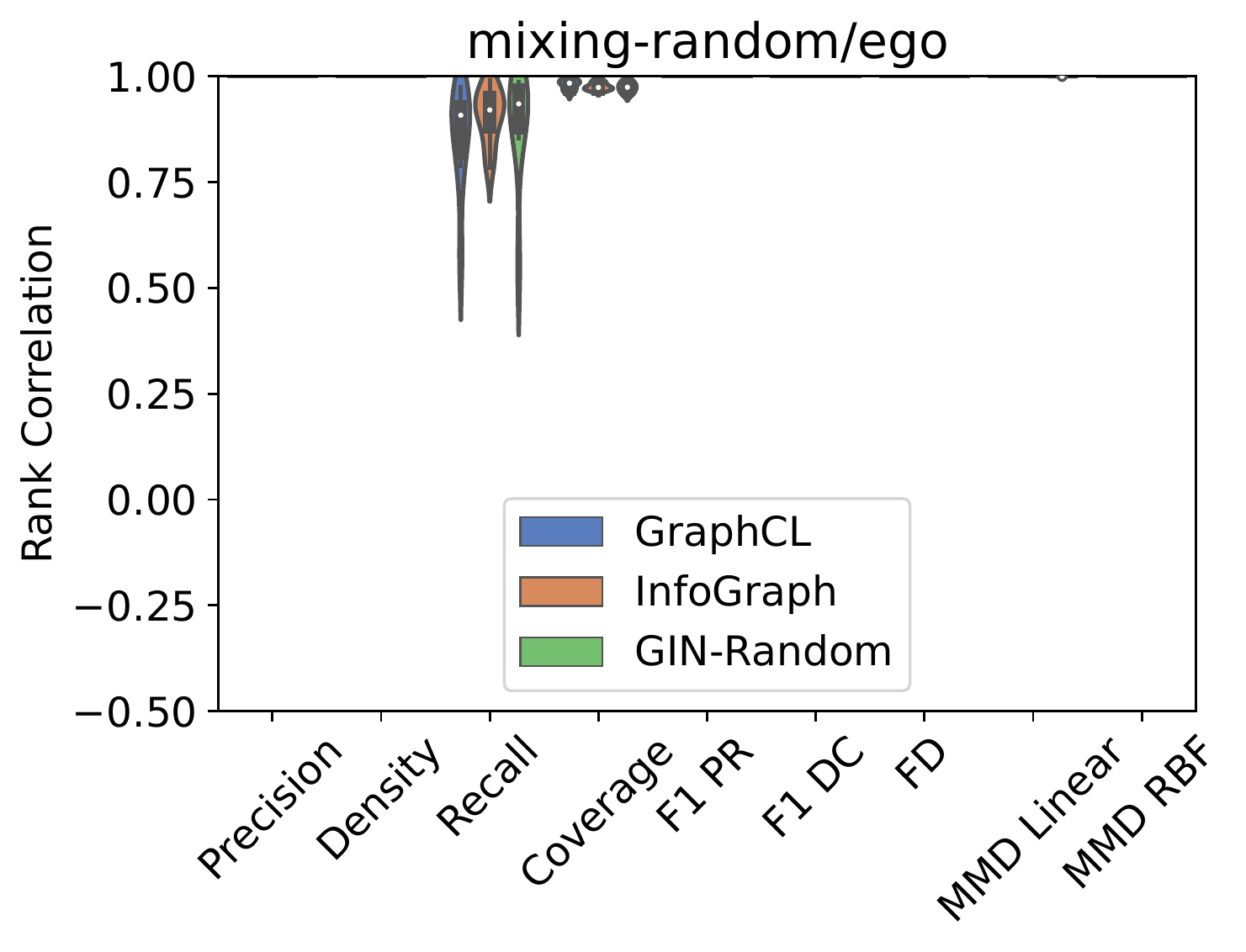}}
    \subfloat[][]{\includegraphics[width = 2.55in]{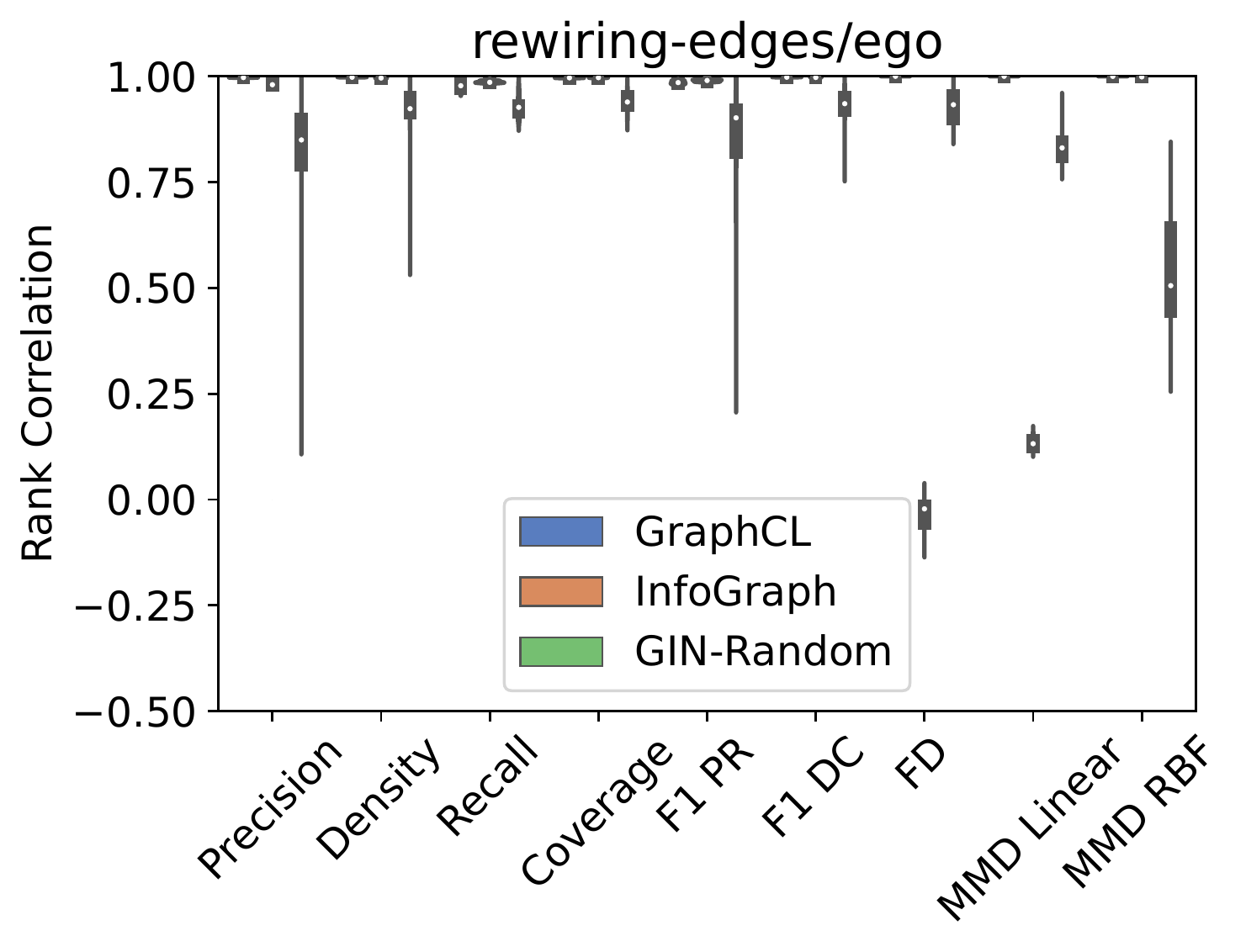}}
    \\
    \subfloat[][]{\includegraphics[width = 2.55in]{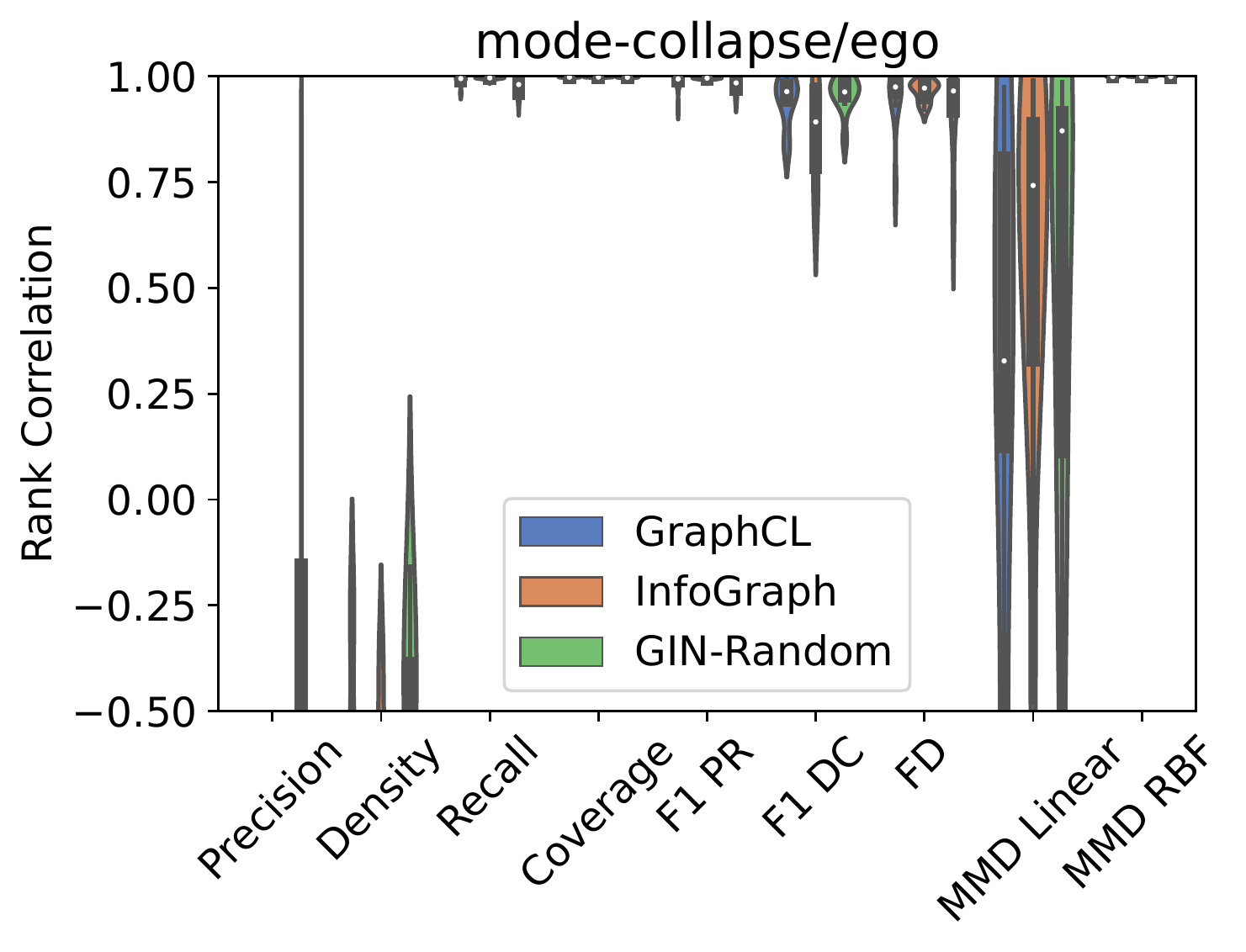}}
    \subfloat[][]{\includegraphics[width = 2.55in]{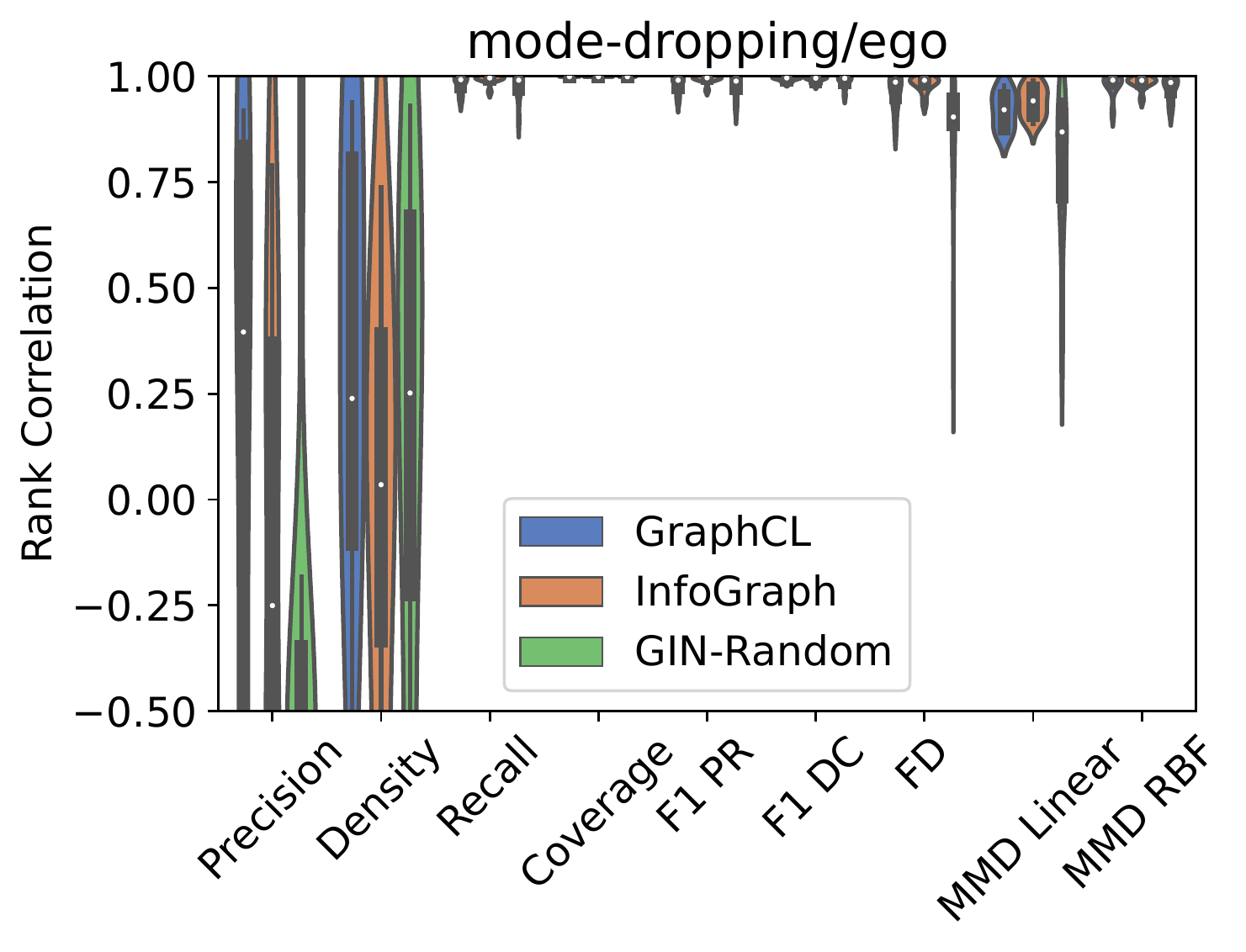}}
    \caption{Violing comparative results among the methods, with no structural features for ego dataset.}
    \label{fig:no_struct_feats_ego}
\end{figure*}
 
\begin{table}[h!]
\centering
\begin{small}
\scalebox{0.7}{
\begin{tabular}{l|l|l|l|l|l|l|l|l|l|l|l} 
\toprule
Model Name               & Experiment                      &        & Precision & Density & Recall & Coverage & F1PR & F1DC & FD & MMD Lin & MMD RBF  \\ 
\hline
\multirow{8}{*}{GraphCL}  &    \multirow{2}{*}{Mixing Random} & Mean & 1.0 & 1.0 & 0.86 & 0.98 & 1.0 & 1.0 & 1.0 & 1.0 & 1.0 \\
\cline{3-12}
                   &     & Median & 1.0 & 1.0 & 0.91 & 0.98 & 1.0 & 1.0 & 1.0 & 1.0 & 1.0 \\ 
\cline{2-12}
 &    \multirow{2}{*}{Rewiring Edges} & Mean & 1.0 & 1.0 & 0.98 & 1.0 & 0.99 & 1.0 & 1.0 & 1.0 & 1.0 \\
\cline{3-12}
                   &     & Median & 1.0 & 1.0 & 0.98 & 1.0 & 0.99 & 1.0 & 1.0 & 1.0 & 1.0 \\ 
\cline{2-12}
 &    \multirow{2}{*}{Mode Collapse} & Mean & -0.97 & -0.8 & 0.99 & 1.0 & 0.99 & 0.94 & 0.95 & 0.3 & 1.0 \\
\cline{3-12}
                   &     & Median & -0.98 & -0.87 & 0.99 & 1.0 & 0.99 & 0.96 & 0.97 & 0.33 & 1.0 \\ 
\cline{2-12}
 &    \multirow{2}{*}{Mode Dropping} & Mean & 0.15 & 0.26 & 0.98 & 1.0 & 0.98 & 0.99 & 0.97 & 0.92 & 0.98 \\
\cline{3-12}
                   &     & Median & 0.4 & 0.24 & 0.99 & 1.0 & 0.99 & 1.0 & 0.99 & 0.92 & 0.99 \\ 
\cline{1-12}
\multirow{8}{*}{InfoGraph}  &    \multirow{2}{*}{Mixing Random} & Mean & 1.0 & 1.0 & 0.91 & 0.98 & 1.0 & 1.0 & 1.0 & 1.0 & 1.0 \\
\cline{3-12}
                   &     & Median & 1.0 & 1.0 & 0.92 & 0.97 & 1.0 & 1.0 & 1.0 & 1.0 & 1.0 \\ 
\cline{2-12}
 &    \multirow{2}{*}{Rewiring Edges} & Mean & 0.98 & 1.0 & 0.99 & 1.0 & 0.99 & 1.0 & -0.03 & 0.13 & 1.0 \\
\cline{3-12}
                   &     & Median & 0.98 & 1.0 & 0.99 & 1.0 & 0.99 & 1.0 & -0.02 & 0.13 & 1.0 \\ 
\cline{2-12}
 &    \multirow{2}{*}{Mode Collapse} & Mean & -0.94 & -0.83 & 1.0 & 1.0 & 0.99 & 0.86 & 0.97 & 0.57 & 1.0 \\
\cline{3-12}
                   &     & Median & -0.95 & -0.92 & 1.0 & 1.0 & 1.0 & 0.89 & 0.97 & 0.74 & 1.0 \\ 
\cline{2-12}
 &    \multirow{2}{*}{Mode Dropping} & Mean & -0.2 & -0.01 & 0.99 & 1.0 & 0.99 & 0.99 & 0.98 & 0.94 & 0.99 \\
\cline{3-12}
                   &     & Median & -0.25 & 0.04 & 1.0 & 1.0 & 1.0 & 1.0 & 0.99 & 0.94 & 0.99 \\ 
\cline{1-12}
\multirow{8}{*}{GIN-Random}  &    \multirow{2}{*}{Mixing Random} & Mean & 1.0 & 1.0 & 0.89 & 0.97 & 1.0 & 1.0 & 1.0 & 1.0 & 1.0 \\
\cline{3-12}
                   &     & Median & 1.0 & 1.0 & 0.93 & 0.97 & 1.0 & 1.0 & 1.0 & 1.0 & 1.0 \\ 
\cline{2-12}
 &    \multirow{2}{*}{Rewiring Edges} & Mean & 0.81 & 0.91 & 0.93 & 0.94 & 0.84 & 0.93 & 0.93 & 0.84 & 0.54 \\
\cline{3-12}
                   &     & Median & 0.85 & 0.92 & 0.93 & 0.94 & 0.9 & 0.94 & 0.93 & 0.83 & 0.51 \\ 
\cline{2-12}
 &    \multirow{2}{*}{Mode Collapse} & Mean & -0.53 & -0.6 & 0.98 & 1.0 & 0.98 & 0.96 & 0.91 & 0.49 & 1.0 \\
\cline{3-12}
                   &     & Median & -0.89 & -0.53 & 0.98 & 1.0 & 0.98 & 0.96 & 0.97 & 0.87 & 1.0 \\ 
\cline{2-12}
 &    \multirow{2}{*}{Mode Dropping} & Mean & -0.41 & 0.12 & 0.98 & 1.0 & 0.98 & 0.99 & 0.86 & 0.79 & 0.97 \\
\cline{3-12}
                   &     & Median & -0.54 & 0.25 & 0.99 & 1.0 & 0.99 & 1.0 & 0.9 & 0.87 & 0.99 \\ \bottomrule
\end{tabular}
}
\end{small}
\caption{Mean and median values for measurements in experiments with no structural features by models.} 
\label{table:no_struct_feats_ego}
\end{table}
 
\begin{figure*}[h!]
    \captionsetup[subfloat]{farskip=-2pt,captionskip=-8pt}
    \centering
    \subfloat[][]{\includegraphics[width = 2.55in]{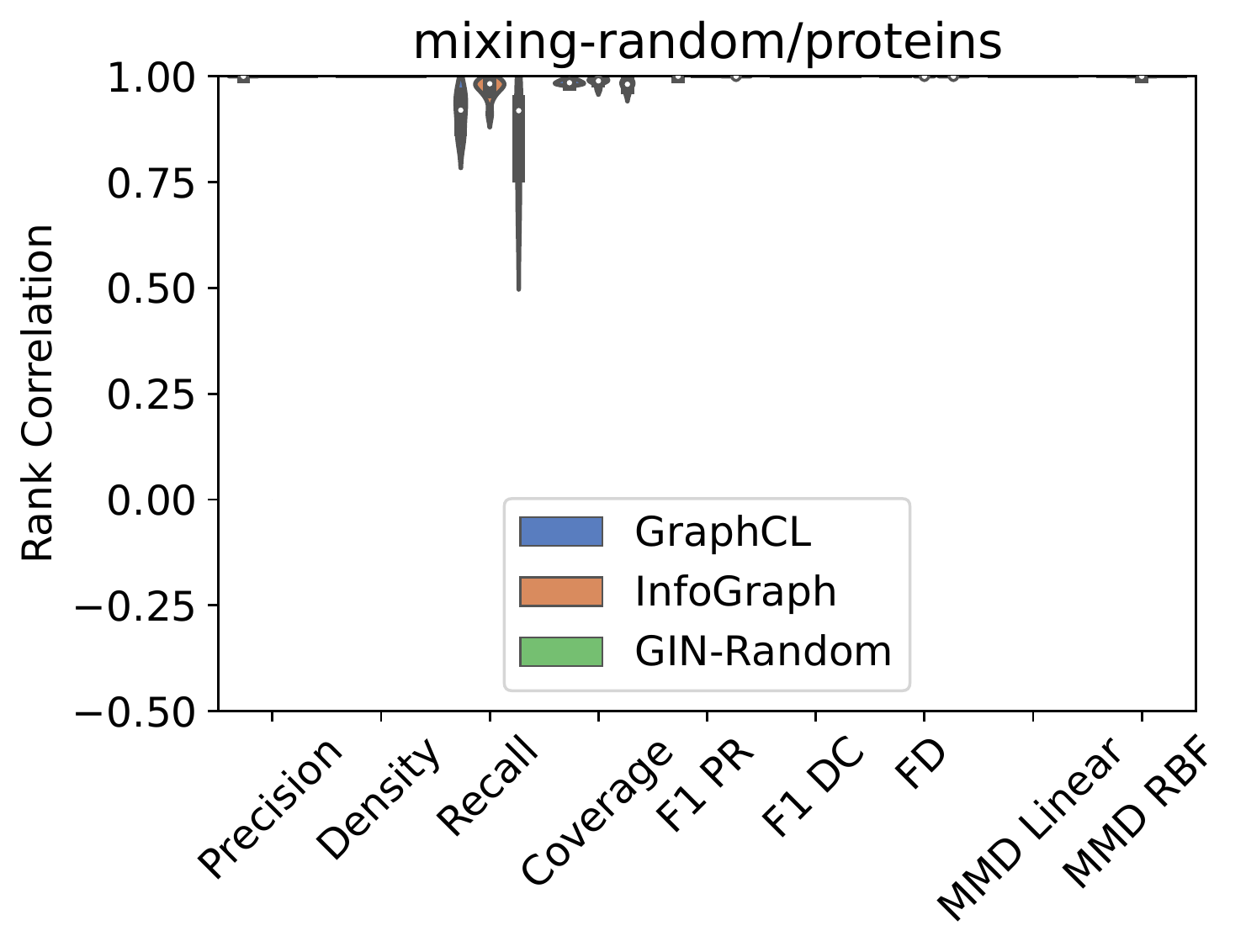}}
    \subfloat[][]{\includegraphics[width = 2.55in]{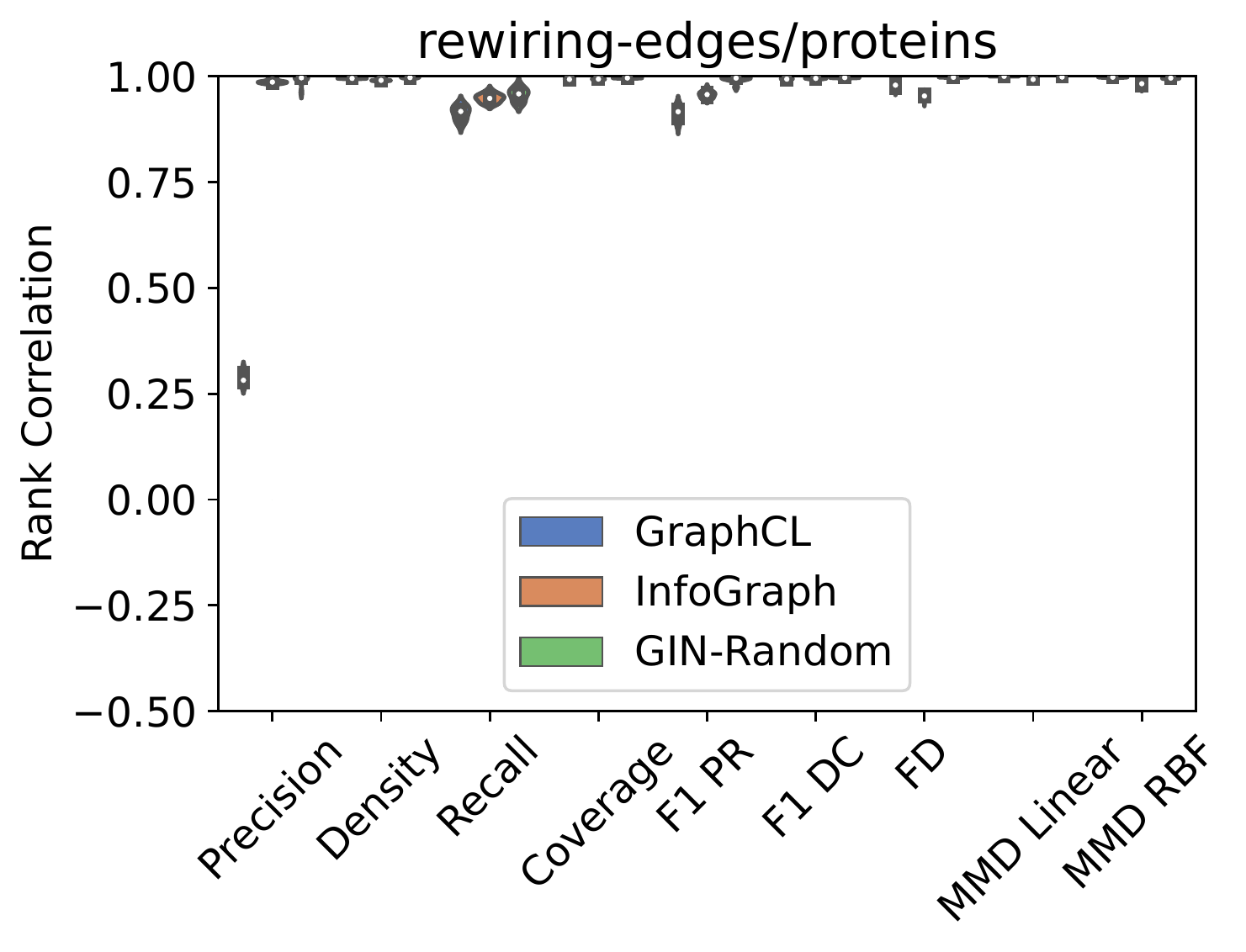}}
    \\
    \subfloat[][]{\includegraphics[width = 2.55in]{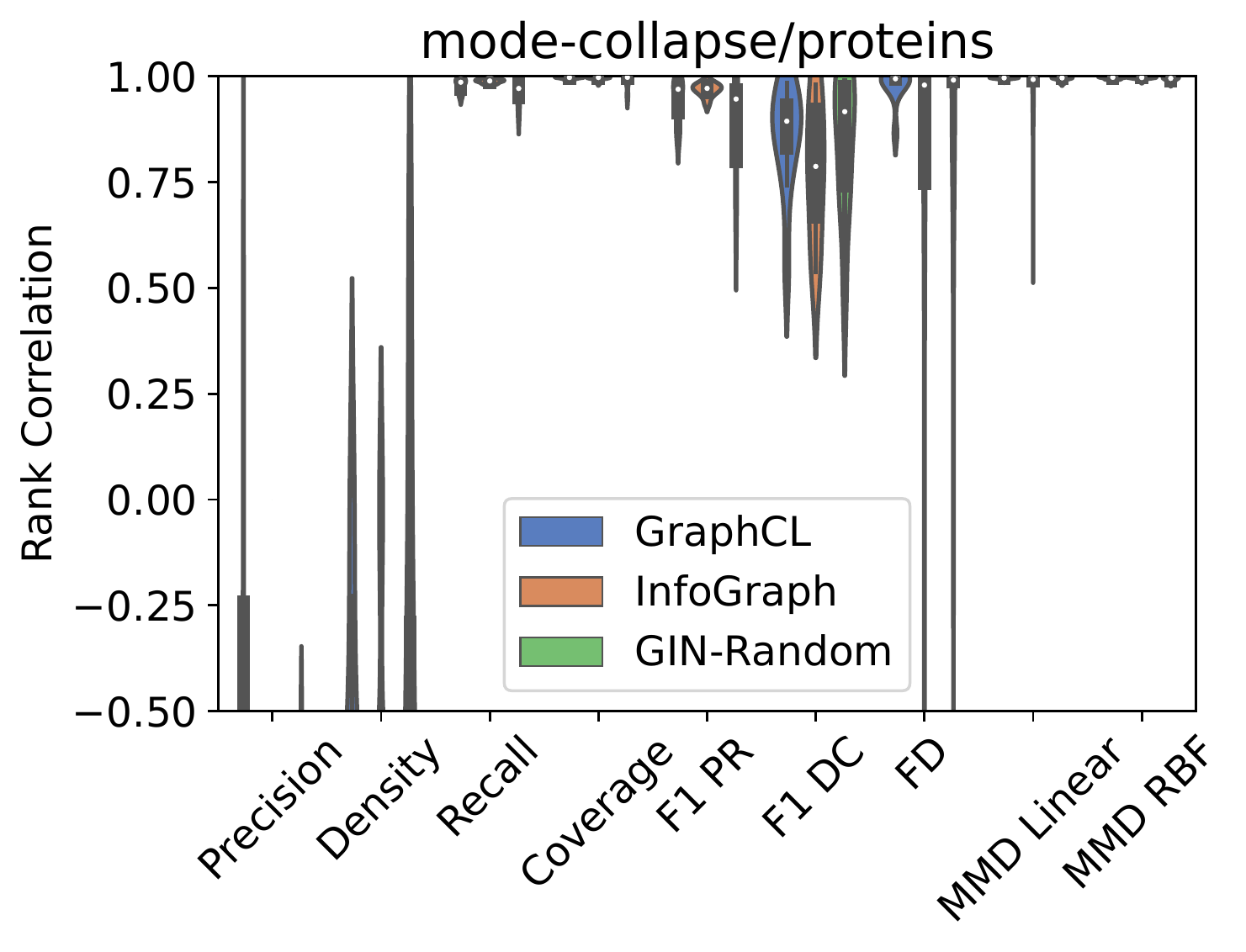}}
    \subfloat[][]{\includegraphics[width = 2.55in]{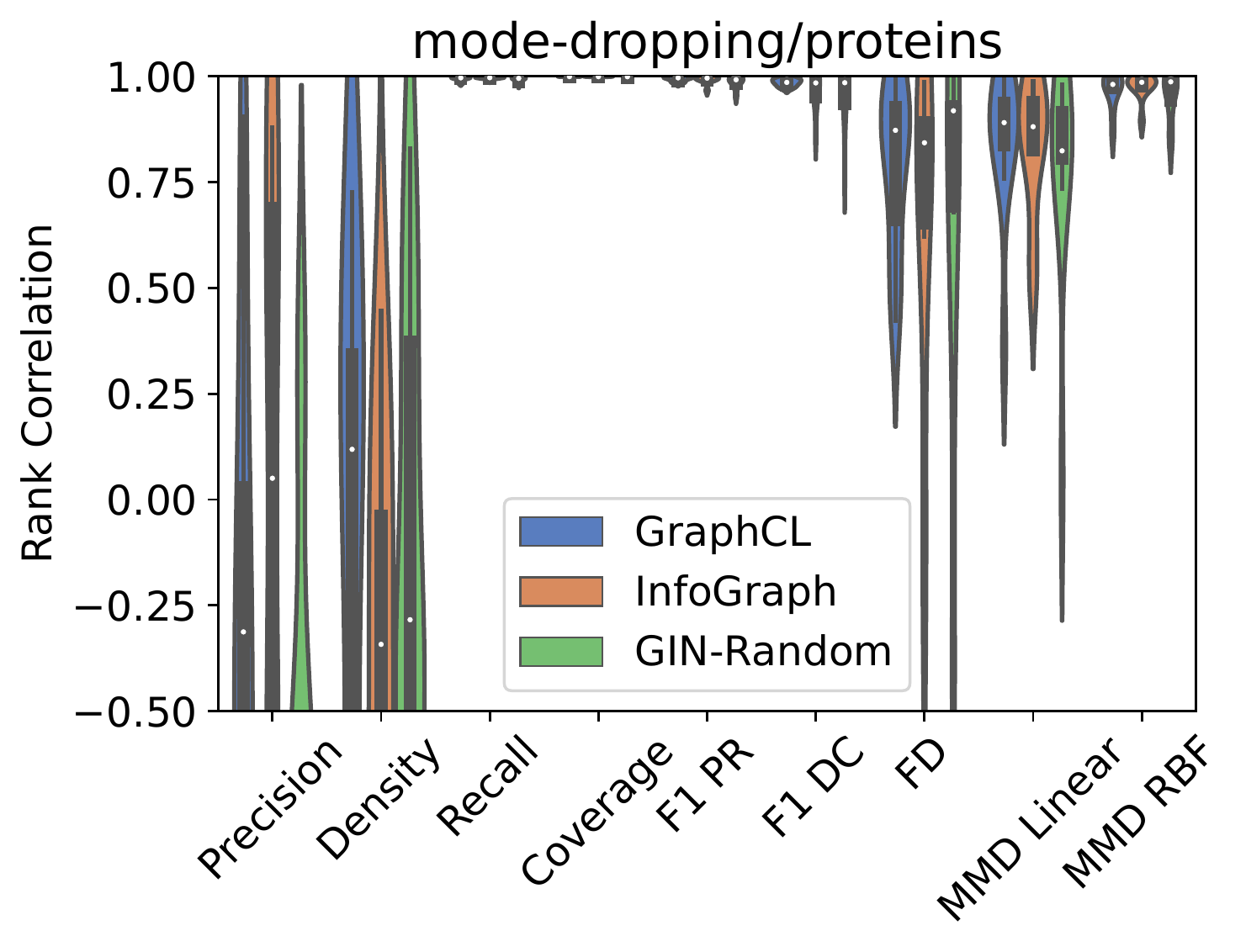}}
    \caption{Violing comparative results among the methods, with no structural features for proteins dataset.}
    \label{fig:no_struct_feats_proteins}
\end{figure*}
 
\begin{table}[h!]
\centering
\begin{small}
\scalebox{0.7}{
\begin{tabular}{l|l|l|l|l|l|l|l|l|l|l|l} 
\toprule
Model Name               & Experiment                      &        & Precision & Density & Recall & Coverage & F1PR & F1DC & FD & MMD Lin & MMD RBF  \\ 
\hline
\multirow{8}{*}{GraphCL}  &    \multirow{2}{*}{Mixing Random} & Mean & 1.0 & 1.0 & 0.91 & 0.99 & 1.0 & 1.0 & 1.0 & 1.0 & 1.0 \\
\cline{3-12}
                   &     & Median & 1.0 & 1.0 & 0.92 & 0.98 & 1.0 & 1.0 & 1.0 & 1.0 & 1.0 \\ 
\cline{2-12}
 &    \multirow{2}{*}{Rewiring Edges} & Mean & 0.29 & 0.99 & 0.91 & 0.99 & 0.91 & 0.99 & 0.98 & 1.0 & 1.0 \\
\cline{3-12}
                   &     & Median & 0.28 & 0.99 & 0.92 & 0.99 & 0.92 & 0.99 & 0.98 & 1.0 & 1.0 \\ 
\cline{2-12}
 &    \multirow{2}{*}{Mode Collapse} & Mean & -0.5 & -0.68 & 0.98 & 1.0 & 0.95 & 0.85 & 0.98 & 1.0 & 1.0 \\
\cline{3-12}
                   &     & Median & -0.72 & -0.87 & 0.99 & 1.0 & 0.97 & 0.89 & 0.99 & 1.0 & 1.0 \\ 
\cline{2-12}
 &    \multirow{2}{*}{Mode Dropping} & Mean & -0.22 & -0.06 & 0.99 & 1.0 & 0.99 & 0.99 & 0.79 & 0.84 & 0.97 \\
\cline{3-12}
                   &     & Median & -0.31 & 0.12 & 1.0 & 1.0 & 1.0 & 0.99 & 0.87 & 0.89 & 0.98 \\ 
\cline{1-12}
\multirow{8}{*}{InfoGraph}  &    \multirow{2}{*}{Mixing Random} & Mean & 1.0 & 1.0 & 0.97 & 0.99 & 1.0 & 1.0 & 1.0 & 1.0 & 1.0 \\
\cline{3-12}
                   &     & Median & 1.0 & 1.0 & 0.98 & 0.99 & 1.0 & 1.0 & 1.0 & 1.0 & 1.0 \\ 
\cline{2-12}
 &    \multirow{2}{*}{Rewiring Edges} & Mean & 0.99 & 0.99 & 0.95 & 0.99 & 0.96 & 0.99 & 0.95 & 0.99 & 0.98 \\
\cline{3-12}
                   &     & Median & 0.99 & 0.99 & 0.95 & 0.99 & 0.96 & 1.0 & 0.95 & 0.99 & 0.98 \\ 
\cline{2-12}
 &    \multirow{2}{*}{Mode Collapse} & Mean & -0.97 & -0.84 & 0.99 & 1.0 & 0.97 & 0.78 & 0.71 & 0.96 & 1.0 \\
\cline{3-12}
                   &     & Median & -0.97 & -0.97 & 0.99 & 1.0 & 0.97 & 0.79 & 0.98 & 0.99 & 1.0 \\ 
\cline{2-12}
 &    \multirow{2}{*}{Mode Dropping} & Mean & -0.0 & -0.34 & 1.0 & 1.0 & 0.99 & 0.97 & 0.65 & 0.83 & 0.98 \\
\cline{3-12}
                   &     & Median & 0.05 & -0.34 & 1.0 & 1.0 & 1.0 & 0.98 & 0.84 & 0.88 & 0.99 \\ 
\cline{1-12}
\multirow{8}{*}{GIN-Random}  &    \multirow{2}{*}{Mixing Random} & Mean & 1.0 & 1.0 & 0.86 & 0.98 & 1.0 & 1.0 & 1.0 & 1.0 & 1.0 \\
\cline{3-12}
                   &     & Median & 1.0 & 1.0 & 0.92 & 0.98 & 1.0 & 1.0 & 1.0 & 1.0 & 1.0 \\ 
\cline{2-12}
 &    \multirow{2}{*}{Rewiring Edges} & Mean & 0.99 & 1.0 & 0.96 & 1.0 & 0.99 & 1.0 & 1.0 & 1.0 & 1.0 \\
\cline{3-12}
                   &     & Median & 1.0 & 1.0 & 0.96 & 1.0 & 1.0 & 1.0 & 1.0 & 1.0 & 1.0 \\ 
\cline{2-12}
 &    \multirow{2}{*}{Mode Collapse} & Mean & -0.89 & -0.5 & 0.96 & 0.99 & 0.89 & 0.85 & 0.84 & 0.99 & 0.99 \\
\cline{3-12}
                   &     & Median & -0.93 & -0.76 & 0.97 & 1.0 & 0.95 & 0.92 & 0.99 & 1.0 & 1.0 \\ 
\cline{2-12}
 &    \multirow{2}{*}{Mode Dropping} & Mean & -0.59 & -0.09 & 0.99 & 1.0 & 0.99 & 0.95 & 0.63 & 0.78 & 0.96 \\
\cline{3-12}
                   &     & Median & -0.82 & -0.28 & 1.0 & 1.0 & 0.99 & 0.98 & 0.92 & 0.82 & 0.99 \\ \bottomrule
\end{tabular}
}
\end{small}
\caption{Mean and median values for measurements in experiments with no structural features by models.} 
\label{table:no_struct_feats_proteins}
\end{table}
 
\begin{figure*}[h!]
    \captionsetup[subfloat]{farskip=-2pt,captionskip=-8pt}
    \centering
    \subfloat[][]{\includegraphics[width = 2.55in]{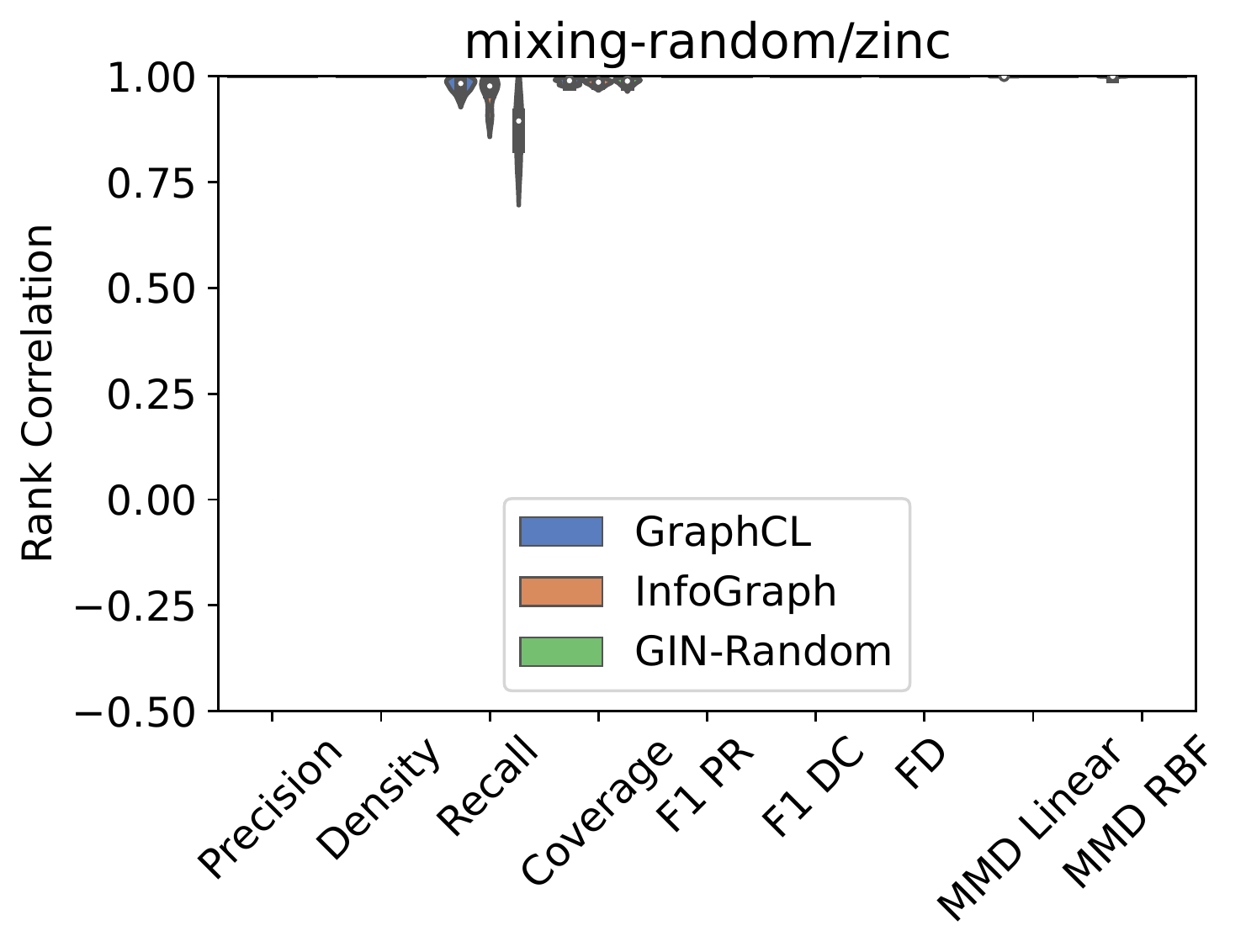}}
    \subfloat[][]{\includegraphics[width = 2.55in]{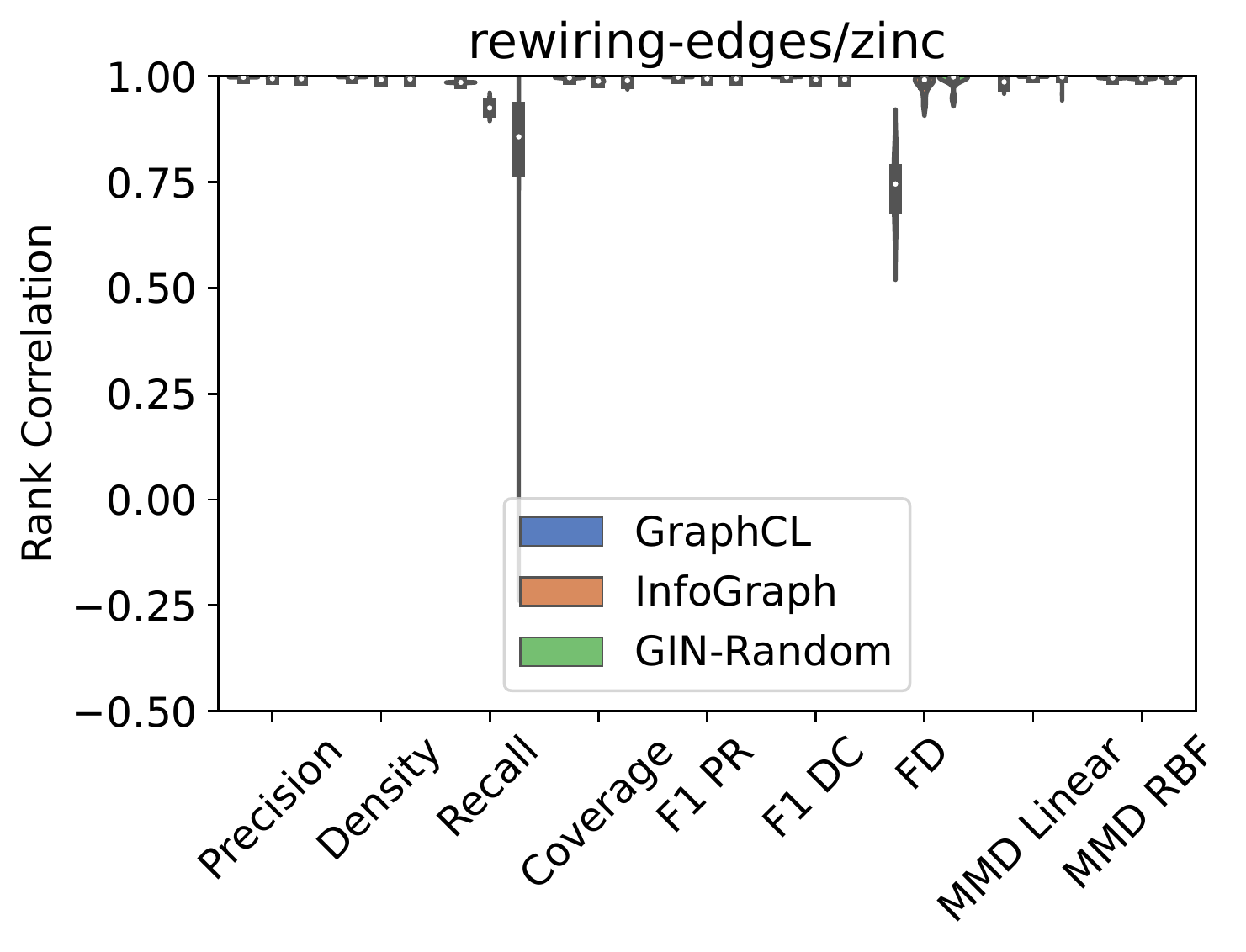}}
    \\
    \subfloat[][]{\includegraphics[width = 2.55in]{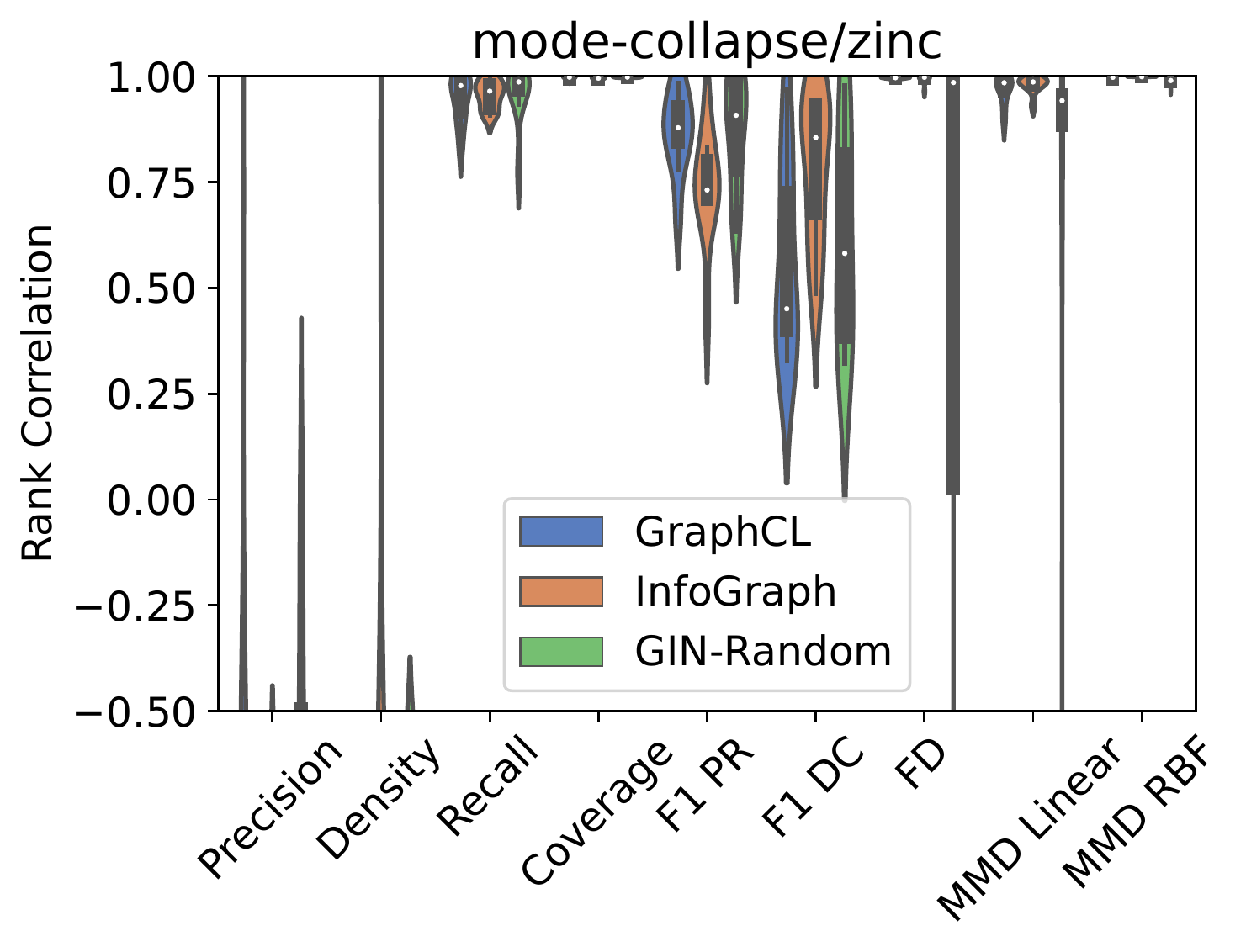}}
    \subfloat[][]{\includegraphics[width = 2.55in]{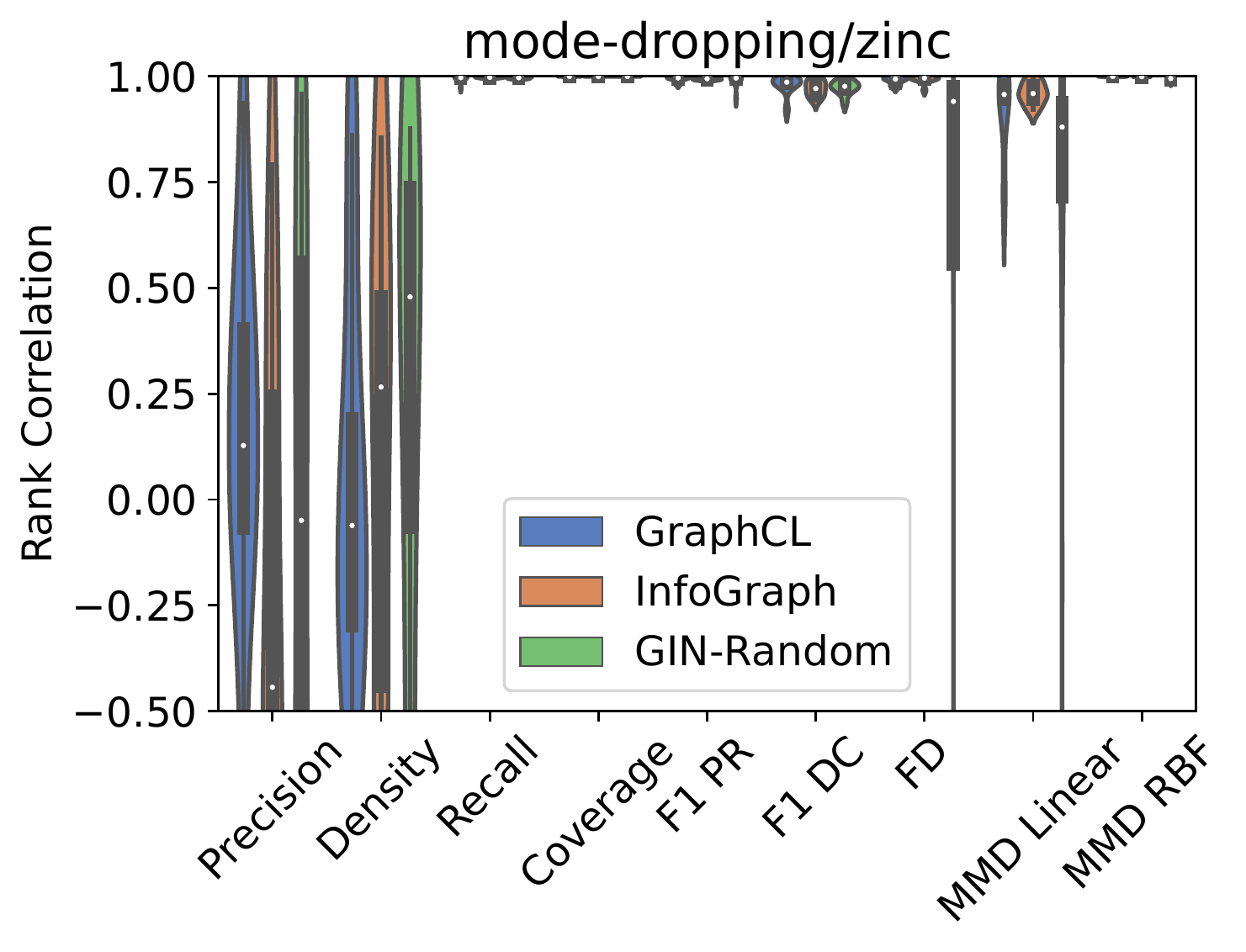}}
    \caption{Violing comparative results among the methods, with no structural features for zinc dataset.}
    \label{fig:no_struct_feats_zinc}
\end{figure*}
 
\begin{table}[h!]
\centering
\begin{small}
\scalebox{0.7}{
\begin{tabular}{l|l|l|l|l|l|l|l|l|l|l|l} 
\toprule
Model Name               & Experiment                      &        & Precision & Density & Recall & Coverage & F1PR & F1DC & FD & MMD Lin & MMD RBF  \\ 
\hline
\multirow{8}{*}{GraphCL}  &    \multirow{2}{*}{Mixing Random} & Mean & 1.0 & 1.0 & 0.98 & 0.99 & 1.0 & 1.0 & 1.0 & 1.0 & 1.0 \\
\cline{3-12}
                   &     & Median & 1.0 & 1.0 & 0.98 & 0.99 & 1.0 & 1.0 & 1.0 & 1.0 & 1.0 \\ 
\cline{2-12}
 &    \multirow{2}{*}{Rewiring Edges} & Mean & 1.0 & 1.0 & 0.99 & 1.0 & 1.0 & 1.0 & 0.73 & 0.98 & 1.0 \\
\cline{3-12}
                   &     & Median & 1.0 & 1.0 & 0.99 & 1.0 & 1.0 & 1.0 & 0.75 & 0.99 & 1.0 \\ 
\cline{2-12}
 &    \multirow{2}{*}{Mode Collapse} & Mean & -0.67 & -0.87 & 0.95 & 1.0 & 0.87 & 0.56 & 1.0 & 0.97 & 1.0 \\
\cline{3-12}
                   &     & Median & -0.8 & -0.92 & 0.98 & 1.0 & 0.88 & 0.45 & 1.0 & 0.98 & 1.0 \\ 
\cline{2-12}
 &    \multirow{2}{*}{Mode Dropping} & Mean & 0.16 & 0.04 & 0.99 & 1.0 & 0.99 & 0.98 & 0.99 & 0.92 & 1.0 \\
\cline{3-12}
                   &     & Median & 0.13 & -0.06 & 1.0 & 1.0 & 1.0 & 0.99 & 0.99 & 0.96 & 1.0 \\ 
\cline{1-12}
\multirow{8}{*}{InfoGraph}  &    \multirow{2}{*}{Mixing Random} & Mean & 1.0 & 1.0 & 0.97 & 0.99 & 1.0 & 1.0 & 1.0 & 1.0 & 1.0 \\
\cline{3-12}
                   &     & Median & 1.0 & 1.0 & 0.98 & 0.99 & 1.0 & 1.0 & 1.0 & 1.0 & 1.0 \\ 
\cline{2-12}
 &    \multirow{2}{*}{Rewiring Edges} & Mean & 0.99 & 0.99 & 0.93 & 0.99 & 0.99 & 0.99 & 0.98 & 1.0 & 0.99 \\
\cline{3-12}
                   &     & Median & 0.99 & 0.99 & 0.93 & 0.99 & 0.99 & 0.99 & 0.99 & 1.0 & 0.99 \\ 
\cline{2-12}
 &    \multirow{2}{*}{Mode Collapse} & Mean & -0.92 & -0.68 & 0.96 & 0.99 & 0.74 & 0.78 & 0.99 & 0.98 & 1.0 \\
\cline{3-12}
                   &     & Median & -0.99 & -0.87 & 0.96 & 1.0 & 0.73 & 0.86 & 1.0 & 0.99 & 1.0 \\ 
\cline{2-12}
 &    \multirow{2}{*}{Mode Dropping} & Mean & -0.21 & 0.06 & 1.0 & 1.0 & 0.99 & 0.97 & 0.99 & 0.96 & 1.0 \\
\cline{3-12}
                   &     & Median & -0.44 & 0.27 & 1.0 & 1.0 & 0.99 & 0.97 & 1.0 & 0.96 & 1.0 \\ 
\cline{1-12}
\multirow{8}{*}{GIN-Random}  &    \multirow{2}{*}{Mixing Random} & Mean & 1.0 & 1.0 & 0.88 & 0.99 & 1.0 & 1.0 & 1.0 & 1.0 & 1.0 \\
\cline{3-12}
                   &     & Median & 1.0 & 1.0 & 0.89 & 0.99 & 1.0 & 1.0 & 1.0 & 1.0 & 1.0 \\ 
\cline{2-12}
 &    \multirow{2}{*}{Rewiring Edges} & Mean & 0.99 & 0.99 & 0.79 & 0.99 & 0.99 & 0.99 & 0.99 & 0.99 & 1.0 \\
\cline{3-12}
                   &     & Median & 0.99 & 0.99 & 0.86 & 0.99 & 0.99 & 0.99 & 1.0 & 1.0 & 1.0 \\ 
\cline{2-12}
 &    \multirow{2}{*}{Mode Collapse} & Mean & -0.68 & -0.85 & 0.96 & 1.0 & 0.87 & 0.62 & 0.56 & 0.78 & 0.99 \\
\cline{3-12}
                   &     & Median & -0.8 & -0.92 & 0.99 & 1.0 & 0.91 & 0.58 & 0.99 & 0.94 & 0.99 \\ 
\cline{2-12}
 &    \multirow{2}{*}{Mode Dropping} & Mean & -0.08 & 0.25 & 1.0 & 1.0 & 0.99 & 0.97 & 0.61 & 0.75 & 0.99 \\
\cline{3-12}
                   &     & Median & -0.05 & 0.48 & 1.0 & 1.0 & 1.0 & 0.98 & 0.94 & 0.88 & 0.99 \\ \bottomrule
\end{tabular}
}
\end{small}
\caption{Mean and median values for measurements in experiments with no structural features by models.} 
\label{table:no_struct_feats_zinc}
\end{table}

%% file: deg_experiments_figs_and_tables.tex
\begin{figure*}[h!]
    \captionsetup[subfloat]{farskip=-2pt,captionskip=-8pt}
    \centering
    \subfloat[][]{\includegraphics[width = 2.55in]{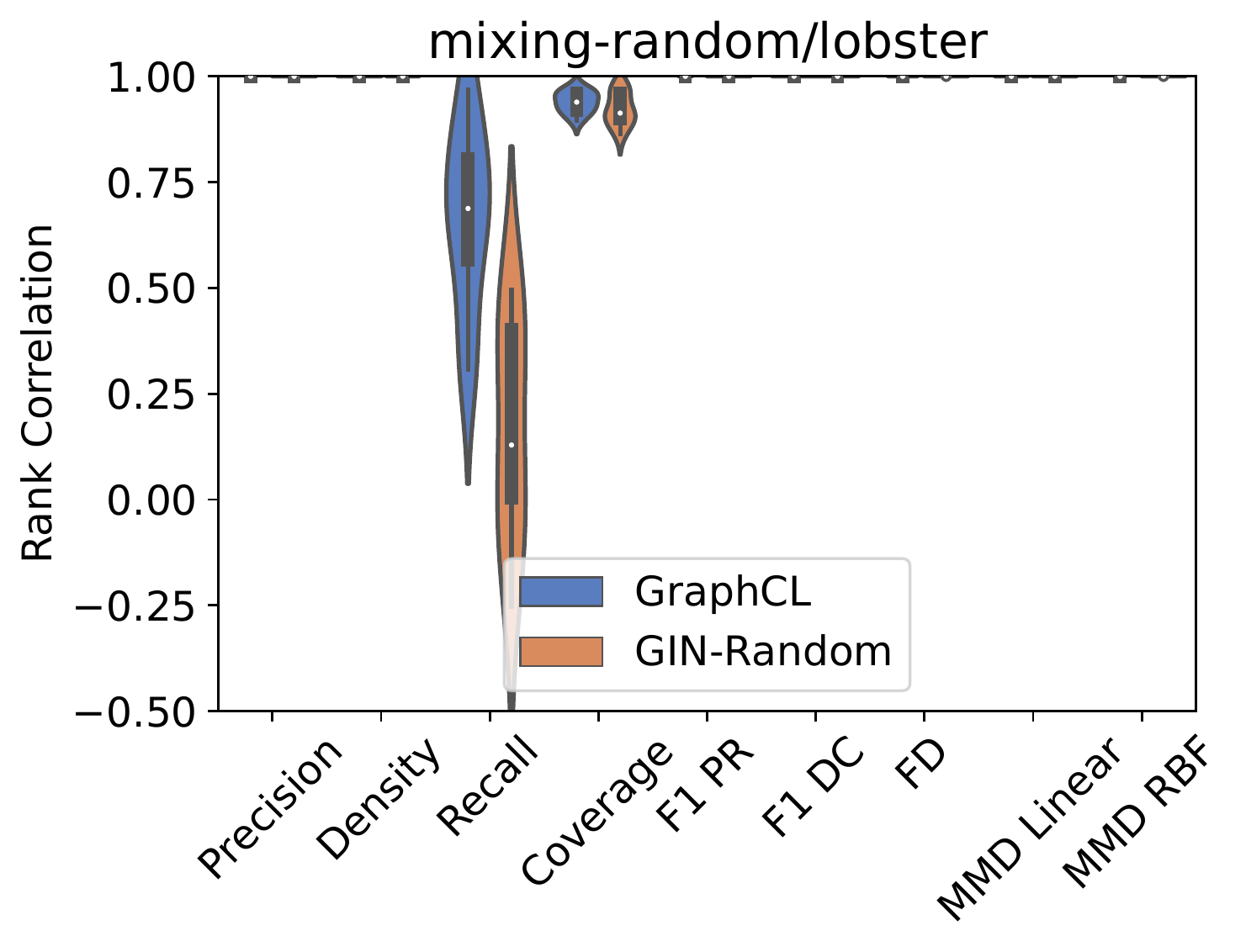}}
    \subfloat[][]{\includegraphics[width = 2.55in]{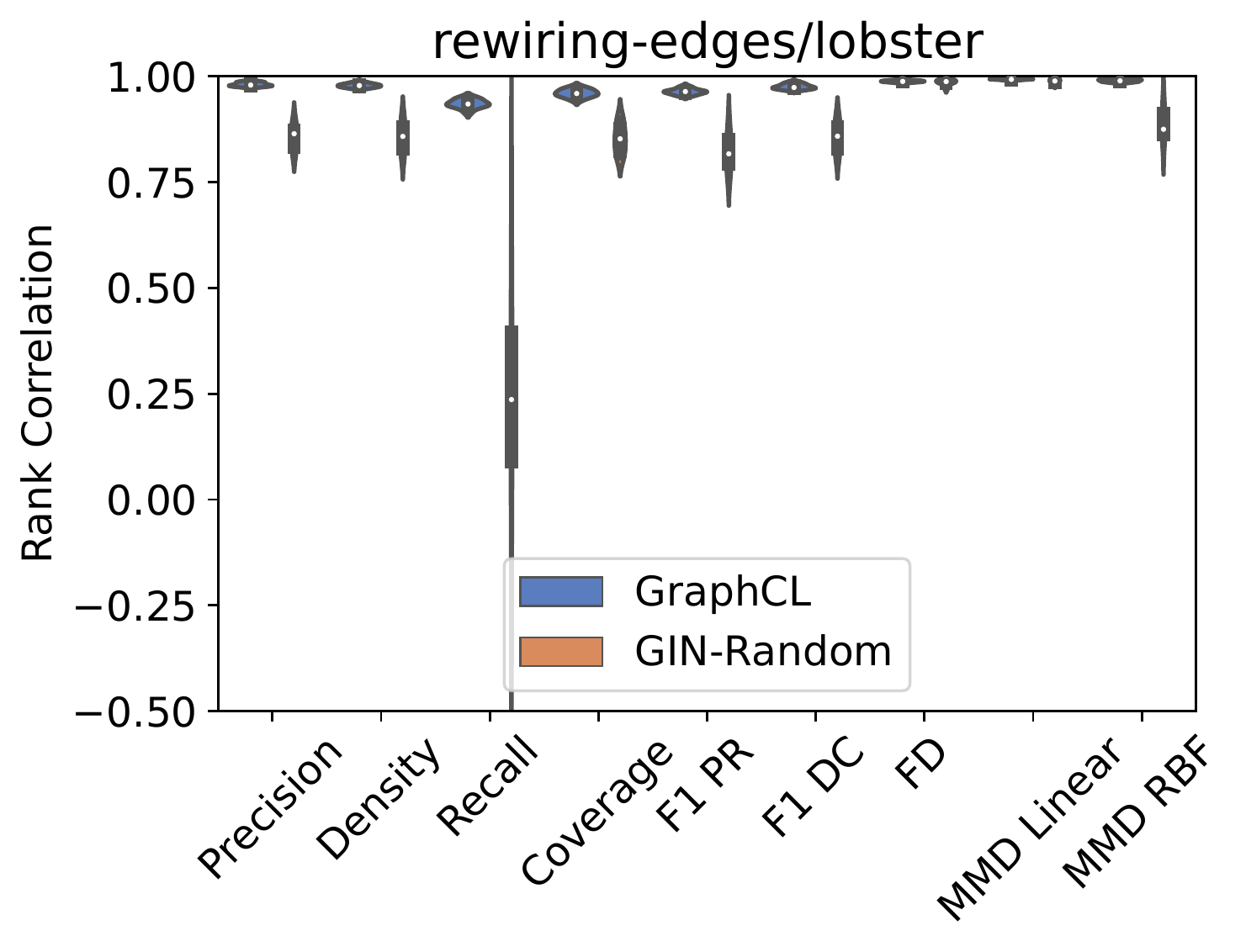}}
    \\
    \subfloat[][]{\includegraphics[width = 2.55in]{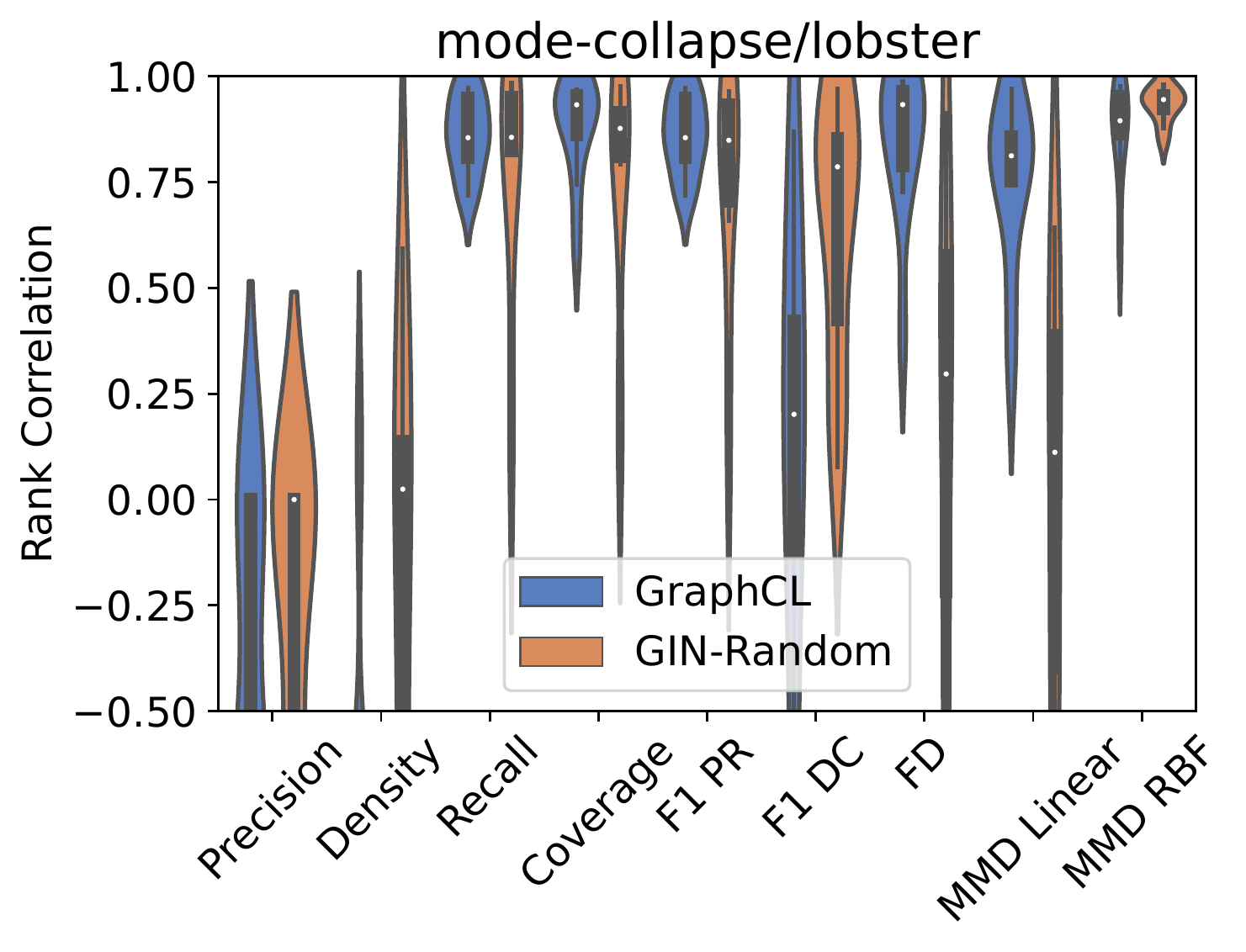}}
    \subfloat[][]{\includegraphics[width = 2.55in]{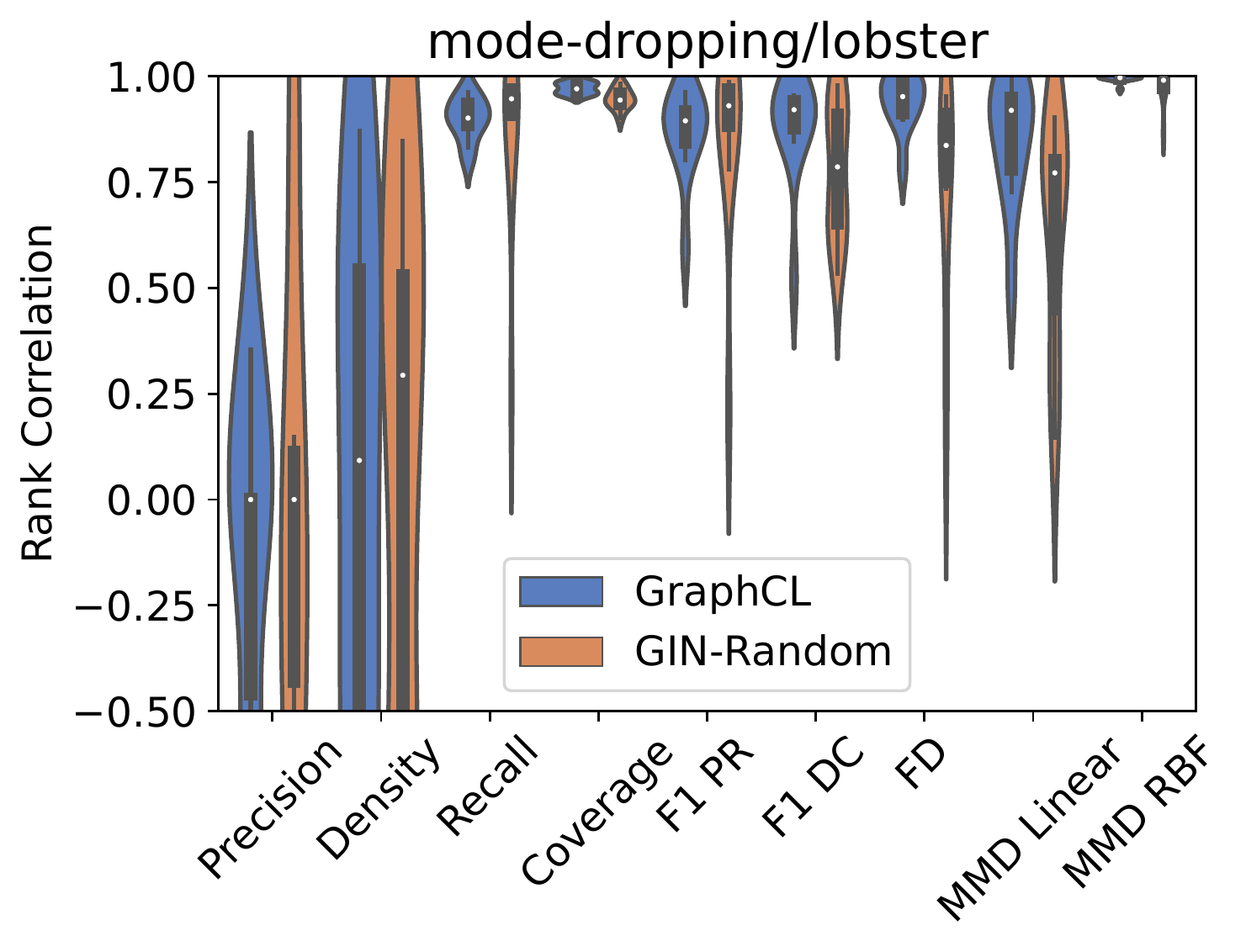}}
    \caption{Violing comparative results among the methods, with with degree features for lobster dataset.}
    \label{fig:deg_feats_lobster}
\end{figure*}
 
\begin{table}[h!]
\centering
\begin{small}
\scalebox{0.7}{
\begin{tabular}{l|l|l|l|l|l|l|l|l|l|l|l} 
\toprule
Model Name               & Experiment                      &        & Precision & Density & Recall & Coverage & F1PR & F1DC & FD & MMD Lin & MMD RBF  \\ 
\hline
\multirow{8}{*}{GraphCL}  &    \multirow{2}{*}{Mixing Random} & Mean & 1.0 & 1.0 & 0.66 & 0.94 & 1.0 & 1.0 & 1.0 & 1.0 & 1.0 \\
\cline{3-12}
                   &     & Median & 1.0 & 1.0 & 0.69 & 0.94 & 1.0 & 1.0 & 1.0 & 1.0 & 1.0 \\ 
\cline{2-12}
 &    \multirow{2}{*}{Rewiring Edges} & Mean & 0.98 & 0.98 & 0.93 & 0.96 & 0.96 & 0.98 & 0.99 & 0.99 & 0.99 \\
\cline{3-12}
                   &     & Median & 0.98 & 0.98 & 0.93 & 0.96 & 0.96 & 0.97 & 0.99 & 0.99 & 0.99 \\ 
\cline{2-12}
 &    \multirow{2}{*}{Mode Collapse} & Mean & -0.47 & -0.82 & 0.86 & 0.88 & 0.86 & 0.1 & 0.85 & 0.75 & 0.87 \\
\cline{3-12}
                   &     & Median & -0.66 & -0.94 & 0.85 & 0.93 & 0.85 & 0.2 & 0.93 & 0.81 & 0.89 \\ 
\cline{2-12}
 &    \multirow{2}{*}{Mode Dropping} & Mean & -0.15 & 0.05 & 0.9 & 0.97 & 0.87 & 0.88 & 0.94 & 0.85 & 0.99 \\
\cline{3-12}
                   &     & Median & 0.0 & 0.09 & 0.9 & 0.97 & 0.89 & 0.92 & 0.95 & 0.92 & 1.0 \\ 
\cline{1-12}
\multirow{8}{*}{GIN-Random}  &    \multirow{2}{*}{Mixing Random} & Mean & 1.0 & 1.0 & 0.16 & 0.92 & 1.0 & 1.0 & 1.0 & 1.0 & 1.0 \\
\cline{3-12}
                   &     & Median & 1.0 & 1.0 & 0.13 & 0.91 & 1.0 & 1.0 & 1.0 & 1.0 & 1.0 \\ 
\cline{2-12}
 &    \multirow{2}{*}{Rewiring Edges} & Mean & 0.86 & 0.86 & 0.22 & 0.85 & 0.82 & 0.86 & 0.99 & 0.99 & 0.89 \\
\cline{3-12}
                   &     & Median & 0.86 & 0.86 & 0.24 & 0.85 & 0.82 & 0.86 & 0.99 & 0.99 & 0.87 \\ 
\cline{2-12}
 &    \multirow{2}{*}{Mode Collapse} & Mean & -0.29 & -0.14 & 0.76 & 0.75 & 0.74 & 0.64 & 0.2 & 0.02 & 0.93 \\
\cline{3-12}
                   &     & Median & 0.0 & 0.02 & 0.86 & 0.88 & 0.85 & 0.79 & 0.3 & 0.11 & 0.94 \\ 
\cline{2-12}
 &    \multirow{2}{*}{Mode Dropping} & Mean & -0.03 & -0.02 & 0.86 & 0.95 & 0.85 & 0.78 & 0.74 & 0.65 & 0.98 \\
\cline{3-12}
                   &     & Median & 0.0 & 0.29 & 0.95 & 0.94 & 0.93 & 0.79 & 0.84 & 0.77 & 0.99 \\ \bottomrule
\end{tabular}
}
\end{small}
\caption{Mean and median values for measurements in experiments with with degree features by models, for dataset lobster} 
\label{table:deg_feats_lobster}
\end{table}
 
\begin{figure*}[h!]
    \captionsetup[subfloat]{farskip=-2pt,captionskip=-8pt}
    \centering
    \subfloat[][]{\includegraphics[width = 2.55in]{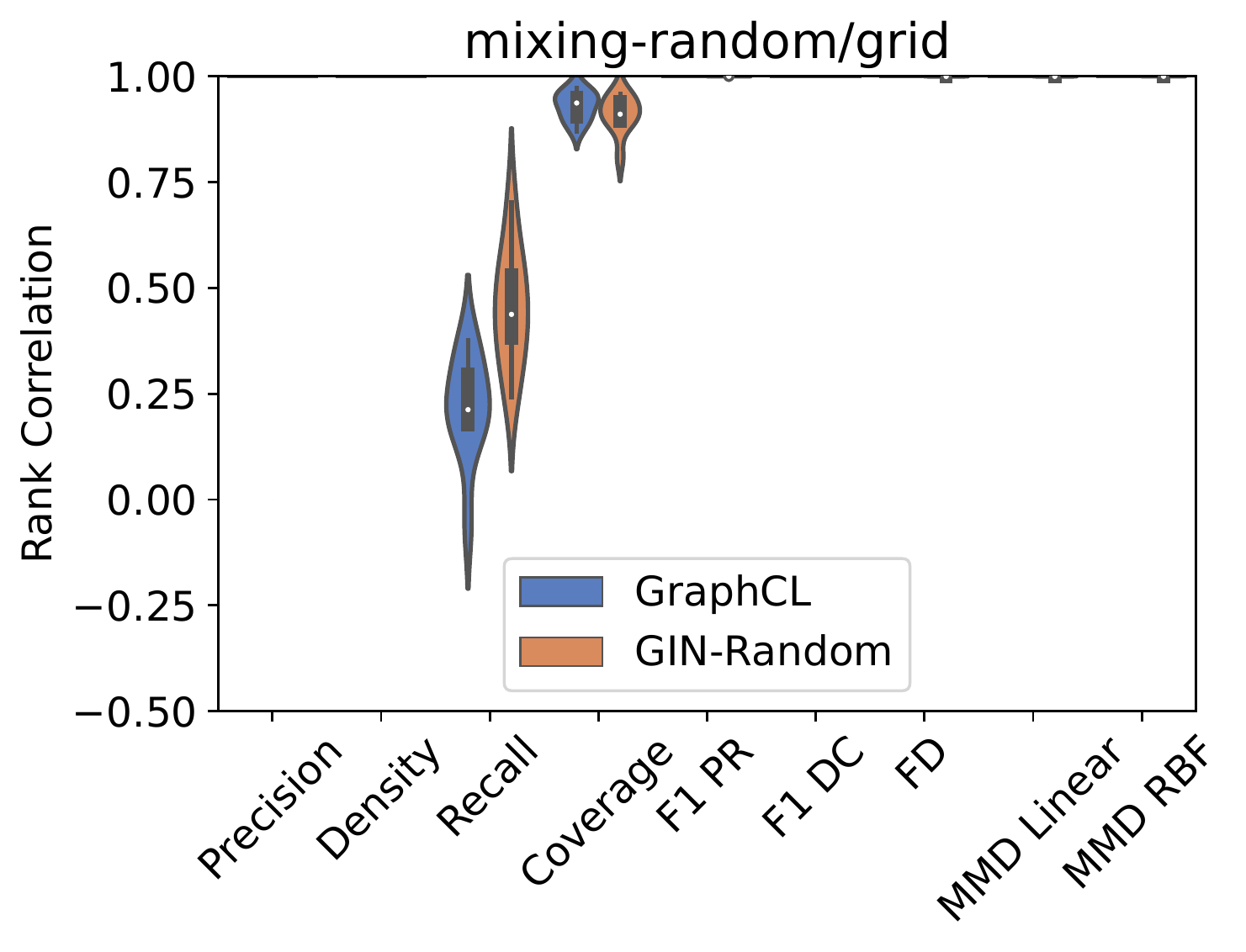}}
    \subfloat[][]{\includegraphics[width = 2.55in]{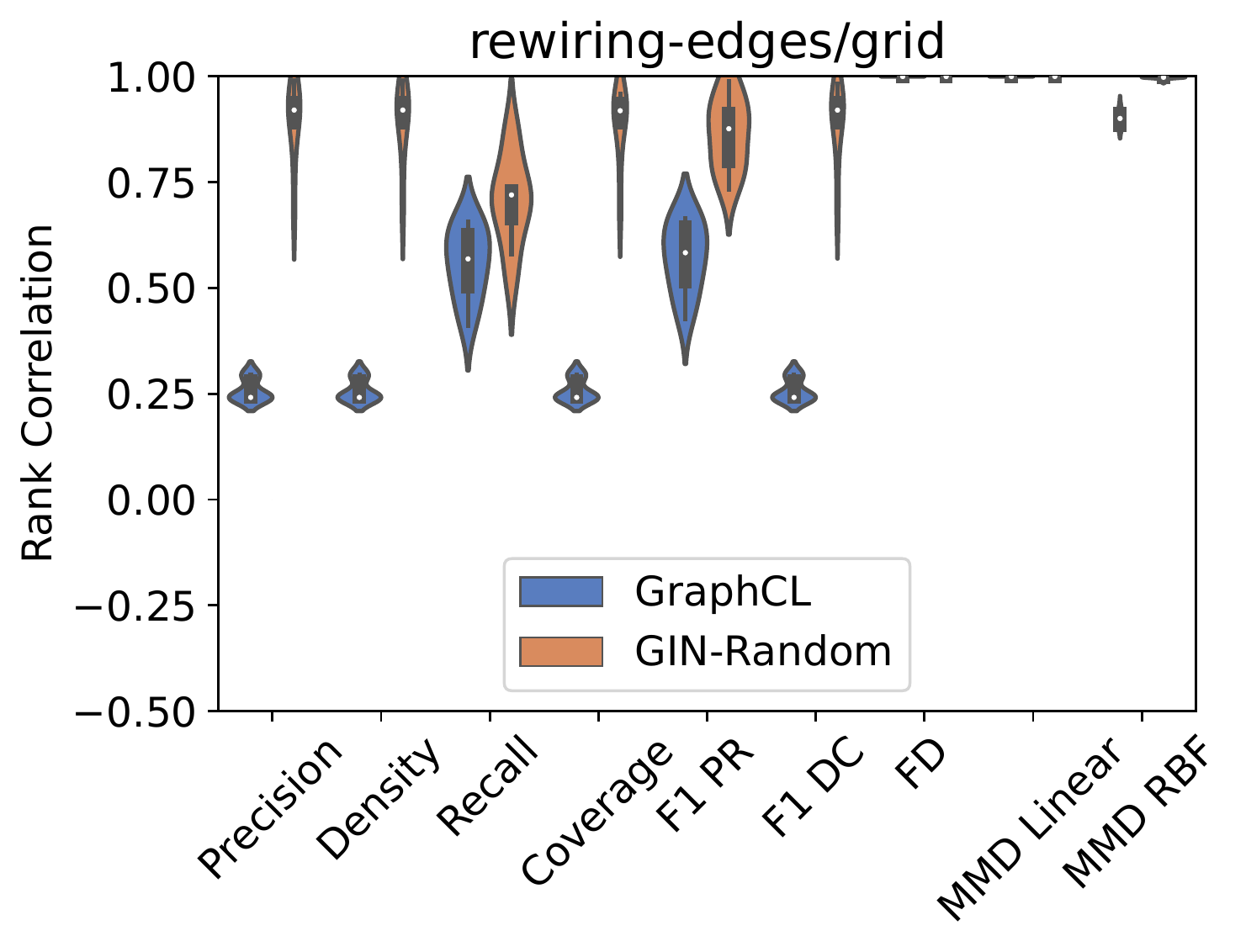}}
    \\
    \subfloat[][]{\includegraphics[width = 2.55in]{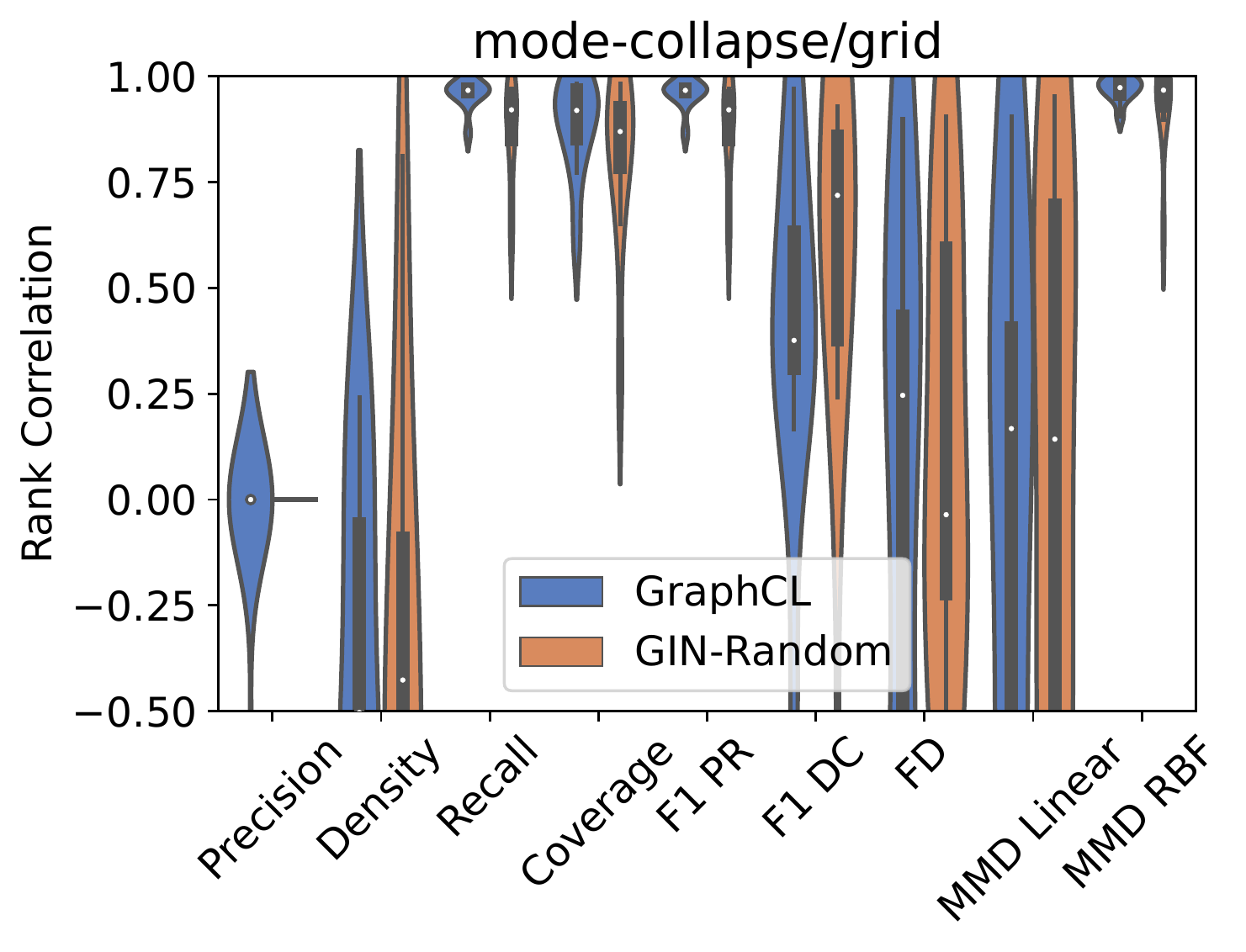}}
    \subfloat[][]{\includegraphics[width = 2.55in]{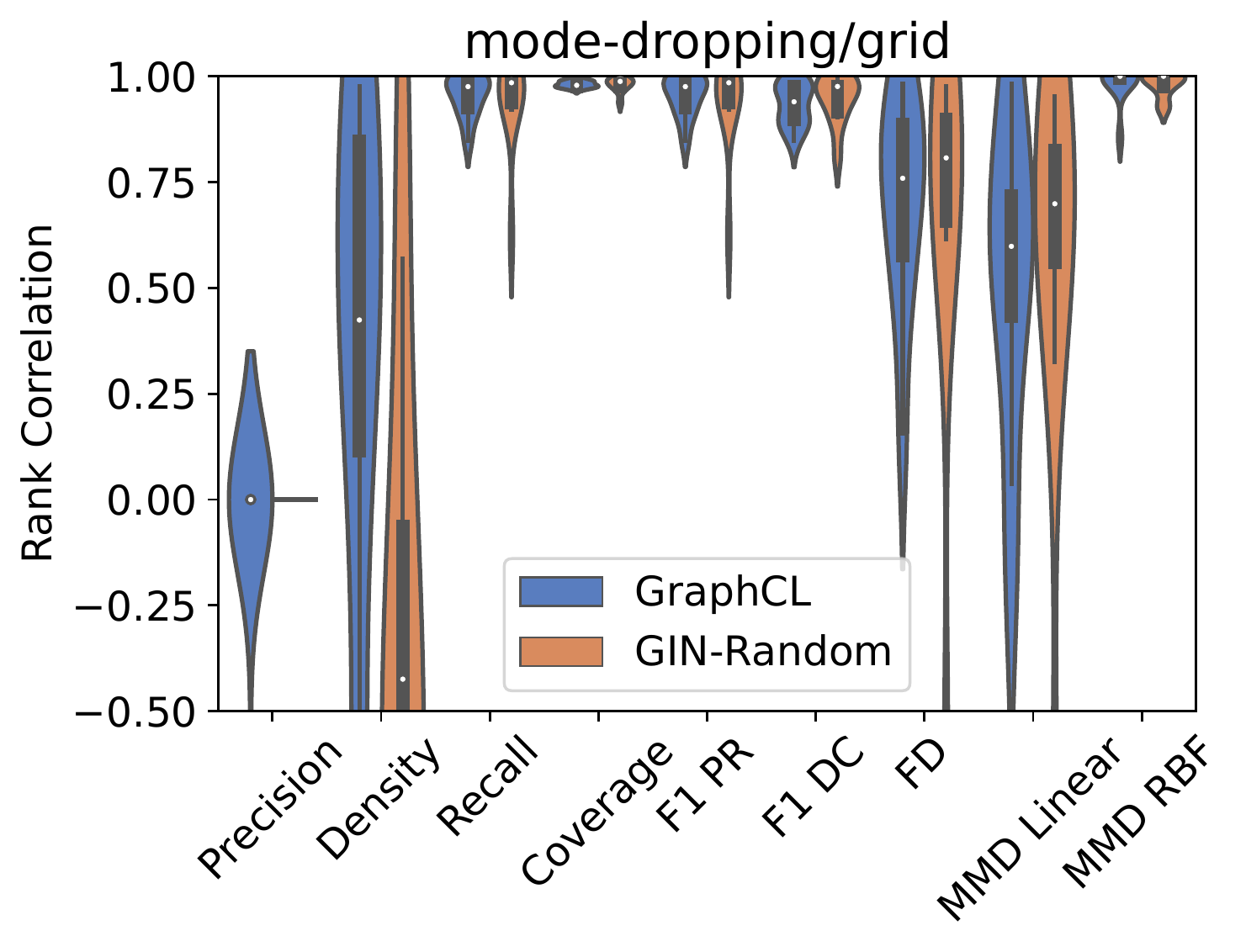}}
    \caption{Violing comparative results among the methods, with with degree features for grid dataset.}
    \label{fig:deg_feats_grid}
\end{figure*}
 
\begin{table}[h!]
\centering
\begin{small}
\scalebox{0.7}{
\begin{tabular}{l|l|l|l|l|l|l|l|l|l|l|l} 
\toprule
Model Name               & Experiment                      &        & Precision & Density & Recall & Coverage & F1PR & F1DC & FD & MMD Lin & MMD RBF  \\ 
\hline
\multirow{8}{*}{GraphCL}  &    \multirow{2}{*}{Mixing Random} & Mean & 1.0 & 1.0 & 0.22 & 0.93 & 1.0 & 1.0 & 1.0 & 1.0 & 1.0 \\
\cline{3-12}
                   &     & Median & 1.0 & 1.0 & 0.21 & 0.94 & 1.0 & 1.0 & 1.0 & 1.0 & 1.0 \\ 
\cline{2-12}
 &    \multirow{2}{*}{Rewiring Edges} & Mean & 0.26 & 0.26 & 0.56 & 0.26 & 0.57 & 0.26 & 1.0 & 1.0 & 0.9 \\
\cline{3-12}
                   &     & Median & 0.24 & 0.24 & 0.57 & 0.24 & 0.58 & 0.24 & 1.0 & 1.0 & 0.9 \\ 
\cline{2-12}
 &    \multirow{2}{*}{Mode Collapse} & Mean & -0.08 & -0.43 & 0.96 & 0.88 & 0.96 & 0.4 & -0.05 & -0.01 & 0.97 \\
\cline{3-12}
                   &     & Median & 0.0 & -0.5 & 0.97 & 0.92 & 0.97 & 0.38 & 0.25 & 0.17 & 0.97 \\ 
\cline{2-12}
 &    \multirow{2}{*}{Mode Dropping} & Mean & -0.09 & 0.31 & 0.95 & 0.98 & 0.95 & 0.93 & 0.7 & 0.53 & 0.98 \\
\cline{3-12}
                   &     & Median & 0.0 & 0.42 & 0.98 & 0.98 & 0.98 & 0.94 & 0.76 & 0.6 & 1.0 \\ 
\cline{1-12}
\multirow{8}{*}{GIN-Random}  &    \multirow{2}{*}{Mixing Random} & Mean & 1.0 & 1.0 & 0.45 & 0.91 & 1.0 & 1.0 & 1.0 & 1.0 & 1.0 \\
\cline{3-12}
                   &     & Median & 1.0 & 1.0 & 0.44 & 0.91 & 1.0 & 1.0 & 1.0 & 1.0 & 1.0 \\ 
\cline{2-12}
 &    \multirow{2}{*}{Rewiring Edges} & Mean & 0.9 & 0.89 & 0.71 & 0.89 & 0.86 & 0.89 & 1.0 & 1.0 & 1.0 \\
\cline{3-12}
                   &     & Median & 0.92 & 0.92 & 0.72 & 0.92 & 0.88 & 0.92 & 1.0 & 1.0 & 1.0 \\ 
\cline{2-12}
 &    \multirow{2}{*}{Mode Collapse} & Mean & 0.0 & -0.34 & 0.87 & 0.81 & 0.87 & 0.51 & 0.09 & 0.1 & 0.93 \\
\cline{3-12}
                   &     & Median & 0.0 & -0.43 & 0.92 & 0.87 & 0.92 & 0.72 & -0.04 & 0.14 & 0.97 \\ 
\cline{2-12}
 &    \multirow{2}{*}{Mode Dropping} & Mean & 0.0 & -0.29 & 0.94 & 0.98 & 0.94 & 0.95 & 0.67 & 0.55 & 0.98 \\
\cline{3-12}
                   &     & Median & 0.0 & -0.42 & 0.98 & 0.99 & 0.98 & 0.98 & 0.81 & 0.7 & 1.0 \\ \bottomrule
\end{tabular}
}
\end{small}
\caption{Mean and median values for measurements in experiments with with degree features by models, for dataset grid} 
\label{table:deg_feats_grid}
\end{table}
 
\begin{figure*}[h!]
    \captionsetup[subfloat]{farskip=-2pt,captionskip=-8pt}
    \centering
    \subfloat[][]{\includegraphics[width = 2.55in]{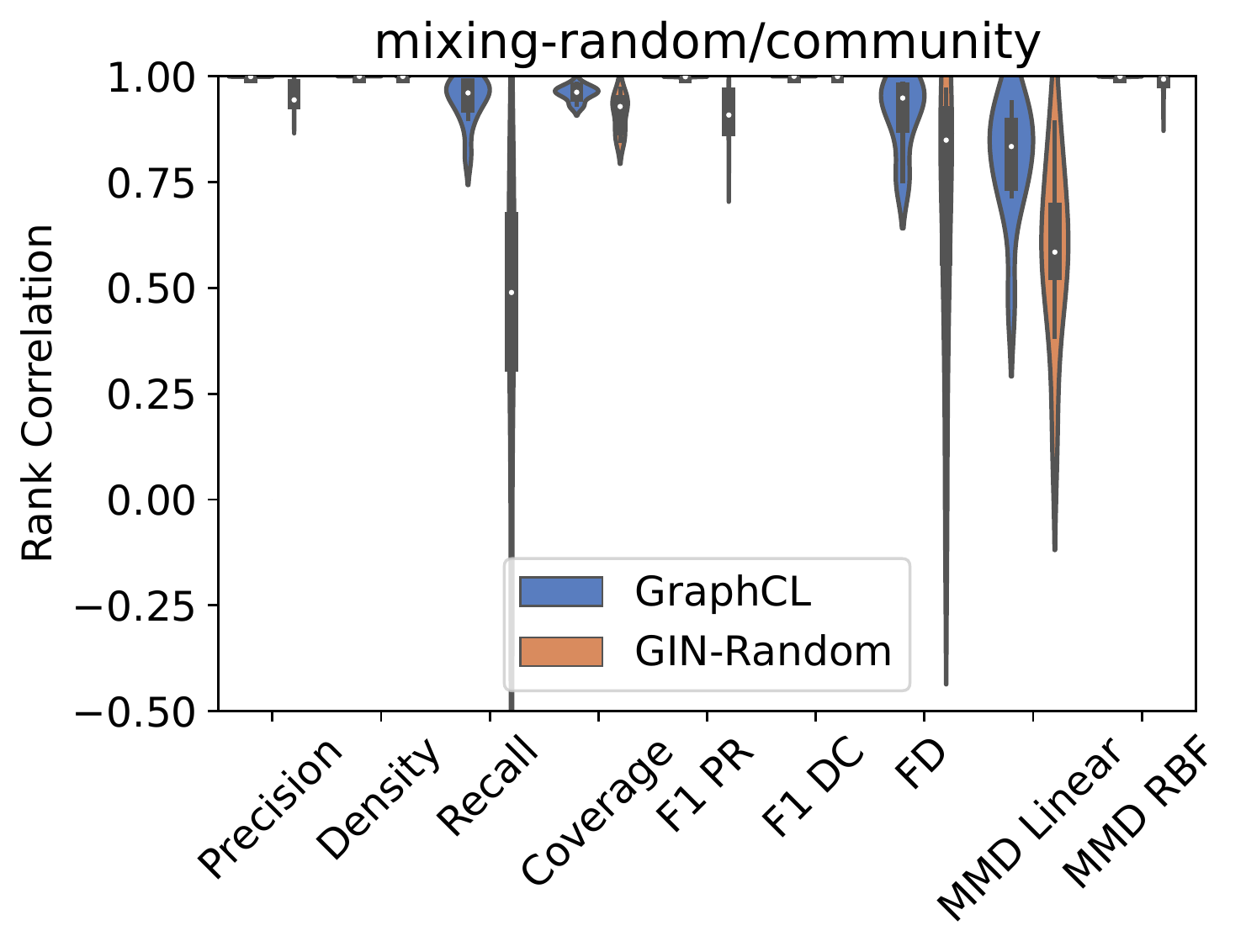}}
    \subfloat[][]{\includegraphics[width = 2.55in]{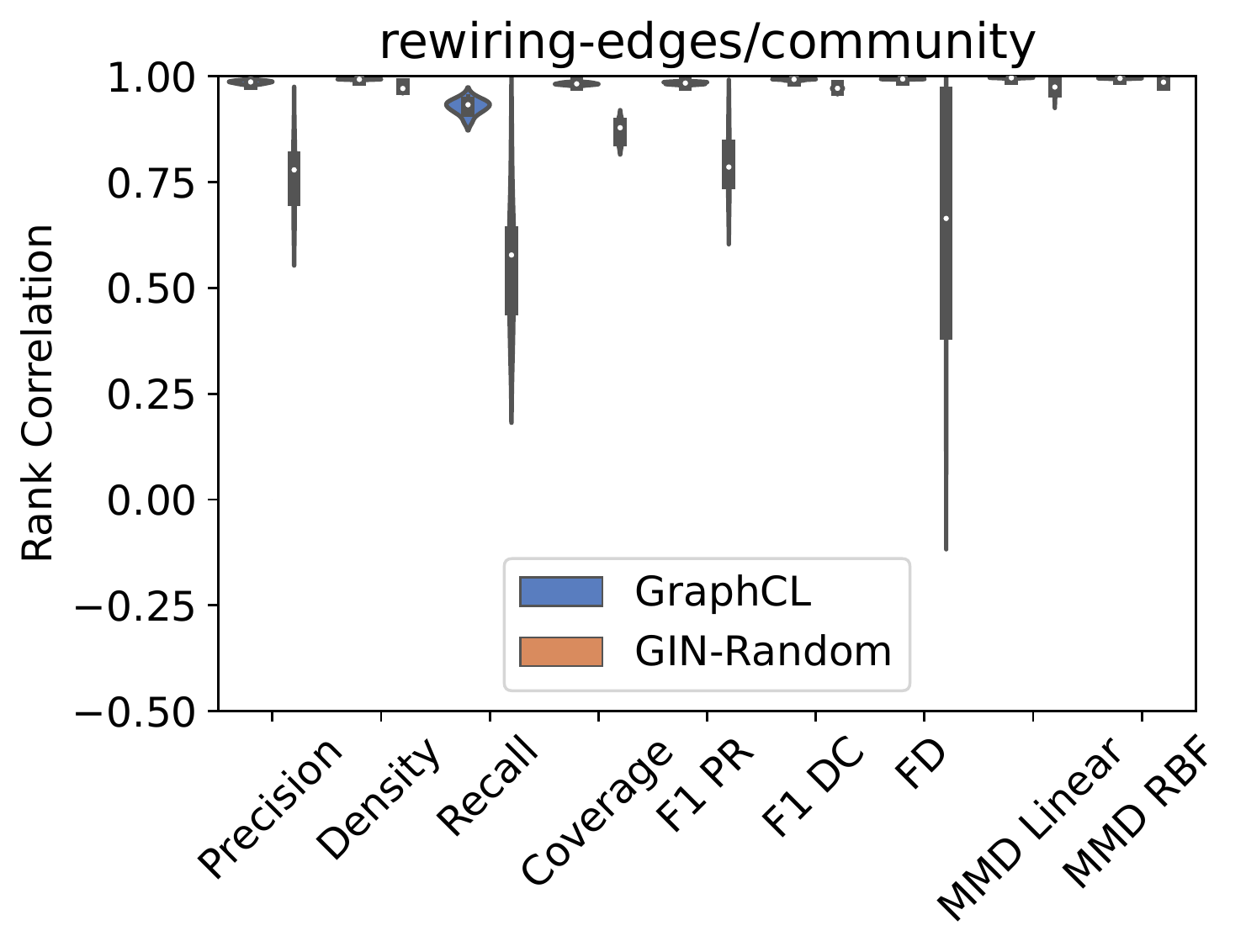}}
    \\
    \subfloat[][]{\includegraphics[width = 2.55in]{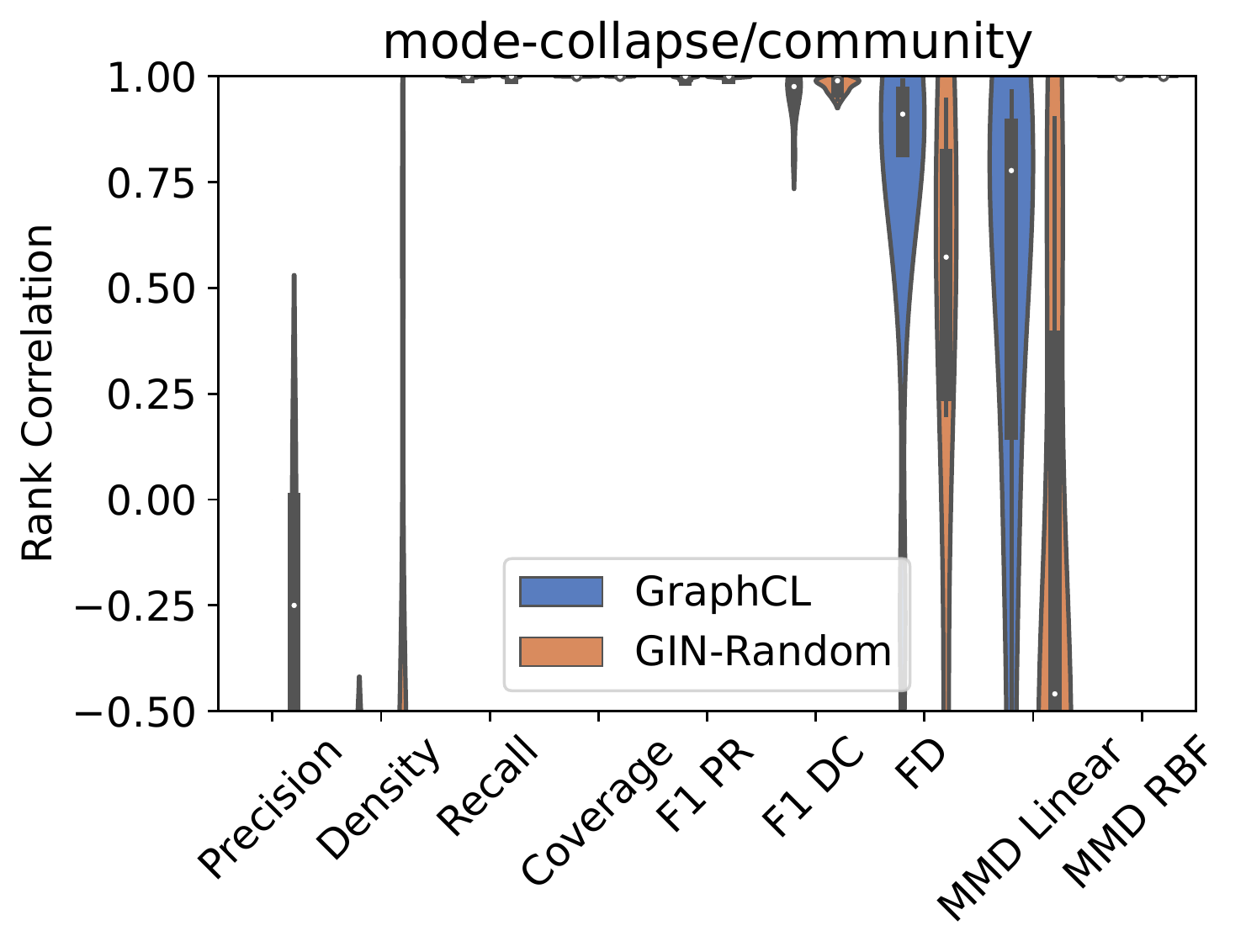}}
    \subfloat[][]{\includegraphics[width = 2.55in]{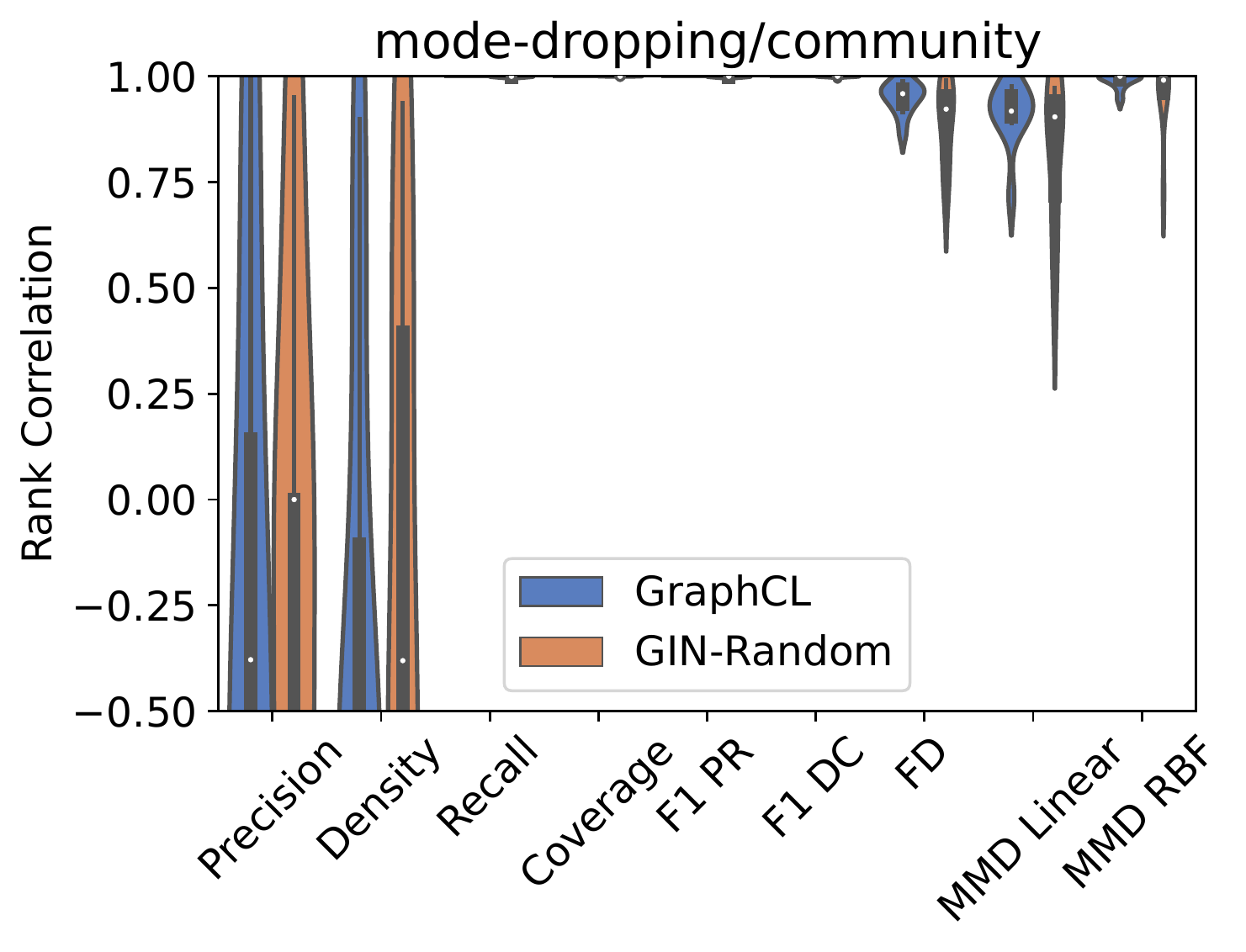}}
    \caption{Violing comparative results among the methods, with with degree features for community dataset.}
    \label{fig:deg_feats_community}
\end{figure*}
 
\begin{table}[h!]
\centering
\begin{small}
\scalebox{0.7}{
\begin{tabular}{l|l|l|l|l|l|l|l|l|l|l|l} 
\toprule
Model Name               & Experiment                      &        & Precision & Density & Recall & Coverage & F1PR & F1DC & FD & MMD Lin & MMD RBF  \\ 
\hline
\multirow{8}{*}{GraphCL}  &    \multirow{2}{*}{Mixing Random} & Mean & 1.0 & 1.0 & 0.94 & 0.96 & 1.0 & 1.0 & 0.91 & 0.8 & 1.0 \\
\cline{3-12}
                   &     & Median & 1.0 & 1.0 & 0.96 & 0.96 & 1.0 & 1.0 & 0.95 & 0.83 & 1.0 \\ 
\cline{2-12}
 &    \multirow{2}{*}{Rewiring Edges} & Mean & 0.99 & 0.99 & 0.93 & 0.98 & 0.98 & 0.99 & 0.99 & 1.0 & 0.99 \\
\cline{3-12}
                   &     & Median & 0.99 & 0.99 & 0.93 & 0.98 & 0.98 & 0.99 & 0.99 & 1.0 & 0.99 \\ 
\cline{2-12}
 &    \multirow{2}{*}{Mode Collapse} & Mean & -0.94 & -0.91 & 1.0 & 1.0 & 1.0 & 0.96 & 0.76 & 0.48 & 1.0 \\
\cline{3-12}
                   &     & Median & -0.96 & -0.95 & 1.0 & 1.0 & 1.0 & 0.98 & 0.91 & 0.78 & 1.0 \\ 
\cline{2-12}
 &    \multirow{2}{*}{Mode Dropping} & Mean & -0.21 & -0.36 & 1.0 & 1.0 & 1.0 & 1.0 & 0.95 & 0.91 & 0.99 \\
\cline{3-12}
                   &     & Median & -0.38 & -0.67 & 1.0 & 1.0 & 1.0 & 1.0 & 0.96 & 0.92 & 1.0 \\ 
\cline{1-12}
\multirow{8}{*}{GIN-Random}  &    \multirow{2}{*}{Mixing Random} & Mean & 0.95 & 1.0 & 0.39 & 0.91 & 0.9 & 1.0 & 0.69 & 0.59 & 0.98 \\
\cline{3-12}
                   &     & Median & 0.94 & 1.0 & 0.49 & 0.93 & 0.91 & 1.0 & 0.85 & 0.58 & 0.99 \\ 
\cline{2-12}
 &    \multirow{2}{*}{Rewiring Edges} & Mean & 0.76 & 0.97 & 0.56 & 0.87 & 0.79 & 0.97 & 0.66 & 0.98 & 0.98 \\
\cline{3-12}
                   &     & Median & 0.78 & 0.97 & 0.58 & 0.88 & 0.79 & 0.97 & 0.66 & 0.97 & 0.99 \\ 
\cline{2-12}
 &    \multirow{2}{*}{Mode Collapse} & Mean & -0.39 & -0.65 & 1.0 & 1.0 & 1.0 & 0.98 & 0.33 & -0.2 & 1.0 \\
\cline{3-12}
                   &     & Median & -0.25 & -0.85 & 1.0 & 1.0 & 1.0 & 0.99 & 0.57 & -0.46 & 1.0 \\ 
\cline{2-12}
 &    \multirow{2}{*}{Mode Dropping} & Mean & -0.17 & -0.16 & 1.0 & 1.0 & 1.0 & 1.0 & 0.89 & 0.83 & 0.96 \\
\cline{3-12}
                   &     & Median & 0.0 & -0.38 & 1.0 & 1.0 & 1.0 & 1.0 & 0.92 & 0.9 & 0.99 \\ \bottomrule
\end{tabular}
}
\end{small}
\caption{Mean and median values for measurements in experiments with with degree features by models, for dataset community} 
\label{table:deg_feats_community}
\end{table}
 
\begin{figure*}[h!]
    \captionsetup[subfloat]{farskip=-2pt,captionskip=-8pt}
    \centering
    \subfloat[][]{\includegraphics[width = 2.55in]{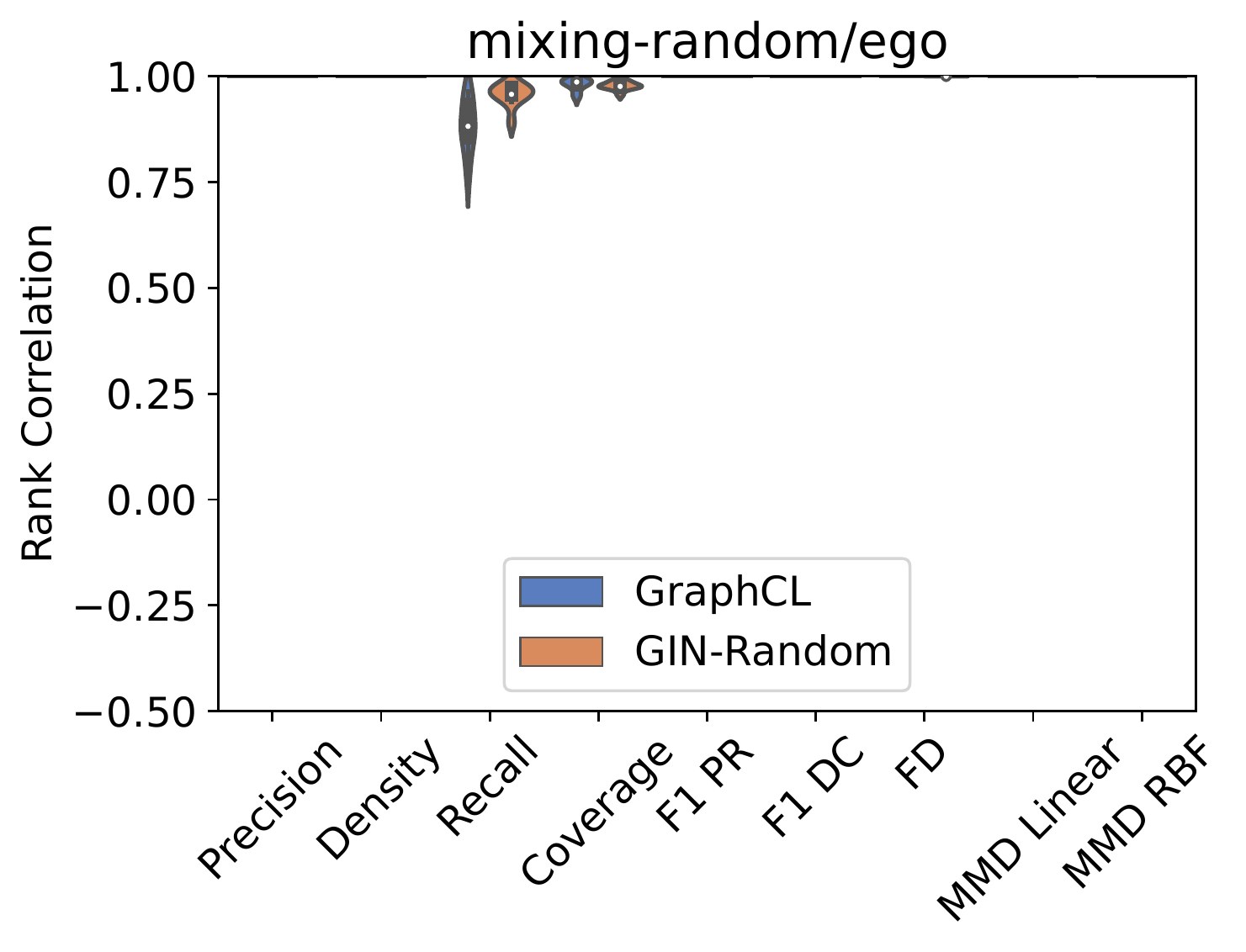}}
    \subfloat[][]{\includegraphics[width = 2.55in]{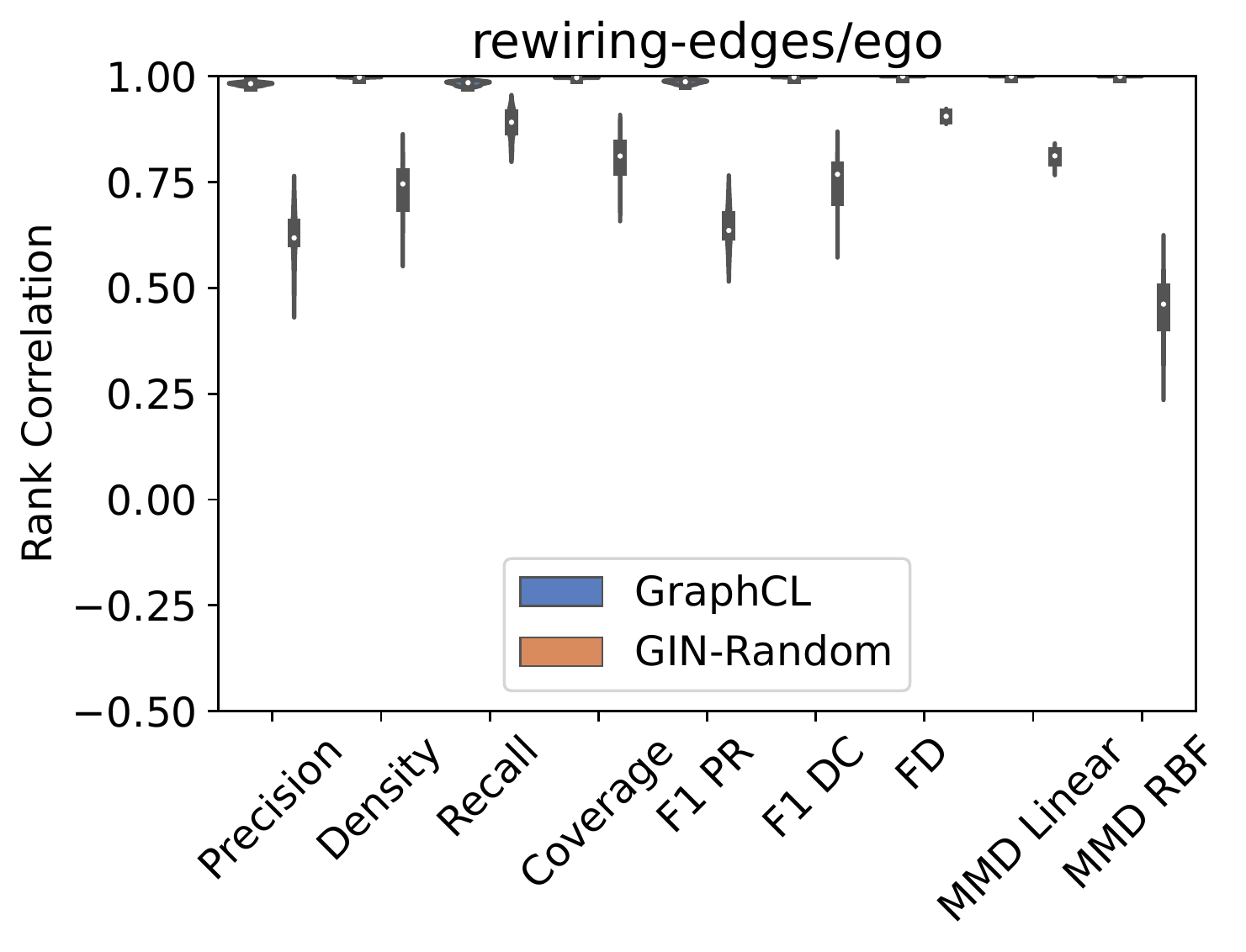}}
    \\
    \subfloat[][]{\includegraphics[width = 2.55in]{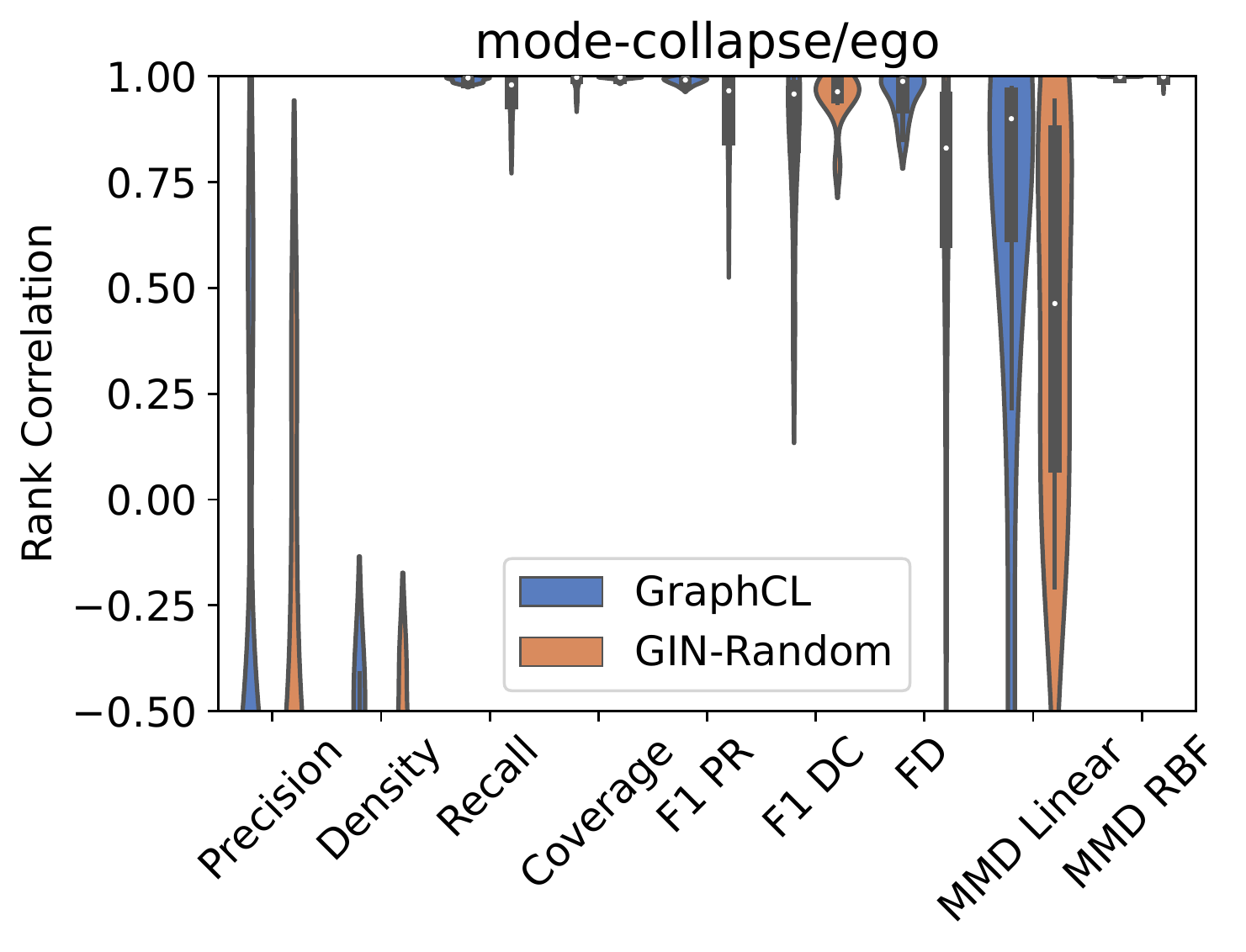}}
    \subfloat[][]{\includegraphics[width = 2.55in]{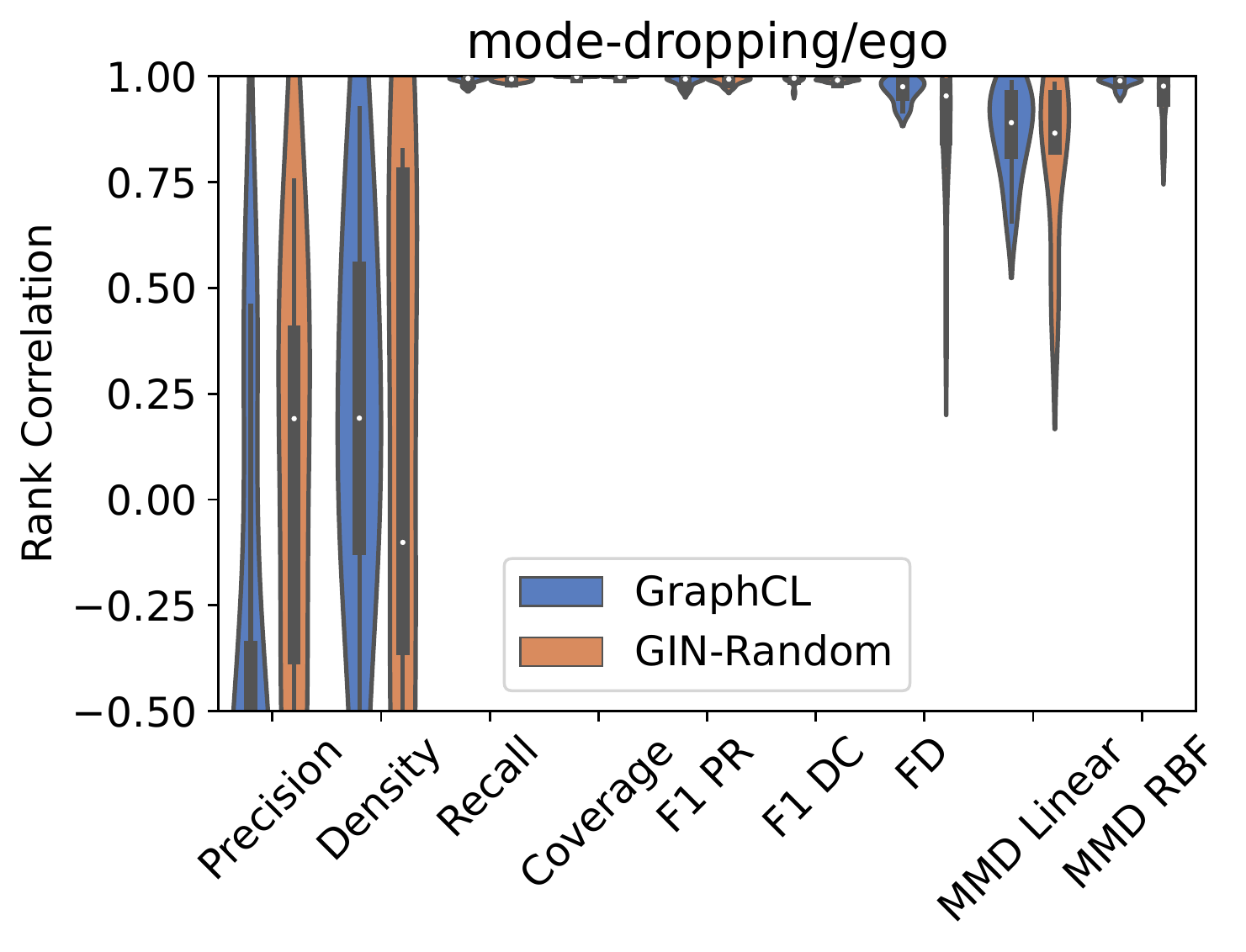}}
    \caption{Violing comparative results among the methods, with with degree features for ego dataset.}
    \label{fig:deg_feats_ego}
\end{figure*}
 
\begin{table}[h!]
\centering
\begin{small}
\scalebox{0.7}{
\begin{tabular}{l|l|l|l|l|l|l|l|l|l|l|l} 
\toprule
Model Name               & Experiment                      &        & Precision & Density & Recall & Coverage & F1PR & F1DC & FD & MMD Lin & MMD RBF  \\ 
\hline
\multirow{8}{*}{GraphCL}  &    \multirow{2}{*}{Mixing Random} & Mean & 1.0 & 1.0 & 0.88 & 0.98 & 1.0 & 1.0 & 1.0 & 1.0 & 1.0 \\
\cline{3-12}
                   &     & Median & 1.0 & 1.0 & 0.88 & 0.99 & 1.0 & 1.0 & 1.0 & 1.0 & 1.0 \\ 
\cline{2-12}
 &    \multirow{2}{*}{Rewiring Edges} & Mean & 0.98 & 1.0 & 0.98 & 1.0 & 0.99 & 1.0 & 1.0 & 1.0 & 1.0 \\
\cline{3-12}
                   &     & Median & 0.98 & 1.0 & 0.98 & 1.0 & 0.99 & 1.0 & 1.0 & 1.0 & 1.0 \\ 
\cline{2-12}
 &    \multirow{2}{*}{Mode Collapse} & Mean & -0.79 & -0.8 & 0.99 & 0.99 & 0.99 & 0.86 & 0.96 & 0.66 & 1.0 \\
\cline{3-12}
                   &     & Median & -0.98 & -0.88 & 1.0 & 1.0 & 0.99 & 0.96 & 0.99 & 0.9 & 1.0 \\ 
\cline{2-12}
 &    \multirow{2}{*}{Mode Dropping} & Mean & -0.5 & 0.14 & 0.99 & 1.0 & 0.99 & 0.99 & 0.97 & 0.87 & 0.99 \\
\cline{3-12}
                   &     & Median & -0.69 & 0.19 & 0.99 & 1.0 & 0.99 & 1.0 & 0.98 & 0.89 & 0.99 \\ 
\cline{1-12}
\multirow{8}{*}{GIN-Random}  &    \multirow{2}{*}{Mixing Random} & Mean & 1.0 & 1.0 & 0.96 & 0.98 & 1.0 & 1.0 & 1.0 & 1.0 & 1.0 \\
\cline{3-12}
                   &     & Median & 1.0 & 1.0 & 0.96 & 0.98 & 1.0 & 1.0 & 1.0 & 1.0 & 1.0 \\ 
\cline{2-12}
 &    \multirow{2}{*}{Rewiring Edges} & Mean & 0.62 & 0.73 & 0.89 & 0.8 & 0.64 & 0.75 & 0.91 & 0.81 & 0.45 \\
\cline{3-12}
                   &     & Median & 0.62 & 0.75 & 0.89 & 0.81 & 0.64 & 0.77 & 0.91 & 0.81 & 0.46 \\ 
\cline{2-12}
 &    \multirow{2}{*}{Mode Collapse} & Mean & -0.77 & -0.76 & 0.95 & 1.0 & 0.91 & 0.95 & 0.65 & 0.45 & 0.99 \\
\cline{3-12}
                   &     & Median & -0.97 & -0.78 & 0.98 & 1.0 & 0.97 & 0.96 & 0.83 & 0.46 & 1.0 \\ 
\cline{2-12}
 &    \multirow{2}{*}{Mode Dropping} & Mean & 0.02 & 0.07 & 0.99 & 1.0 & 0.99 & 0.99 & 0.87 & 0.82 & 0.95 \\
\cline{3-12}
                   &     & Median & 0.19 & -0.1 & 0.99 & 1.0 & 0.99 & 0.99 & 0.95 & 0.87 & 0.98 \\ \bottomrule
\end{tabular}
}
\end{small}
\caption{Mean and median values for measurements in experiments with with degree features by models, for dataset ego} 
\label{table:deg_feats_ego}
\end{table}
 
\begin{figure*}[h!]
    \captionsetup[subfloat]{farskip=-2pt,captionskip=-8pt}
    \centering
    \subfloat[][]{\includegraphics[width = 2.55in]{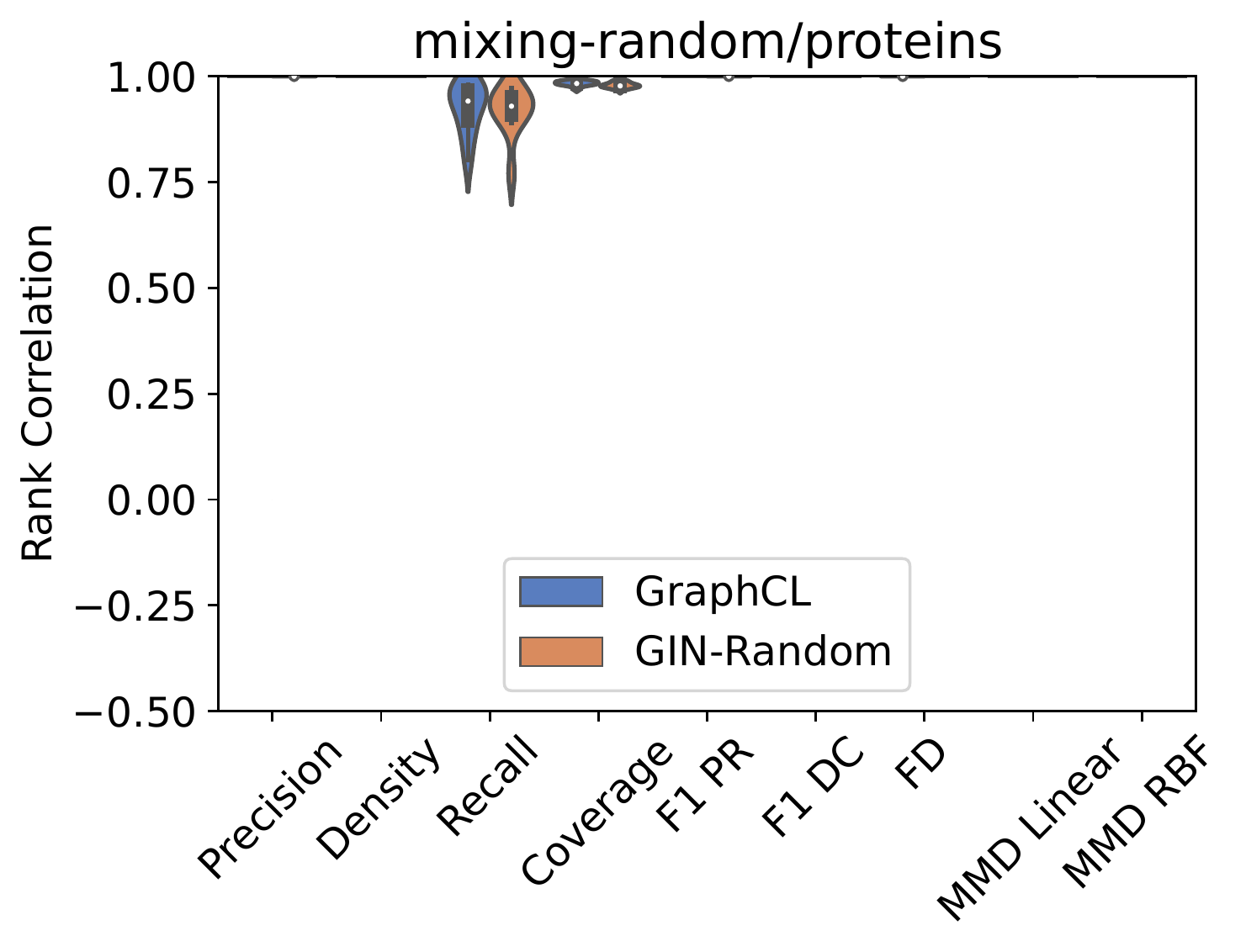}}
    \subfloat[][]{\includegraphics[width = 2.55in]{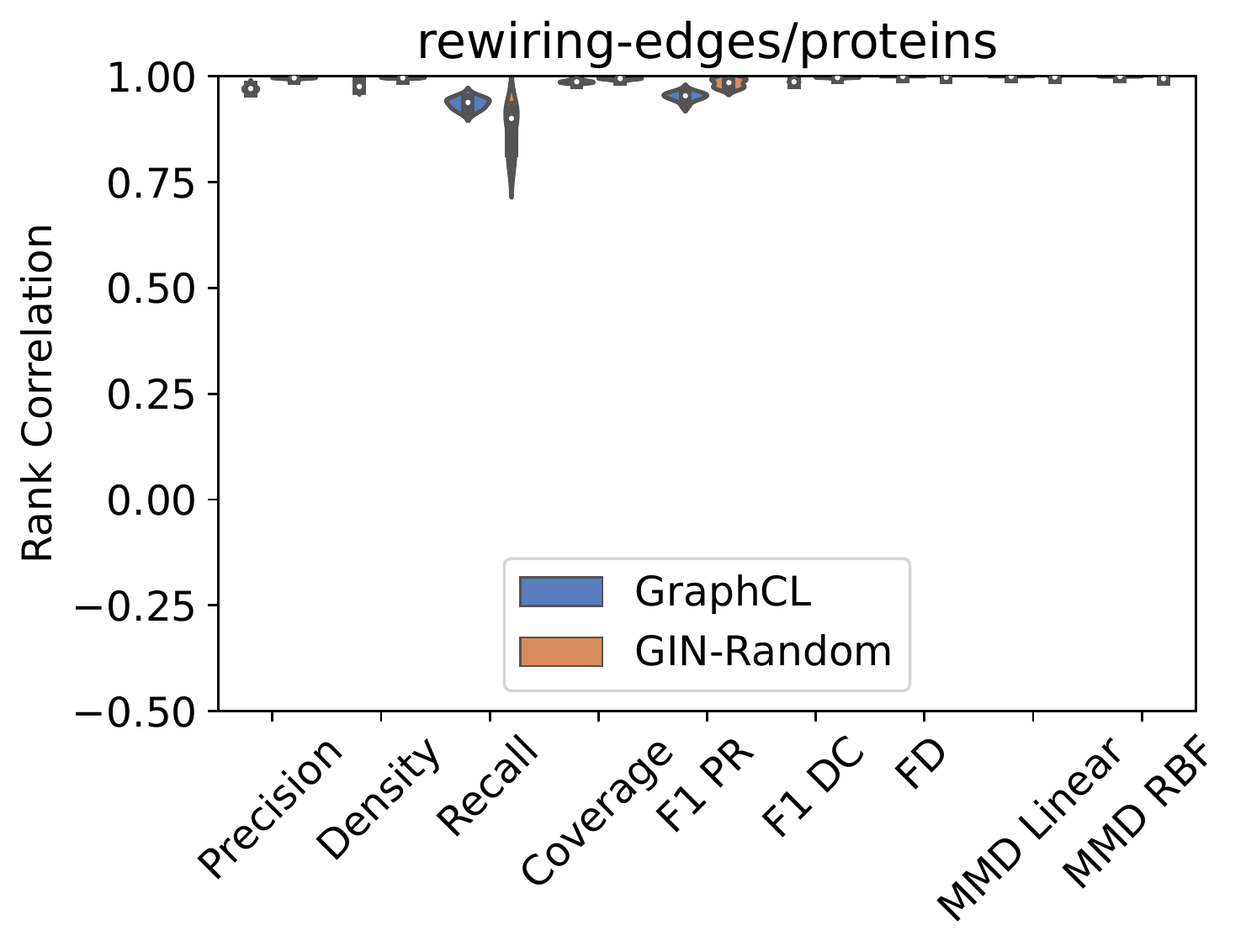}}
    \\
    \subfloat[][]{\includegraphics[width = 2.55in]{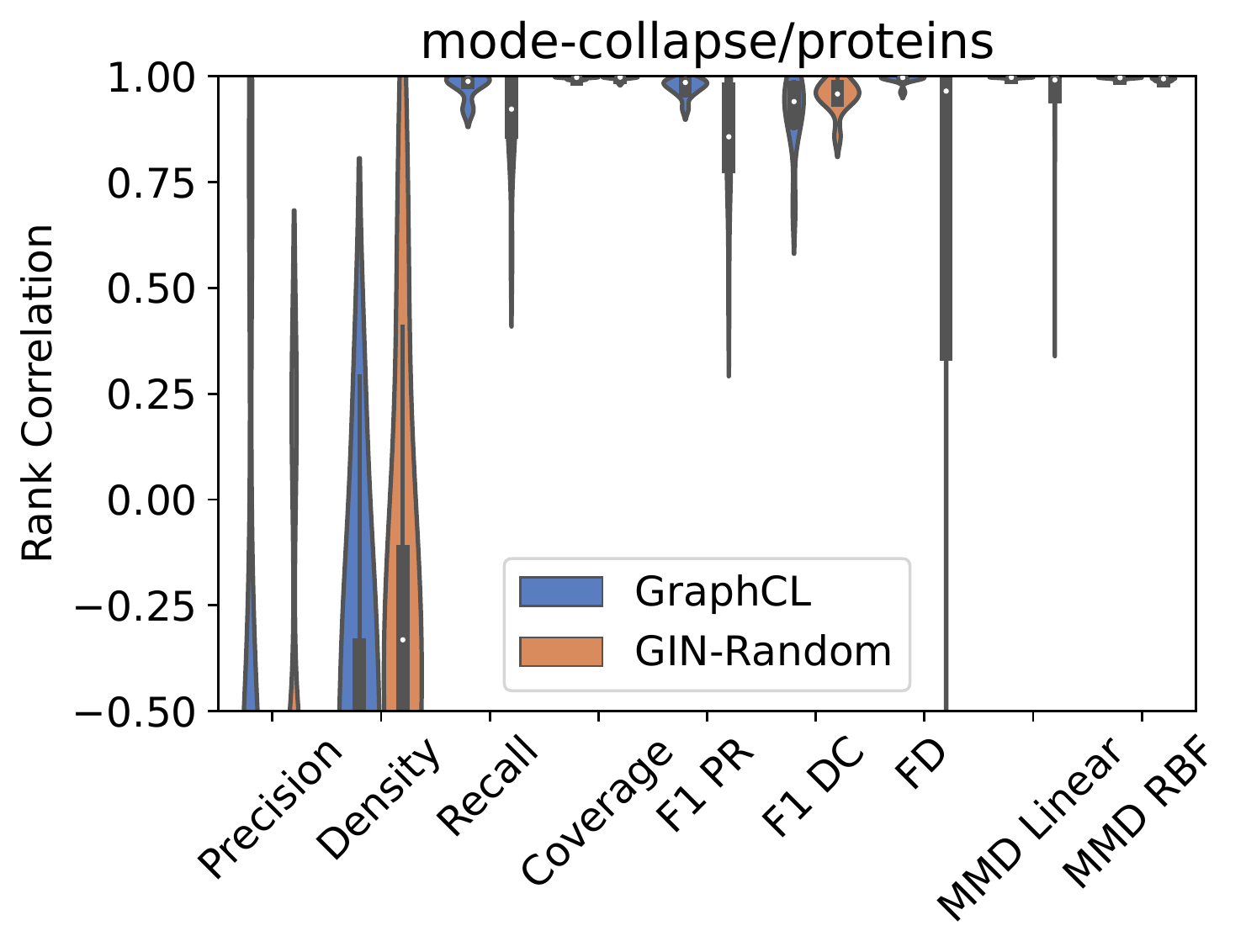}}
    \subfloat[][]{\includegraphics[width = 2.55in]{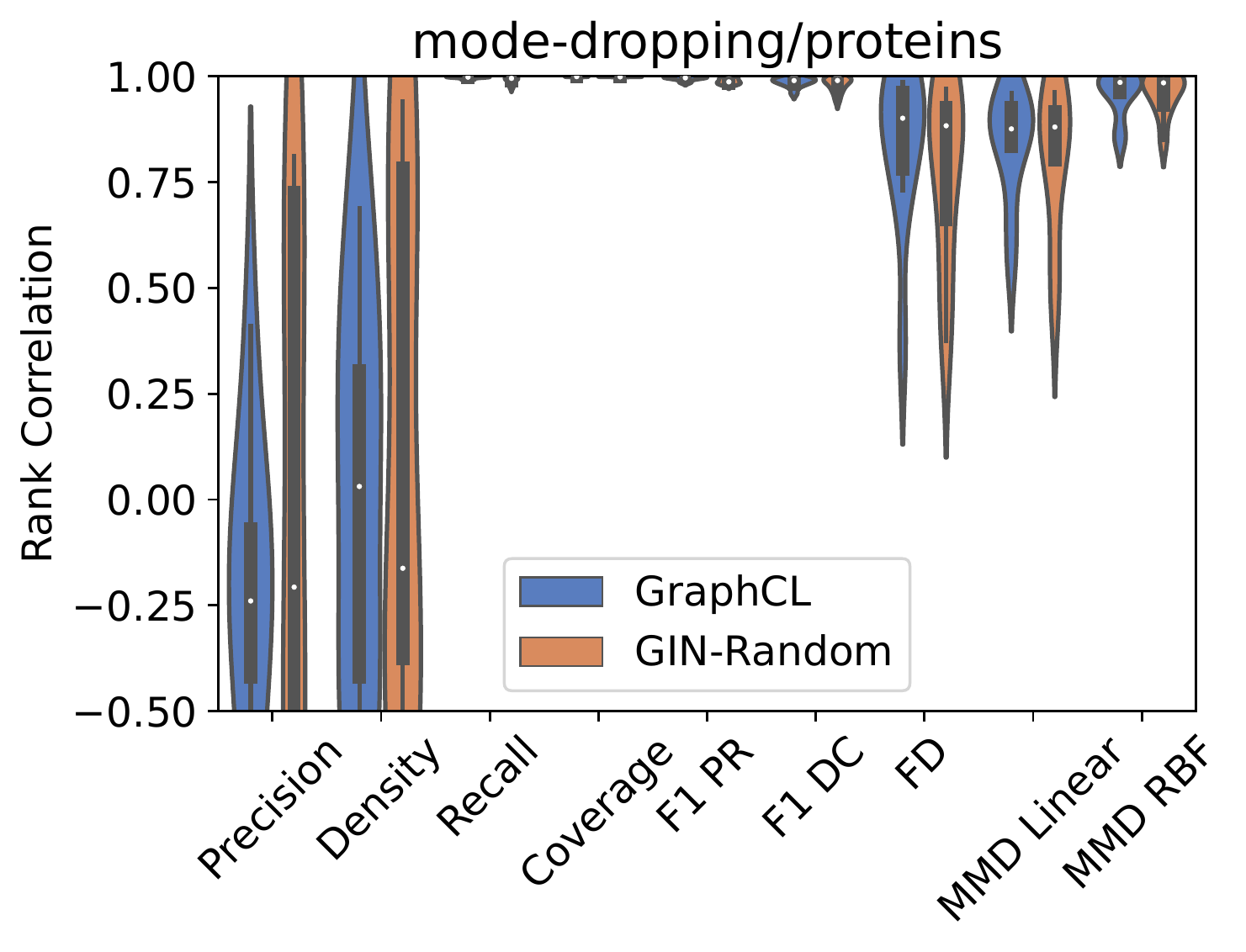}}
    \caption{Violing comparative results among the methods, with with degree features for proteins dataset.}
    \label{fig:deg_feats_proteins}
\end{figure*}
 
\begin{table}[h!]
\centering
\begin{small}
\scalebox{0.7}{
\begin{tabular}{l|l|l|l|l|l|l|l|l|l|l|l} 
\toprule
Model Name               & Experiment                      &        & Precision & Density & Recall & Coverage & F1PR & F1DC & FD & MMD Lin & MMD RBF  \\ 
\hline
\multirow{8}{*}{GraphCL}  &    \multirow{2}{*}{Mixing Random} & Mean & 1.0 & 1.0 & 0.92 & 0.98 & 1.0 & 1.0 & 1.0 & 1.0 & 1.0 \\
\cline{3-12}
                   &     & Median & 1.0 & 1.0 & 0.94 & 0.98 & 1.0 & 1.0 & 1.0 & 1.0 & 1.0 \\ 
\cline{2-12}
 &    \multirow{2}{*}{Rewiring Edges} & Mean & 0.97 & 0.98 & 0.94 & 0.99 & 0.95 & 0.99 & 1.0 & 1.0 & 1.0 \\
\cline{3-12}
                   &     & Median & 0.97 & 0.98 & 0.94 & 0.99 & 0.95 & 0.99 & 1.0 & 1.0 & 1.0 \\ 
\cline{2-12}
 &    \multirow{2}{*}{Mode Collapse} & Mean & -0.76 & -0.5 & 0.98 & 1.0 & 0.97 & 0.92 & 0.99 & 1.0 & 1.0 \\
\cline{3-12}
                   &     & Median & -0.98 & -0.54 & 0.99 & 1.0 & 0.99 & 0.94 & 1.0 & 1.0 & 1.0 \\ 
\cline{2-12}
 &    \multirow{2}{*}{Mode Dropping} & Mean & -0.29 & -0.01 & 1.0 & 1.0 & 0.99 & 0.99 & 0.84 & 0.84 & 0.96 \\
\cline{3-12}
                   &     & Median & -0.24 & 0.03 & 1.0 & 1.0 & 1.0 & 0.99 & 0.9 & 0.88 & 0.99 \\ 
\cline{1-12}
\multirow{8}{*}{GIN-Random}  &    \multirow{2}{*}{Mixing Random} & Mean & 1.0 & 1.0 & 0.92 & 0.98 & 1.0 & 1.0 & 1.0 & 1.0 & 1.0 \\
\cline{3-12}
                   &     & Median & 1.0 & 1.0 & 0.93 & 0.98 & 1.0 & 1.0 & 1.0 & 1.0 & 1.0 \\ 
\cline{2-12}
 &    \multirow{2}{*}{Rewiring Edges} & Mean & 0.99 & 1.0 & 0.87 & 0.99 & 0.98 & 1.0 & 1.0 & 1.0 & 0.99 \\
\cline{3-12}
                   &     & Median & 1.0 & 1.0 & 0.9 & 0.99 & 0.98 & 1.0 & 1.0 & 1.0 & 0.99 \\ 
\cline{2-12}
 &    \multirow{2}{*}{Mode Collapse} & Mean & -0.82 & -0.25 & 0.89 & 1.0 & 0.84 & 0.95 & 0.59 & 0.92 & 0.99 \\
\cline{3-12}
                   &     & Median & -0.96 & -0.33 & 0.92 & 1.0 & 0.86 & 0.96 & 0.97 & 0.99 & 0.99 \\ 
\cline{2-12}
 &    \multirow{2}{*}{Mode Dropping} & Mean & 0.02 & 0.13 & 0.99 & 1.0 & 0.99 & 0.98 & 0.77 & 0.81 & 0.96 \\
\cline{3-12}
                   &     & Median & -0.21 & -0.16 & 1.0 & 1.0 & 0.99 & 0.99 & 0.88 & 0.88 & 0.98 \\ \bottomrule
\end{tabular}
}
\end{small}
\caption{Mean and median values for measurements in experiments with with degree features by models, for dataset proteins} 
\label{table:deg_feats_proteins}
\end{table}
 
\begin{figure*}[h!]
    \captionsetup[subfloat]{farskip=-2pt,captionskip=-8pt}
    \centering
    \subfloat[][]{\includegraphics[width = 2.55in]{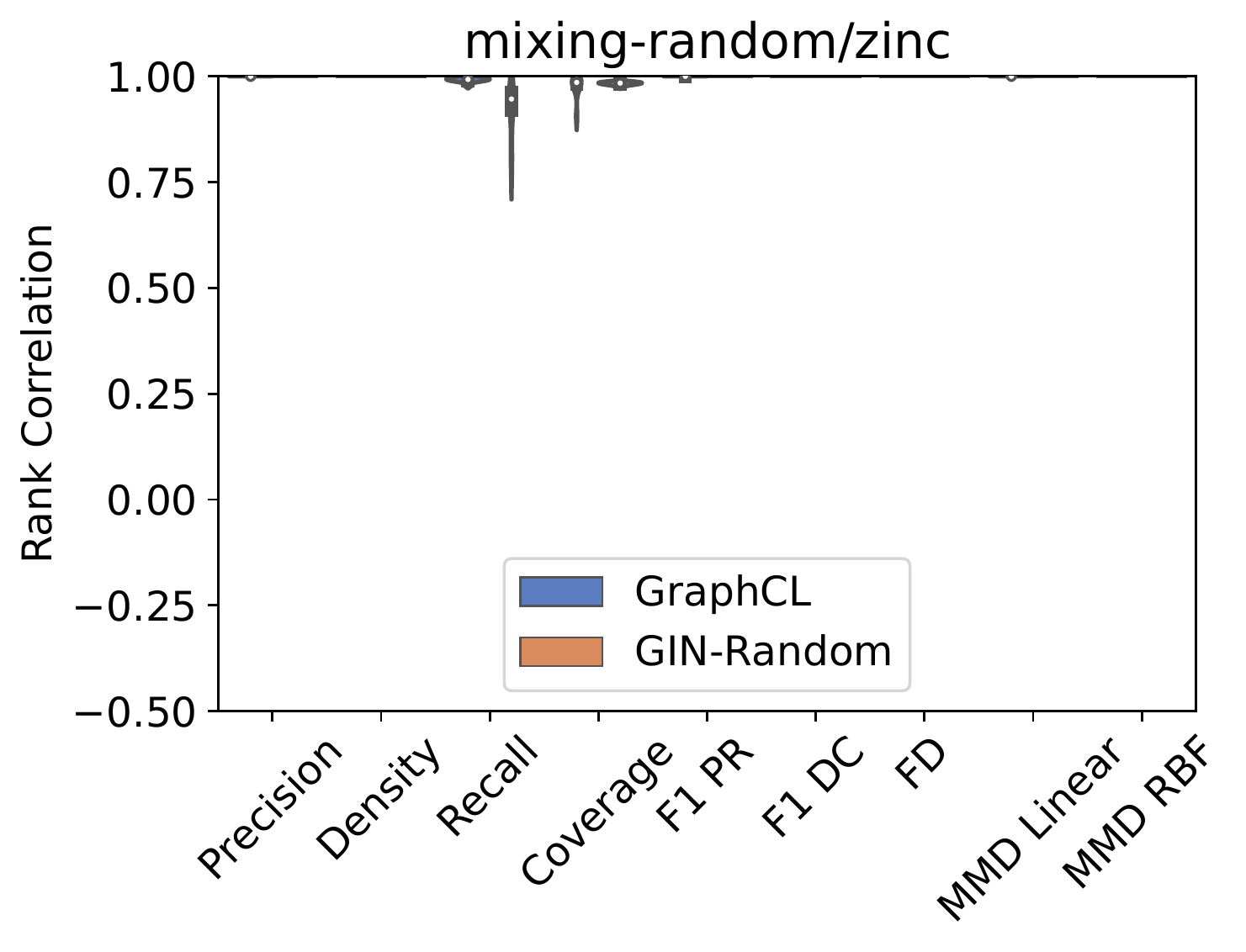}}
    \subfloat[][]{\includegraphics[width = 2.55in]{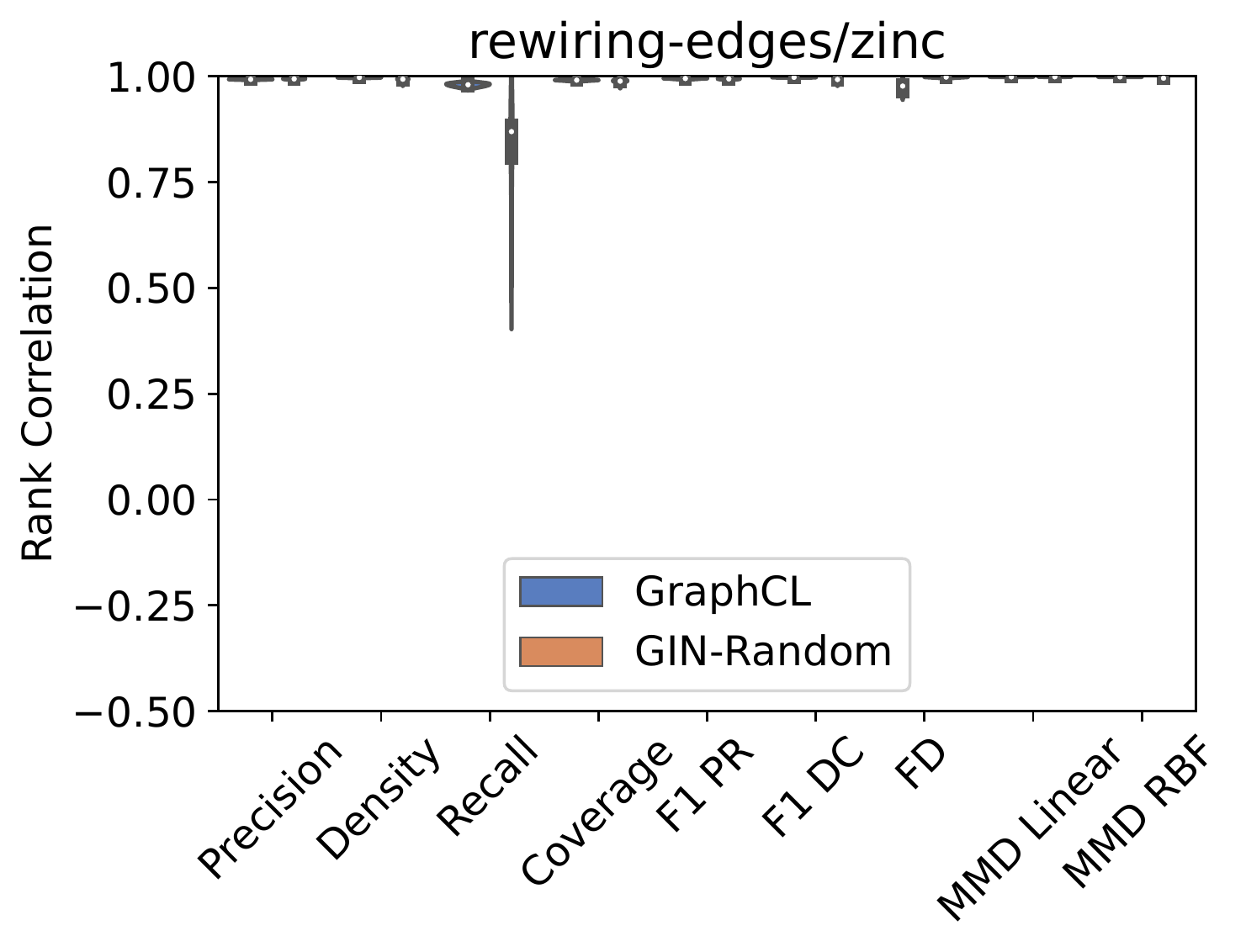}}
    \\
    \subfloat[][]{\includegraphics[width = 2.55in]{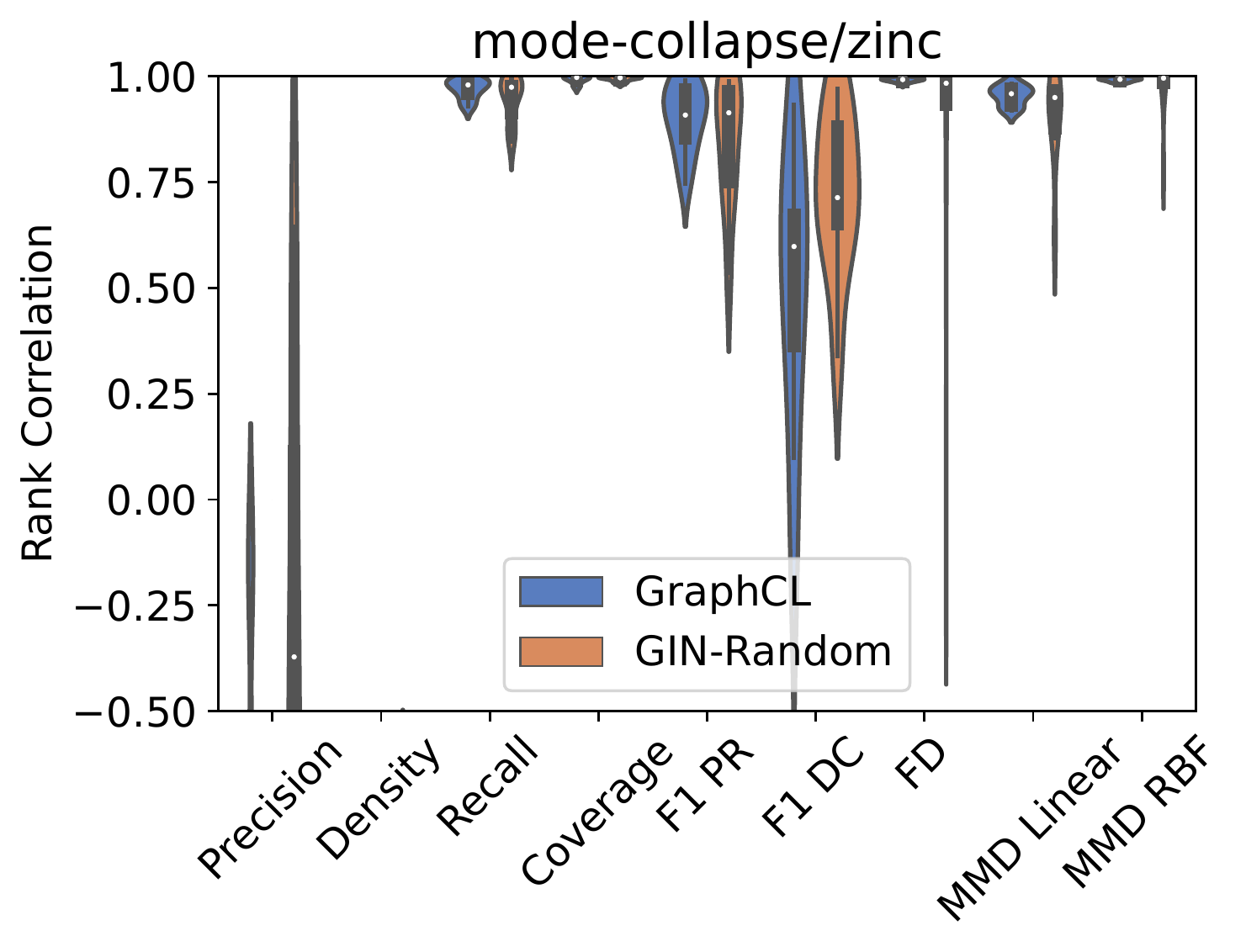}}
    \subfloat[][]{\includegraphics[width = 2.55in]{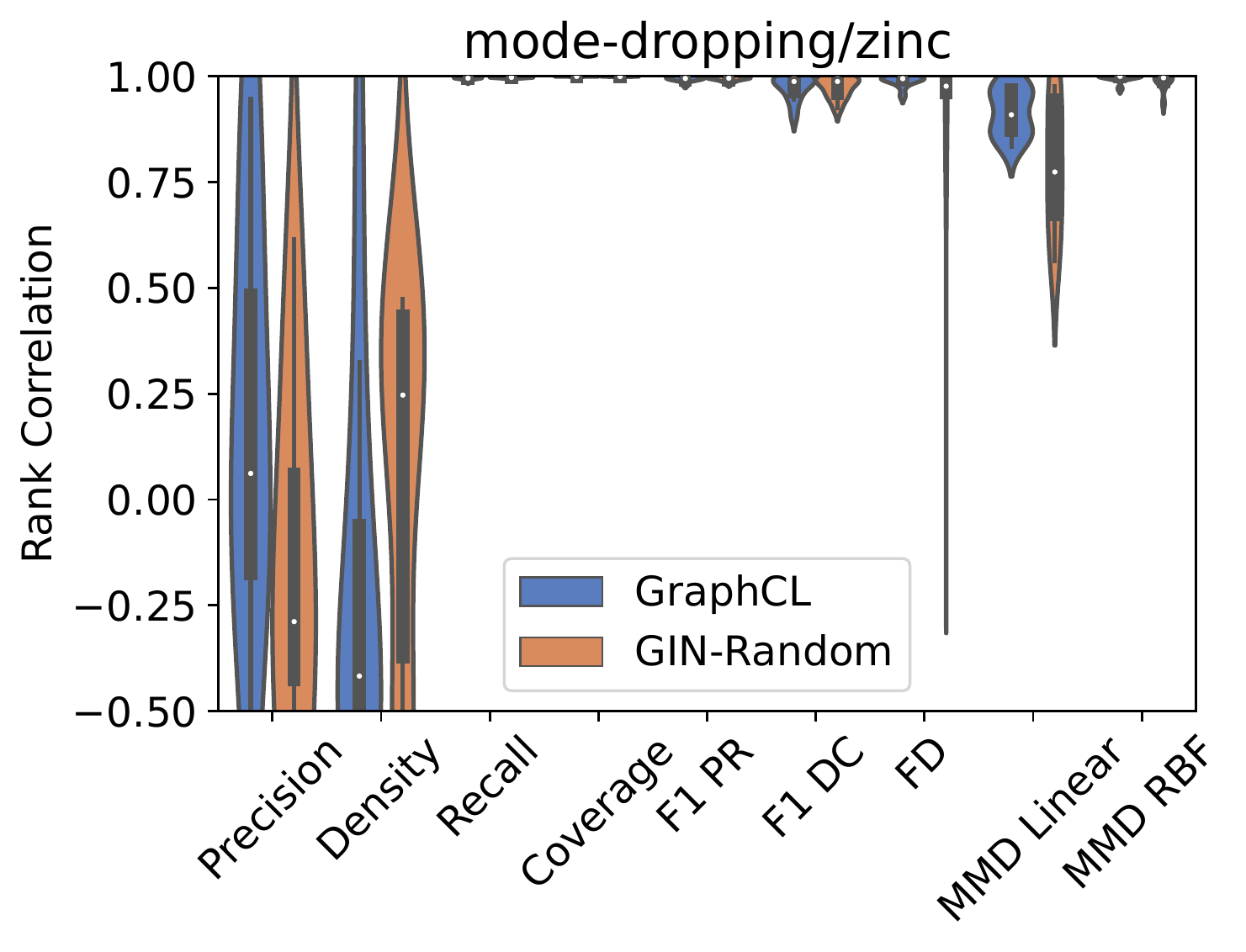}}
    \caption{Violing comparative results among the methods, with with degree features for zinc dataset.}
    \label{fig:deg_feats_zinc}
\end{figure*}
 
\begin{table}[h!]
\centering
\begin{small}
\scalebox{0.7}{
\begin{tabular}{l|l|l|l|l|l|l|l|l|l|l|l} 
\toprule
Model Name               & Experiment                      &        & Precision & Density & Recall & Coverage & F1PR & F1DC & FD & MMD Lin & MMD RBF  \\ 
\hline
\multirow{8}{*}{GraphCL}  &    \multirow{2}{*}{Mixing Random} & Mean & 1.0 & 1.0 & 0.99 & 0.98 & 1.0 & 1.0 & 1.0 & 1.0 & 1.0 \\
\cline{3-12}
                   &     & Median & 1.0 & 1.0 & 0.99 & 0.99 & 1.0 & 1.0 & 1.0 & 1.0 & 1.0 \\ 
\cline{2-12}
 &    \multirow{2}{*}{Rewiring Edges} & Mean & 0.99 & 1.0 & 0.98 & 0.99 & 0.99 & 1.0 & 0.97 & 1.0 & 1.0 \\
\cline{3-12}
                   &     & Median & 0.99 & 1.0 & 0.98 & 0.99 & 0.99 & 1.0 & 0.98 & 1.0 & 1.0 \\ 
\cline{2-12}
 &    \multirow{2}{*}{Mode Collapse} & Mean & -0.86 & -0.97 & 0.97 & 0.99 & 0.9 & 0.5 & 0.99 & 0.95 & 0.99 \\
\cline{3-12}
                   &     & Median & -0.94 & -0.97 & 0.98 & 1.0 & 0.91 & 0.6 & 0.99 & 0.96 & 0.99 \\ 
\cline{2-12}
 &    \multirow{2}{*}{Mode Dropping} & Mean & 0.15 & -0.22 & 1.0 & 1.0 & 0.99 & 0.97 & 0.99 & 0.91 & 1.0 \\
\cline{3-12}
                   &     & Median & 0.06 & -0.42 & 1.0 & 1.0 & 0.99 & 0.99 & 0.99 & 0.91 & 1.0 \\ 
\cline{1-12}
\multirow{8}{*}{GIN-Random}  &    \multirow{2}{*}{Mixing Random} & Mean & 1.0 & 1.0 & 0.92 & 0.98 & 1.0 & 1.0 & 1.0 & 1.0 & 1.0 \\
\cline{3-12}
                   &     & Median & 1.0 & 1.0 & 0.95 & 0.98 & 1.0 & 1.0 & 1.0 & 1.0 & 1.0 \\ 
\cline{2-12}
 &    \multirow{2}{*}{Rewiring Edges} & Mean & 0.99 & 0.99 & 0.82 & 0.99 & 0.99 & 0.99 & 1.0 & 1.0 & 1.0 \\
\cline{3-12}
                   &     & Median & 0.99 & 0.99 & 0.87 & 0.99 & 0.99 & 0.99 & 1.0 & 1.0 & 1.0 \\ 
\cline{2-12}
 &    \multirow{2}{*}{Mode Collapse} & Mean & -0.32 & -0.86 & 0.95 & 0.99 & 0.85 & 0.72 & 0.86 & 0.9 & 0.97 \\
\cline{3-12}
                   &     & Median & -0.37 & -0.9 & 0.97 & 1.0 & 0.91 & 0.71 & 0.98 & 0.95 & 1.0 \\ 
\cline{2-12}
 &    \multirow{2}{*}{Mode Dropping} & Mean & -0.21 & 0.06 & 1.0 & 1.0 & 0.99 & 0.98 & 0.88 & 0.79 & 0.99 \\
\cline{3-12}
                   &     & Median & -0.29 & 0.25 & 1.0 & 1.0 & 1.0 & 0.99 & 0.98 & 0.77 & 1.0 \\ \bottomrule
\end{tabular}
}
\end{small}
\caption{Mean and median values for measurements in experiments with with degree features by models, for dataset zinc} 
\label{table:deg_feats_zinc}
\end{table}

%% file: clustering_experiments_figs_and_tables.tex
\begin{figure*}[h!]
    \captionsetup[subfloat]{farskip=-2pt,captionskip=-8pt}
    \centering
    \subfloat[][]{\includegraphics[width = 2.55in]{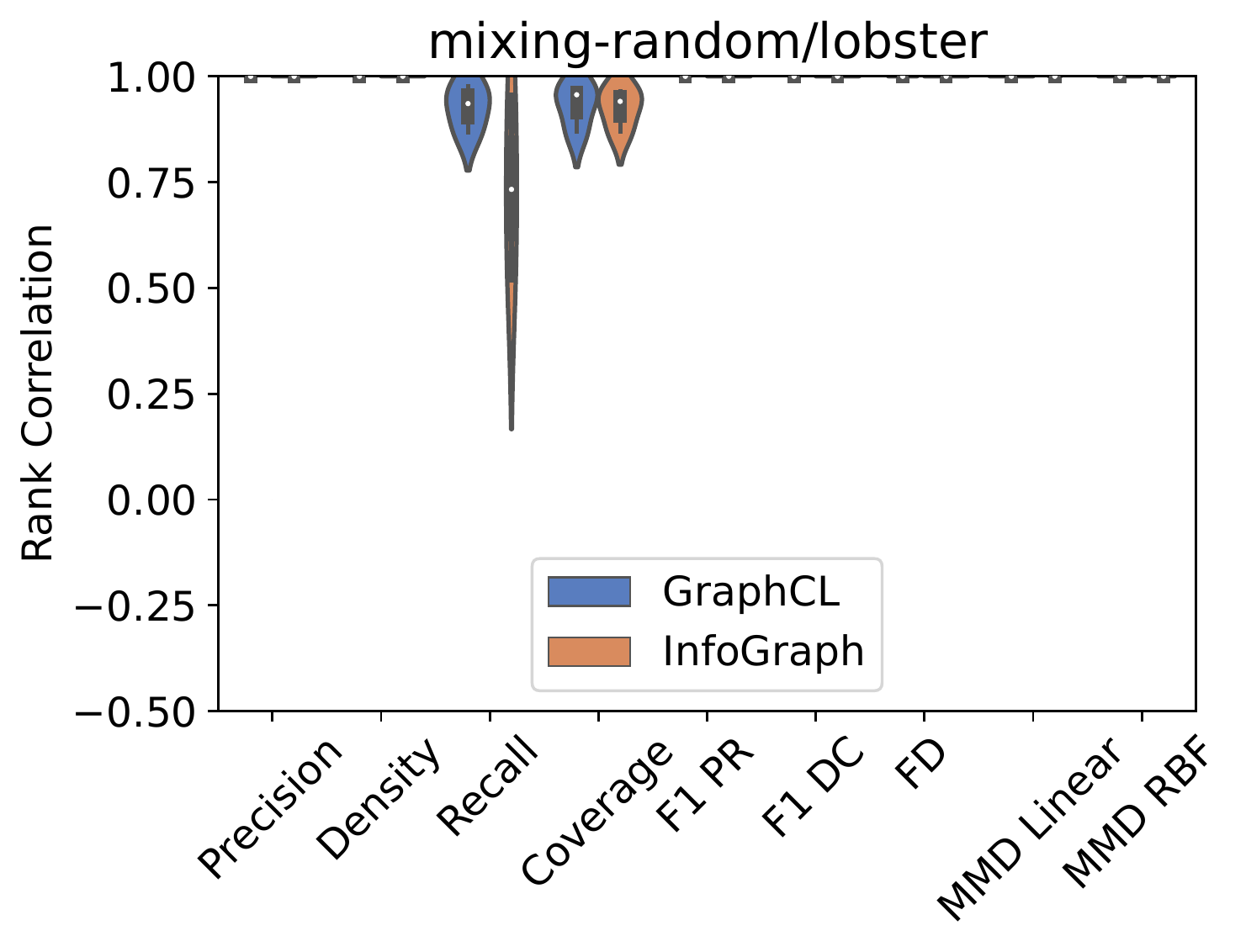}}
    \subfloat[][]{\includegraphics[width = 2.55in]{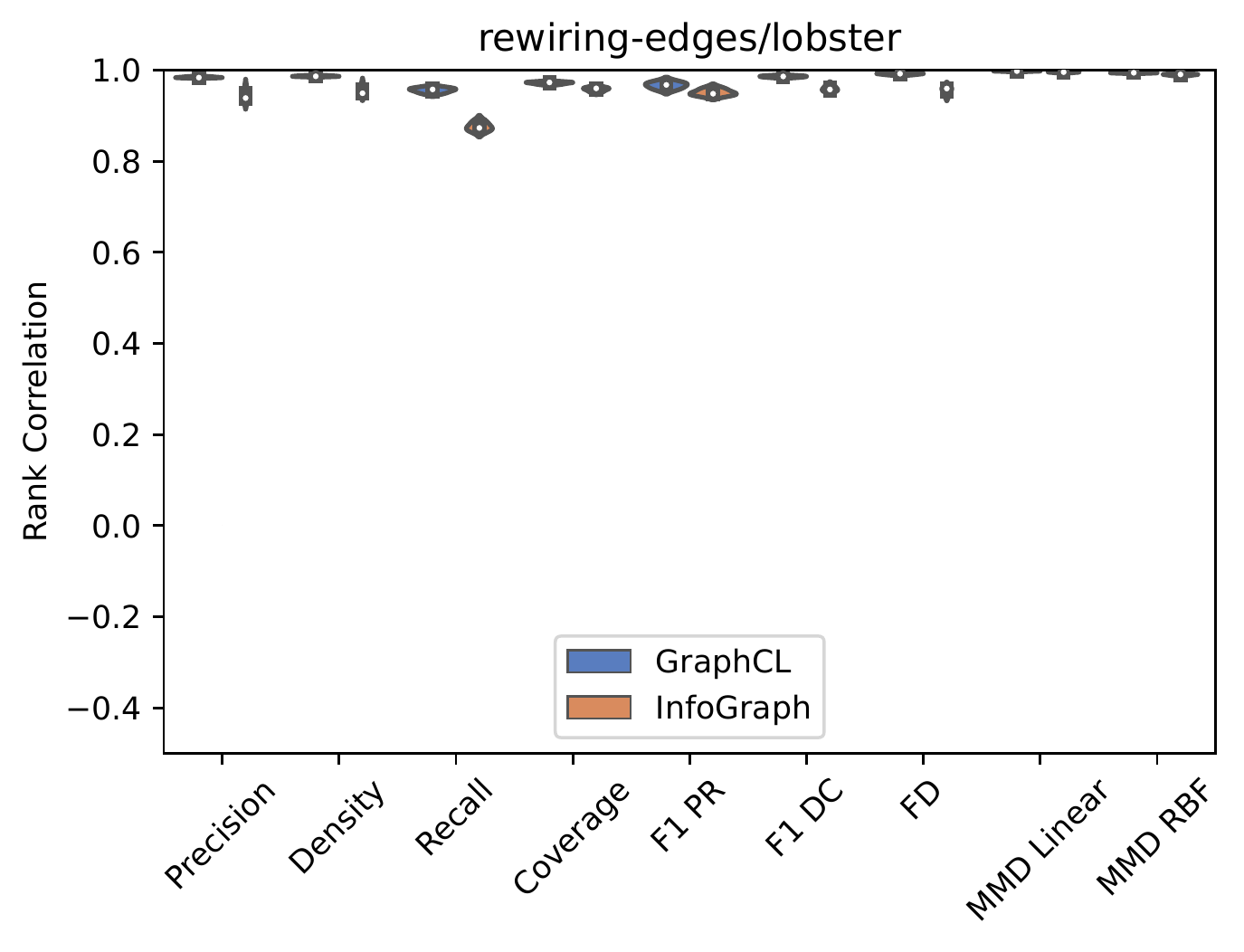}}
    \\
    \subfloat[][]{\includegraphics[width = 2.55in]{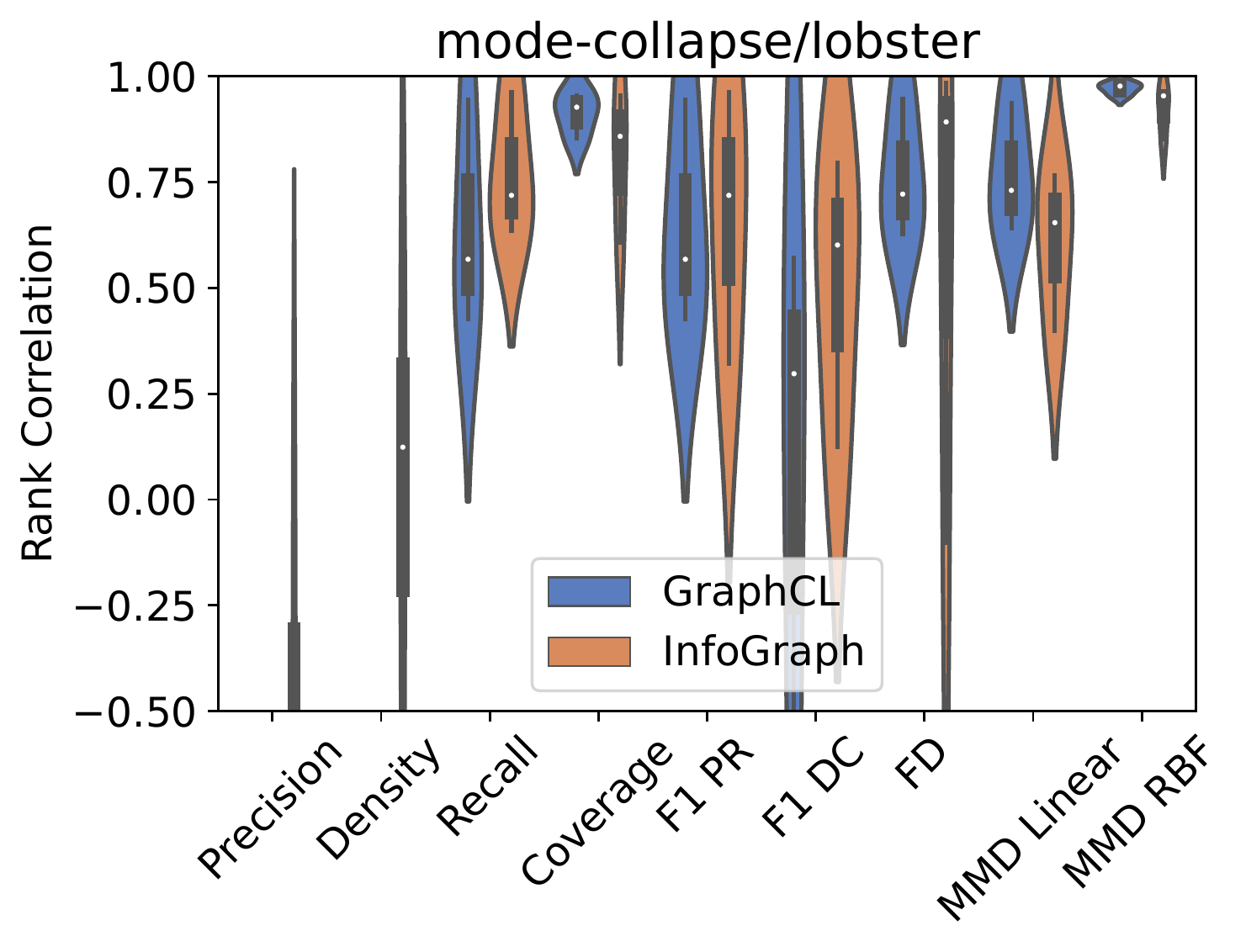}}
    \subfloat[][]{\includegraphics[width = 2.55in]{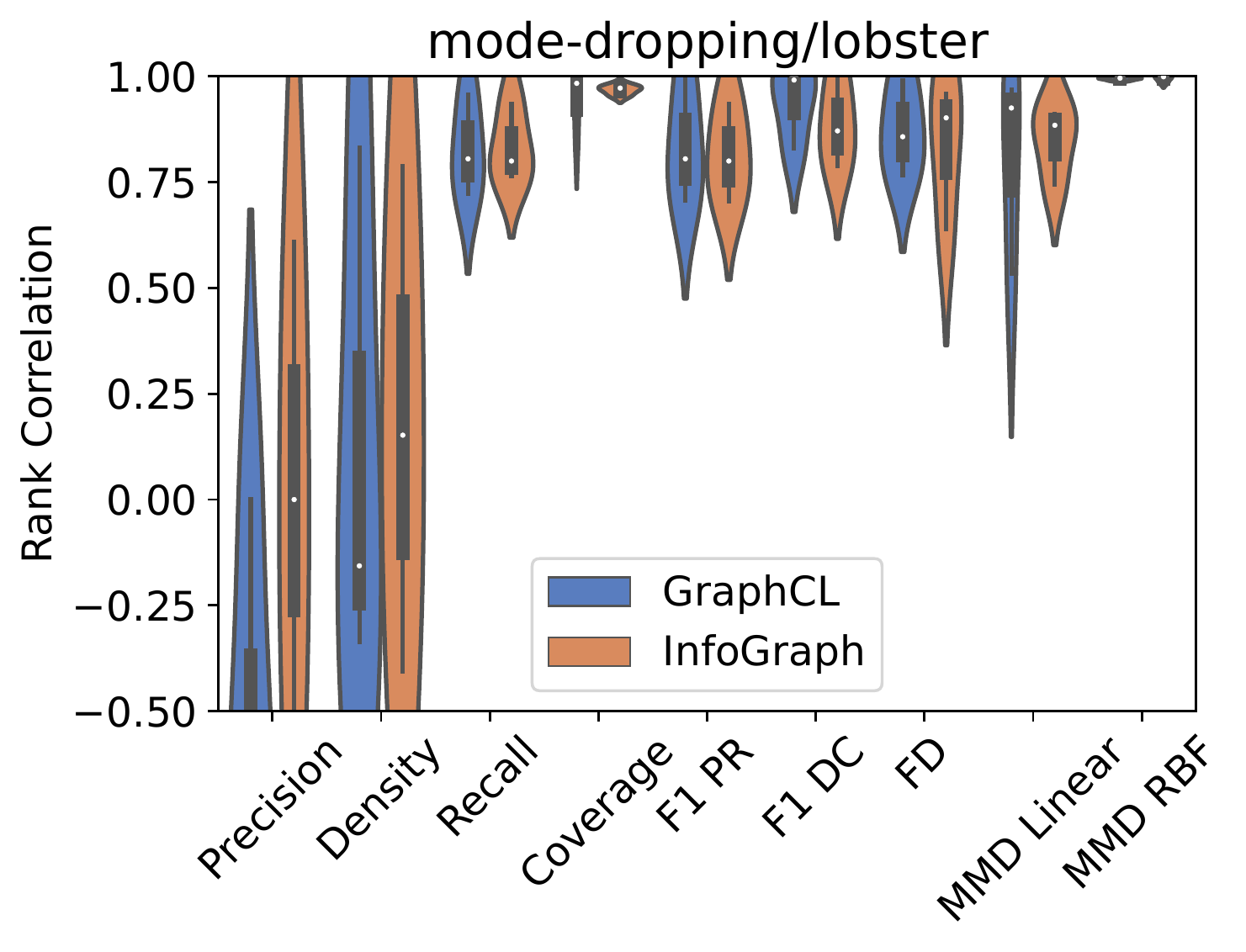}}
    \caption{Violing comparative results among the methods, with with clustering structural features for lobster dataset.}
    \label{fig:clustering_feats_lobster}
\end{figure*}
 
\begin{table}[h!]
\centering
\begin{small}
\scalebox{0.7}{
\begin{tabular}{l|l|l|l|l|l|l|l|l|l|l|l} 
\toprule
Model Name               & Experiment                      &        & Precision & Density & Recall & Coverage & F1PR & F1DC & FD & MMD Lin & MMD RBF  \\ 
\hline
\multirow{8}{*}{GraphCL}  &    \multirow{2}{*}{Mixing Random} & Mean & 1.0 & 1.0 & 0.93 & 0.93 & 1.0 & 1.0 & 1.0 & 1.0 & 1.0 \\
\cline{3-12}
                   &     & Median & 1.0 & 1.0 & 0.94 & 0.96 & 1.0 & 1.0 & 1.0 & 1.0 & 1.0 \\ 
\cline{2-12}
 &    \multirow{2}{*}{Rewiring Edges} & Mean & 0.98 & 0.99 & 0.96 & 0.97 & 0.97 & 0.99 & 0.99 & 1.0 & 0.99 \\
\cline{3-12}
                   &     & Median & 0.98 & 0.99 & 0.96 & 0.97 & 0.97 & 0.99 & 0.99 & 1.0 & 0.99 \\ 
\cline{2-12}
 &    \multirow{2}{*}{Mode Collapse} & Mean & -0.87 & -0.8 & 0.65 & 0.91 & 0.65 & 0.02 & 0.77 & 0.77 & 0.97 \\
\cline{3-12}
                   &     & Median & -0.86 & -0.82 & 0.57 & 0.93 & 0.57 & 0.3 & 0.72 & 0.73 & 0.98 \\ 
\cline{2-12}
 &    \multirow{2}{*}{Mode Dropping} & Mean & -0.49 & 0.11 & 0.83 & 0.94 & 0.83 & 0.94 & 0.87 & 0.81 & 0.99 \\
\cline{3-12}
                   &     & Median & -0.73 & -0.16 & 0.8 & 0.98 & 0.8 & 0.99 & 0.86 & 0.93 & 1.0 \\ 
\cline{1-12}
\multirow{8}{*}{InfoGraph}  &    \multirow{2}{*}{Mixing Random} & Mean & 1.0 & 1.0 & 0.74 & 0.93 & 1.0 & 1.0 & 1.0 & 1.0 & 1.0 \\
\cline{3-12}
                   &     & Median & 1.0 & 1.0 & 0.73 & 0.94 & 1.0 & 1.0 & 1.0 & 1.0 & 1.0 \\ 
\cline{2-12}
 &    \multirow{2}{*}{Rewiring Edges} & Mean & 0.94 & 0.96 & 0.88 & 0.96 & 0.95 & 0.96 & 0.96 & 1.0 & 0.99 \\
\cline{3-12}
                   &     & Median & 0.94 & 0.95 & 0.87 & 0.96 & 0.95 & 0.96 & 0.96 & 1.0 & 0.99 \\ 
\cline{2-12}
 &    \multirow{2}{*}{Mode Collapse} & Mean & -0.52 & 0.03 & 0.77 & 0.81 & 0.67 & 0.51 & 0.59 & 0.61 & 0.92 \\
\cline{3-12}
                   &     & Median & -0.61 & 0.12 & 0.72 & 0.86 & 0.72 & 0.6 & 0.89 & 0.65 & 0.95 \\ 
\cline{2-12}
 &    \multirow{2}{*}{Mode Dropping} & Mean & 0.03 & 0.18 & 0.83 & 0.97 & 0.81 & 0.89 & 0.83 & 0.85 & 0.99 \\
\cline{3-12}
                   &     & Median & 0.0 & 0.15 & 0.8 & 0.97 & 0.8 & 0.87 & 0.9 & 0.88 & 1.0 \\ \bottomrule
\end{tabular}
}
\end{small}
\caption{Mean and median values for measurements in experiments with with clustering structural features by models, for dataset lobster} 
\label{table:clustering_feats_lobster}
\end{table}
 
\begin{figure*}[h!]
    \captionsetup[subfloat]{farskip=-2pt,captionskip=-8pt}
    \centering
    \subfloat[][]{\includegraphics[width = 2.55in]{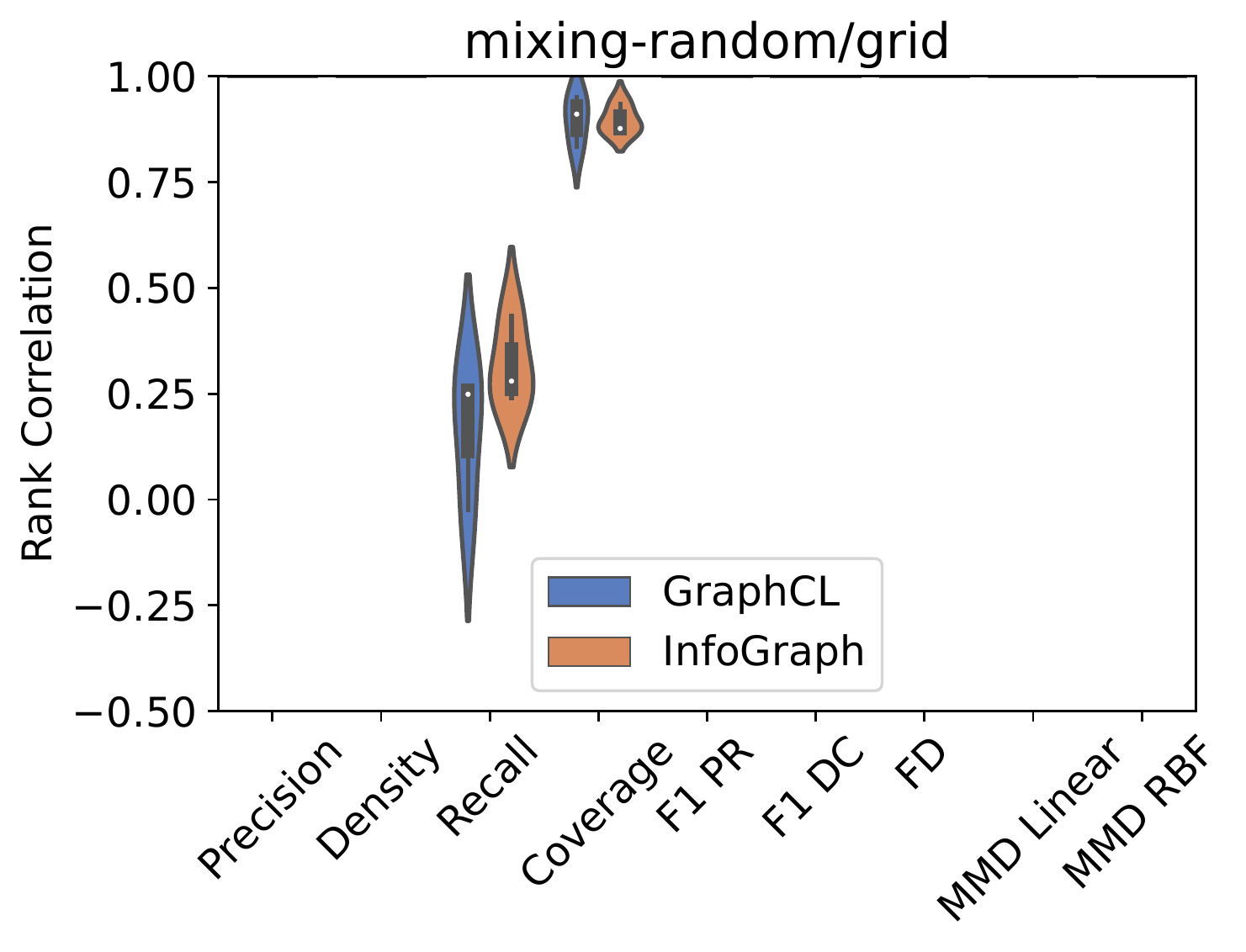}}
    \subfloat[][]{\includegraphics[width = 2.55in]{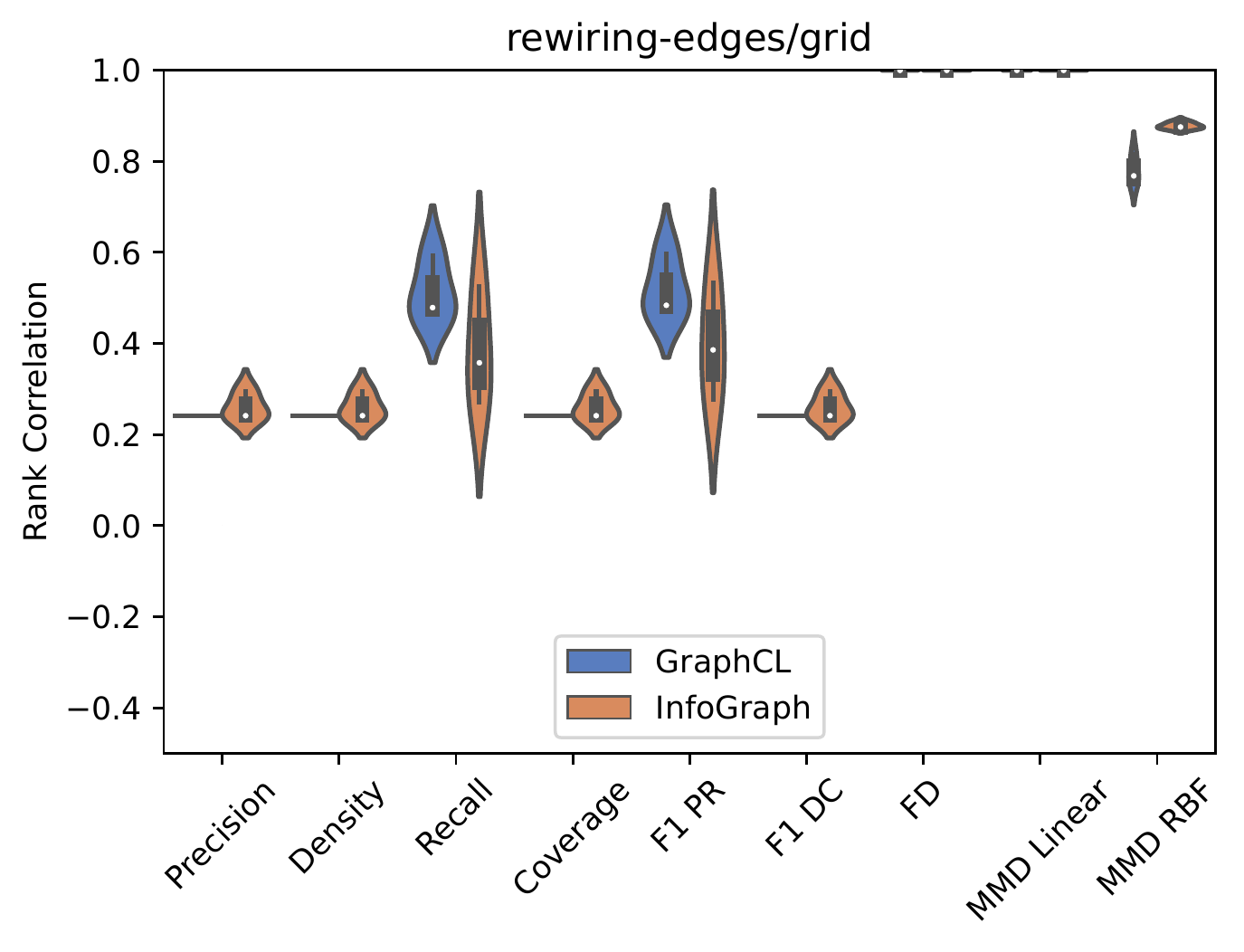}}
    \\
    \subfloat[][]{\includegraphics[width = 2.55in]{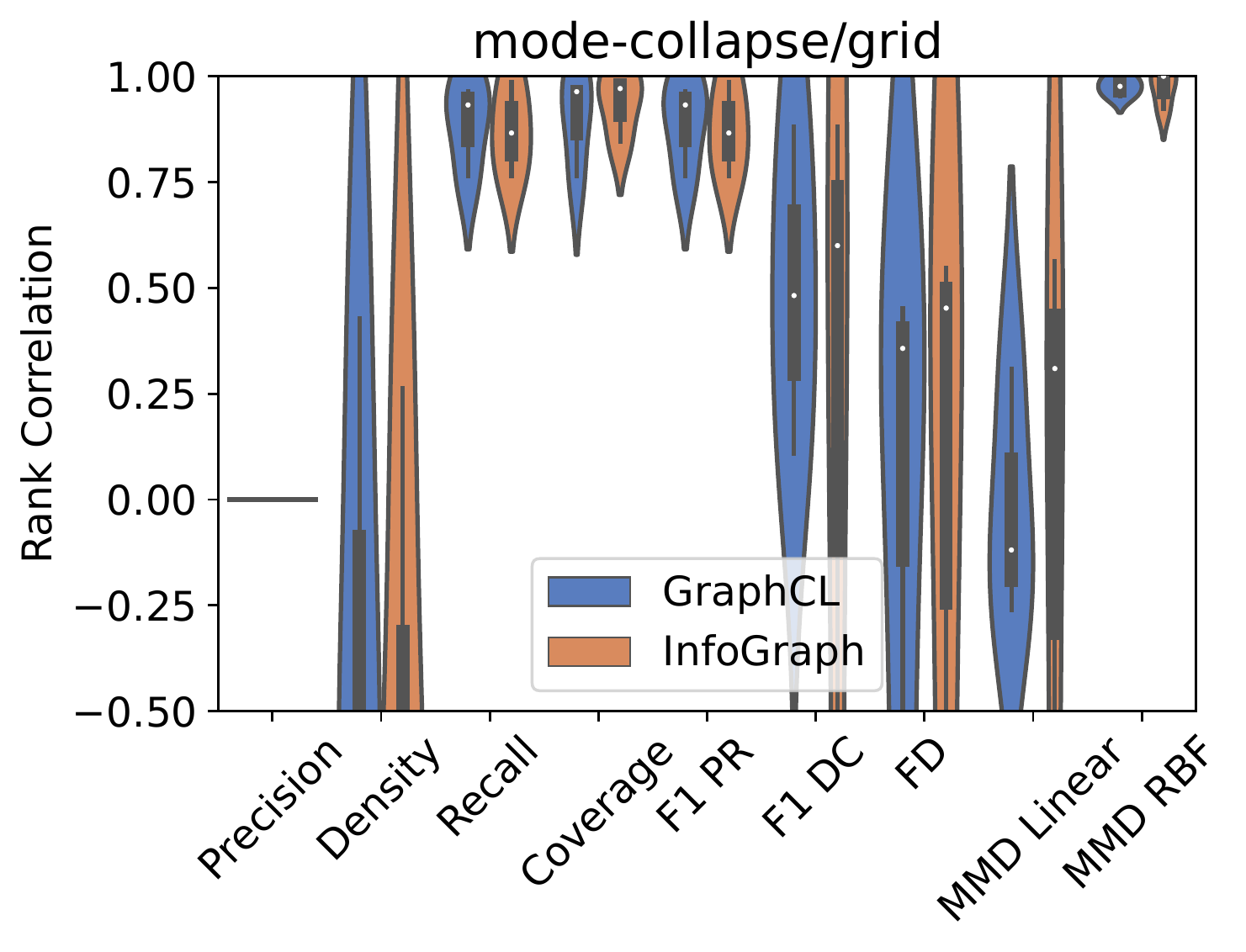}}
    \subfloat[][]{\includegraphics[width = 2.55in]{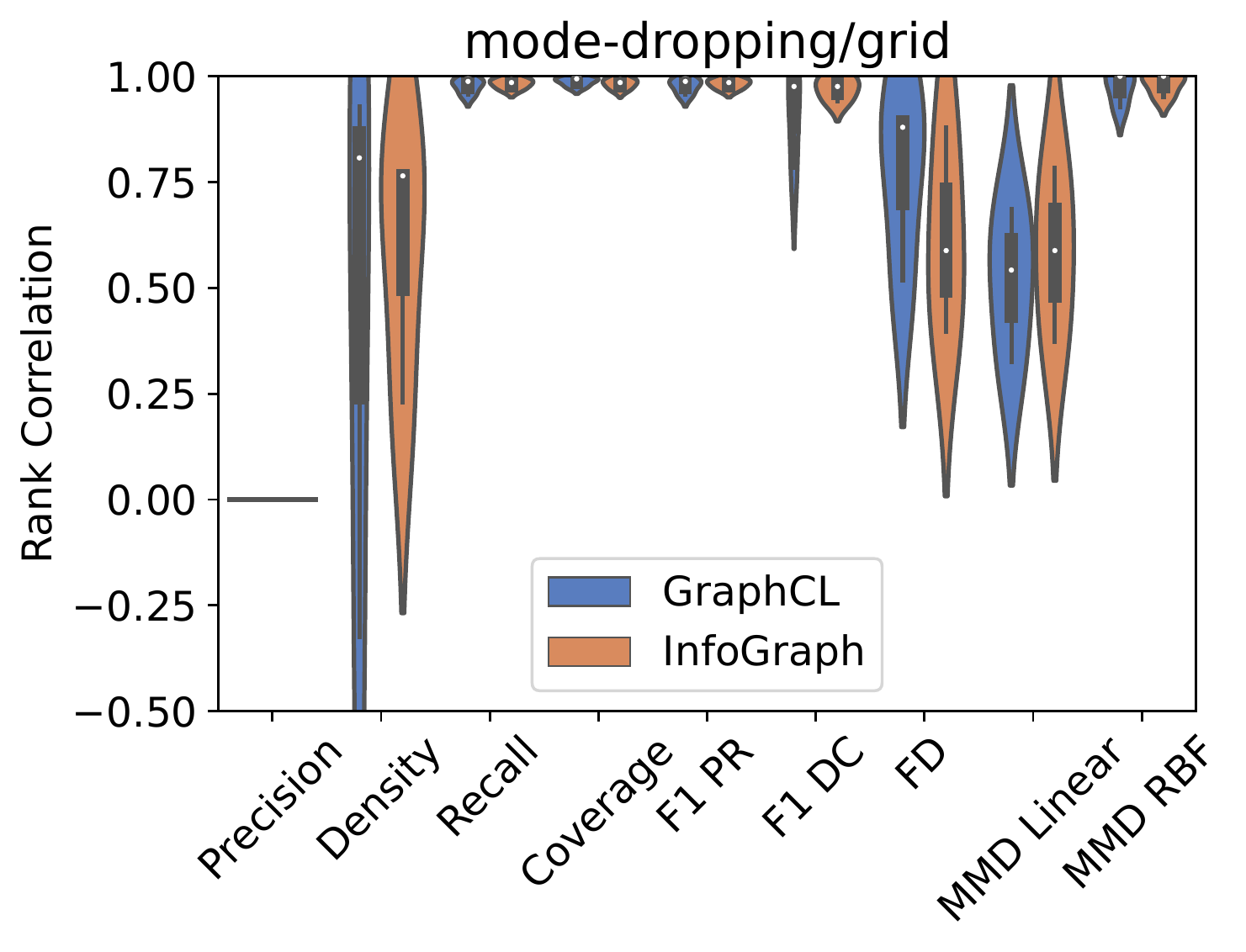}}
    \caption{Violing comparative results among the methods, with with clustering structural features for grid dataset.}
    \label{fig:clustering_feats_grid}
\end{figure*}
 
\begin{table}[h!]
\centering
\begin{small}
\scalebox{0.7}{
\begin{tabular}{l|l|l|l|l|l|l|l|l|l|l|l} 
\toprule
Model Name               & Experiment                      &        & Precision & Density & Recall & Coverage & F1PR & F1DC & FD & MMD Lin & MMD RBF  \\ 
\hline
\multirow{8}{*}{GraphCL}  &    \multirow{2}{*}{Mixing Random} & Mean & 1.0 & 1.0 & 0.16 & 0.9 & 1.0 & 1.0 & 1.0 & 1.0 & 1.0 \\
\cline{3-12}
                   &     & Median & 1.0 & 1.0 & 0.25 & 0.91 & 1.0 & 1.0 & 1.0 & 1.0 & 1.0 \\ 
\cline{2-12}
 &    \multirow{2}{*}{Rewiring Edges} & Mean & 0.24 & 0.24 & 0.51 & 0.24 & 0.52 & 0.24 & 1.0 & 1.0 & 0.78 \\
\cline{3-12}
                   &     & Median & 0.24 & 0.24 & 0.48 & 0.24 & 0.48 & 0.24 & 1.0 & 1.0 & 0.77 \\ 
\cline{2-12}
 &    \multirow{2}{*}{Mode Collapse} & Mean & 0.0 & -0.37 & 0.89 & 0.9 & 0.89 & 0.49 & 0.06 & -0.02 & 0.98 \\
\cline{3-12}
                   &     & Median & 0.0 & -0.6 & 0.93 & 0.96 & 0.93 & 0.48 & 0.36 & -0.12 & 0.98 \\ 
\cline{2-12}
 &    \multirow{2}{*}{Mode Dropping} & Mean & 0.0 & 0.47 & 0.98 & 0.99 & 0.98 & 0.92 & 0.77 & 0.52 & 0.98 \\
\cline{3-12}
                   &     & Median & 0.0 & 0.81 & 0.99 & 0.99 & 0.99 & 0.98 & 0.88 & 0.54 & 1.0 \\ 
\cline{1-12}
\multirow{8}{*}{InfoGraph}  &    \multirow{2}{*}{Mixing Random} & Mean & 1.0 & 1.0 & 0.32 & 0.9 & 1.0 & 1.0 & 1.0 & 1.0 & 1.0 \\
\cline{3-12}
                   &     & Median & 1.0 & 1.0 & 0.28 & 0.88 & 1.0 & 1.0 & 1.0 & 1.0 & 1.0 \\ 
\cline{2-12}
 &    \multirow{2}{*}{Rewiring Edges} & Mean & 0.26 & 0.26 & 0.38 & 0.26 & 0.4 & 0.26 & 1.0 & 1.0 & 0.88 \\
\cline{3-12}
                   &     & Median & 0.24 & 0.24 & 0.36 & 0.24 & 0.39 & 0.24 & 1.0 & 1.0 & 0.88 \\ 
\cline{2-12}
 &    \multirow{2}{*}{Mode Collapse} & Mean & 0.0 & -0.54 & 0.87 & 0.93 & 0.87 & 0.21 & 0.02 & -0.02 & 0.97 \\
\cline{3-12}
                   &     & Median & 0.0 & -0.89 & 0.87 & 0.97 & 0.87 & 0.6 & 0.45 & 0.31 & 1.0 \\ 
\cline{2-12}
 &    \multirow{2}{*}{Mode Dropping} & Mean & 0.0 & 0.59 & 0.98 & 0.98 & 0.98 & 0.97 & 0.62 & 0.58 & 0.98 \\
\cline{3-12}
                   &     & Median & 0.0 & 0.76 & 0.99 & 0.99 & 0.99 & 0.98 & 0.59 & 0.59 & 1.0 \\ \bottomrule
\end{tabular}
}
\end{small}
\caption{Mean and median values for measurements in experiments with with clustering structural features by models, for dataset grid} 
\label{table:clustering_feats_grid}
\end{table}
 
\begin{figure*}[h!]
    \captionsetup[subfloat]{farskip=-2pt,captionskip=-8pt}
    \centering
    \subfloat[][]{\includegraphics[width = 2.55in]{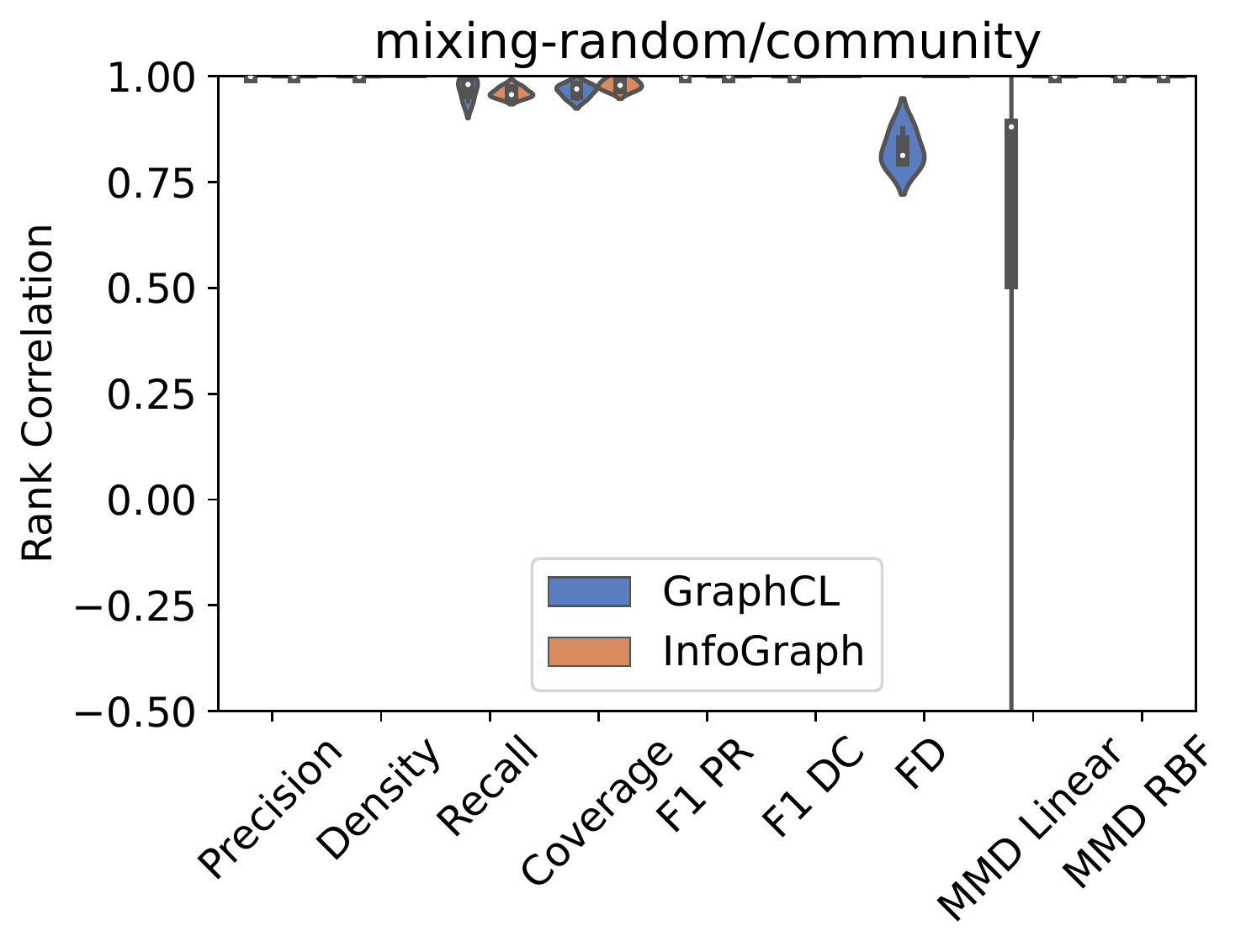}}
    \subfloat[][]{\includegraphics[width = 2.55in]{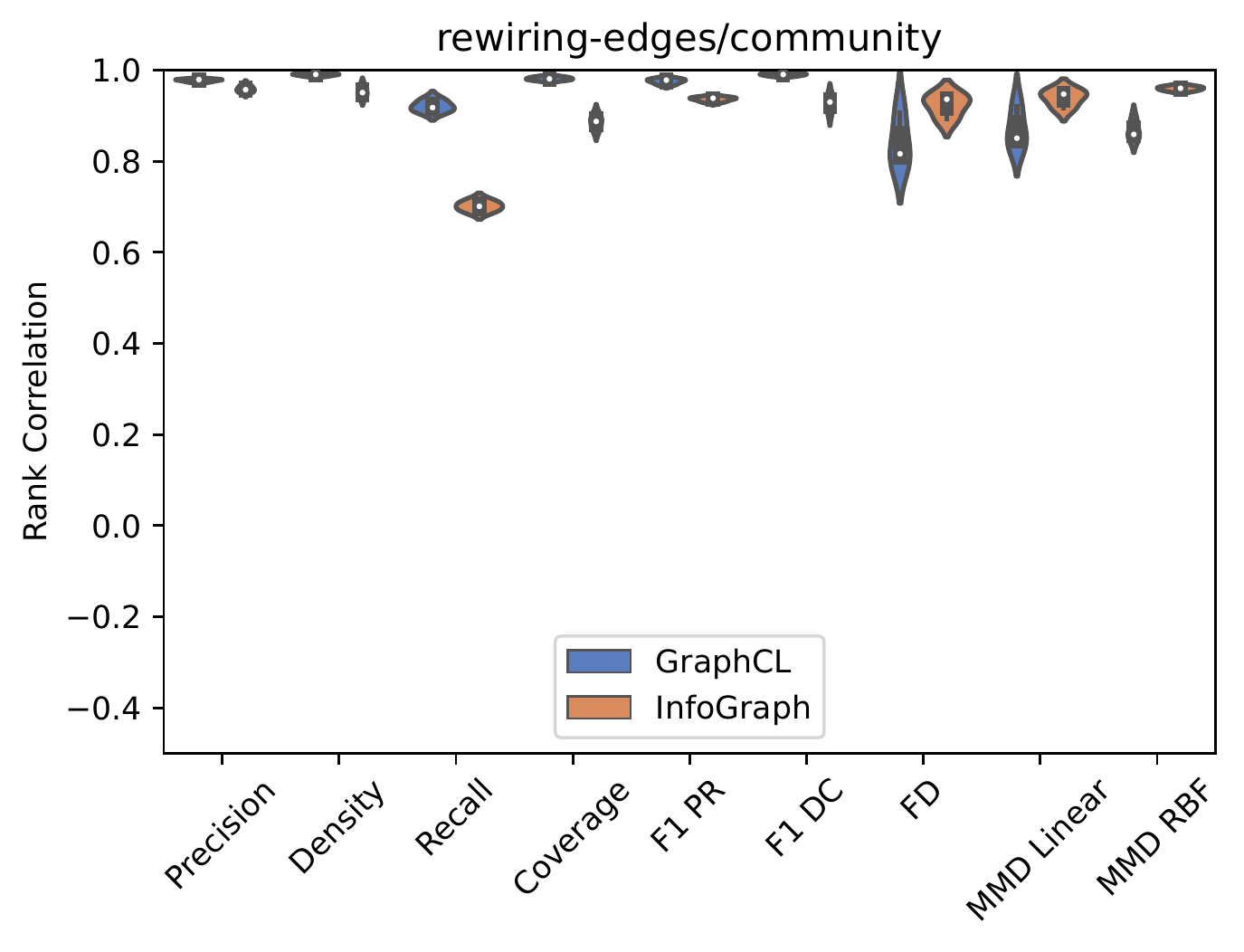}}
    \\
    \subfloat[][]{\includegraphics[width = 2.55in]{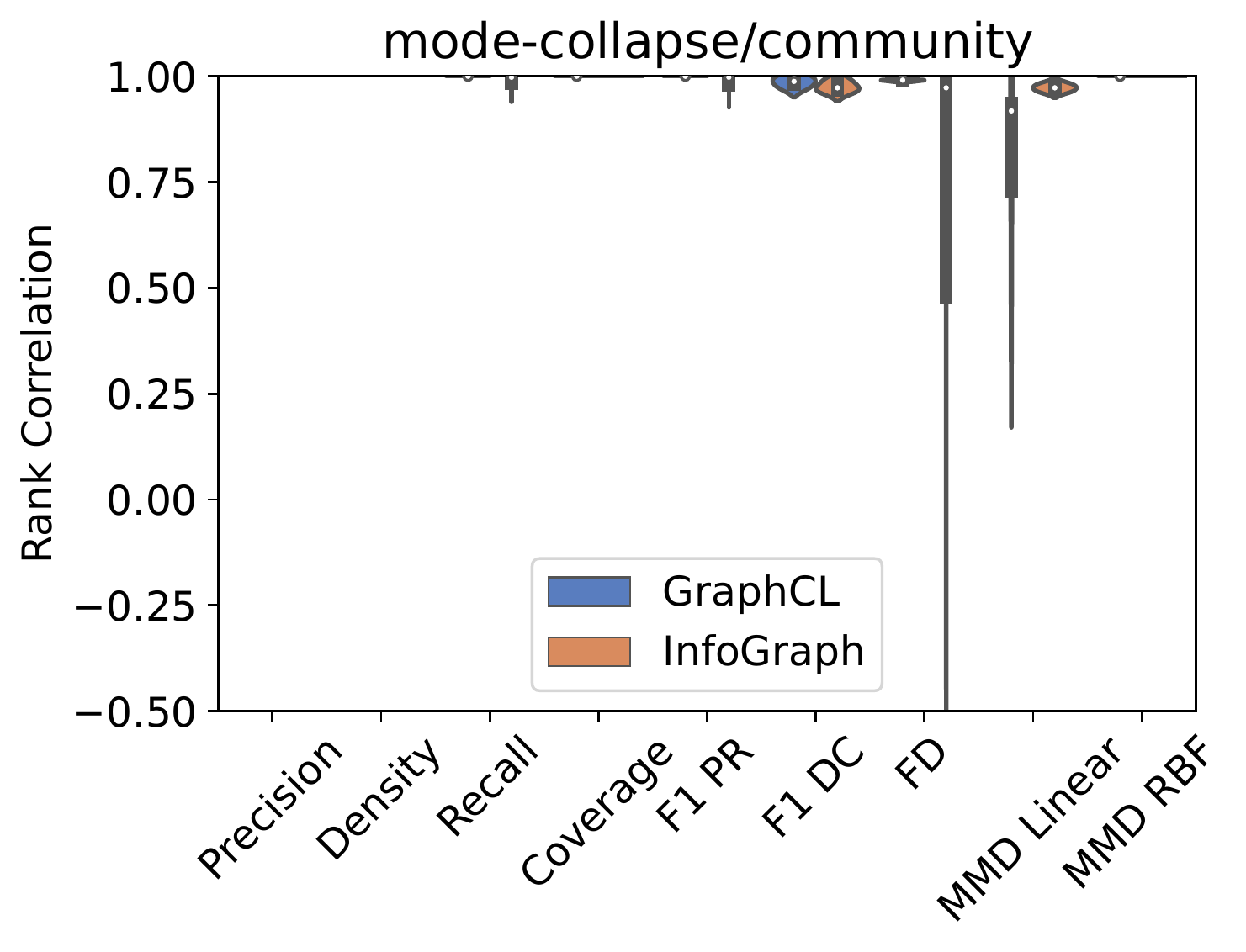}}
    \subfloat[][]{\includegraphics[width = 2.55in]{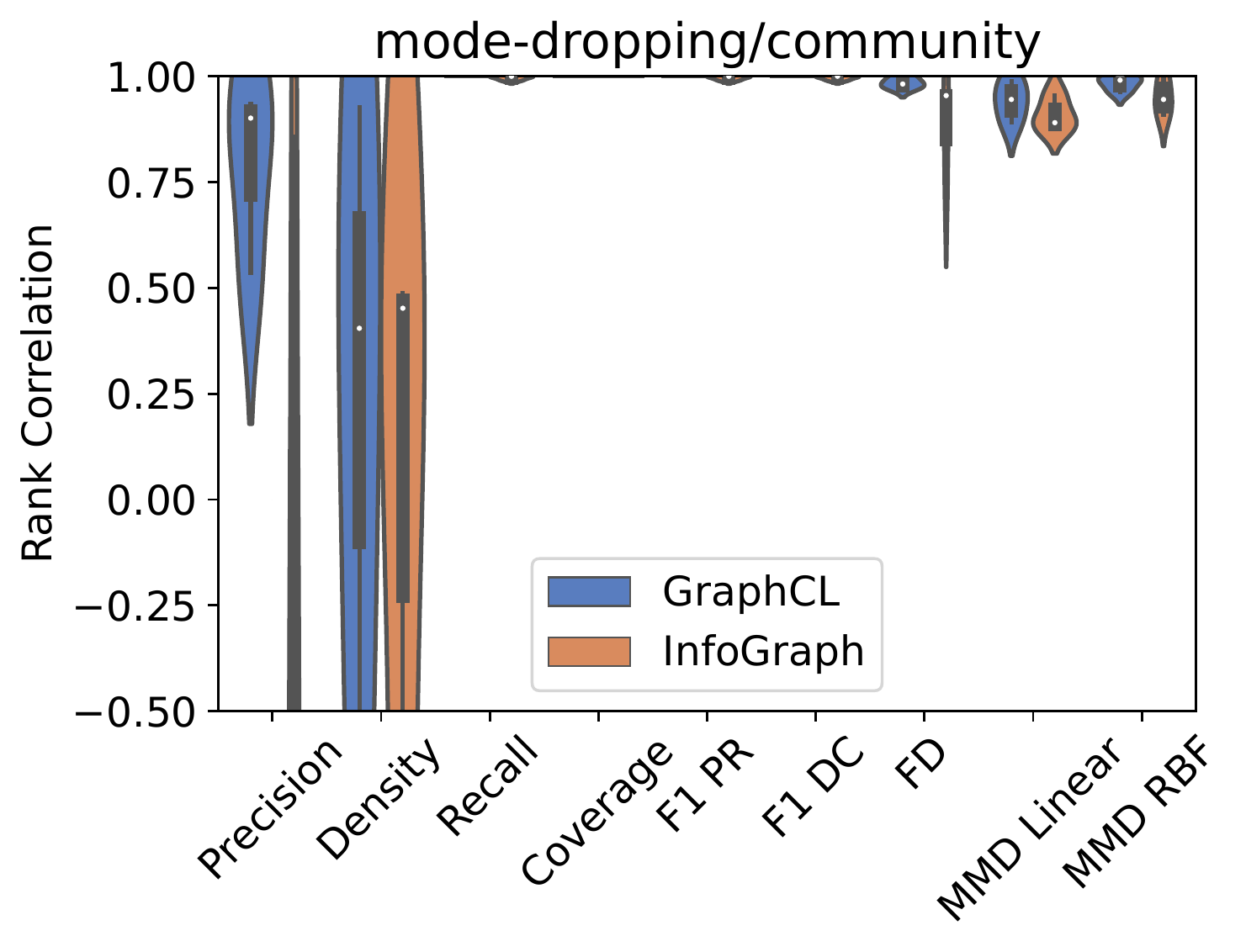}}
    \caption{Violing comparative results among the methods, with with clustering structural features for community dataset.}
    \label{fig:clustering_feats_community}
\end{figure*}
 
\begin{table}[h!]
\centering
\begin{small}
\scalebox{0.7}{
\begin{tabular}{l|l|l|l|l|l|l|l|l|l|l|l} 
\toprule
Model Name               & Experiment                      &        & Precision & Density & Recall & Coverage & F1PR & F1DC & FD & MMD Lin & MMD RBF  \\ 
\hline
\multirow{8}{*}{GraphCL}  &    \multirow{2}{*}{Mixing Random} & Mean & 1.0 & 1.0 & 0.97 & 0.96 & 1.0 & 1.0 & 0.83 & 0.64 & 1.0 \\
\cline{3-12}
                   &     & Median & 1.0 & 1.0 & 0.98 & 0.97 & 1.0 & 1.0 & 0.81 & 0.88 & 1.0 \\ 
\cline{2-12}
 &    \multirow{2}{*}{Rewiring Edges} & Mean & 0.98 & 0.99 & 0.92 & 0.98 & 0.98 & 0.99 & 0.84 & 0.87 & 0.87 \\
\cline{3-12}
                   &     & Median & 0.98 & 0.99 & 0.92 & 0.98 & 0.98 & 0.99 & 0.82 & 0.85 & 0.86 \\ 
\cline{2-12}
 &    \multirow{2}{*}{Mode Collapse} & Mean & -0.96 & -0.77 & 1.0 & 1.0 & 1.0 & 0.99 & 0.99 & 0.8 & 1.0 \\
\cline{3-12}
                   &     & Median & -0.96 & -0.79 & 1.0 & 1.0 & 1.0 & 0.99 & 0.99 & 0.92 & 1.0 \\ 
\cline{2-12}
 &    \multirow{2}{*}{Mode Dropping} & Mean & 0.79 & 0.24 & 1.0 & 1.0 & 1.0 & 1.0 & 0.98 & 0.94 & 0.98 \\
\cline{3-12}
                   &     & Median & 0.9 & 0.4 & 1.0 & 1.0 & 1.0 & 1.0 & 0.98 & 0.95 & 0.99 \\ 
\cline{1-12}
\multirow{8}{*}{InfoGraph}  &    \multirow{2}{*}{Mixing Random} & Mean & 1.0 & 1.0 & 0.96 & 0.98 & 1.0 & 1.0 & 1.0 & 1.0 & 1.0 \\
\cline{3-12}
                   &     & Median & 1.0 & 1.0 & 0.96 & 0.98 & 1.0 & 1.0 & 1.0 & 1.0 & 1.0 \\ 
\cline{2-12}
 &    \multirow{2}{*}{Rewiring Edges} & Mean & 0.96 & 0.95 & 0.7 & 0.89 & 0.94 & 0.93 & 0.92 & 0.94 & 0.96 \\
\cline{3-12}
                   &     & Median & 0.96 & 0.95 & 0.7 & 0.89 & 0.94 & 0.93 & 0.94 & 0.95 & 0.96 \\ 
\cline{2-12}
 &    \multirow{2}{*}{Mode Collapse} & Mean & -0.96 & -0.95 & 0.99 & 1.0 & 0.99 & 0.98 & 0.65 & 0.97 & 1.0 \\
\cline{3-12}
                   &     & Median & -0.98 & -0.95 & 1.0 & 1.0 & 1.0 & 0.97 & 0.97 & 0.97 & 1.0 \\ 
\cline{2-12}
 &    \multirow{2}{*}{Mode Dropping} & Mean & -0.23 & 0.01 & 1.0 & 1.0 & 1.0 & 1.0 & 0.88 & 0.91 & 0.95 \\
\cline{3-12}
                   &     & Median & -0.65 & 0.45 & 1.0 & 1.0 & 1.0 & 1.0 & 0.95 & 0.89 & 0.95 \\ \bottomrule
\end{tabular}
}
\end{small}
\caption{Mean and median values for measurements in experiments with with clustering structural features by models, for dataset community} 
\label{table:clustering_feats_community}
\end{table}
 
\begin{figure*}[h!]
    \captionsetup[subfloat]{farskip=-2pt,captionskip=-8pt}
    \centering
    \subfloat[][]{\includegraphics[width = 2.55in]{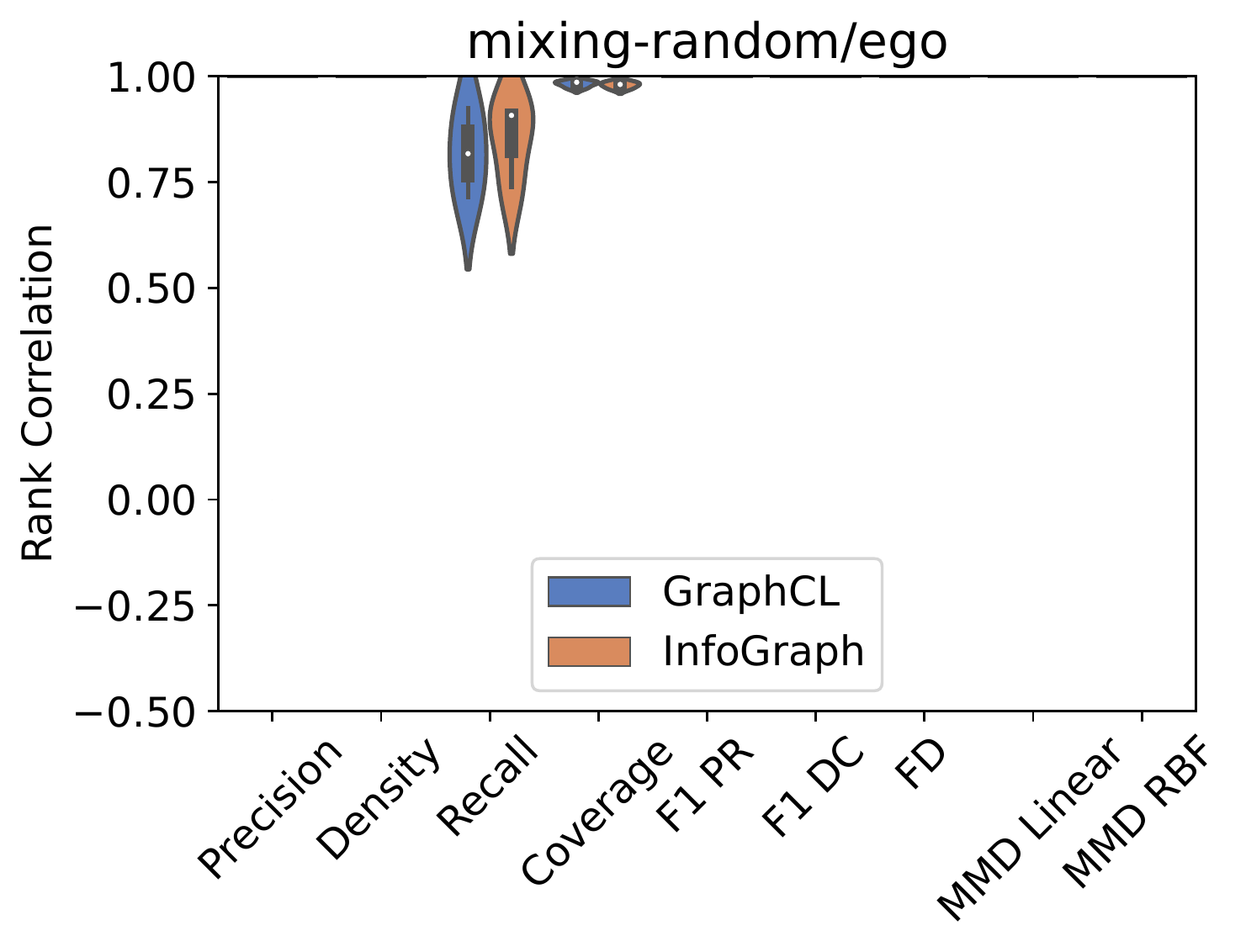}}
    \subfloat[][]{\includegraphics[width = 2.55in]{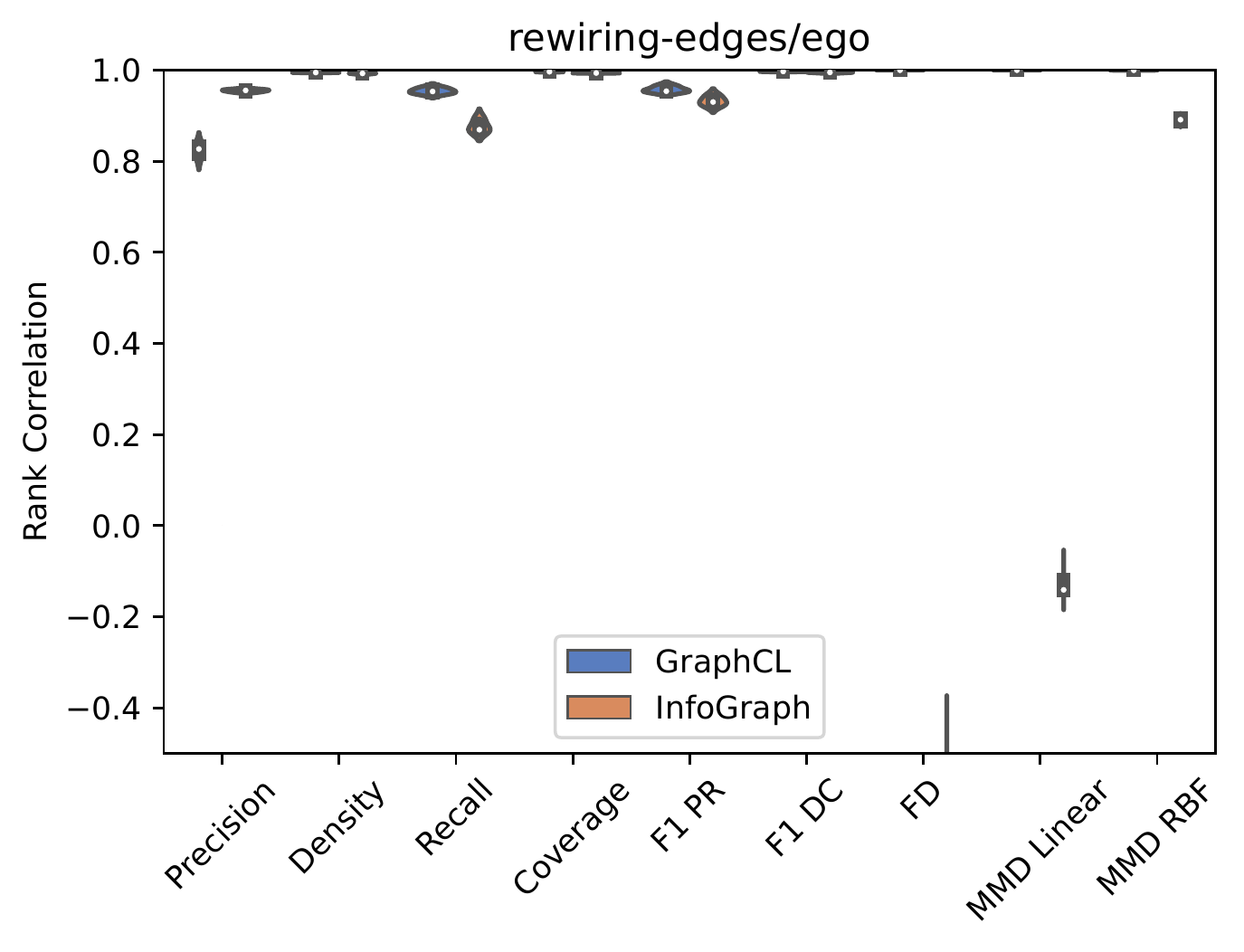}}
    \\
    \subfloat[][]{\includegraphics[width = 2.55in]{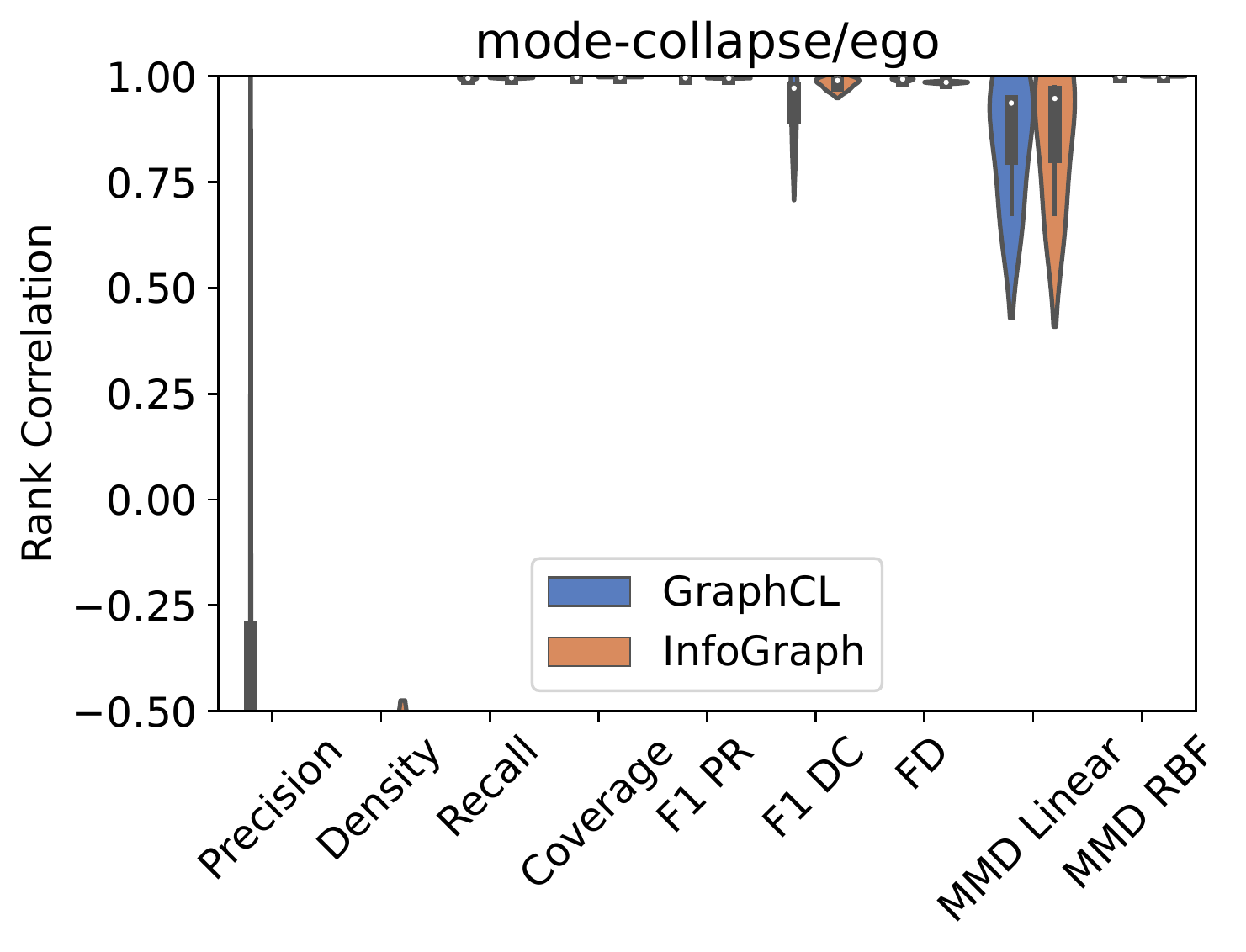}}
    \subfloat[][]{\includegraphics[width = 2.55in]{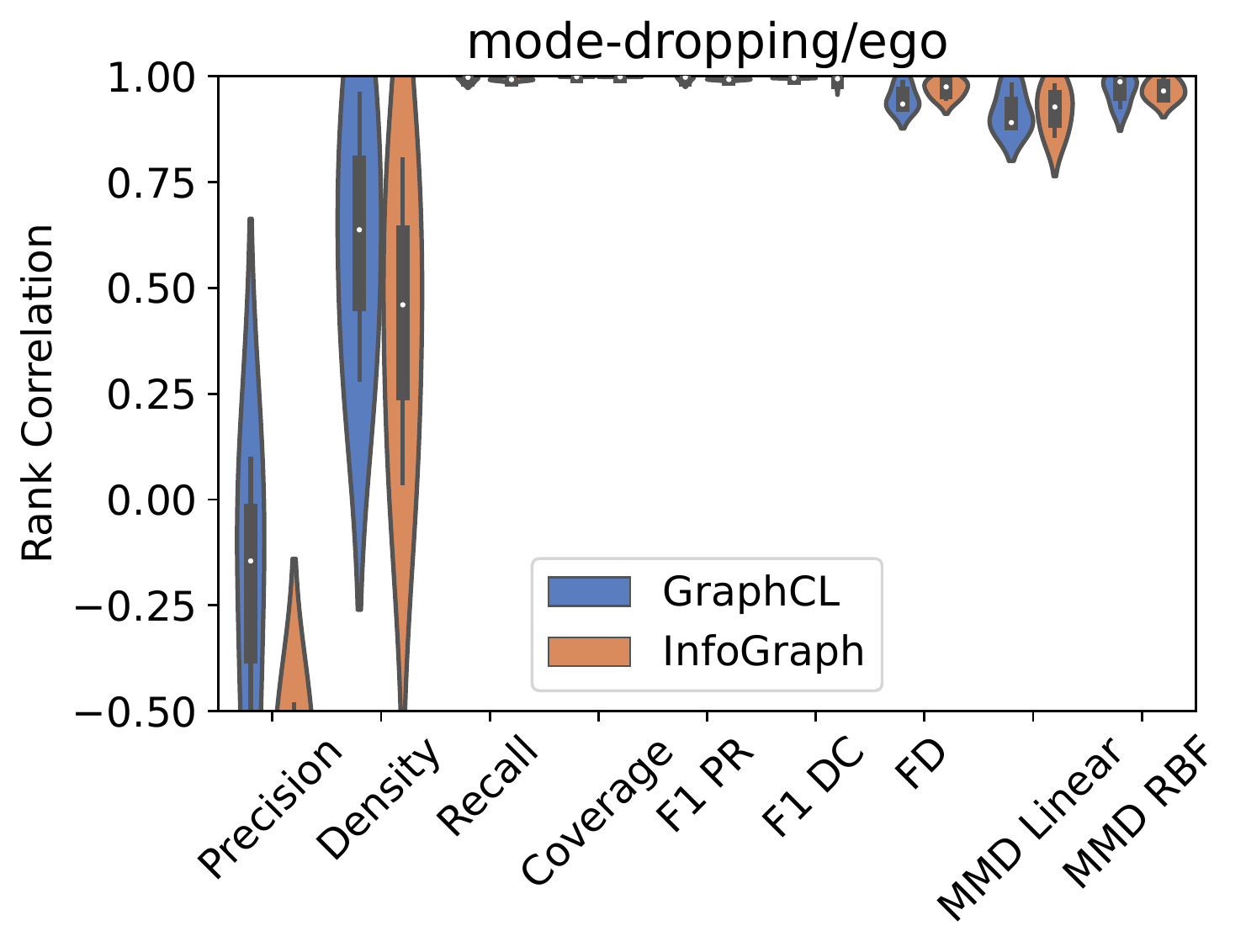}}
    \caption{Violing comparative results among the methods, with with clustering structural features for ego dataset.}
    \label{fig:clustering_feats_ego}
\end{figure*}
 
\begin{table}[h!]
\centering
\begin{small}
\scalebox{0.7}{
\begin{tabular}{l|l|l|l|l|l|l|l|l|l|l|l} 
\toprule
Model Name               & Experiment                      &        & Precision & Density & Recall & Coverage & F1PR & F1DC & FD & MMD Lin & MMD RBF  \\ 
\hline
\multirow{8}{*}{GraphCL}  &    \multirow{2}{*}{Mixing Random} & Mean & 1.0 & 1.0 & 0.82 & 0.98 & 1.0 & 1.0 & 1.0 & 1.0 & 1.0 \\
\cline{3-12}
                   &     & Median & 1.0 & 1.0 & 0.82 & 0.99 & 1.0 & 1.0 & 1.0 & 1.0 & 1.0 \\ 
\cline{2-12}
 &    \multirow{2}{*}{Rewiring Edges} & Mean & 0.82 & 1.0 & 0.95 & 1.0 & 0.96 & 1.0 & 1.0 & 1.0 & 1.0 \\
\cline{3-12}
                   &     & Median & 0.83 & 0.99 & 0.95 & 1.0 & 0.95 & 1.0 & 1.0 & 1.0 & 1.0 \\ 
\cline{2-12}
 &    \multirow{2}{*}{Mode Collapse} & Mean & -0.53 & -0.88 & 1.0 & 1.0 & 1.0 & 0.93 & 0.99 & 0.85 & 1.0 \\
\cline{3-12}
                   &     & Median & -0.91 & -0.94 & 1.0 & 1.0 & 1.0 & 0.97 & 0.99 & 0.94 & 1.0 \\ 
\cline{2-12}
 &    \multirow{2}{*}{Mode Dropping} & Mean & -0.22 & 0.63 & 0.99 & 1.0 & 0.99 & 1.0 & 0.95 & 0.92 & 0.97 \\
\cline{3-12}
                   &     & Median & -0.14 & 0.64 & 1.0 & 1.0 & 1.0 & 1.0 & 0.93 & 0.89 & 0.99 \\ 
\cline{1-12}
\multirow{8}{*}{InfoGraph}  &    \multirow{2}{*}{Mixing Random} & Mean & 1.0 & 1.0 & 0.85 & 0.98 & 1.0 & 1.0 & 1.0 & 1.0 & 1.0 \\
\cline{3-12}
                   &     & Median & 1.0 & 1.0 & 0.91 & 0.98 & 1.0 & 1.0 & 1.0 & 1.0 & 1.0 \\ 
\cline{2-12}
 &    \multirow{2}{*}{Rewiring Edges} & Mean & 0.95 & 0.99 & 0.88 & 0.99 & 0.93 & 0.99 & -0.55 & -0.13 & 0.89 \\
\cline{3-12}
                   &     & Median & 0.96 & 0.99 & 0.87 & 0.99 & 0.93 & 1.0 & -0.57 & -0.14 & 0.89 \\ 
\cline{2-12}
 &    \multirow{2}{*}{Mode Collapse} & Mean & -0.93 & -0.76 & 1.0 & 1.0 & 1.0 & 0.98 & 0.99 & 0.87 & 1.0 \\
\cline{3-12}
                   &     & Median & -0.92 & -0.7 & 1.0 & 1.0 & 1.0 & 0.99 & 0.99 & 0.95 & 1.0 \\ 
\cline{2-12}
 &    \multirow{2}{*}{Mode Dropping} & Mean & -0.7 & 0.43 & 0.99 & 1.0 & 0.99 & 0.99 & 0.97 & 0.92 & 0.97 \\
\cline{3-12}
                   &     & Median & -0.71 & 0.46 & 0.99 & 1.0 & 0.99 & 0.99 & 0.97 & 0.93 & 0.97 \\ \bottomrule
\end{tabular}
}
\end{small}
\caption{Mean and median values for measurements in experiments with with clustering structural features by models, for dataset ego} 
\label{table:clustering_feats_ego}
\end{table}
 
\begin{figure*}[h!]
    \captionsetup[subfloat]{farskip=-2pt,captionskip=-8pt}
    \centering
    \subfloat[][]{\includegraphics[width = 2.55in]{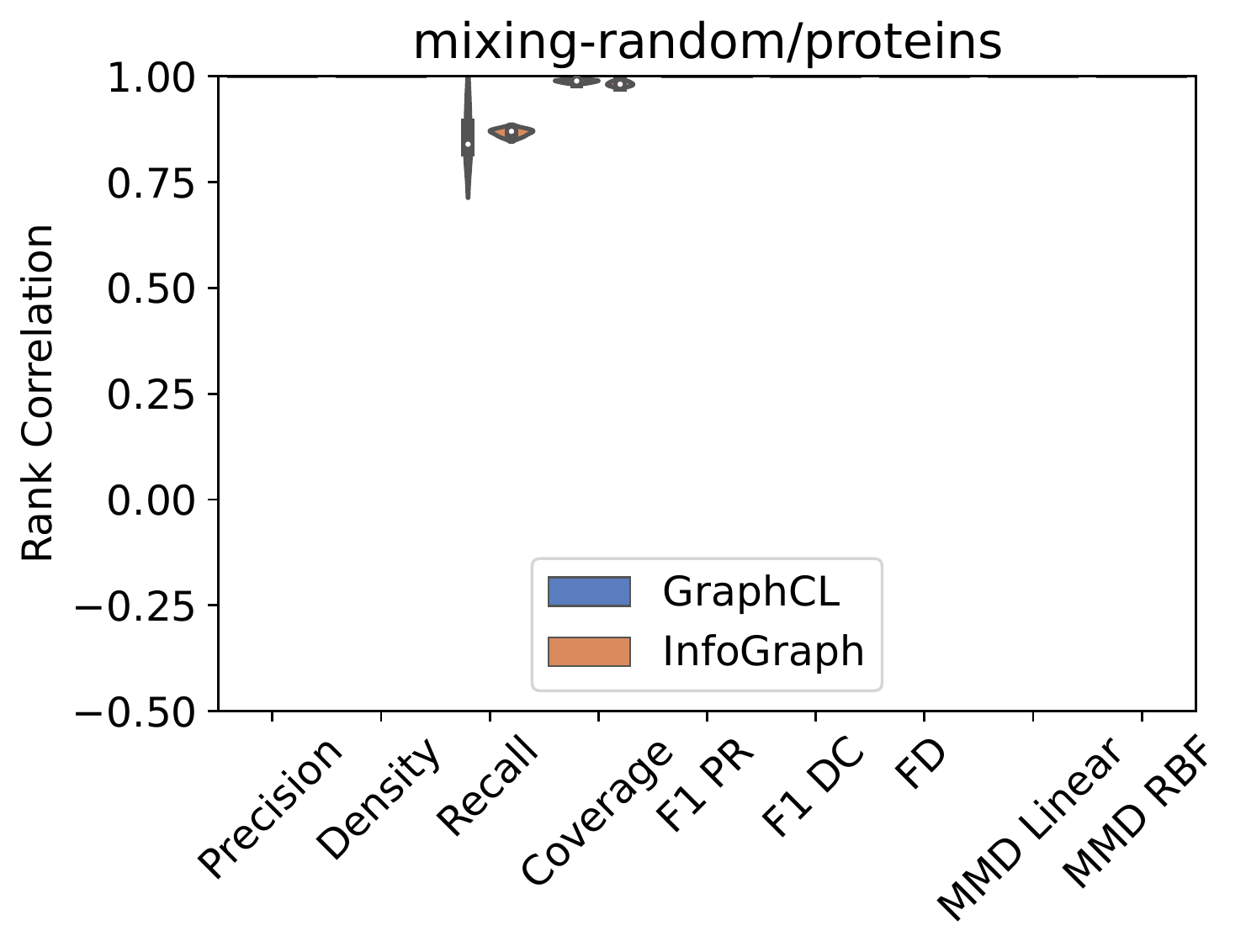}}
    \subfloat[][]{\includegraphics[width = 2.55in]{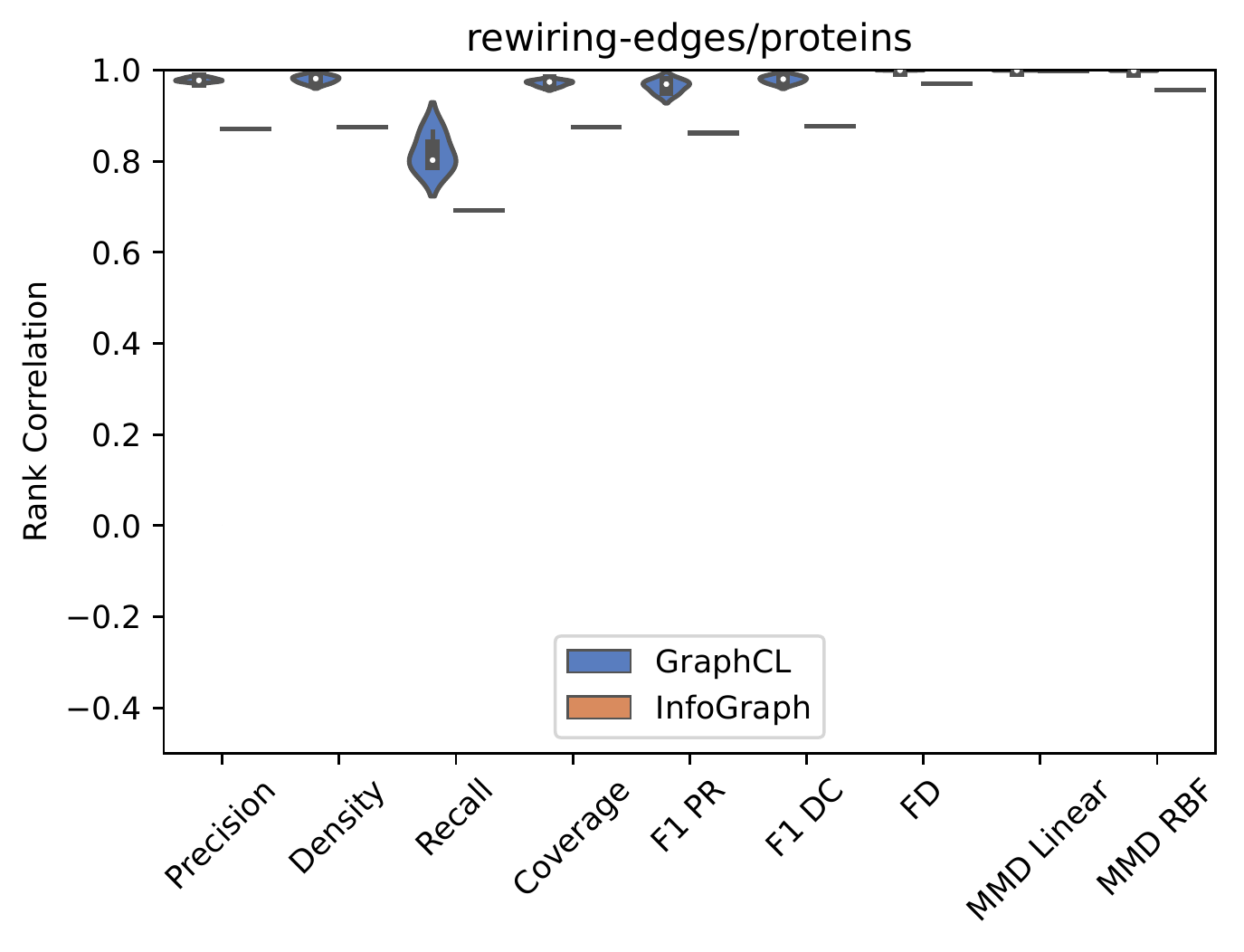}}
    \\
    \subfloat[][]{\includegraphics[width = 2.55in]{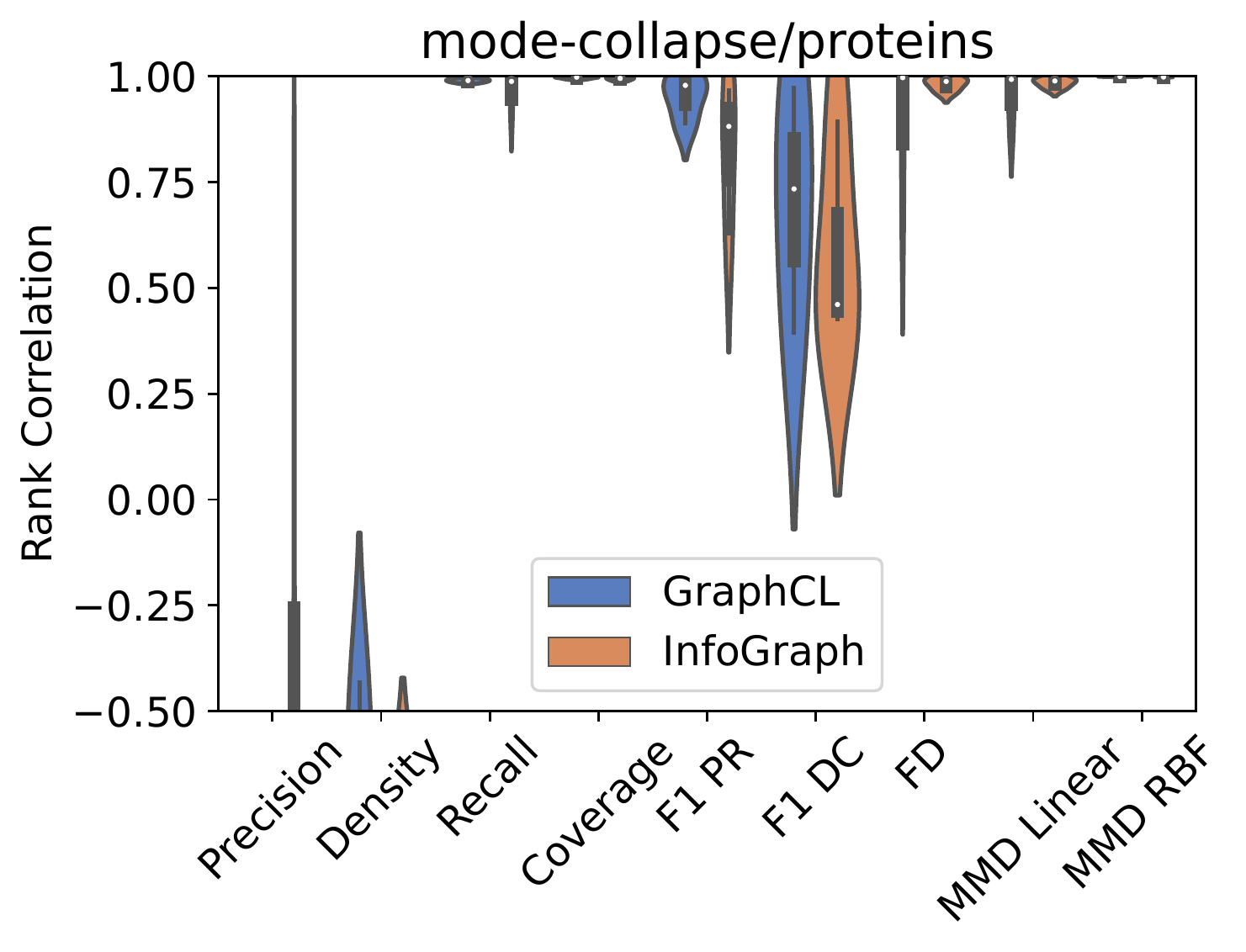}}
    \subfloat[][]{\includegraphics[width = 2.55in]{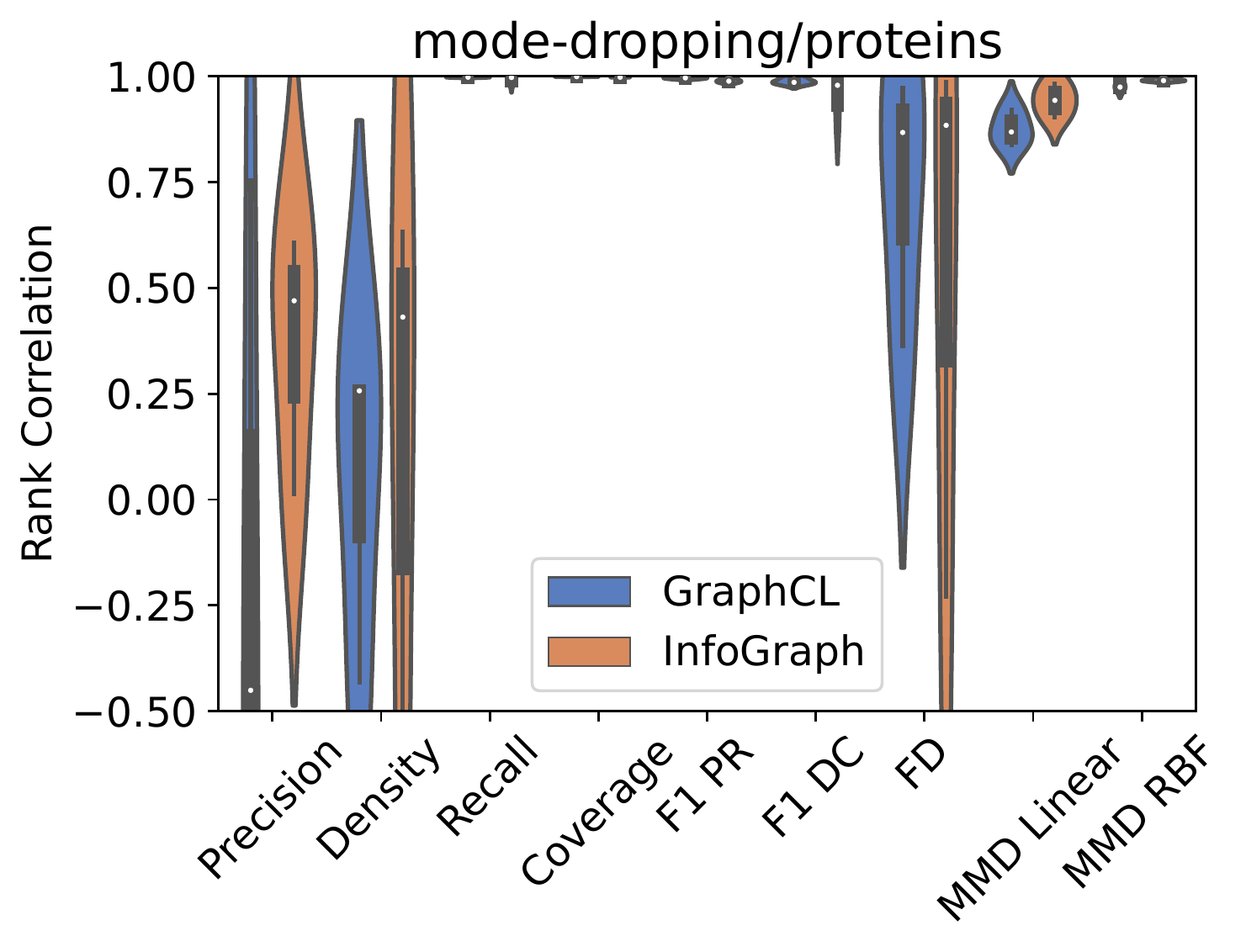}}
    \caption{Violing comparative results among the methods, with with clustering structural features for proteins dataset.}
    \label{fig:clustering_feats_proteins}
\end{figure*}
 
\begin{table}[h!]
\centering
\begin{small}
\scalebox{0.7}{
\begin{tabular}{l|l|l|l|l|l|l|l|l|l|l|l} 
\toprule
Model Name               & Experiment                      &        & Precision & Density & Recall & Coverage & F1PR & F1DC & FD & MMD Lin & MMD RBF  \\ 
\hline
\multirow{8}{*}{GraphCL}  &    \multirow{2}{*}{Mixing Random} & Mean & 1.0 & 1.0 & 0.86 & 0.99 & 1.0 & 1.0 & 1.0 & 1.0 & 1.0 \\
\cline{3-12}
                   &     & Median & 1.0 & 1.0 & 0.84 & 0.99 & 1.0 & 1.0 & 1.0 & 1.0 & 1.0 \\ 
\cline{2-12}
 &    \multirow{2}{*}{Rewiring Edges} & Mean & 0.98 & 0.98 & 0.82 & 0.97 & 0.96 & 0.98 & 1.0 & 1.0 & 1.0 \\
\cline{3-12}
                   &     & Median & 0.98 & 0.98 & 0.8 & 0.97 & 0.97 & 0.98 & 1.0 & 1.0 & 1.0 \\ 
\cline{2-12}
 &    \multirow{2}{*}{Mode Collapse} & Mean & -0.97 & -0.68 & 0.99 & 1.0 & 0.95 & 0.7 & 0.89 & 0.95 & 1.0 \\
\cline{3-12}
                   &     & Median & -0.96 & -0.73 & 0.99 & 1.0 & 0.98 & 0.73 & 1.0 & 0.99 & 1.0 \\ 
\cline{2-12}
 &    \multirow{2}{*}{Mode Dropping} & Mean & -0.22 & 0.03 & 1.0 & 1.0 & 1.0 & 0.99 & 0.73 & 0.88 & 0.98 \\
\cline{3-12}
                   &     & Median & -0.45 & 0.26 & 1.0 & 1.0 & 1.0 & 0.99 & 0.87 & 0.87 & 0.97 \\ 
\cline{1-12}
\multirow{8}{*}{InfoGraph}  &    \multirow{2}{*}{Mixing Random} & Mean & 1.0 & 1.0 & 0.87 & 0.98 & 1.0 & 1.0 & 1.0 & 1.0 & 1.0 \\
\cline{3-12}
                   &     & Median & 1.0 & 1.0 & 0.87 & 0.98 & 1.0 & 1.0 & 1.0 & 1.0 & 1.0 \\ 
\cline{2-12}
 &    \multirow{2}{*}{Rewiring Edges} & Mean & 0.87 & 0.87 & 0.69 & 0.88 & 0.86 & 0.88 & 0.97 & 1.0 & 0.96 \\
\cline{3-12}
                   &     & Median & 0.87 & 0.87 & 0.69 & 0.88 & 0.86 & 0.88 & 0.97 & 1.0 & 0.96 \\ 
\cline{2-12}
 &    \multirow{2}{*}{Mode Collapse} & Mean & -0.5 & -0.84 & 0.96 & 0.99 & 0.83 & 0.59 & 0.98 & 0.99 & 1.0 \\
\cline{3-12}
                   &     & Median & -1.0 & -0.88 & 0.99 & 0.99 & 0.88 & 0.46 & 0.99 & 0.99 & 1.0 \\ 
\cline{2-12}
 &    \multirow{2}{*}{Mode Dropping} & Mean & 0.36 & 0.1 & 0.99 & 1.0 & 0.99 & 0.95 & 0.55 & 0.94 & 0.99 \\
\cline{3-12}
                   &     & Median & 0.47 & 0.43 & 1.0 & 1.0 & 0.99 & 0.98 & 0.88 & 0.94 & 0.99 \\ \bottomrule
\end{tabular}
}
\end{small}
\caption{Mean and median values for measurements in experiments with with clustering structural features by models, for dataset proteins} 
\label{table:clustering_feats_proteins}
\end{table}
 
\begin{figure*}[h!]
    \captionsetup[subfloat]{farskip=-2pt,captionskip=-8pt}
    \centering
    \subfloat[][]{\includegraphics[width = 2.55in]{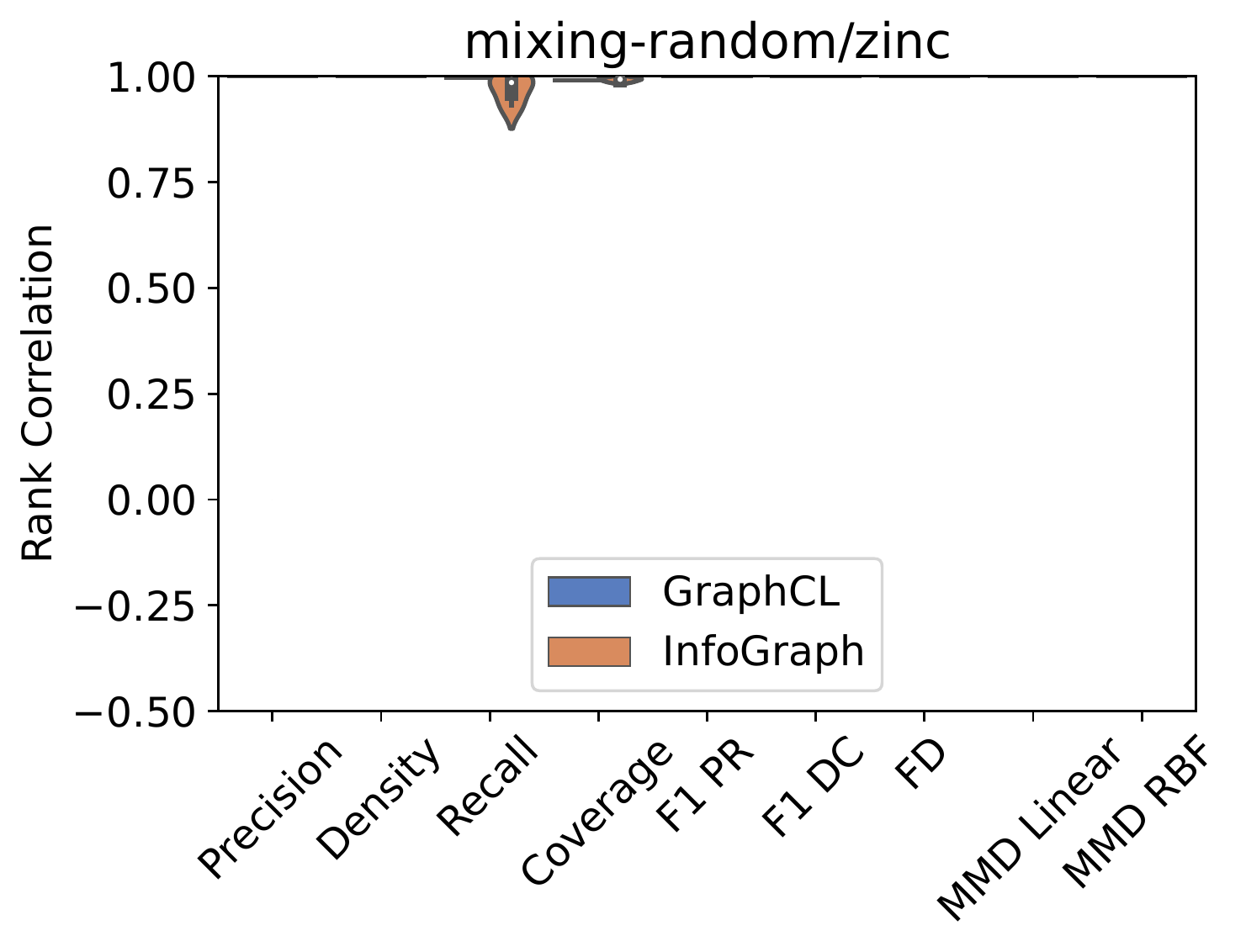}}
    \subfloat[][]{\includegraphics[width = 2.55in]{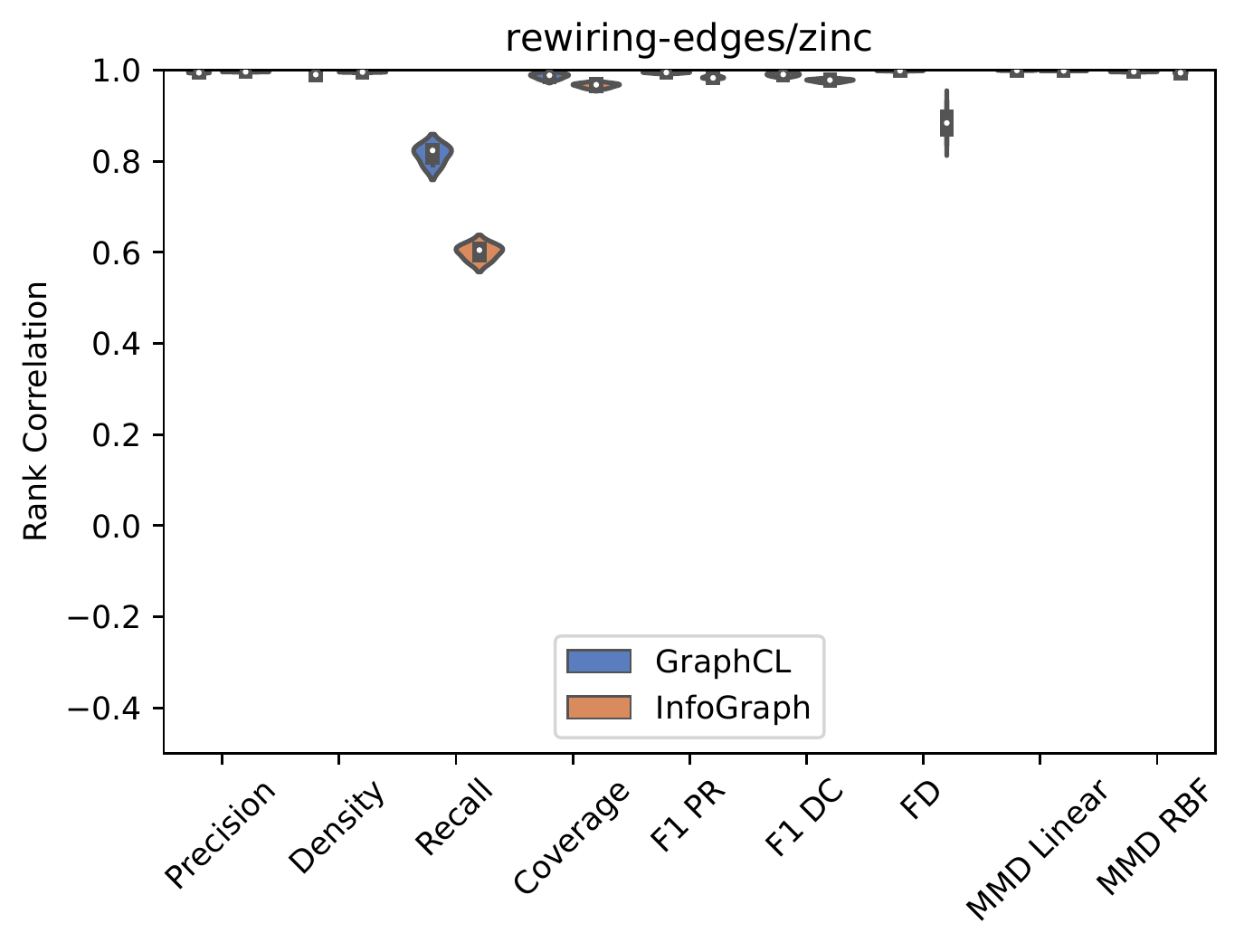}}
    \\
    \subfloat[][]{\includegraphics[width = 2.55in]{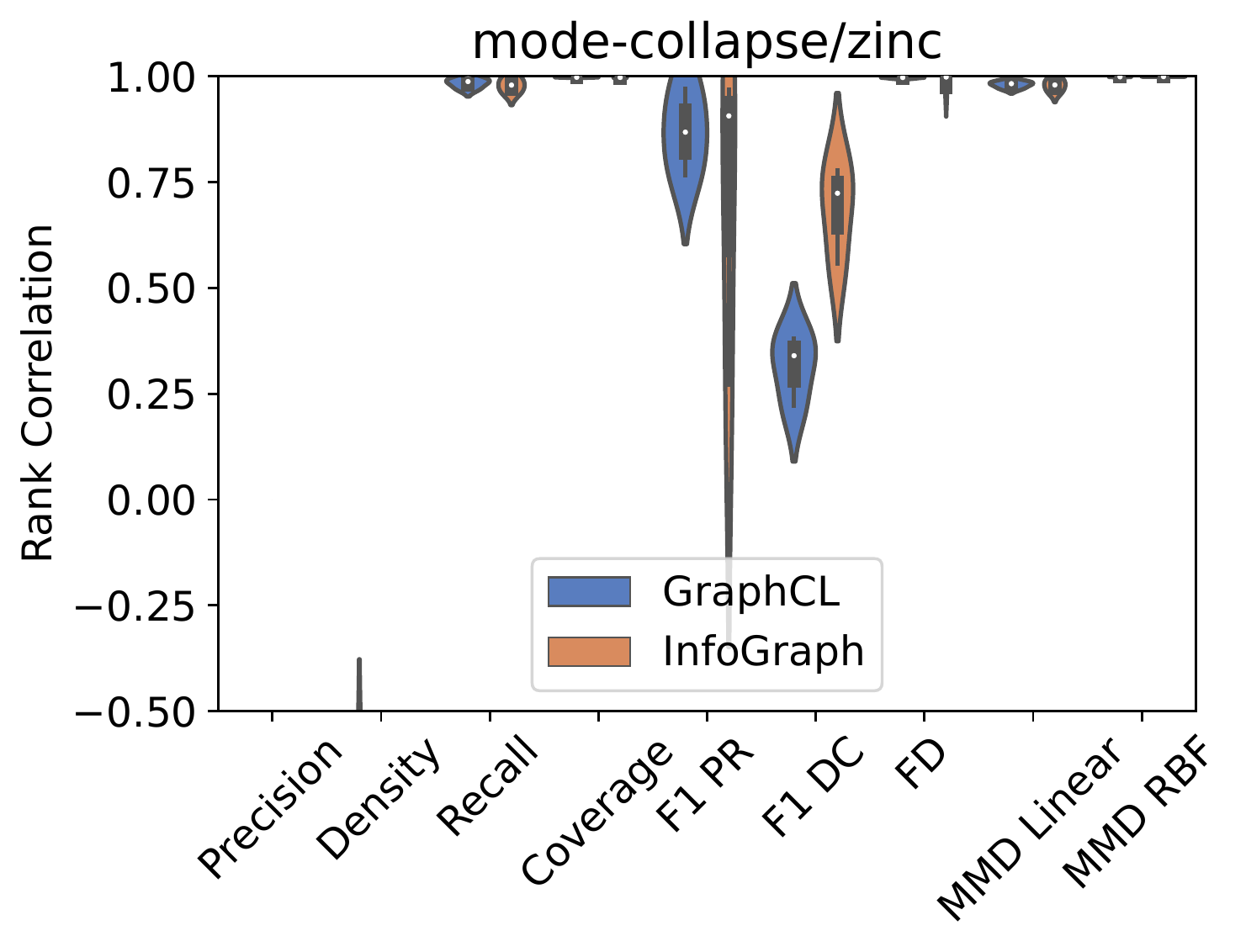}}
    \subfloat[][]{\includegraphics[width = 2.55in]{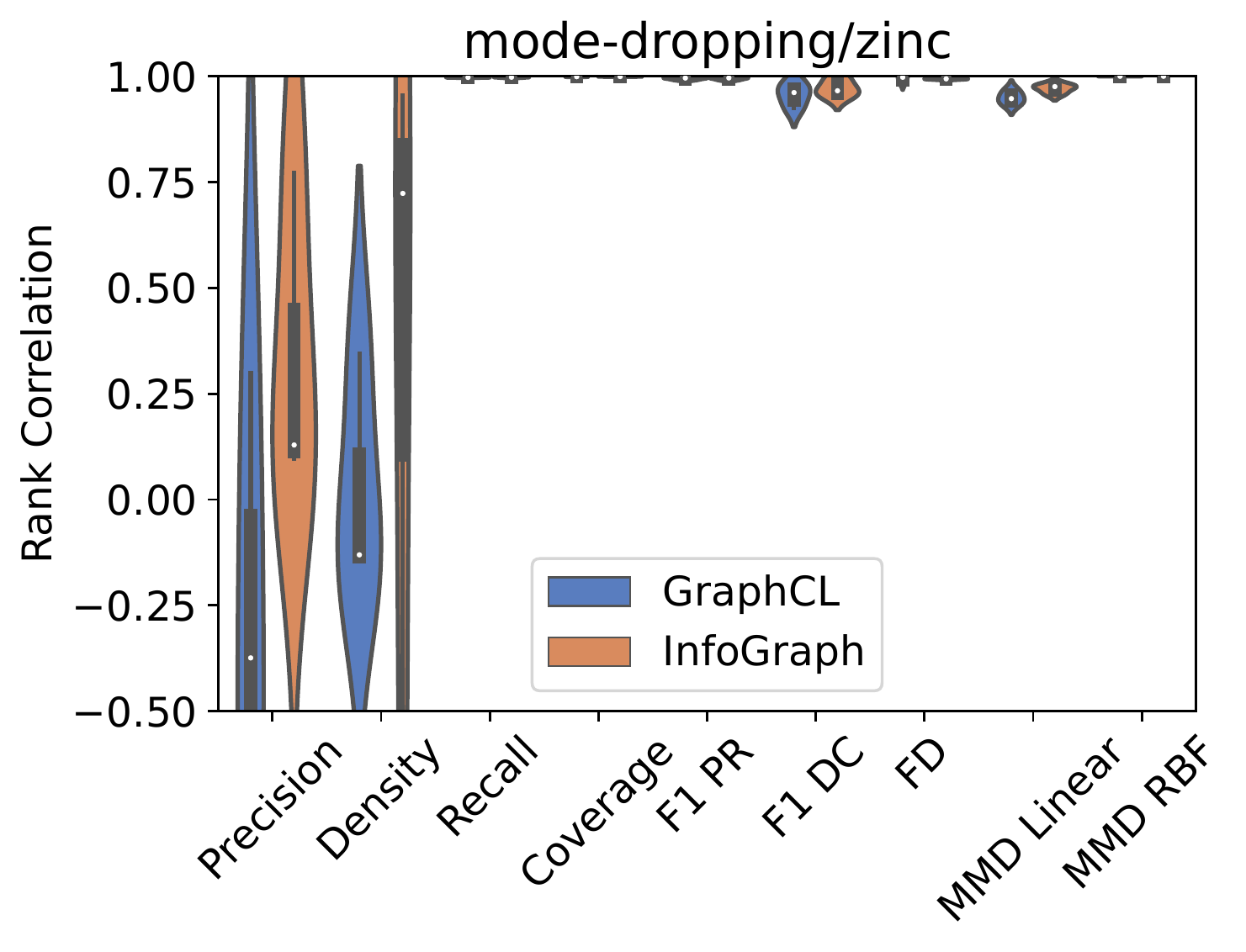}}
    \caption{Violing comparative results among the methods, with with clustering structural features for zinc dataset.}
    \label{fig:clustering_feats_zinc}
\end{figure*}
 
\begin{table}[h!]
\centering
\begin{small}
\scalebox{0.7}{
\begin{tabular}{l|l|l|l|l|l|l|l|l|l|l|l} 
\toprule
Model Name               & Experiment                      &        & Precision & Density & Recall & Coverage & F1PR & F1DC & FD & MMD Lin & MMD RBF  \\ 
\hline
\multirow{8}{*}{GraphCL}  &    \multirow{2}{*}{Mixing Random} & Mean & 1.0 & 1.0 & 1.0 & 0.99 & 1.0 & 1.0 & 1.0 & 1.0 & 1.0 \\
\cline{3-12}
                   &     & Median & 1.0 & 1.0 & 1.0 & 0.99 & 1.0 & 1.0 & 1.0 & 1.0 & 1.0 \\ 
\cline{2-12}
 &    \multirow{2}{*}{Rewiring Edges} & Mean & 0.99 & 0.99 & 0.81 & 0.99 & 0.99 & 0.99 & 1.0 & 1.0 & 1.0 \\
\cline{3-12}
                   &     & Median & 0.99 & 0.99 & 0.82 & 0.99 & 0.99 & 0.99 & 1.0 & 1.0 & 1.0 \\ 
\cline{2-12}
 &    \multirow{2}{*}{Mode Collapse} & Mean & -0.99 & -0.87 & 0.98 & 1.0 & 0.87 & 0.31 & 1.0 & 0.98 & 1.0 \\
\cline{3-12}
                   &     & Median & -1.0 & -0.97 & 0.99 & 1.0 & 0.87 & 0.34 & 1.0 & 0.98 & 1.0 \\ 
\cline{2-12}
 &    \multirow{2}{*}{Mode Dropping} & Mean & -0.31 & 0.02 & 1.0 & 1.0 & 0.99 & 0.96 & 0.99 & 0.95 & 1.0 \\
\cline{3-12}
                   &     & Median & -0.37 & -0.13 & 1.0 & 1.0 & 1.0 & 0.96 & 1.0 & 0.95 & 1.0 \\ 
\cline{1-12}
\multirow{8}{*}{InfoGraph}  &    \multirow{2}{*}{Mixing Random} & Mean & 1.0 & 1.0 & 0.97 & 0.99 & 1.0 & 1.0 & 1.0 & 1.0 & 1.0 \\
\cline{3-12}
                   &     & Median & 1.0 & 1.0 & 0.99 & 0.99 & 1.0 & 1.0 & 1.0 & 1.0 & 1.0 \\ 
\cline{2-12}
 &    \multirow{2}{*}{Rewiring Edges} & Mean & 1.0 & 1.0 & 0.6 & 0.97 & 0.98 & 0.98 & 0.88 & 1.0 & 0.99 \\
\cline{3-12}
                   &     & Median & 1.0 & 1.0 & 0.6 & 0.97 & 0.98 & 0.98 & 0.88 & 1.0 & 0.99 \\ 
\cline{2-12}
 &    \multirow{2}{*}{Mode Collapse} & Mean & -1.0 & -0.91 & 0.98 & 1.0 & 0.72 & 0.69 & 0.98 & 0.98 & 1.0 \\
\cline{3-12}
                   &     & Median & -1.0 & -0.93 & 0.98 & 1.0 & 0.91 & 0.72 & 1.0 & 0.98 & 1.0 \\ 
\cline{2-12}
 &    \multirow{2}{*}{Mode Dropping} & Mean & 0.33 & 0.39 & 1.0 & 1.0 & 0.99 & 0.97 & 0.99 & 0.97 & 1.0 \\
\cline{3-12}
                   &     & Median & 0.13 & 0.72 & 1.0 & 1.0 & 1.0 & 0.97 & 0.99 & 0.98 & 1.0 \\ \bottomrule
\end{tabular}
}
\end{small}
\caption{Mean and median values for measurements in experiments with with clustering structural features by models, for dataset zinc} 
\label{table:clustering_feats_zinc}
\end{table}